\documentclass[10pt,journal,compsoc]{IEEEtran}

\ifCLASSOPTIONcompsoc
  \usepackage[nocompress]{cite}
\else
  \usepackage{cite}
\fi

\usepackage{amssymb}
\usepackage{graphicx}
\usepackage{enumitem}
\usepackage{mathrsfs}  
\usepackage{threeparttable}
\usepackage{booktabs}
\usepackage{amsmath}
\usepackage{dsfont}
\usepackage{amsthm}
\usepackage[boxruled,linesnumbered,vlined,lined]{algorithm2e}
\usepackage{url}
\usepackage{xcolor}
\usepackage{float}
\usepackage{wrapfig}
\usepackage{longtable}
\usepackage{colortbl}
\usepackage{multirow}
\usepackage{pdflscape}
\usepackage{mathtools}
\usepackage{makecell}

\usepackage{tikz}
\newcommand*\circled[1]{\tikz[baseline=(char.base)]{
            \node[shape=circle,draw,inner sep=0.5pt] (char) {#1};}}

\newtheorem{theorem}{Theorem}[section]
\newtheorem{proposition}[theorem]{Proposition}
\theoremstyle{definition}
\newtheorem{definition}[theorem]{Definition}
\theoremstyle{remark}
\newtheorem*{remark}{Remark}
\newenvironment{remarkqed}
  {\pushQED{\qed}\remark}
  {\popQED\endremark}

\ifCLASSINFOpdf
\else
\fi
\begin{document}

\title{h-calibration: Rethinking Classifier Recalibration with Probabilistic Error-Bounded Objective}

\author{Wenjian~Huang, Guiping~Cao, Jiahao~Xia, Jingkun~Chen,
        Hao~Wang, and~Jianguo~Zhang
        
\IEEEcompsocitemizethanks{
\IEEEcompsocthanksitem W. Huang, G. Cao, H. Wang and J. Zhang are with Research Inst. of Trustworthy Autonomous Systems \& Dept. of Computer Science and Engineering, SUSTech, China. E-mail: wjhuang@pku.edu.cn (also: \protect\\ huangwj@sustech.edu.cn); \{12131099,12232399\}@mail.sustech.edu.cn; zhangjg@sustech.edu.cn. \textit{(Corresponding author: J. Zhang)}
\IEEEcompsocthanksitem J. Zhang is also with Guangdong Provincial Key Laboratory of Brain-inspired Intelligent Computation, Dept. of Computer Science and Engineering, SUSTech, and Peng Cheng Lab, China.
\IEEEcompsocthanksitem J. Xia is with Faculty of Engineering and Information Technology, University of Technology Sydney, Australia and J. Chen is with Inst. of Biomedical Engineering, Dept. of Engineering Science, University of Oxford, UK. E-mail: Jiahao.Xia-1@uts.edu.au; jingkun.chen@eng.ox.ac.uk.
}

\thanks{This work was supported in part by the Shenzhen Postdoctoral Research Grant (granted to Wenjian Huang), National Natural Science Foundation of China (Grant No. 62276121), the TianYuan funds for Mathematics of the National Science Foundation of China (Grant No. 12326604) and the Shenzhen International Research Cooperation Project (Grant No. GJHZ20220913142611021).}
}

%
%

\markboth{}%
{Shell \MakeLowercase{\textit{et al.}}: Bare Demo of IEEEtran.cls for Computer Society Journals}
%



\IEEEtitleabstractindextext{%
\begin{abstract}
Deep neural networks have demonstrated remarkable performance across numerous learning tasks but often suffer from miscalibration, resulting in unreliable probability outputs. This has inspired many recent works on mitigating miscalibration, particularly through post-hoc recalibration methods that aim to obtain calibrated probabilities without sacrificing the classification performance of pre-trained models. In this study, we summarize and categorize previous works into three general strategies: intuitively designed methods, binning-based methods, and methods based on formulations of ideal calibration. Through theoretical and practical analysis, we highlight ten common limitations in previous approaches. To address these limitations, we propose a probabilistic learning framework for calibration called $h$-calibration, which theoretically constructs an equivalent learning formulation for canonical calibration with boundedness. On this basis, we design a simple yet effective post-hoc calibration algorithm. Our method not only overcomes the ten identified limitations but also achieves markedly better performance than traditional methods, as validated by extensive experiments. We further analyze, both theoretically and experimentally, the relationship and advantages of our learning objective compared to traditional proper scoring rule. In summary, our probabilistic framework derives an approximately equivalent differentiable objective for learning error-bounded calibrated probabilities, elucidating the correspondence and convergence properties of computational statistics with respect to theoretical bounds in canonical calibration. The theoretical effectiveness is verified on standard post-hoc calibration benchmarks by achieving state-of-the-art performance. This research offers valuable reference for learning reliable likelihood in related fields. The code is available at https://github.com/WenjianHuang93/h-Calibration.
\end{abstract}

}

\maketitle

\IEEEdisplaynontitleabstractindextext

%
\IEEEpeerreviewmaketitle

\IEEEraisesectionheading{\section{Introduction}}
\label{sec:introduction}

%
%
%
%
\IEEEPARstart{T}{he} notion of calibration study has a rich history, with roots going back to the weather or general statistical forecasting \cite{no.155,no.156,no.157}, and predates the birth of machine learning by decades \cite{no.5,no.33,no.68}. For classification task, two fundamental and complementary criteria by which we judge the quality and reliability of a probabilistic predictor are accuracy and calibration \cite{no.11,no.66,no.79,no.45}. A probabilistic predictor is considered ``well calibrated" when its predicted probabilities closely align with the actual likelihoods/frequencies of the corresponding events \cite{no.5,no.21,no.158,no.46,no.79,no.64,no.11,no.51,no.22}. A relatively narrow concept is confidence calibration (or termed top-label calibration by \cite{no.16,no.133,no.23,no.18,no.51,no.52,no.58,no.60,no.75}), which refers to the predicted confidence (maximal classification probability) matching the likelihood of correct class assignments \cite{no.50,no.4,no.77,no.34,no.58,no.66,no.60,no.37,no.16}. For example, among samples on which the model predicts a class with 0.9 probability confidence, approximately 90\% of them should indeed be classified correctly. Calibration ensures that machine learning models provide meaningful and interpretable predicted probabilities, consistent with realized outcomes, making it a key mathematical formulation for model reliability \cite{no.13,no.76,no.24,no.30,no.16,no.159,no.158,no.76,no.182}. This importance is underscored in various fields, including healthcare diagnosis \cite{no.160,no.60,no.41,no.161,no.162,no.33,no.163,no.164,no.165,no.166}, meteorological forecasting \cite{no.156,no.167,no.157,no.168,no.169}, economics analysis \cite{no.170,no.171,no.172}, natural language processing \cite{no.144,no.173,no.174}, fairness studies \cite{no.175,no.176,no.177} and many others \cite{no.178,no.179,no.180,no.181}. Moreover, since well-calibrated probability enables adjusting decision rules in a standardized way \cite{no.5}, such as applying universal decision threshold, it has been shown to be beneficial in numerous machine learning scenarios, including knowledge distillation \cite{no.26,no.208}, curriculum learning \cite{no.42}, multimodal learning \cite{no.45}, out-of-distribution learning \cite{no.82,no.56,no.190}, object detection \cite{no.25,no.65,no.95} and segmentation \cite{no.41,no.57}, domain adaptation \cite{no.183}, dynamic network learning \cite{no.15}, reinforcement learning \cite{no.184}, zero-shot learning \cite{no.185}, ensemble learning \cite{no.186}, improving explainability \cite{no.187} and active learning \cite{no.188,no.189}. Unfortunately, many machine learning models lack inherent calibration \cite{no.109,no.21,no.5,no.71,no.46}. Furthermore, the seminal work of Guo et al. \cite{no.109} empirically demonstrated that \emph{popular modern neural networks often suffer from severer miscalibration issue}, particularly tending towards over-confidence, than shallow models, despite having significantly improved accuracy in diverse classification tasks over the past decade. This has subsequently inspired many recent works on model calibration for deep classifiers.

\subsection{Training-Time Calibration}
The underlying cause for overconfidence of modern neural networks is hypothesised or empirically observed to be associated with model overfitting \cite{no.1,no.12,no.16,no.22,no.24,no.41,no.52,no.55,no.60,no.62,no.64,no.71,no.109}, or the assignment of high confidence to misclassified samples (generally unobserved events/outcomes) \cite{no.12,no.18,no.47,no.64,no.66}. Given these observations, numerous regularization or ensemble techniques have been employed during training of the target task to alleviate overfitting or prevent overconfident predictions by penalizing high-confidence outputs \cite{no.3,no.16,no.47,no.67} or by encouraging high entropy \cite{no.13,no.16,no.48,no.58,no.60,no.64,no.73,no.78} of the predicted distribution. Specific approaches include applying \emph{implicit regularizations}, such as mixup \cite{no.3,no.6,no.24,no.32,no.49,no.55,no.73,no.90,no.91}, label smoothing \cite{no.72,no.93,no.111,no.41,no.42,no.55}, early stopping \cite{no.12}, weight decay \cite{no.109,no.24}, \emph{explicit regularization} terms, such as entropy-regularized loss (ERL) \cite{no.94}, S-AvUC loss \cite{no.58}, focal loss \cite{no.64},  norm in function/logit space \cite{no.53,no.103}, VWCI loss \cite{no.102}, CS-KD loss \cite{no.107}, DWB loss \cite{no.48}, as well as model ensembling \cite{no.2,no.70,no.113,no.106,no.114,no.115,no.20,no.24,no.33,no.77,no.139}. Addtionally, some modality-specific augmentation techniques, e.g, AutoLabel \cite{no.43}, Augmix \cite{no.95}, Cutmix \cite{no.32}, and augmentation methods in \cite{no.24,no.34,no.104}, or structure-dependent regularization techniques, e.g., LRSA \cite{no.18} or SGPA \cite{no.44} for Transformer-based models, have been found beneficial in mitigating overconfidence. However, these empirical augmentation, regularization or ensemble techniques based on reducing overfitting, increasing entropy, or discouraging overconfident outputs still lack compelling theoretical guarantees, such as an inherent direct connection to miscalibration metrics \cite{no.63}. This leads to controversies regarding the effectiveness of related methods in many scenarios, as detailed in \emph{ Appendix \ref{sec:controversies_training_regularizer}} with negative reports \cite{no.90,no.57,no.2,no.24,no.48,no.120,no.12,no.119,no.61,no.53,no.3,no.47} on these strategies' effectiveness.
Moreover, these modified training schemes requires retraining models for recalibration. This incurs high computational costs and diminishes their effectiveness, especially when the model has already been deployed in real-world scenarios \cite{no.40,no.7,no.16,no.60,no.1,no.29}. More importantly, modified training schemes suffer from a decline in classification accuracy compared to the original models specifically trained for improved classification \cite{no.24,no.29,no.33,no.53}.

\subsection{Post-Hoc Recalibration}
Differing from above training-time calibration studies, another category of methods is post-hoc recalibration\footnote{Recalibration inherently implies a post-hoc context. We use `calibration' and `recalibration' interchangeably in post-hoc setting.} \cite{no.1,no.17,no.20,no.21,no.23,no.29,no.34,no.40,no.46,no.57,no.60,no.67,no.68,no.70,no.71,no.75,no.78,no.80,no.85,no.86,no.87,no.78,no.80,no.85,no.86,no.87,no.88,no.89,no.109,no.124,no.109,no.124,no.109,no.124,no.125,no.126,no.127,no.128,no.133,no.135,no.136,no.138,no.145,no.146,no.147,no.150,no.151}, which rectify miscalibrated predictions by applying calibration mapping (fitted on a held-out dataset) to the output (probabilities or logits) of an already trained model. As post-hoc methods can reduce calibration error without requiring retraining and potentially keeping classification accuracy (when strictly monotonic mappings applied), it has become a primary research direction in calibration studies. Recalibrators are frequently acquired through the optimization of \emph{proper scoring rules (PSR)}, such as cross-entropy loss \cite{no.20,no.40,no.57,no.67,no.78,no.80,no.109} (also referred to as logarithmic score, ignorance score, predictive deviance, negative Shannon entropy, NLL for negative log-likelihood) and mean squared error loss \cite{no.21,no.23,no.70,no.75,no.86,no.88} (MSE, also called Brier or quadratic score). This preference is grounded in the theoretical underpinning that the expected score of a PSR is minimized when the model's sample-wise classification probabilities align with the actual probabilities, which is indeed calibrated. \emph{However, extensive research has revealed empirically that optimizing common PSR can produce miscalibrated predictions} \cite{no.24,no.1,no.16,no.22,no.52,no.58,no.60,no.62,no.64,no.14-1,no.24,no.41,no.76,no.13,no.53,no.63}. This emphasizes the central role of learning objective design in calibration.
\emph{As a response to the theoretical preference-experimental finding discrepancy, this paper will offer a potential theoretical explanation in Section \ref{sec:psranalysis} and \ref{method:sec5} to illustrate why PSR is not suitable as a learning objective for recalibration.} It will be shown that scoring rules are prone to overfitting, influenced by unquantifiable approximation errors, thereby giving rise to uncontrollable calibration errors. In contrast, the estimation error of our learning objective is both quantifiable and controllable, ensuring manageable calibration errors. Detailed comparison with PSR will be presented in Section \ref{method:sec5} and \ref{sec:ablation-psr-comparison}.

\subsection{Categorizing Calibration Learning Strategies}
Due to the inadequacy in calibration by directly optimizing networks using PSRs, various alternative learning strategies have been proposed in the literature, encompassing both posthoc and training phases. After systematically reviewing the existing literature, we broadly classify relevant works into the following \emph{three categories}: \textbf{(1)} \emph{Intuitively designed or empirically validated methods}, such as DFL \cite{no.13}, CALL \cite{no.24}, SCTL \cite{no.29}, MHML \cite{no.33}, ATTA \cite{no.34}, SBTS \cite{no.40}, AutoLabel \cite{no.43}, LECE \cite{no.46}, DWB \cite{no.48}, IFL \cite{no.53}, EOW-Softmax \cite{no.54}, MiSLAS \cite{no.55}, AvUC \cite{no.62}, S-AvUC \cite{no.58}, DBLE \cite{no.63}, FL \cite{no.64}, CRL \cite{no.66}, Relaxed Softmax \cite{no.81}, ERL \cite{no.94}, VWCI \cite{no.102}, CS-KD \cite{no.107}, MbLS \cite{no.111}, AdaTS \cite{no.128}, GSD \cite{no.130}, ATS \cite{no.147}, and the approaches in \cite{no.14-0,no.42,no.78,no.103,no.150}; or methods based on posterior estimation of parametric models, such as Dirichlet \cite{no.71}, GP \cite{no.68}, GPD \cite{no.84}, Beta \cite{no.85}, Bayes-Iso \cite{no.127}, and the approach in \cite{no.151}. It is noteworthy that these methods, classified as intuitively designed or empirically validated, do not imply no theory involved, but exhibit a deficiency in establishing inherent direct connection with common calibration evaluation metrics; \textbf{(2)} \emph{Non-parametric binning-based methods}, inspired by binning-based calibration error evaluators, aligning average confidences and the frequencies of event occurrence, such as I-Max \cite{no.1}, PCS \cite{no.12}, EC \cite{no.17}, Mix-n-Match \cite{no.70}, SB-ECE \cite{no.58}, Scaling-binning \cite{no.75}, Histogram Binning \cite{no.86}, BBQ \cite{no.87}, Isotonic Regression \cite{no.88}, ENIR \cite{no.89}, DECE \cite{no.99}, DCA \cite{no.100}, MDCA \cite{no.101},  RB \cite{no.125}, CBT \cite{no.126}, M2B \cite{no.133}, and \cite{no.21}; \textbf{(3)} Methods grounded in \emph{equivalent formulation of ideal calibration}, including ESD \cite{no.22}, KDE-XE \cite{no.52}, Spline \cite{no.60}, and MMCE \cite{no.83}.

\subsection{Limitations for Existing Calibrators}

\subsubsection{Theoretical Gaps}
However, above three learning strategies confront specific theoretical challenges. As mentioned above, the \emph{first category} of methods suffer from a \textbf{deficiency in statistical guarantees to establish a bridge between the learning objective and common evaluation criteria} (\textbf{limitation \#1}). 

The \emph{second category} of strategies is prone to overfitting. Specifically, unlike common PSRs assigning an anchor target (e.g., one-hot vector) to each predicted probability, binning-based methods aim to align the mean statistics of bin-wise confidences and event occurrences. This absence of a unique target per prediction can cause overfitting to the binning scheme, 
failing to guarantee genuinely effective calibration. For instance, predictions with zero calibration error are numerous under specific binning setups, and the error under another binning setup is not assured to be small. The variability under different binnings has been observed in many studies \cite{no.50,no.51,no.52,no.31,no.60,no.70}. Researchers broadly term the problem arising from the non-differentiability of binning, making binning-based methods unsuitable as gradient-based optimization objectives, as the ``non-differentiability" problem \cite{no.12,no.19,no.52,no.58,no.101}. This problem worsens with small batch sizes, potentially due to the impact of the bias-variance tradeoff induced by binning operations \cite{no.1,no.21,no.52,no.70,no.76,no.51}, resulting in a larger bias in bin-wise statistics with fewer bin-wise samples. Here, we prefer to frame this problem as a form of \textbf{overfitting}, as \textbf{achieving low (even zero) calibration error under a specific binning is only a necessary condition for being well-calibrated} (\textbf{limitation \#2}). In such cases, the intrinsic calibrated error is uncontrolled, posing the problem of underestimation, as empirically validated by \cite{no.79,no.75,no.68,no.1,no.74}. This susceptibility to overfitting makes binning-based metrics unsuitable as learning objectives, though they can serve as common evaluators. In \emph{Appendix \ref{sec:comparing_nll_ece_training}}, we demonstrate 
shows further evidence using cross-entropy and binning-based ECE as post-hoc learning objectives. Although some empirical variants, such as SB-ECE \cite{no.58} and DECE \cite{no.99}, have been proposed to address the non-differentiability of hard binning, along with evaluators like FCE \cite{no.37}, the effectiveness of these soft variants still lacks clear theoretical guarantees. Additionally, determining the membership function of soft binning, which can greatly affect results, remains challenging. 

The \emph{third category} of strategies solely outlines the equivalent forms associated with ideally calibrated probabilities, treating them as learning objectives. In practice, however, probabilities are not perfectly calibrated \cite{no.27,no.5,no.136} due to the influence of various inductive biases, such as network structure \cite{no.11}. Achieving ideally zero calibration error with loss function forms theoretically equivalent to perfect calibration is essentially unattainable. Accordingly, the \textbf{character and extent of imperfect calibration in the calibrated probabilities, with respect to objective statistics, remain unclear} (\textbf{limitation \#3}). The study of equivalent forms of realistically imperfectly calibrated probabilities is still underexplored. Furthermore, existing strategies in modeling ideal calibration, i.e., $P(y_q|p_q(x)=v)=v$, involve empirical estimation of distributions or conditional distributions and subsequently deriving the learning objectives, creating inherent tradeoff. Specifically, when focusing on one-dimensional distribution cases, like top-label calibration (e.g., ESD \cite{no.22}, MMCE \cite{no.83}, where $q$ represents the top-label variable), or classwise calibration (e.g., Spline \cite{no.60}, where $q$ represents the class variable), errors in empirical distribution or expectation estimations in one-dimension are relatively controllable. However, the drawback is that both top-label calibration and classwise calibration are \textbf{weaker than canonical calibration} (\textbf{limitation \#4}). In contrast, focusing on high-dimensional situations for canonical calibration, such as the optimization objective in KDE-XE \cite{no.52}, and calibration error evaluators in KDE-ECE \cite{no.70} and SKCE \cite{no.79}, presents the \textbf{challenge of the curse of dimensionality} \cite{no.52,no.70,no.46,no.88} (\textbf{limitation \#5}). For instance, the estimation of conditional density $\mathbb E[Y|p(X)]$ in KDE-XE \cite{no.52} and KDE-ECE \cite{no.70}, as well as SKCE's estimation of the joint distribution of $(e_Y, p(X))$ (where $e_Y$ represents the one-hot class vector), entails biases that are difficult to control in high-dimensional cases (e.g., ImageNet task with a dimensionality of 1000). To illustrate, the expected bias in canonical calibration error estimates based on kernel density estimation (KDE) increases with the increase in class dimensions, necessitating an exponential growth in sample size to counteract this bias growth (see theorem in \cite{no.70}). In comparison, our approach can model strong canonical calibration while avoiding the issue of high-dimensional density estimation, providing an equivalent constraint form for one-dimensional scalar directly from asymptotic theory.
 
It is noteworthy that \textbf{limitation \#4} and \textbf{limitation \#5} \emph{are not confined to methods in category} \textbf{(3)}; \emph{rather, they extend to} \textbf{(1)} \emph{intuitively designed} or \textbf{(2)} \emph{binning-based strategies.} Concerning \textbf{limitation \#4}, numerous methods focus on top-label or classwise calibration rather than canonical calibration in their modeling approaches and evaluations. \emph{Appendix \ref{sec:addtional_comment_limit4}} provides a summary for these numerous studies. Regarding \textbf{limitation \#5}, binning-based or empirical methods that directly model the predicted probabilistic vector in high-dimension can also encounter this issue. As discussed in \cite{no.52,no.65,no.70,no.74,no.133}, high-dimensional binning, like the Sierpinski, Grid-style, or projection-based binning in \cite{no.133}, or other strategies directly constructing estimators and constraints from high-dimensional neighborhoods, as in \cite{no.46}, inherently suffer from the curse of dimensionality. This issue underscores a substantial need for extensive data to counterbalance the sparsity in sample distribution resulting from the increased dimensionality.

\subsubsection{Methodological Dependencies}
In addition to the aforementioned theoretical challenges, we identify five other common deficiencies shared across \emph{different categories} of methods. Due to page limit, we briefly summarize them in this and the next section, categorized as methodological and practical issues. Detailed discussions can be found in \emph{Appendix \ref{sec:limitationdiscussion}}. \textbf{Limitation \#6} highlights the \textbf{reliance on many unverified assumptions} to achieve calibration, such as assuming Gaussian, Beta, or Dirichlet distributions for learned representations, as seen in \cite{no.4, no.68, no.84, no.44, no.28, no.106, no.19, no.49, no.50, no.85, no.151, no.71, no.87, no.52}. These assumptions can even \emph{contradict} one another across different studies.
\textbf{Limitation \#7} emphasizes that many methods involve specific \textbf{settings or hyperparameters that are non-universal or difficult to determine} directly through theory or experience. Examples include choices within implicitly regularized strategies, such as augmentation, ensemble, and other configurations \cite{no.3, no.6, no.24, no.32, no.49, no.55, no.73, no.90, no.91, no.72, no.93, no.111, no.41, no.42, no.55, no.12, no.109}; binning configurations \cite{no.1, no.12, no.17, no.70, no.58, no.75, no.86, no.87, no.88, no.89, no.99, no.100, no.101, no.125, no.126, no.133, no.21}; kernel selection \cite{no.83, no.79, no.19, no.52, no.70}; and the weighting of explicit regularizers \cite{no.12, no.16, no.19, no.22, no.41, no.52, no.54, no.58, no.62, no.66, no.83, no.99, no.100, no.101, no.102, no.103, no.107, no.111, no.128, no.145, no.150}. 

\subsubsection{Practical Limitations}
\textbf{Limitation \#8} pertains to the \textbf{trade-off between the probabilistic unit measure property and calibration}. Many methods fail to ensure the unit measure property when implementing calibration, and additional rectified normalization may lead to probabilities that are no longer calibrated, as noted in \cite{no.5, no.71, no.46, no.57, no.1}. For example, methods in \cite{no.1, no.70, no.145, no.60, no.75, no.133, no.88, no.46} are prone to this issue. Furthermore, many calibrators that were initially proposed for binary 
classifier
\cite{no.86, no.87, no.89, no.85, no.21, no.125, no.151, no.127}, when extended to multiclass calibration using strategies like one-vs-rest \cite{no.65}, also face this problem. \textbf{Limitation \#9} concerns the issue of \textbf{non-accuracy preservation} for calibration methods. First, methods employing modified training schemes \cite{no.3, no.4, no.6, no.12, no.13, no.14-0, no.16, no.18, no.19, no.22, no.24, no.28, no.33, no.39, no.41, no.42, no.43, no.44, no.48, no.49, no.52, no.53, no.54, no.55, no.58, no.61, no.62, no.63, no.64, no.66, no.72, no.73, no.81, no.83, no.84, no.90, no.91, no.93, no.95, no.99, no.100, no.101, no.102, no.103, no.104, no.106, no.107, no.111, no.120, no.129, no.130, no.139} inherently do not guarantee accuracy preservation. Secondly, many post-hoc methods \cite{no.1, no.21, no.71, no.85, no.86, no.68, no.87, no.88, no.89, no.151, no.125, no.127, no.146, no.150, no.109, no.70, no.75, no.126, no.133} fail to ensure monotonicity in the recalibration mapping of the sample-wise probabilities, frequently leading to decreased classification accuracy. \textbf{Limitation \#10} pertains to the issue of \textbf{applicability}. Some calibration methods are limited to specific models or require modifications to the original network structure or training procedures, e.g., \cite{no.12, no.15, no.18, no.19, no.23, no.24, no.28, no.33, no.44, no.54, no.55, no.63, no.84, no.106, no.129, no.130, no.139, no.145}, thereby restricting their applicability.

\begin{figure}[ht]
\centering
\includegraphics[width=\columnwidth, trim={0cm 0.4cm 0cm 0.2cm}, clip]{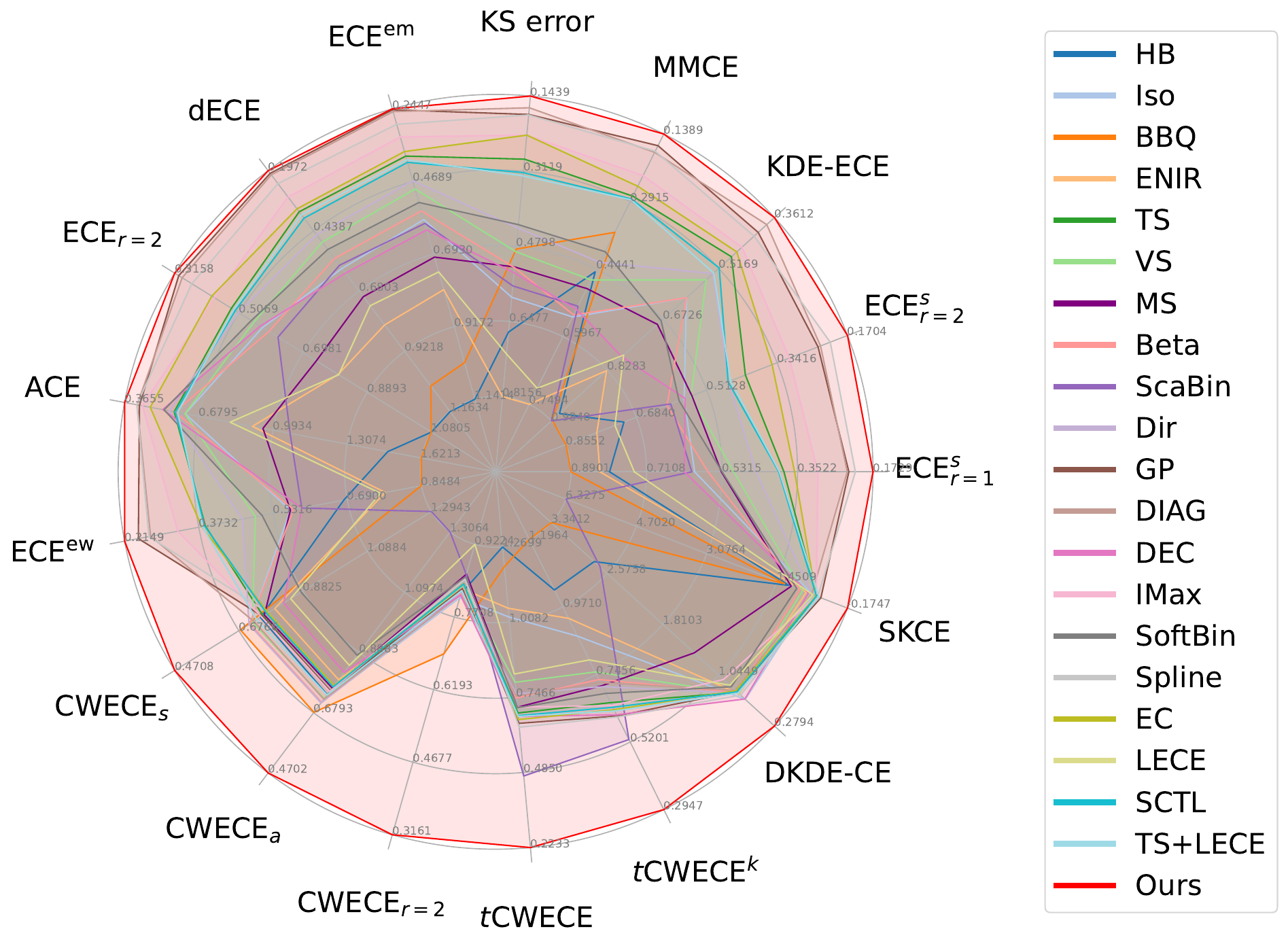}
\caption{Average relative calibration error (ARE) across all metrics for all methods, with our approach achieving the best performance on different metrics. (Scaled-down version, full-size image in \emph{Appendix} \ref{sec:apdx-visulaizationcomparison})}
\label{fig:RadarARE}
\end{figure}

Fig. \ref{fig:limitation} outlines these limitations. \emph{Appendix \ref{sec:apdx-limitationsummary}} provides a tabular summary of each limitation and how it is addressed by $h$-calibration.

\begin{figure*}[htbp]
\centering
\includegraphics[width=\textwidth]{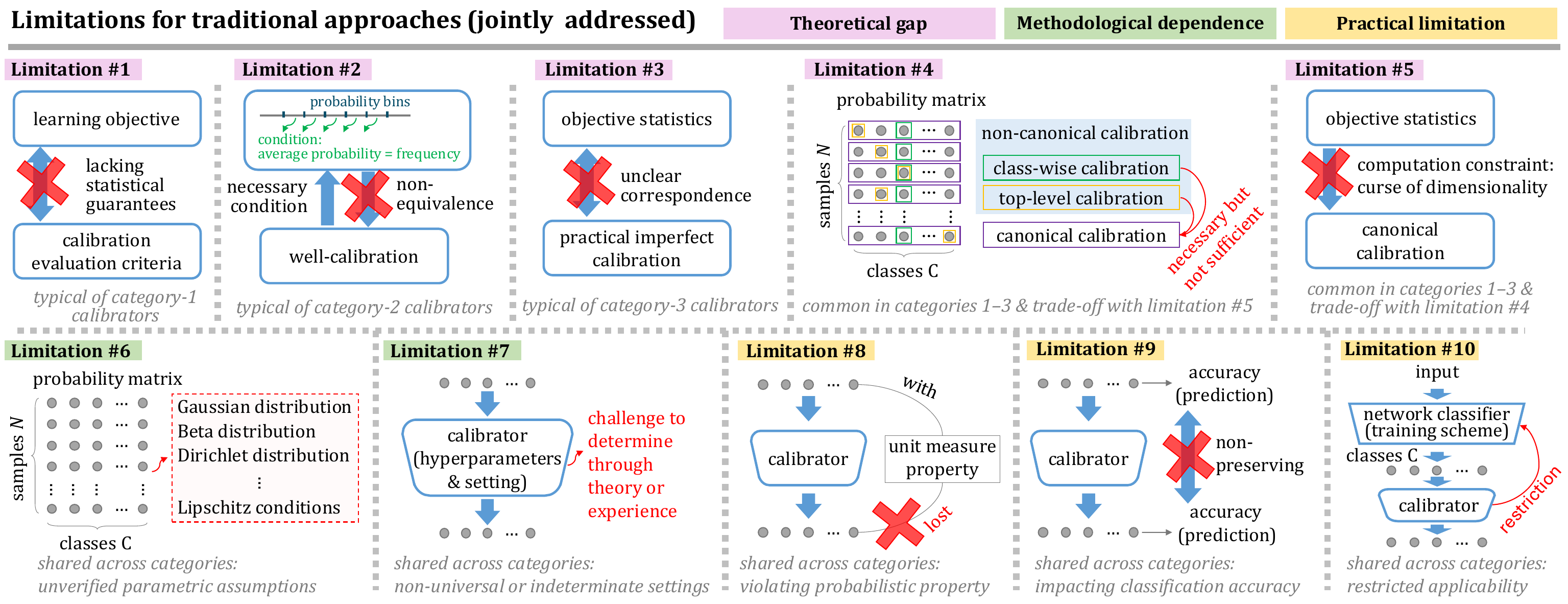}
\caption{Illustration of limitations specific to each strategy, and those shared across methods from different strategies}
\label{fig:limitation}
\end{figure*}

\subsection{Motivation and Contribution of This Study}

In this study, we aim to address these challenges by concentrating on developing learning objective for post-hoc recalibration without modifying the classification model or compromising classification accuracy (resolving \textbf{limitations \#9} and \textbf{\#10}). We: 
(a) propose a definition of uniformly error-bounded calibration compatible with inperfect calibration, offering a more realistic representation of real-world imperfect calibration (addressing \textbf{limitation \#3}); 
(b) construct a theoretical framework directly linking it to common theoretical definitions and empirical evaluators for calibration (resolving \textbf{limitations \#1} and \textbf{\#4}); 
(c) derive statistically error-controllable equivalent forms of error-bounded calibration, along with an equivalent differentiable learning criterion (solving \textbf{limitation \#2}). Notably, our differentiable learning criterion for canonical calibration imposes constrains directly on one-dimensional scalars, avoiding the estimation of high-dimensional distributions (addressing \textbf{limitation \#5}). Furthermore, the proposed approach does not rely on any parametric assumptions (resolving \textbf{limitation \#6}) and does not compromise the unit measure property (resolving \textbf{limitation \#8}). The designed objectives can be independently optimized without the need of setting complex non-interpretable hyperparameters (resolving \textbf{limitation \#7}). Fig. \ref{fig:studystructure} illustrates the study structure. Our contributions can be outlined as follows: 
\begin{itemize}[wide]
    \item A detailed overview of prior research, summarizing and analyzing the existing deficiencies and their underlying causes in learning calibration for classification, with the primary focus on learning objective design.
    \item Introducing a novel practical error-bounded form of well-calibrated probability, compatible with ideal and realistic inperfect calibration, and a probabilistic framework revealing its theoretical relationship with existing definitions and empirical evaluations of calibration.
    \item Providing a series of tools, including constructing equivalent constraining statistics for hypotheses of error-bounded calibration based on large deviation theory and deriving differentiable based on intergral transformations, yielding a differentiable equivalent optimization objective for error-bounded calibration. Unlike existing research, such as empirical methods, learning necessary condition for calibration by binning-based approaches, or investigating conditions for ideal calibration with limited effectiveness (in a sense of non-canonicality or uncontronable errors by the curse-of-dimensionality), our proposed method avoids the diverse deficiencies in prior research by providing a approach for converting error-bounded hypothesis of canonical calibration into equivalent differentiable learning objectives with controllable errors.
    \item Based on the above theoretical analysis, we provide a simple yet effective implementation algorithm of post-hoc recalibration and validate its state-of-the-art performance through extensive experiments across models, metrics, and tasks, as shown in Fig. \ref{fig:RadarARE}.
\end{itemize}

\begin{figure}[htbp]
\centering
\includegraphics[width=\linewidth]{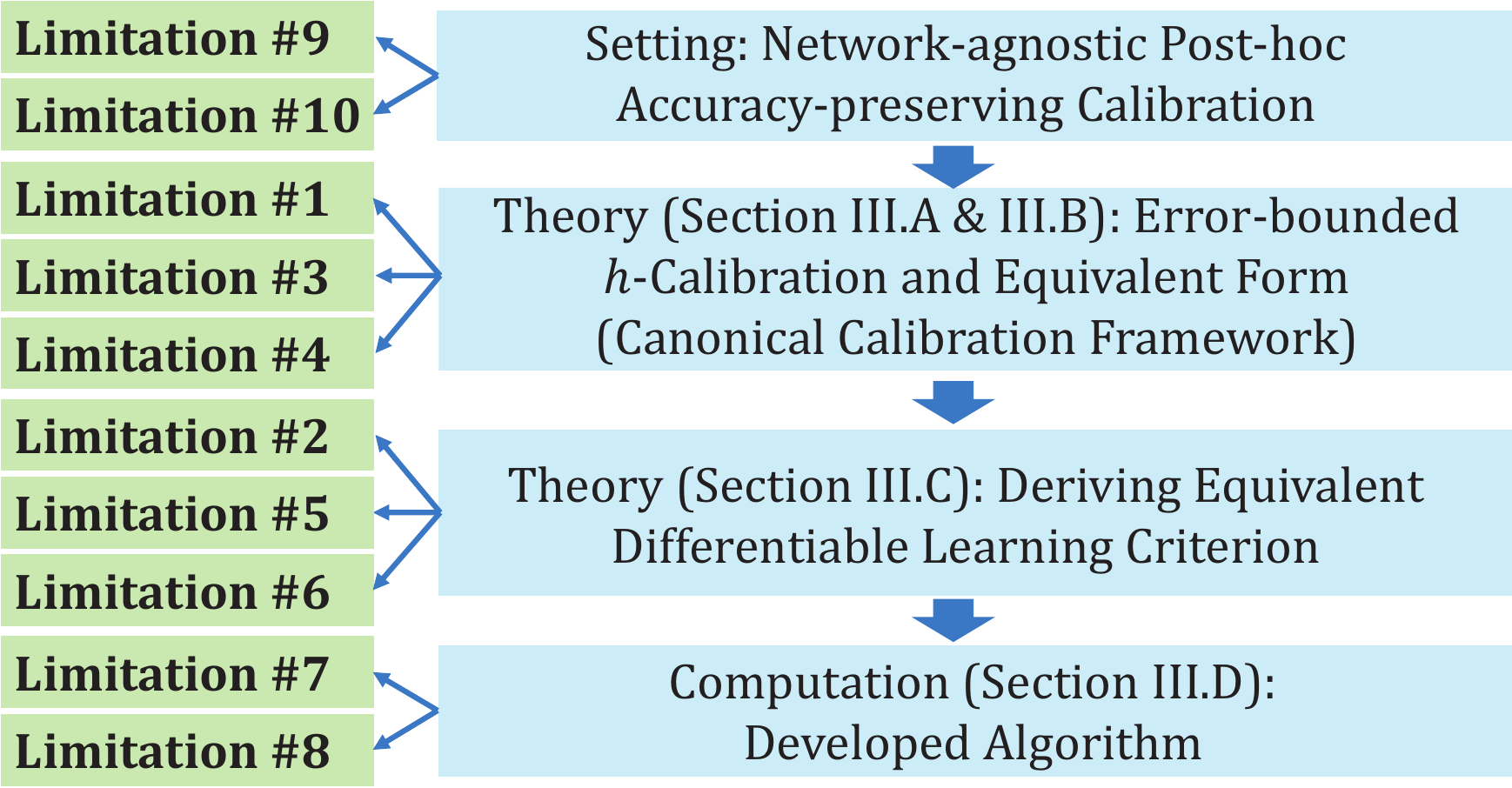}
\caption{Study structure resolving highlighted limitations}
\label{fig:studystructure}
\end{figure}

In the subsequent section of related works, we initially review different levels of calibration and evaluation metrics. Following this, we present an overview of PSR and provide an explanation for why they may fall short in ensuring effective calibration as an response to the experimental conclusion by previous research. Subsequently, we provide a comprehensive summary of both training-time and post-hoc calibration methods. Within the method section, we first introduce a probabilistic framework for error-bounded calibrated (EBA) probability, elucidating the connection between EBA and theoretical calibration definitions and common evaluators. We further formulate the equivalent statistics constraints for EBA hypotheses and the corresponding differentiable learning objectives. This section will conclude with the implementation of a simple recalibration algorithm, along with a interpretation of its relation to PSR. Finally, we demonstrate the experiments results and comparisons.

\section{Related Work and Analysis}
\subsection{Definitions and Evaluators for Calibration}
\label{II.A}
There exists variation in the definitions and evaluation metrics used for calibration in numerous previous works. This paper adopts a nomenclature consistent with many existing literature. Fundamental symbolic definition of calibration can be unified as $p_\mu(\mathscr{E}|\mathscr{H}(F))$$=$$\mathscr{H}(F)$. Here, $\mathscr{H}(F)$ represents specific predicted probability based on feature $F$, and $\mathscr{E}$ denotes the event or random variable (r.v.) corresponding to the predicted probability. $p_\mu$ represents ground-truth probability measure. Different choices for $\mathscr E$ and $\mathscr H(F)$ lead to different calibration definitions, with the most common ones being top-label, classwise, and canonical calibrations, as detailed in Table \ref{tab1}.
\begin{table*}[htbp]
\centering
\caption{Different types of definitions regarding calibration $p_\mu(\mathscr{E}|\mathscr{H}(F)) = \mathscr{H}(F)$}
\begin{threeparttable} 
\setlength{\tabcolsep}{3pt} 
\begin{tabular}{llll}
\toprule
Definition  & Condition $\mathscr H(F)$ & Event or random variable $\mathscr E$ & References\tnote{2} \\ 
\midrule
Top-label    & $\max\limits_l p_c( Y=l| F)$   & $Y = \underset{l}{\operatorname {argmax}} ~p_c( Y=l| F)$ & \begin{tabular}{@{}l} \cite{no.2, no.5, no.12, no.46, no.52, no.71, no.75} \\ \cite{no.34, no.79, no.60, no.136, no.133, no.51, no.16} \end{tabular} \\ 
Classwise (for any $l$) & $p_c(Y=l | F)$  & $Y=l$  & \cite{no.2, no.5, no.46, no.52, no.71, no.75, no.60, no.136, no.133}  \\
Canonical & $\big[ p_c( Y=1 | F),...,p_c( Y=L | F) \big]^\top$ & $\big[ \mathds 1_{\{Y=1\}},...,\mathds 1_{\{Y=L\}} \big]^\top$                                                & \begin{tabular}{@{}l}\cite{no.2, no.70, no.5, no.46, no.52, no.71, no.74, no.75} \\ \cite{no.79, no.60, no.136, no.75, no.103, no.51, no.39, no.102} \end{tabular} \\ \bottomrule
\end{tabular}
\begin{tablenotes}  
\footnotesize 
\item[2] sharing similar mathematical definitions, but possibly employing different terminology.
\end{tablenotes} 
\end{threeparttable}
\label{tab1}
\end{table*}
The $p_c$ denotes the predicted probability. For notation convenience, we do not distinguish between the event and its indicator function representation (e.g., $\mathds 1_{\{Y=l\}}$ and $Y$$=$$l$). 

Both top-label and classwise calibrations are weaker forms of the canonical calibration, with the latter being a sufficient but not necessary condition for the former. Various discretized approximations for these definitions serve as metrics for evaluation. Existing metrics related to top-label calibration include ECE, ACE, MCE and their variants, as well as KS error \cite{no.60}, KDE-ECE \cite{no.70}, MMCE \cite{no.83}. For classwise calibration, CWECE metric and its variants are commonly adopted for evaluation. Metrics associated with canonical calibration include SKCE \cite{no.79}, DKDE-CE \cite{no.52}. While PSRs such as NLL and Brier score have been used in some studies for evaluation, they can be decomposed into multiple factors beyond just calibration and thus do not serve as direct metrics for calibration \cite{no.3,no.10,no.51,no.125,no.130}. In addition, another commonly used evaluation technique is the visualization method of reliability diagram, which is often employed to show top-label calibration. We will introduce and summarize these various metrics within a unified probabilistic framework in \emph{Appendix \ref{sec:apdx-evaluationsummary}}.

\subsection{Calibration by Modified Training Scheme}

Current calibration approaches can be roughly divided into post-hoc recalibrators and training-time calibrations by modified training schemes. Modified training schemes aim to enhance calibration during the training of classifiers and can be broadly categorized into four types: (a) augmentation or implicit regularization, e.g., \cite{no.73,no.49,no.97,no.72,no.93,no.111,no.94,no.95,no.118,no.43,no.24,no.104,no.109,no.12,no.18,no.44} (b) model ensembling, e.g., \cite{no.77,no.139,no.114,no.113,no.106,no.28,no.115,no.20,no.33}, (c) regularization by explicit loss functions, e.g., \cite{no.64,no.83,no.94,no.62,no.58,no.99,no.102,no.100,no.101,no.53,no.84,no.103,no.107,no.13,no.4,no.52,no.66,no.22,no.60,no.48,no.14-0,no.54,no.63,no.11}and (d) some hybrid methods, e.g., \cite{no.24,no.16,no.19,no.55,no.39,no.42}. A detailed summary of the studies can be found in \emph{Appendix \ref{sec:trainingtime-calibration-summary}}. 

However, modified training schemes are subject to \textbf{limitation \#9} of requiring retraining, incurring substantial computational costs and non-preservation of original network's accuracy. For other limitations associated with these methods, please refer to the introduction. For example, methods within ensemble, augmentation or implicit regularization categories, often lack theoretical interpretations directly related to common evaluation metrics (\textbf{limitation \#1}).  Additionally, many augmentation methods depend on specific input modalities, and ensemble methods often involve modifying networks, limiting their applicability (\textbf{limitation \#10}). Moreover, numerous explicit loss-based regularizations serve as auxiliary objectives rather than independent optimization targets, posing challenges in empirically determining appropriate loss weights (\textbf{limitation \#7}).

\subsection{Proper Scoring Rules (PSR): A Revisited Analysis}
\label{sec:psranalysis}
As discussed in the introduction, PSRs are widely employed as loss functions for recalibrators, such as cross-entropy \cite{no.20,no.40,no.57,no.67,no.78,no.80,no.109} and MSE \cite{no.21,no.23,no.70,no.75,no.86,no.88} losses, mainly attributed to the theoretical property that optimal values of PSR are achieved when the forecaster predicts the true probabilities of events. However, recent studies suggest that, experimentally, PSR do not guarantee reliable calibration \cite{no.24,no.1,no.16,no.22,no.52,no.58,no.60,no.62,no.64,no.14-1,no.24,no.41,no.76,no.13,no.53,no.63}. Yet, some preliminary theoretical studies, conducted from the perspective of model families and training procedures \cite{no.11}, suggest that when optimizing certain networks by appropriate scoring rules reaches a state of local optimality, such that the loss cannot be significantly reduced by adding a few more layers, calibration performance can be ensured. To the best of our knowledge, there is currently no widely accepted explanation for why PSR often fail to yield satisfactory calibration. This paper provides an explanation, from the perspective of computational errors, suggesting that such miscalibration might be attributed to intrinsic estimation errors arising from single observation of the conditional distribution of $Y$ given $F$.

Let us begin by revisiting the definition of a PSR \cite{no.123}. Let $\Omega$ denote the sample space, $\mathcal D$ be an $\sigma$-algebra of subsets of $\Omega$, and $\mathcal P$ be a convex class of probability measures/forecasts on $(\Omega, \mathcal D)$. A scoring rule is any function $S$:$\mathcal P \times \Omega$$\rightarrow$$\bar R=[-\infty,\infty]$ such that $S(P,\cdot)$ is measurable with respect to $\mathcal D$ and quasi-integrable with respect to all $Q \in \mathcal P$. $\mathcal S(P,Q)=\int S(P,\omega) d Q(\omega)$ is defined for the expected score under $Q$ when the probabilistic forecast is $P$. The scoring rule $S$ is proper relative to $\mathcal P$ if $\mathcal S(Q,Q)$$\geq$$\mathcal S(P,Q)$ for all $P,Q \in \mathcal P$. Strictly proper implies the equality holds if and only if $P=Q$. According to Savage's representation of PSR \cite{no.123}, it can be shown that both the Brier score $S(\mathbf{p}, \omega)=-\sum_{\upsilon \in \Omega} (\delta_{\omega \upsilon}-p_\upsilon)^2$ and logarithmic score $S(\mathbf{p},\omega)=\log(p_\omega)$ are strictly PSRs, where $\delta_{\omega,\upsilon}=\mathds 1_{\{\omega=\upsilon\}}$.

Although PSR are theoretically effective for estimating true discriminative probabilities, there are practical challenges due to the fact that the true conditional probabilities corresponding to each sample feature (such as logits or probabilities from an uncalibrated classifier) do not align perfectly with the one-hot vector labels. In this context, the target label can be considered as label obtained from a single sampling and the computation of the expectation in proper scoring is based solely on a single observation, introducing estimation error. This error can lead to different biases when selecting different PSRs, resulting in the generation of distinct classifiers. It can explain why theoretically effective PSRs may not guarantee calibration and why different PSRs theoretically yield unique solution but result in different outcomes. Detailed explanation is provided below.

\begin{proposition}
\label{pro1}
 For a feature representation $F_i$ in a network and its corresponding observation label $Y_i$, the true conditional probability $p(Y|F_i)$ cannot be guaranteed to be the one-hot vector of the target label $Y_i$.
\end{proposition}

The discussion for Prop. \ref{pro1} is given in \emph{Appendix \ref{sec:apdx-propsition-hardlabeling}}. When performing post-hoc recalibration on $F$ obtained from a trained neural network, it is common to select $F$ as the logit representation and introduce a mapping $g$ such that $g(F_i) \approx p_{Y|F}(F_i,*)=p(Y=*|F=F_i)$. By the definition of a PSR, we have
\begin{equation}
\label{eq1}
    \mathcal S(g(F_i),p_{Y|F}(F_i,*))=\int S(g(F_i),y)p_{Y|F}(F_i,dy)
\end{equation}
and, theoretically, optimization by $\inf\limits_{g}\mathcal S(g(F_i),p_{Y|F}(F_i,*))$ can lead to $g(F_i) = p_{Y|F}(F_i,*)$. However, the integration on the right side cannot be directly computed since the true distribution $p_{Y|F}(F_i,*)$ is unknown, requiring sampling instead. If $p(Y|F=F_i)$ is guaranteed to be the one-hot vector of the target label $Y_i$, the expectation estimation in Eq.\eqref{eq1} would only require a single sampling. Yet, by Prop. \ref{pro1}, this assumption is not guaranteed. For the logit-label pair $(F_i,Y_i)$ from the calibration set, it can only be regarded as a single sampling from distribution $p_{Y|F}(F_i,*)$. The optimization objective of Eq.\eqref{eq1} reduces to $S(g(F_i),Y_i)$, showing clear approximation errors.

Therefore, while theoretically, PSR can yield well-calibrated probabilities, in practice, inherent inductive bias arising from the disparity between network families and the true posterior probability distribution, coupled with the impact of approximation errors, makes it challenging to capture the genuine conditional distribution. Moreover, there is a tendency to overfit to observed label, resulting in miscalibration, particularly leaning towards overconfidence. Here, we present a conceptual framework from the perspective of approximation error, elucidating why, in practical applications, theoretically effective PSR may fall short of ensuring robust calibration.

While prior works \cite{carrell2022calibration,berta2025rethinking} have examined the effect of using PSRs on test calibration error from a population loss perspective, they do not explore in depth why minimizing PSR reduces training calibration error but fails to generalize. In contrast, we interpret this gap through the lens of per-sample label bias. Specifically, each training sample is typically annotated with a single one-hot label that captures only the dominant foreground semantics, while ignoring relevant covariate context (e.g., labeling an image as “car” may neglecting “road”). Minimizing PSR on such labels can lead to overfitting to dominant semantics and overlooking meaningful context. We formalize this as an insufficient label sampling issue from statistical perspective, where single one-hot semantic representation can lead to overfitting to biased training labels and to a generalization gap.

\subsection{Post-hoc Calibration}

Post-hoc calibration methods aim to improve the calibration of previously trained models by transforming predictions using hold-out validation/calibration data. One advantage of post-hoc methods is that they do not involve retraining, mitigating the impact on model's classification accuracy, even preserving it unaffected. Post-hoc methods can be categorized into parametric and non-parametric methods: (a) the former use parametric models to design learning objectives, e.g., \cite{no.85,no.151,no.65,no.71,no.127,no.150,no.68}, while non-parametric methods can be further classified into five categories by learning objectives, including (b) objectives inspired by binning-based evaluation metrics, e.g., \cite{no.86,no.87,no.88,no.89,no.21,no.125,no.65,no.17,no.58,no.126,no.133}, (c) constructing equivalent forms for ideal calibration, e.g., \cite{no.60,no.70,no.145,no.52}, (d) other methods using PSRs, e.g., \cite{no.20,no.23,no.40,no.57,no.67,no.78,no.80,no.124,no.135,no.146,no.109}, (e) other empirical methods, e.g., \cite{no.29,no.34,no.46,no.128,no.147,no.138,no.145}, and (f) hybrid strategies, e.g., \cite{no.70,no.75}. 

A detailed summary of the above methods is provided in \emph{Appendix \ref{sec:posthoc-calibration-summary}}. However, these existing methods have corresponding limitations. For example, methods in (a) involve unproven parametric assumptions and lack direct connections between methodological design and calibration evaluators (\textbf{limitation \#6} and \textbf{\#1}). Methods in (b) are prone to overfitting necessary conditions for calibration (\textbf{limitation \#2}) and face challenges in determining binning schemes (\textbf{limitation \#7}). Methods in (c) fall short in describing the equivalent form for real-world imperfect calibration (\textbf{limitation \#3}) and in learning canonical calibration (\textbf{limitation \#4}). Methods in (d), based on PSRs, are susceptible to estimation errors, and empirical designs in (e) also lack direct connections to evaluation metrics (\textbf{limitation \#1}). Further dissussions on limitations specific to each method (e.g., the connection of (b) and (c) with \textbf{limitations \#4 and \#5}, one-vs-rest extensions and individual prediction in (b) leading to \textbf{limitation \#8}, as well as numerous method-specific shortcomings in \textbf{limitation \#9)} are outlined in the introduction,  \emph{Appendixes \ref{sec:addtional_comment_limit4}} and \emph{\ref{sec:limitationdiscussion}}.

\begin{figure*}[!h]
\centering
\includegraphics[width=\textwidth]{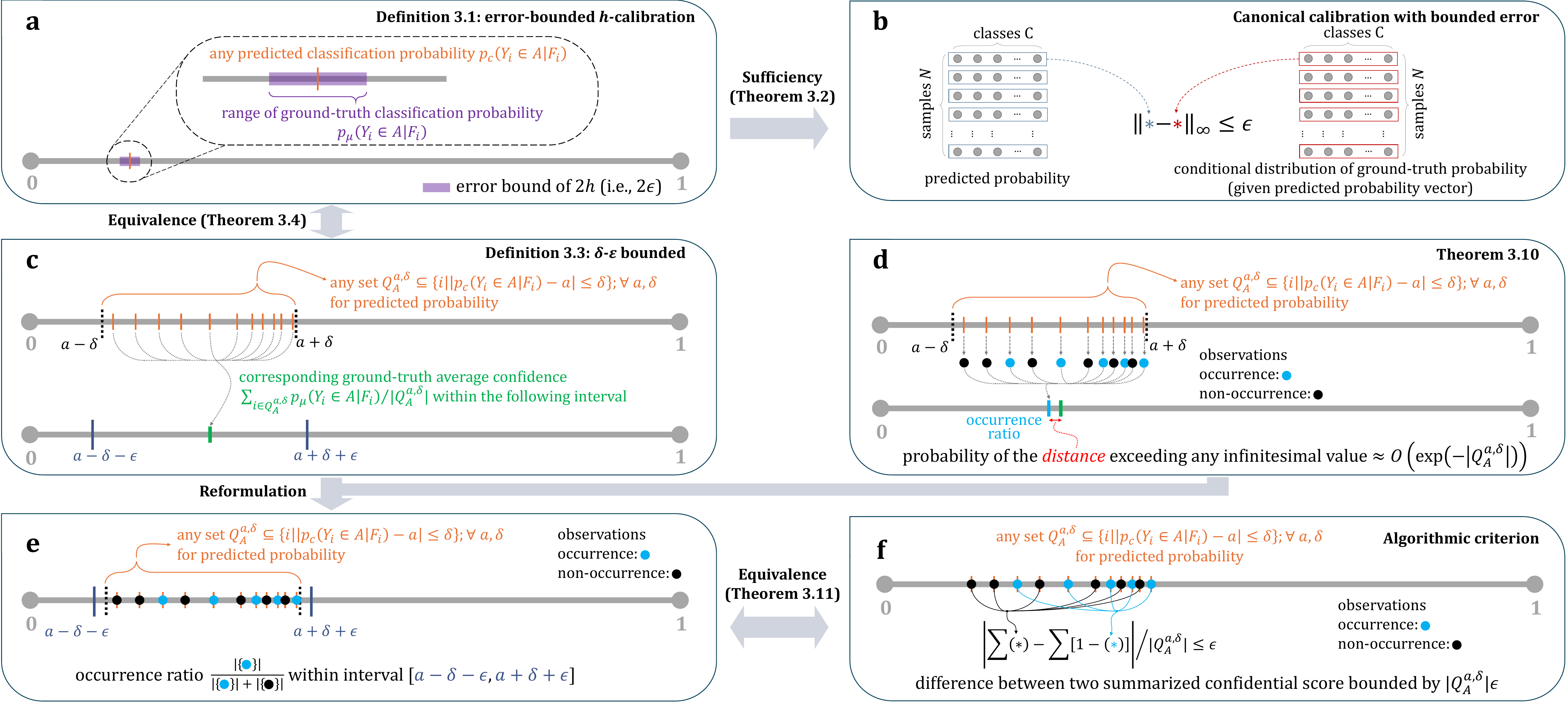}
\caption{A schematic diagram illustrating the $h$-calibration framework, showing the relationships among key definitions and theorems. It shows the sufficiency of $h$-calibration over the traditional canonical calibration definition with bounded error, and how a differentiable algorithmic criterion is designed to learn $h$-calibration.}
\label{fig:keyidea}
\end{figure*}

\section{Method}
In the following sections, we begin by introducing the concept of error-bounded $h$-calibration (EBC), followed by a discussion of its relationship with existing calibration definitions (Section \ref{method:sec1}). Given the inobservability of true sample-wise classification probabilities, which precludes direct attainment of EBC based on its definition, we investigate an equivalent formulation of EBC (Section \ref{method:sec2}) and formulate corresponding approximation statistics with controlled error margins (Section \ref{method:sec3}). Subsequently, we introduce an integral transform technique to convert the non-differentiable constraint statistics into a differentiable form (Section \ref{method:sec3}). We then design a simple algorithmic implementation based on this differentiable form (Section \ref{method:sec4}). Finally, we explore the nexus between our approach and proper scoring from multiple theoretical perspectives (Section \ref{method:sec5}), highlighting that MSE loss can be considered as degenerate forms of our method. Our approach can be regarded as an error-controlled PSR by introducing pseudo sampling. Fig. \ref{fig:keyidea} illustrates the core ideas discussed from Sections \ref{method:sec2} to \ref{method:sec3}. \emph{Appendix \ref{sec:apdx-notation}} presents a table of key notations.

\subsection{Error-bounded $h$-calibration}
\label{method:sec1}
Effective calibration of predictive probabilities inherently implies the controllability of the deviation between predicted and true probabilities. In light of this, we bring forth the concept of error-bounded calibration, called $h$-calibration, predicated on the idea that deviations are uniformly bounded, as elucidated in Fig. \ref{fig:keyidea} (a). Formal definition is provided below.
\begin{definition}[$h$-calibrated]
\label{def:h-calibration}
Let feature-label pair $(F,Y)$ be the r.v. in 
testing 
space $\Omega_F \times\Omega_Y$ and $\mathscr F_Y$ be the $\sigma$-field of $Y$. A calibrated probability $p_c$ is called $h$-calibrated if and only if there exists $h \in (0,1)$ for any event $A \in \mathscr F_Y$,
\begin{equation}
    \label{eq:h-calibration-v1}
    |p_{\mu}(Y \in A|F) - p_c(Y \in A|F)| \leq h,
\end{equation}
where $p_\mu$ denotes the true conditional probability of classification. For finite samples, it becomes for any $1 \leq i \leq N$,
\begin{equation}
    \label{eq:h-calibration-v2}
    |p_{\mu}(Y_i \in A|F_i) - p_c(Y_i \in A|F_i)| \leq h.
\end{equation}
\end{definition}

Our Def. \ref{def:h-calibration} is natural for calibrated probabilities with bounded error. We first discuss the relationship between $h$-calibration and canonical calibration.

\begin{theorem}
\label{thm:1}
$h$-calibration is a sufficient condition for generalized canonical calibration with bounded error, i.e., 
\begin{equation}
    \label{eq:cano-calibraton}
    \| p_\mu(\mathscr E|\mathscr H(F))-\mathscr H(F) \|_\infty \leq h,
\end{equation}
where $\mathscr H(F)=[ p_c( Y=1 | F),\dots,p_c( Y=L | F) ]^\top$ and $\mathscr E $$=$$[ \mathds 1_{\{Y=1\}},\dots,\mathds 1_{\{Y=L\}} ]^\top$.
\end{theorem}

In this context, setting $h$ to zero in Eq. \eqref{eq:cano-calibraton} corresponds to the established definition of canonical calibration. Thus, our $h$-calibration presents a more generalized definition, accommodating both imperfect and perfect calibration within the canonical framework (addressing \textbf{limitation \#3}). Although study \cite{no.133} similarly relaxes the calibration error in an error-bounded form, their definition focuses on the weaker class-wise calibration and bounds the gap between $p_\mu(\mathscr E | \mathscr H(F))$ and $\mathscr H(F)$ (see notation in Table \ref{tab1}), which constitutes a necessary condition for our $h$-calibration.
For a detailed proof of Thm. \ref{thm:1}, please see \emph{Appendix \ref{sec:a1}}. \emph{Appendix \ref{sec:apdx-hcalibrationillustrate}} visually compares $h$-calibration and canonical calibration.

\subsection{Equivalent Form for $h$-calibration}
\label{method:sec2}
Since the true probability $p_\mu$ is not directly observable for enforcing the constraint in the $h$-calibration definition, we propose devising equivalent forms that ensure reliable statistical estimation with controllable error margins, thereby enabling the design of effective learning constraints. In this context, we present a relatively complex definition for well-calibrated classification probability. Thm. \ref{thm:2} establishes its equivalence to $h$-calibration, which will be leveraged in later sections to develop differentiable optimization objectives for learning calibrated probabilities. The proofs for Thm. \ref{thm:2} and similar Thm. \ref{thm:3} can be found in Appendices \ref{sec:a2} and \ref{sec:a3}, respectively.

\begin{definition}[$\delta$-$\epsilon$ bounded]
\label{def:del-eps-calib}
    A calibrated probability $p_c$ is said to be $\delta$-$\epsilon$ bounded if and only if there exists $\epsilon \in (0,1)$, for any interval
    $B_\delta(a)  \triangleq [a-\delta,a+\delta] \subseteq [0,1]$, any $A \in \mathscr F_Y$ and any $Q_A^{B_\delta(a)} \subseteq \left\{ i|p_c(Y_i \in A|F_i) \in B_\delta(a),1\leq i \leq N \right\}$ with $|Q_A^{B_\delta(a)}| \geq1$, we have
    \begin{equation}
        \label{eq:del-eps-calib}
        \left| a -  \frac{ \sum_{i \in Q_A^{B_\delta(a)}}\mathbb{E}_\mu[\mathds 1_A(Y_i)|F_i]}{|Q_A^{B_\delta(a)}|}  \right| \leq \epsilon+\delta,
    \end{equation}
    where $\mathds 1_A(*)$ represents the indicator function for set $A$ and operator $|*|$ computes the cardinal number of a set.
\end{definition}
\begin{theorem}
\label{thm:2}
For finite samples, a calibrated probability $p_c$ is $h$-calibrated if and only if $p_c$ is $\delta$-$\epsilon$ bounded.
\end{theorem}

Fig. \ref{fig:keyidea} (a) and (c) illustrate the core idea of Def. \ref{def:del-eps-calib} and Thm. \ref{thm:2}. Specifically, we construct the reformulated Def. \ref{def:del-eps-calib} because Def. \ref{def:h-calibration} involves the true classification probability $p_\mu$ of a single sample, which is challenging to estimate. In comparison, the expectation term reformulated in Def. \ref{def:del-eps-calib} can be reliably estimated using the statistics of multiple observations. This potentially enables us to construct a differentiable loss function for learning $h$-calibrated probability. Furthermore, the corresponding estimation errors can be analyzed using asymptotic statistical theories, such as the law of large numbers, large deviation theory, etc. Section \ref{method:sec3} will develop Thms. \ref{thm:5}, \ref{thm:6}, and \ref{thm:7} for these issues.

Before proceeding to the next section \ref{method:sec3}, we digress to discuss a pertinent topic: the concept of $h$-calibration can be extended to establish a connection with non-canonical calibration definitions, including top-label and classwise calibrations. In this context, we introduce two concepts: $h$-$\mathcal A$ calibration and $\delta$-$\epsilon$-$\mathcal{A}$ boundedness. Thm. \ref{thm:3} establishes the equivalence between these two concepts.

\begin{definition}[$h$-$\mathcal A$ calibrated]
\label{def:h-a-calib}
    With the notations in Def. \ref{def:h-calibration}, a calibrated probability $p_c$ is called $h$-$\mathcal A$ calibrated if and only if there exists $\mathcal A_i \in \mathscr F_Y$, $1 \leq i \leq N$, we have
    \begin{equation}
        |p_{\mu}(Y_i \in \mathcal A_i|F_i) - p_c(Y_i \in \mathcal A_i|F_i)| \leq h.
    \end{equation}
\end{definition}

\begin{definition}[$\delta$-$\epsilon$-$\mathcal{A}$ bounded] 
\label{def:del-eps-a-calib}
    A calibrated probability $p_c$ is said to be $\delta$-$\epsilon$-$\mathcal{A}$ bounded if and only if there exists $\epsilon \in (0,1)$ and $\mathcal{A}_i \in \mathscr F_Y$, $1\leq i\leq N$, for any 
    $B_\delta(a)$ $\triangleq [a-\delta,a+\delta] \subseteq [0,1] $ and any $Q_{\mathcal{A}}^{B_\delta(a)} \subseteq \left\{ i|p_c(Y_i \in \mathcal{A}_i|F_i) \in B_\delta(a)
    \right\}$ with $|Q_{\mathcal{A}}^{B_\delta(a)} |\geq 1$, we have
    \begin{equation}
    \label{eq:del-eps-a-calib}
    \left| a - \frac{ \sum_{i \in Q_{\mathcal{A}}^{B_\delta(a)}} \mathbb{E}_{\mu}[1_{\mathcal{A}_i}(Y_i)|F_i] }{|Q_{\mathcal{A}}^{B_\delta(a)}|} \right| \leq \epsilon+\delta.
    \end{equation}
\end{definition}

\begin{theorem}
\label{thm:3}
A calibrated probability $p_c$ is $h$-$\mathcal A$ calibrated if and only if $p_c$ is $\delta$-$\epsilon$-$\mathcal A$ bounded. Both conditions are necessary but not sufficient for $h$-calibration.
\end{theorem}

In fact, the notion of $h$-$\mathcal{A}$ boundedness under different $\mathcal{A}_i$, $1 \leq i \leq N$, corresponds to different non-canonical definitions of calibration, such as top-label and classwise calibration. Specifically, if the focus lies in maximal classification probabilities, i.e., 
\begin{equation}
\label{eq:top-label-event}
  \mathcal{A}_i \triangleq \{l | \mathrm{argmax}_{\,l~} p_c( Y_i = l|F_i)\},
\end{equation}
then the corresponding $h$-$\mathcal{A}$ boundedness pertains to top-label calibration. Alternatively, If the interest is in the prediction probabilities of specific classes, i.e., 
\begin{equation}
\label{eq:classwise-event}
\mathcal{A}_i = \{l\}
\end{equation}
for fixed class $l$, the corresponding $h$-$\mathcal{A}$ boundedness relates to the classwise calibration for class $l$. The formal description is as follows and the proof is given in \emph{Appendix \ref{sec:a4}}.

\begin{theorem}
\label{thm:4}
For the $\mathcal{A}_i$ specified in Eq. \eqref{eq:top-label-event} or Eq. \eqref{eq:classwise-event}, the corresponding $h$-$\mathcal{A}$ calibrations are sufficient for top-label or classwise calibrations, respectively, with uniform error bound $h$. That is, it holds that $| p_\mu (Y_i \in \mathcal A_i | p_c (Y_i \in \mathcal A_i | F_i)) - p_c (Y_i \in \mathcal A_i | F_i)| \leq h$, for the $\mathcal A_i$ defined in Eq. \eqref{eq:top-label-event} or Eq. \eqref{eq:classwise-event}, respectively.
\end{theorem}

Above theorems show our $h$-calibration implies canonical calibration, and the weaker $h$-$\mathcal A$-calibration implies non-canonical cases, e.g., top-label and classwise calibration. Furthermore, common calibration estimators are essentially approximations of these three theoretical calibration definitions, as detailed in \emph{Appendix \ref{sec:apdx-evaluationsummary}}. Thus, this study addresses \textbf{limitations \#1} and \textbf{\#4} highlighted in the introduction.

In the following, we will construct approximation of the equivalent form of $h$-calibration given in Def. \ref{def:del-eps-calib}, with controllable error. This allows us to design effective learning objectives that yield canonically calibrated probabilities.

\subsection{Differentiable Objective to Learn $h$-calibration}
\label{method:sec3}

The preceding section presents rigorous probabilistic definitions of well-calibrated probability. However, how to learn a $\delta$-$\epsilon$ bounded probability (equivalently, $h$-calibration) still needs to be addressed. This section will show the conversion of $\delta$-$\epsilon$ boundedness into a differentiable objective. We first give Thms. \ref{thm:5} and \ref{thm:6}, showing that the statistics of indicator functions can approximate the expectation term in $\delta$-$\epsilon$ boundedness, as depicted in Fig. \ref{fig:keyidea} (d).
\begin{theorem}
\label{thm:5}
With the notations in Def. \ref{def:del-eps-calib},
\begin{equation*}
\label{eq:thm:5}
\textstyle
\left| \frac{ \sum_{i \in Q_{A}^{B_\delta(a)}}\mathds 1_{A}(Y_i)}{|Q_{A}^{B_\delta(a)}|}
- \frac{ \sum_{i \in Q_A^{B_\delta(a)}}\mathbb{E}_\mu[\mathds 1_A(Y_i)|F_i]}{|Q_A^{B_\delta(a)}|}  \right| 
\xrightarrow[p_\mu\&L^2\&a.s.]{|Q_A^{B_\delta(a)}| \rightarrow \infty}   0.
\end{equation*}
\end{theorem}
\begin{theorem}
\label{thm:6}
With the notations in Def. \ref{def:del-eps-calib}, The difference term in Thm. \ref{thm:5} converges exponentially to zero in $p_\mu$ as $|Q_A^{B_\delta(a)}|\rightarrow \infty$, i.e., for any $\kappa>0$,
\begin{equation*}
\textstyle
    p_\mu \left(
       \left| \frac{ \sum_{i \in Q_{A}^{B_\delta(a)}}\mathds 1_{A}(Y_i)}{|Q_{A}^{B_\delta(a)}|} - \frac{ \sum_{i \in Q_A^{B_\delta(a)}}\mathbb{E}_\mu[\mathds 1_A(Y_i)|F_i]}{|Q_A^{B_\delta(a)}|}   \right| > \kappa \right) 
\end{equation*}
converges to zero exponentially as $|Q_A^{B_\delta(a)}|\rightarrow \infty$.
\end{theorem}

The proofs for the above Thms. \ref{thm:5} and \ref{thm:6} are given in \emph{Appendix \ref{sec:a5}} and \emph{\ref{sec:a6}}. Thms. \ref{thm:5} and \ref{thm:6} ensure that, for a large $|Q_A^{B_\delta(a)}|$, the expectation term in the $\delta$-$\epsilon$ boundedness can be conveniently replaced with the statistics on the left-hand side (LHS) of Eq.\eqref{eq:thm:5}. Then the inequality in Eq.\eqref{eq:del-eps-calib} of $\delta$-$\epsilon$ boundedness can be rewritten as:
\begin{equation}
  \textstyle
    \label{eq:rm-exp}
    \left| a -  \frac{ \sum_{i \in Q_A^{B_\delta(a)}}\mathds 1_A(Y_i)}{|Q_A^{B_\delta(a)}|} \right| \leq \epsilon + \delta.
\end{equation}

The indicator function-based statistics $\sum\nolimits_{i \in Q_A^{B_\delta(a)}} \mathds 1_A(Y_i)$ in Eq.\eqref{eq:rm-exp} are not differentiable with respect to $p_c(Y_i \in A|F_i)$. This precludes the effective application of gradient backpropagation. To turn it into a differentiable form, we further present a necessary and sufficient condition for Eq.\eqref{eq:rm-exp} in Thm. \ref{thm:7}. Eq.\eqref{eq:rm-exp} and Thm. \ref{thm:7} provide an approximately equivalent condition and its differentiable form with controllable error, respectively, to learn $h$-calibration. By optimizing this effective differentiable surrogate, we can circumvent overfitting to necessary condition for calibration, thereby overcoming \textbf{limitation \#2}.
\begin{theorem}
\label{thm:7}
    With the notations in Def. \ref{def:del-eps-calib}, a necessary and sufficient condition for
    \begin{equation}
    \textstyle
        \left| a -  \frac{ \sum_{i \in Q_A^{B_\delta(a)}}\mathds 1_A(Y_i)}{|Q_A^{B_\delta(a)}|} \right| \leq \epsilon + \delta,
    \end{equation}
    i.e., Eq. \eqref{eq:rm-exp}, is that the following inequality holds:
    \begin{equation}
    \label{eq:thm7regterm}
    \mathscr T(A, [R_1,R_2], Q_A^{\overline{R_1 R_2}}) \triangleq  
    \left|     T_1   -    T_2    \right| /    T_3     ~~~~\leq \epsilon 
    \end{equation}
    for any $A \in \mathscr F_Y$, $0\leq R_1 < R_2 \leq 1$, $Q_A^{\overline{R_1 R_2}} \subseteq \{i | R_1 \leq {^cp_i^A} \leq R_2\}$ with $|Q_A^{\overline{R_1R_2}}|>0$, where
    {\small
    \begin{gather}
        T_1 = \mathop{2\sum(1-{^cp_i^A})}\limits_{\hspace{-2em} Q_A^{\overset{\frown}{R_1R_2}} \cap O_A}  + \mathop{\sum (1-{^cp_i^A})}\limits_{\hspace{-2em} ( Q_A^{R_1} \cup Q_A^{R_2} ) \cap O_A}, \label{eq:t1_definition} \\
        T_2 =  \mathop{2\sum {^cp_i^A}}\limits_{ Q_A^{\overset{\frown}{R_1R_2}} \cap O_{A^{\mathbf C}}} + \mathop{\sum {^cp_i^A},}\limits_{(Q_A^{R_1} \cup Q_A^{R_2}) \cap O_{A^{\mathbf C}} } \label{eq:t2_definition} \\
        T_3 = 2 \sum\nolimits_{Q_A^{\overset{\frown}{R_1R_2}} } 1 + \sum\nolimits_{Q_A^{R_1} \cup Q_A^{R_2}} 1 , \label{eq:t3_definition} \\
        ^cp_i^A \triangleq p_c(Y_i \in A|F_i), \\
        Q_A^{R_1} \triangleq  \{i| {^cp_i^A=R_1},i \in Q_A^{\overline{R_1R_2}}  \}, 
        \label{eq:qr1}\\
        Q_A^{R_2} \triangleq  \{i| {^cp_i^A=R_2},i \in Q_A^{\overline{R_1R_2}}  \}, \label{eq:qr2} \\
        Q_A^{\overset{\frown}{R_1 R_2}} \triangleq  \{i|  R_1 < {^cp_i^A}< R_2,i \in Q_A^{\overline{R_1R_2}}  \}, \label{eq:qri_qr2_open}\\
        O_A \triangleq \{i|Y_i \in A\}, ~~O_{A^{\mathbf C}} \triangleq \{i|Y_i \notin A\}.
    \end{gather}}
\leavevmode
\end{theorem}

See \emph{Appendix \ref{sec:a7}} for detailed proof of Thm. \ref{thm:7}. 
The $T_1$-$T_3$ terms can be simplified (see Eq. \ref{eq:rewriteT1} – \ref{eq:rewriteT3}), yielding the criterion in Fig. \ref{fig:keyidea}(f). It is noteworthy that Thm. \ref{thm:7}'s constraints avoid the need to estimate high-dimensional sample densities, focusing instead on 1-D scalars. This circumvents the issue of uncontrollable errors in high-dimensional density or distribution estimation (addressing \textbf{limitation \#5}). Moreover, the proofs for above theorems do not introduce any unverified parametric assumptions (addressing \textbf{limitation \#6}). 

\emph{Appendix \ref{sec:apdx-probinterp}} presents a probabilistic explanation for the above theorems. In short, \textbf{Thm. \ref{thm:7} turns a non-differentiable constraint into an equivalent differentiable objective through an integral transformation, which enables training a calibration model to learn $h$-calibrated probability.} Fig. \ref{fig:keyidea} (f) is provided to intuitively illustrate this objective.

\begin{algorithm}
\caption{Calibration by Our Objective}
\label{algo1}       
\SetAlgoLined
\KwIn{Calibration training set $D_{trn}$, calibration mapping $g_\theta$, constants $\epsilon$, $r$, $M$, batch size $B$, learning rate $\eta$, training iteration $T$}
\KwOut{Calibrator $g_{\hat{\theta}}$}

\For{$t \leftarrow 1$ \KwTo $T$}{
    Sample logit-label pairs $\{(F_i, Y_i)\}_{i=1}^B$ from $D_{trn}$\;

    Compute calibrated probability $p \triangleq { \{g_\theta (F_i) \}}_{i=1}^B$\;
    Sort $\{p(Y_i \in A|F_i) | 1\leq i \leq B, A \in \mathscr A \}$ to obtain sorting index $\mathbf{u}$ for $\{(i,A)\}$\;
    Apply convolution with kernel $[1, \ldots, 1]$ of length $M$ to $(1-{p_i^A}) \mathds 1_{\{Y_i \in A\}}$ and  $p_i^A \mathds 1_{\{Y_i \notin A\}}$ flattened by $\mathbf u$, yielding $\mathscr{V}_1$ and $\mathscr{V}_2$ vectors\footnotemark[3], respectively\;

    Calculate $\mathscr L$ vector\footnotemark[3] by Eq.\eqref{eq:EventEnsembleLoss}, i.e., \\
    ~~$\mathscr L = \mathrm{ReLU}(  {|\mathscr{V}_1  -  \mathscr{V}_2 |}/{M} - \epsilon) $ \;

    Yield $w$ vector\footnotemark[3] by feeding every $M$ consecutive elements of sorted ${p_i^A}$ to weighting function $w$\;
    
    Compute inner product $\mathcal L = r  \langle w, \mathscr L \rangle $ by Eq.\eqref{eq:finalloss}\;

    Update parameter $\theta_{t+1}$ with gradient $\eta \cdot \nabla_{\theta_{t}} \mathcal L$ \;

    Evaluate $g_{\theta_{t+1}}$ calibration performance on $D_{trn}$. 
}
\Return{Best performed model $g_{\hat{\theta}}$}
\end{algorithm}

\footnotetext[3]{Each element in the vector is associated with a specific $\mathscr R$. Symbols $\mathscr V_1$ and $\mathscr V_2$ correspond to $ \sum\nolimits_{A \in \mathscr A}\widetilde T_1^A $ and $ \sum\nolimits_{A \in \mathscr A}\widetilde T_2^A$, respectively, with each value in the vector likewise associated with a specific $\mathscr R$.}
\setcounter{footnote}{3}

\subsection{Calibration Algorithm by Differentiable Objective}
\label{method:sec4}
The preceding section has established the theoretical foundation for learning $h$-calibrated probability with a differentiable objective. However, when designing a computational algorithm, specific settings have to be determined, including the collection of event set $A$, the regularizing intervals $[R_1,R_2]$ in Thm. \ref{thm:7}, and the specific form of the loss function. Here we present a simple, specific implementation algorithm to verify the effectiveness of our theory. 

\begin{description}[style=unboxed,leftmargin=0cm]
  \item[(Event $A$ and Set $Q_A^{\overline{R_1 R_2}}$)] Regarding subset $A$ in space $\Omega_Y$, there are a total of $2^L$ distinct event sets for classification tasks with $L$ categories. For computational efficiency, here we consider atomic events as the constrained events, denoted as $A \in \mathscr{A} =\{A'|A'\in \mathscr F_Y, |A'|=1\}$, and directly set $Q_A^{\overline{R_1R_2}} = \{i | R_1 \leq {^cp}_i^A \leq R_2\}$ for any given range $[R_1,R_2]$. It is worth noting that it can be proven that substituting $A \in \mathscr{F}_Y$ with $A \in \mathscr{A}$ in Thm. \ref{thm:7} is theoretically equivalent.
  
  \item[(Loss Form)] For any given $[R_1,R_2]$, the reformulated form for $T_1$, $T_2$ and $T_3$ in Thm. \ref{thm:7} can be calculated as follows:
  {
    \begin{equation}
    \label{eq:rewriteT1}
        {\widetilde T^A}_1 = \mathop{\sum (1-{^cp_i^A})}_{Q_A^{\overline{R_1R_2}}\cap \{i|Y_i \in A\}} = \sum\limits_{Q_A^{\overline{R_1R_2}}} (1-{^cp_i^A}) \mathds 1_{\{Y_i \in A\}}
    \end{equation}
    \begin{equation}
    \label{eq:rewriteT2}
        {\widetilde T^A}_2 = \mathop{\sum {^cp_i^A}}_{ Q_A^{\overline{R_1R_2}} \cap \{i|Y_i \in A^{\mathbf C}\}} = \sum\limits_{ Q_A^{\overline{R_1R_2}}}{^cp_i^A} \mathds 1_{\{Y_i \in A^{\mathbf C}\}}
    \end{equation}
    \begin{equation}
    \label{eq:rewriteT3}
        \widetilde T^A_3 = \sum\nolimits_{Q_A^{\overline{R_1R_2}} } 1 = |Q_A^{\overline{R_1R_2}}|
    \end{equation}}
It is noteworthy that, despite being concise compared to the original form of $T_1$, $T_2$ and $T_3$ in Thm. \ref{thm:7} by removing boundary terms for $Q^{R_1}_A$ and $Q^{R_2}_A$, there exists an infinitesimal $\nu$, ensuring that $T_{*}^A$ under $[R_1+\nu,R_2-\nu]$ yields the same value as $\widetilde T_*^A$. Hence these two forms of definition are essentially equivalent in this sense.
With Eq.\eqref{eq:thm7regterm}, we have
\begin{equation}
\label{eq:rewritecaliberror}
    |  \widetilde T_1^A - \widetilde T_2^A| \big/ \widetilde T_3^A  \leq \epsilon, \forall A.
\end{equation}

For a fixed $[R_1, R_2]$, since the distribution of calibrated probabilities for a single event $A$ is relatively sparse compared to multiple events, we propose integrating the probabilities of multiple events into one loss term to obtain a more accurate estimate of the calibration error. For an event set $\mathbb A$, the following inequality can be derived based on Eq.\eqref{eq:rewritecaliberror}:
\begin{equation}
\label{eq:EventLoss}
\textstyle
    \big|  \sum\nolimits_{A \in \mathbb A}\widetilde T_1^A - \sum\nolimits_{A \in \mathbb A}\widetilde T_2^A \big|  \Big/ \sum\nolimits_{A \in \mathbb A}\widetilde T_3^A \leq \epsilon.
\end{equation}
In this study, the event set $\mathbb A$ is set as $\mathscr A$.
Then we define the loss function for $[R_1,R_2]$ of the form
\begin{equation}
\label{eq:EventEnsembleLoss}
\textstyle
  \mathscr L (\mathscr R) =  \max \Big( \Big|  \sum\limits_{A \in \mathscr A}\widetilde T_1^A - \sum\limits_{A \in \mathscr A}\widetilde T_2^A \Big|/{\sum\limits_{A \in \mathscr A}\widetilde T_3^A}-\epsilon,0 \Big)
\end{equation}
For ease of notation, $[R_1,R_2]$ is henceforth referred to as $\mathscr R$.
  \item[(Interval $\mathscr R$)] 
  Here, we demonstrate how to configure $\mathscr R$ for efficient computation of Eq. \eqref{eq:EventEnsembleLoss} for all $\mathscr R$ intervals.
  By ordering the values of ${^cp}_i^A, 1\leq i\leq N, A \in \mathscr A$, we generate the vector $\mathbf q$ (with the sorting index $\mathbf u$), and the span from the minimum to the maximum values among the $M$ consecutive numbers in $\mathbf{q}$ is identified as $\mathscr{R}$.   
  Accordingly, the set for interval $\mathscr R$ is denoted as $\mathcal R \triangleq \{ [\mathbf q_i,\mathbf q_j]| j-i=M\}$. 
  The selection of such $\mathscr R$ enables the direct computation of Eq.\eqref{eq:EventEnsembleLoss} for all $\mathscr R$ intervals using the convolution operation, which can thus be efficiently implemented by deep learning libraries using GPU. Specifically, by Eq.\eqref{eq:rewriteT1}, Eq.\eqref{eq:rewriteT2} and Eq.\eqref{eq:EventEnsembleLoss}, a 1-D convolution operation with constant kernel $[1,1,..,1]$ of length $M$ is applied to vectors $(1-{^cp_i^A}) \mathds 1_{\{Y_i \in A\}}$ and ${^cp_i^A} \mathds 1_{\{Y_i \in A^{\mathbf C}\}}$, sorted by index $\mathbf u$, to calculate $ \sum\nolimits_{A \in \mathscr A}\widetilde T_1^A $ and $ \sum\nolimits_{A \in \mathscr A}\widetilde T_2^A$ for all $\mathscr R$ intervals, respectively. 
  Additionally, by defining $\mathscr R$ as such, $\sum\nolimits_{A \in \mathscr A}\widetilde T_3^A$ is equal to $M$.

\item[(Loss Function)] 
Finally, the weighted average value over all  $\mathscr R$ is is used as the training loss:
\begin{equation}
    \label{eq:finalloss}
    \mathcal L = r  \sum_{\mathscr R \in \mathcal R} w(\mathscr R) \mathscr L(\mathscr R).
\end{equation}
Given that the prediction probabilities for many unobserved atomic events are significantly low, and the ratio of low-probability to high-probability predictions nearly mirrors the ratio of non-occurring to occurring atomic events—this ratio linearly increases as the number of task classes rises. This leads to an extensive collection of regularizing intervals $\mathscr R$ prioritizing low-probability events in multi-class scenarios. Yet, concentrating excessively on low-probability events can bias the evaluation of calibration \cite{no.3,no.16,no.76}. To remedy this bias, we apply a simple k-means based weighting function $w(\mathscr R)$ that adjusts dynamically, based on the idea of assigning adaptive weights to counteract the imbalance distribution of $\mathscr R$ in multi-class scenarios, as detailed in \emph{Appendix \ref{sec:apdx-weighitingfunction}}. A constant multiplier $r$ is introduced to increase the loss value. Please note that $\sum_{\mathscr R \in \mathcal R} w(\mathscr R) \mathscr L(\mathscr R)$ can be efficiently computed by using the convolution operation, as described above. 
The loss function in Eq.\eqref{eq:finalloss} will finally be used to train the calibration mapping. 

\item[(Calibration Mapping)] To preserve original classification accuracy (addressing \textbf{limitation \#9}),
we consider monotonic transformation as the calibration mapping for logits. Due to the inaccessibility of the ideal transformation from uncalibrated logits to their authentic logits, with both the transformation form variations and complexity discrepancies potentially impairing calibration, we follow previous studies to explore a family of mappings and automatically select the optimal one \cite{no.67,no.71,no.84,no.46,no.88}. Specifically, we examine a set of learnable monotonic mappings (see \emph{Appendix \ref{sec:apdx-monotonicmapping}}) and determine the optimal mapping based on calibration performance of the calibration training set. By transforming the logit on a samplewise basis, our method generates probabilistic predictions that sum to one, preserving the unit measure property (avoiding \textbf{limitation \#8}).
\end{description}

Algorithm \ref{algo1} presents the pseudo-code for training our calibrator. \textbf{It is noteworthy that although our theory is generic and relatively complicated from a mathematical point of view, this specific algorithm is surprisingly simple and easy to implement.} Notably, the proposed objective incorporates only two extra hyperparameters, $M$ and $\epsilon$, with each providing a clear theoretical interpretation: $M$ controls the approximation error of the constraint and $\epsilon$ regulates the the upper bound on the calibration error, ensuring parameter setting with intuitive interpretation (tackling \textbf{limitation \#7}). Additionally, the proposed method is applicable to calibrate any trained classifier (avoiding \textbf{limitation \#10}).

\subsection{Relationship to Proper Scoring Rule}
\label{method:sec5}
This section focuses on demonstrating a particular PSR, i.e., MSE, is essentially a degenerate form of our framework from an algorithmic standpoint, and it outlines the theoretical strengths of our general framework in comparison to PSR.

\emph{1) First, we show from the algorithmic perspective that MSE is a degenerate version within our framework:}
When our approach undergoes certain degenerations or slight implementation modifications, it aligns with the Brier scoring rule. 

Specifically, this occurs when reducing the window length to $M$ = $1$, setting the error margin $\epsilon~$=$~0$, replacing the $L_1$ norm in Eq.\eqref{eq:EventEnsembleLoss} with the squared $L_2$ norm, and omitting the adaptive weighting term $w(\mathscr R)$.
Consequently, the the proposed objective in Eq.\eqref{eq:finalloss} reduces to
\begin{equation}
\label{eq:ReduceToMSE}
    r\frac{1}{L} \sum\nolimits_{1\leq l\leq L} \frac{1}{N} \sum\nolimits_{1\leq i \leq N} (^cp_i^l - \mathds 1_{\{Y_i=l\}})^2,
\end{equation}
where $^cp_i^l = p_c(Y_i=l|F_i)$. Eq. \eqref{eq:ReduceToMSE} corresponds to the Brier scoring rule (i.e., the MSE loss function).

\emph{2) Secondly, we illustrate that our general approach holds theoretical advantages over the PSR, which we analyze in two aspects:} In one aspect, under the $\| * \|_{M,\omega}$ distance discussed below, our approach can yield probability forecasts that are closer to the true probabilities than learning through standard PSR. In another, we will interpret the edge of our approach over standard PSR calculation from the standpoint of errors due to insufficient sampling, unveiling our method as a solution to mitigate the overfitting and overconfidence issues in standard PSR through a pseudo sampling strategy.
\begin{itemize}[wide]
    \item Regarding the first aspect, according to Thms. \ref{thm:5} and \ref{thm:6}, the constraint in Thm. \ref{thm:7} proves effective for large values of $|Q_A^{B_\delta(a)}|$. Nonetheless, for small values of $|Q_A^{B_\delta(a)}|$, the constraint may not be reliably met due to approximation error. In our algorithmic implementation, we set a large $M$ to ensure the effectiveness of the constraint. In such cases, it is possible that the calibrated probability might not strictly adhere to a uniform error-bound. In fact, theoretically, in the standard setting where a single label is observed for a given feature without any prior distribution over classification probabilities, no method can strictly guarantee uniformly error-bounded probabilistic forecasts. This naturally raises the question of how effectively our algorithm can minimize the discrepancy between predicted and actual probabilities. Moreover, it is pertinent to explore how the error margin of our method compares with that derived from traditional PSR. Clarifications on these topics are provided below.

    To begin, we introduce a vector distance metric $\| * \|_{M,\omega}$: 
    \begin{equation}
        \label{eq:Mnorm}
        \| \xi - \psi \|_{M,\omega} \triangleq \sum_{\mathscr D \in \mathbb D} \omega(\mathscr D) \Big|
        \frac{\sum_{i \in \mathscr D} \xi_i}{M}
          - \frac{\sum_{i \in \mathscr D} \psi_i}{M}  \Big|
    \end{equation}
    where $\mathbb D$$=$$\big\{\mathscr D$$\subset$$\{1,2,...,\text{dim}(\xi)\} \big| \lvert \mathscr D \rvert$$=$$M\big\}$ and $\omega$ represents a weighting function. It means vector distance is constructed via the weighted averages of mean discrepancies across sub-vectors of length $M$ (reader can verify this definition satisfies the metric axioms). This distance can also be extended to matrix space as $\|  \text{vec}(\xi) - \text{vec}(\psi) \|_{M,w}$, by flattening matrices into vectors before distance calculation. 
    
    We denote the family of mappings for recalibration as $\{g_\theta | \theta$$\in$$\Theta\}$. 
    The proposed learning Algorithm \ref{algo1} can be interpreted as optimizing the Eq. \eqref{eq:equiobj} (where $M$ is the hyperparameter in the algorithm, and $\omega$ corresponds with the specifications of $\mathbb A$, $\mathcal R$ and $\{ w(\mathscr R) | \mathscr R$$\in$$\mathcal R \}$, with a detailed explanation in \emph{Appendix \ref{sec:apdx-explain-obj}}).   
    {\small
    \begin{equation}
    \label{eq:equiobj}
        \min\limits_{g} \Big\|  [ g^l(F_i) ]_{1\leq i\leq N, 1 \leq l \leq  L}\! - [\mathds 1_{\{Y_i = l\}} ]_{1\leq i\leq N, 1 \leq l \leq L} \Big\|_{M,\omega}.
    \end{equation}}
    From the optimization perspective, the prediction probabilities are expressed as $p_c(Y$$=$$l |F_i)$$=$$g^l(F_i)$, with $g$$\in$$\{ g_\theta | \theta$$\in$$\Theta\}$. We denote the optimal solution for the above objective as $p_{M}$.
    In comparison, the PSR objective can be expressed as:
    \begin{equation}
     \min_{g} \iint S(g(f),y) p_{Y|F}(dy,f) p_F(df).
    \end{equation}
    The discrete equivalent (substituting the inner integral by Eq. \eqref{eq:singlesampling} with a single sampling from $p_{Y|F}$) is:
    \begin{equation}
    \label{eq:discretepsr}
        \min_{g} \frac{1}{N} \sum_{1\leq i\leq N}  S(g(F_i),Y_i)
    \end{equation}
    The probabilistic prediction corresponding to the optimal solution $g$ is denoted as $p_{\text{psr}}$.
    For ease of notation, we abbreviate the probability matrix $ \big[ p(Y_i$$=$$l | F_i) \big]_{ 1\leq i\leq N,1 \leq l\leq  L}$ as $[$\,$p$\,$]$. As detailed in \emph{Appendix \ref{sec:apdx-proof-proposition}} and \emph{\ref{sec:apdx-proof-proposition1}}, we prove that:
    \begin{proposition}
    \label{prop:betterthanpsr}
        For any $\alpha>0$, 
        \begin{equation}
            \Big\| \big[ p_\mu \big] - \big [p_M \big] \Big\|_{M,\omega} \leq \Big\| \big[ p_\mu \big] - \big[ p_{\mathrm{psr}}\big] \Big\|_{M,\omega} + \alpha
        \end{equation}
        holds with high probability (failure probability below $\frac{2}{\alpha\sqrt{M}}$), where $p_{\mu}$ refers to the ground truth classification probability. 
    \end{proposition}
    This implies that in terms of $\| * \|_{M,\omega}$ distance, the deviation of our estimated probabilities from the true probabilities is highly likely to be smaller than that derived from traditional PSR methods. Additionally, we also show
    \begin{proposition}
    \label{prop:controllableerror}
        For any $\alpha>0$, 
        \begin{equation}
            \Big\| \big[ p_\mu \big] - \big [p_M \big] \Big\|_{M,\omega}  \leq  \Xi + \alpha
        \end{equation}
        holds with high probability (failure probability below $\frac{1}{\alpha\sqrt{M}}$), where $\Xi$ reflects the learning loss.
    \end{proposition}
    The above propositions indicate that our approach maintains a controllable error margin relative to the ground truth and is highly likely to outperform traditional PSR under the $\| * \|_{M,\omega}$ distance metric.

    \item Regarding the second aspect, as analyzed in Section \ref{sec:psranalysis}, traditional implementations of PSRs, such as NLL and MSE losses, are seen as approximations under single sampling conditions. Such a single sampling approximation, illustrated by Eq. \eqref{eq:singlesampling}, can lead to uncontrollable errors, contributing to overfitting and thus overconfidence. 
    \begin{equation}
    \label{eq:singlesampling}
    \fontsize{9.5pt}{4pt}\selectfont
    \int S(g^A(F_i),y)p_{\mathds 1_A(Y)|F}(dy,F_i) \approx S(g^A(F_i),\mathds 1_A(Y_i)),
    \end{equation}
    where $\mathds 1_A(Y_i) = \mathds 1_{\{ Y_i \in A\}}$ and $g^A(F_i)$ denotes the predicted distribution for $\mathds 1_A(Y_i)$, for any sample $i$ and event $A$. 
    
    In contrast, from an algorithmic perspective, our method can be interpreted as leveraging a pseudo sampling-based estimate of PSR to tackle the aforementioned issue. Specifically, our algorithm sorts predicted probabilities, then aligns the average probability within a sliding window with its event occurrence rate. This procedure approximately aligns the predicted probability of the event $\mathds 1_A(Y_i)$ in central element with the occurrence rate of other events in the window. 
    Comparatively, the approximation of PSR in the LHS of Eq. \eqref{eq:singlesampling} with multiple samplings $\mathcal J$ is
    \begin{equation}
        \sum\nolimits_{u \in \mathcal J } S(g^A(F_i),\mathds 1_A(\widetilde Y_{i_u})),
    \label{eq:psr_multiple_sampling}
    \end{equation}
    where $\mathds 1_A(\widetilde Y_{i_u})$ represents multiple latent samplings from the distribution $\mathds 1_A(Y) | F_i$. 
    Without loss of generality, using the Brier score as an example, Eq. \eqref{eq:psr_multiple_sampling} produces,
    \begin{equation}
    \label{eq:brierexample}
        -\sum\nolimits_{u \in \mathcal J} \big(p_c(Y \in A | F_i) - \mathds 1_A({\widetilde Y}_{i_u}) \big)^2,
    \end{equation}
    optimizing the PSR yields the solution $p_c(Y$$\in$$A | F_i)=\frac{1}{|\mathcal J|}\sum_{ u \in \mathcal J} \mathds 1_A({\widetilde Y}_{i_u})$.
    This solution closely aligns with our method by considering observation labels for events with similar predicted probabilities within the window as multiple pseudo-samplings for the distribution of event $\mathds 1_A(Y_i)$. 
    \textbf{Thus, our algorithm can be understood as a pseudo sampling-based extension of the PSR.}
\end{itemize}

\section{Experiments}
The proposed method is validated using established post-hoc calibration tasks. For comprehensive evaluation, we employ diverse metrics at multiple levels, including top-label, classwise, and canonical calibration levels. Extensive comparisons have been conducted against numerous prior methods. Details of the experiments are presented below.

\subsection{Datasets}
We assess the effectiveness of our method using established post-hoc calibration benchmark tasks from \cite{no.71} and \cite{no.67}, which are publicly available and widely adopted for calibrator evaluation. These benchmarks provide uncalibrated probabilities of test samples generated by various networks for datasets including CIFAR10 \cite{no.192}, CIFAR100\cite{no.192}, SVHN\cite{no.193}, CARS\cite{no.194}, BIRDS \cite{no.195}, and ImageNet\cite{no.196}. The uncalibrated probabilities produced by the network classifiers on the original test image samples were further divided by \cite{no.71,no.67} into a training and a test set, which were used to train and evaluate the calibrator. Details of the class numbers (ranging from 10 to 1000) and the set sizes are summarized in Table \ref{tab:dataset} (\emph{Appendix \ref{subsec:apdx-dataset}}). 
14 calibration tasks were included for the datasets, each for calibrating a pretrained network classifier in the benchmark \cite{no.71,no.67}. Specific networks are summarized in \emph{Appendix \ref{subsec:apdx-dataset}}.
Since the established benchmark primarily focus on convolutional networks, we further added a calibration task for Transformer network to broaden the scope of our evaluation. Specifically, the SwinTransformer \cite{no.203} classifier trained by \cite{no.204}\footnote{\url{https:huggingface.co/timm/swin_tiny_patch4_window7_224.ms_in22k}} for ImageNet was incorporated, following the calibration protocols by \cite{no.71,no.67}, to enrich our overall assessment.

\subsection{Comparison Methods}
We compared our approach with a comprehensive set of post-hoc calibration methods, including Histogram Binning (HB) \cite{no.86}, Isotonic Regression (Iso) \cite{no.88}, Bayesian Binning into Quantiles (BBQ) \cite{no.87}, Ensemble of Near Isotonic Regression (ENIR) \cite{no.89}, Temperature Scaling (TS) \cite{no.109}, Vector Scaling (VS) \cite{no.109}, Matrix Scaling (MS) \cite{no.109,no.71}, Beta Calibration (Beta) \cite{no.85}, Scaling Binning (ScaBin) \cite{no.75}, Dirichlet Calibration (Dir) \cite{no.71}, Gaussian Process Calibration (GP) \cite{no.68}, Diagonal Intra Order-preserving Calibration (DIAG) \cite{no.67}, Decision Calibration (DEC) \cite{no.136}, Mutual Information Maximization-based Binning (IMax) \cite{no.1}, Soft Binning Calibration (SoftBin) \cite{no.58}, Spline Calibration \cite{no.60}, Expectation Consistency (EC) \cite{no.17}, Locally Equal Calibration Error (LECE) Calibration \cite{no.46}, LECE combined with TS (TS+LECE) \cite{no.46}, and Scaling of Classwise Training Losses (SCTL) \cite{no.29}. A detailed review of these methods is provided in the Introduction and Related Works sections. 

We reproduced these calibrators based on their open-source code. Specifically, the implementations of BBQ, Beta, ENIR, HB, Iso, TS, and VS were sourced from \cite{no.65}. IMax, Spline, EC, ScaBin, SCTL, and DIAG\footnote{Original DIAG code contained a ligical error omitting fold models in multi-fold ensembling, which was corrected in our reproduction.} were reproduced using the codes from the respective publications. Implementations of GP, LECE, DEC, and TS+LECE were obtained from \cite{no.46}. For MS and Dir, we used the code provided in \cite{no.71}, where MS was enhanced with Off-diagonal and Intercept Regularisation to mitigate overfitting issues identified in the original MS model in \cite{no.109}. SoftBin was implemented within our codebase using the original open-source loss function from \cite{no.58}.

\begin{table*}[h!]
\centering
\caption{ECE$_{r=1}^s$ Metric for Top-label Calibration Comparison (Best in \textcolor{red}{Red}, Second-best in \textcolor{blue}{Blue})}
\label{tab:sweepece1}
\setlength{\tabcolsep}{2pt}
\resizebox{\linewidth}{!}{%
\begin{tabular}{l l  c c c c c c c c c c c c c c c c c c c c c c}
\hline
Dataset & Model & Uncal & HB & Iso & BBQ & ENIR & TS & VS & MS & Beta & ScaBin & Dir & GP & DIAG & DEC & IMax & SoftBin & Spline & EC & LECE & SCTL & TS+LECE & Ours \\
\hline
CIFAR10 & ResNet110 & 0.0475 & 0.0111 & 0.0093 & 0.0095 & 0.0183 & 0.0091 & 0.0107 & 0.0170 & 0.0096 & 0.0087 & 0.0106 & 0.0083 & 0.0081 & 0.0091 & 0.0078 & \textcolor{red}{0.0030} & 0.0131 & 0.0063 & 0.0331 & 0.0061 & 0.0054 & \textcolor{blue}{0.0043} \\
CIFAR10 & WideResNet32 & 0.0448 & 0.0072 & 0.0077 & 0.0062 & 0.0139 & 0.0021 & 0.0034 & 0.0114 & 0.0061 & 0.0090 & 0.0025 & \textcolor{red}{0.0013} & 0.0030 & 0.0045 & 0.0108 & 0.0036 & 0.0089 & 0.0020 & 0.0205 & 0.0032 & 0.0027 & \textcolor{blue}{0.0015} \\
CIFAR10 & DenseNet40 & 0.0549 & 0.0160 & 0.0161 & 0.0167 & 0.0261 & 0.0090 & 0.0112 & 0.0211 & 0.0105 & \textcolor{red}{0.0020} & 0.0098 & 0.0082 & 0.0075 & 0.0108 & 0.0147 & \textcolor{blue}{0.0029} & 0.0161 & 0.0104 & 0.0366 & 0.0067 & 0.0067 & 0.0073 \\
SVHN & ResNet152(SD) & 0.0084 & 0.0061 & \textcolor{blue}{0.0029} & 0.0043 & \textcolor{red}{0.0027} & 0.0057 & 0.0057 & 0.0045 & 0.0060 & 0.0054 & 0.0058 & 0.0050 & 0.0042 & 0.0069 & 0.0045 & 0.0224 & 0.0050 & 0.0078 & 0.0056 & 0.0071 & 0.0078 & 0.0059 \\
CIFAR100 & ResNet110 & 0.1848 & 0.0950 & 0.0605 & 0.0863 & 0.0889 & 0.0187 & 0.0250 & 0.0413 & 0.0334 & 0.0308 & 0.0346 & \textcolor{blue}{0.0140} & 0.0285 & 0.0314 & 0.0446 & 0.0152 & 0.0160 & 0.0154 & 0.0637 & 0.0154 & 0.0182 & \textcolor{red}{0.0118} \\
CIFAR100 & WideResNet32 & 0.1878 & 0.0808 & 0.0564 & 0.0831 & 0.0755 & 0.0134 & 0.0177 & 0.0360 & 0.0257 & 0.0385 & 0.0177 & 0.0107 & 0.0103 & 0.0326 & 0.0517 & 0.0136 & 0.0259 & 0.0107 & 0.1523 & \textcolor{blue}{0.0101} & \textcolor{red}{0.0088} & 0.0149 \\
CIFAR100 & DenseNet40 & 0.2116 & 0.0768 & 0.0515 & 0.0751 & 0.0821 & 0.0074 & 0.0150 & 0.0371 & 0.0201 & 0.0397 & 0.0207 & 0.0093 & \textcolor{red}{0.0067} & 0.0373 & 0.0494 & \textcolor{blue}{0.0072} & 0.0172 & 0.0078 & 0.1030 & 0.0087 & 0.0111 & 0.0084 \\
CARS & ResNet50pre & 0.0213 & 0.0304 & 0.0285 & 0.0428 & 0.0380 & 0.0105 & 0.0269 & 0.0213 & 0.0251 & 0.0367 & 0.0168 & 0.0076 & 0.0084 & 0.0286 & \textcolor{red}{0.0060} & 0.0154 & \textcolor{blue}{0.0062} & 0.0085 & 0.0095 & 0.0132 & 0.0144 & 0.0070 \\
CARS & ResNet101pre & 0.0168 & 0.0586 & 0.0359 & 0.0572 & 0.0455 & 0.0297 & 0.0235 & 0.0238 & 0.0269 & 0.0367 & 0.0217 & 0.0159 & 0.0160 & 0.0451 & 0.0216 & 0.0338 & \textcolor{blue}{0.0109} & 0.0301 & 0.0228 & 0.0333 & 0.0380 & \textcolor{red}{0.0052} \\
CARS & ResNet101 & 0.0362 & 0.0274 & 0.0266 & 0.0389 & 0.0351 & 0.0148 & 0.0251 & 0.0166 & 0.0218 & 0.0476 & 0.0173 & 0.0091 & 0.0109 & 0.0272 & 0.0095 & 0.0228 & \textcolor{red}{0.0050} & 0.0132 & 0.0150 & 0.0157 & 0.0125 & \textcolor{blue}{0.0082} \\
BIRDS & ResNet50(NTS) & 0.0696 & 0.0466 & 0.0447 & 0.0581 & 0.0561 & 0.0312 & 0.0284 & 0.0265 & 0.0302 & 0.0238 & 0.0435 & 0.0142 & 0.0206 & 0.0304 & 0.0181 & 0.0278 & \textcolor{blue}{0.0103} & 0.0307 & 0.0618 & 0.0284 & 0.0245 & \textcolor{red}{0.0096} \\
ImageNet & ResNet152 & 0.0654 & 0.0721 & 0.0511 & 0.0771 & 0.0694 & 0.0213 & 0.0320 & 0.0610 & 0.0719 & 0.0312 & 0.0391 & 0.0119 & \textcolor{red}{0.0085} & 0.0328 & 0.0214 & 0.0213 & \textcolor{blue}{0.0096} & 0.0205 & 0.0610 & 0.0205 & 0.0205 & \textcolor{red}{0.0085} \\
ImageNet & DenseNet161 & 0.0572 & 0.0725 & 0.0464 & 0.0712 & 0.0652 & 0.0188 & 0.0259 & 0.0580 & 0.0661 & 0.0302 & 0.0373 & 0.0185 & 0.0110 & 0.0367 & 0.0164 & 0.0187 & \textcolor{blue}{0.0095} & 0.0187 & 0.0520 & 0.0186 & 0.0187 & \textcolor{red}{0.0076} \\
ImageNet & PNASNet5large & 0.0584 & 0.0456 & 0.0329 & 0.0548 & 0.0484 & 0.0452 & 0.0466 & 0.0266 & 0.0266 & 0.0146 & 0.0404 & 0.0111 & 0.0120 & 0.0610 & 0.0109 & 0.0369 & \textcolor{blue}{0.0078} & 0.0365 & 0.0675 & 0.0434 & 0.0414 & \textcolor{red}{0.0068} \\
ImageNet & SwinTransformer & 0.0730 & 0.0563 & 0.0261 & 0.0612 & 0.0442 & 0.0298 & 0.0389 & 0.0332 & 0.0335 & 0.0198 & 0.0186 & 0.0068 & 0.0067 & 0.0475 & 0.0068 & 0.0172 & \textcolor{blue}{0.0052} & 0.0149 & 0.0821 & 0.0255 & 0.0248 & \textcolor{red}{0.0050} \\
\hline
\multicolumn{2}{c}{\bf{Average Error}} & \bf{0.0759} & \bf{0.0468} & \bf{0.0331} & \bf{0.0495} & \bf{0.0473} & \bf{0.0178} & \bf{0.0224} & \bf{0.0290} & \bf{0.0276} & \bf{0.0250} & \bf{0.0224} & \bf{\textcolor{blue}{0.0101}} & \bf{0.0108} & \bf{0.0295} & \bf{0.0196} & \bf{0.0174} & \bf{0.0111} & \bf{0.0156} & \bf{0.0524} & \bf{0.0171} & \bf{0.0170} & \bf{\textcolor{red}{0.0075}} \\
\multicolumn{2}{c}{\bf{Average Relative Error}} & \bf{1.0000} & \bf{0.7986} & \bf{0.5966} & \bf{0.8901} & \bf{0.8216} & \bf{0.3844} & \bf{0.5038} & \bf{0.5330} & \bf{0.5648} & \bf{0.6033} & \bf{0.4509} & \bf{0.2302} & \bf{0.2339} & \bf{0.6208} & \bf{0.3039} & \bf{0.5313} & \bf{\textcolor{blue}{0.2146}} & \bf{0.3591} & \bf{0.7398} & \bf{0.3968} & \bf{0.3951} & \bf{\textcolor{red}{0.1729}} \\
\hline
\end{tabular}%
}
\end{table*}
\begin{figure*}[h!]
    \centering
    \includegraphics[width=0.331\linewidth, trim=8 0 8 0, clip]{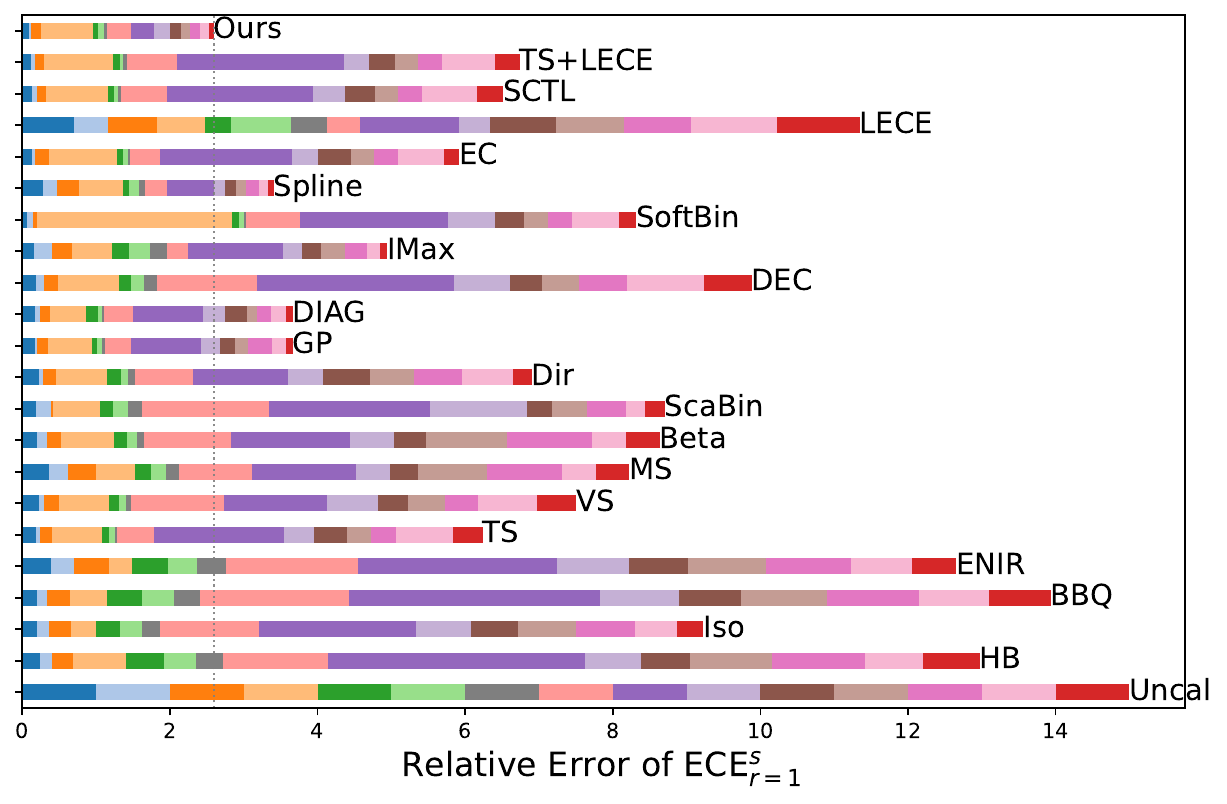}%
    \includegraphics[width=0.329\linewidth, trim=8 0 15 0, clip]{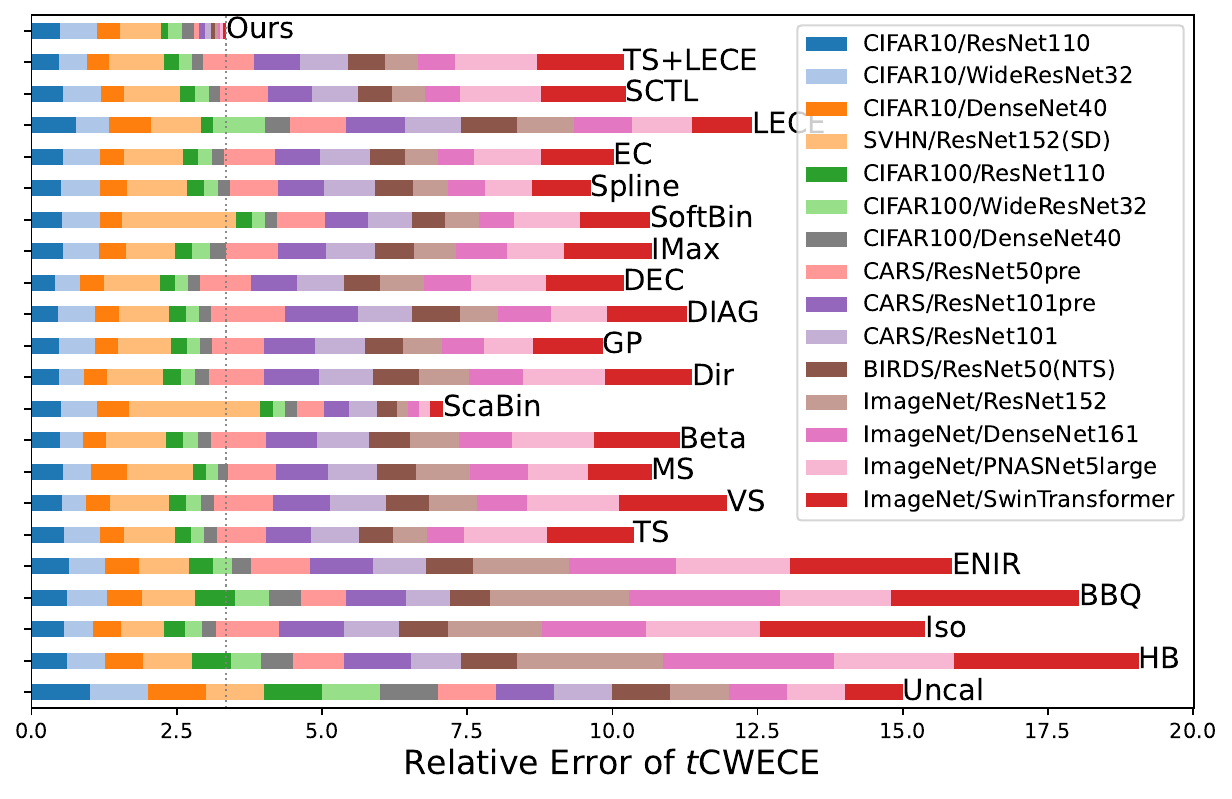}%
    \includegraphics[width=0.33\linewidth, trim=0 0 9 0, clip]{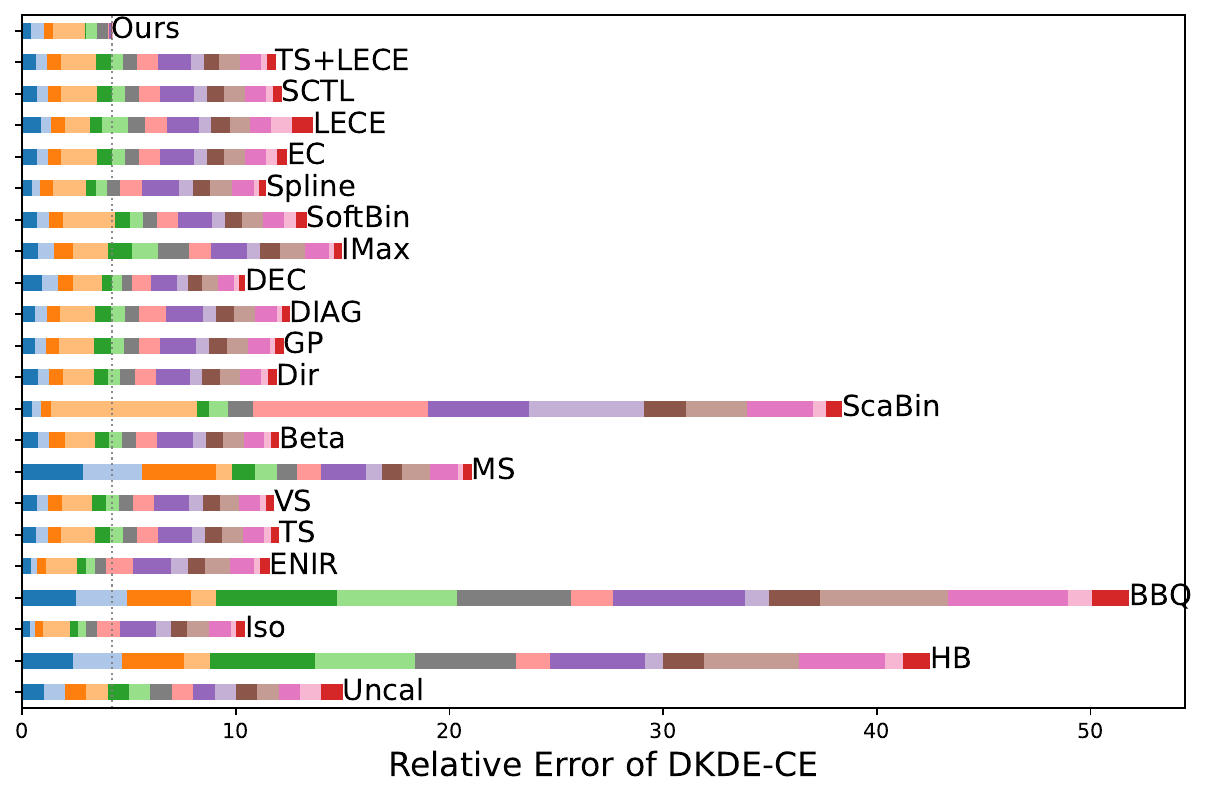}
    \caption{Relative calibration error of different methods across 15 tasks on three exemplar metrics}
    \label{fig:stackbarcombine}
\end{figure*}
\begin{table*}[h!]
\centering
\caption{$t$CWECE Metric for Classwise Calibration Comparison (Best in \textcolor{red}{Red}, Second-best in \textcolor{blue}{Blue})}
\label{tab:cwecethr}
\setlength{\tabcolsep}{2pt}
\resizebox{\linewidth}{!}{%
\begin{tabular}{l l  c c c c c c c c c c c c c c c c c c c c c c}
\hline
Dataset & Model & Uncal & HB & Iso & BBQ & ENIR & TS & VS & MS & Beta & ScaBin & Dir & GP & DIAG & DEC & IMax & SoftBin & Spline & EC & LECE & SCTL & TS+LECE & Ours \\
\hline
CIFAR10 & ResNet110 & 0.0523 & 0.0319 & 0.0289 & 0.0315 & 0.0338 & 0.0289 & 0.0277 & 0.0281 & 0.0255 & 0.0269 & 0.0250 & 0.0249 & \textcolor{blue}{0.0240} & \textcolor{red}{0.0215} & 0.0283 & 0.0276 & 0.0268 & 0.0287 & 0.0399 & 0.0287 & 0.0246 & 0.0258 \\
CIFAR10 & WideResNet32 & 0.0516 & 0.0341 & 0.0262 & 0.0360 & 0.0322 & 0.0325 & \textcolor{blue}{0.0210} & 0.0256 & \textcolor{red}{0.0207} & 0.0321 & 0.0219 & 0.0317 & 0.0330 & 0.0219 & 0.0320 & 0.0336 & 0.0343 & 0.0323 & 0.0295 & 0.0335 & 0.0249 & 0.0325 \\
CIFAR10 & DenseNet40 & 0.0610 & 0.0397 & 0.0291 & 0.0370 & 0.0356 & 0.0247 & 0.0250 & 0.0377 & 0.0239 & 0.0332 & 0.0245 & 0.0240 & 0.0243 & 0.0247 & 0.0283 & \textcolor{blue}{0.0231} & 0.0281 & 0.0256 & 0.0438 & 0.0236 & \textcolor{red}{0.0229} & 0.0247 \\
SVHN & ResNet152(SD) & 0.0133 & 0.0113 & \textcolor{blue}{0.0099} & 0.0120 & 0.0114 & 0.0118 & 0.0136 & 0.0151 & 0.0137 & 0.0300 & 0.0127 & 0.0121 & 0.0115 & 0.0130 & 0.0113 & 0.0262 & 0.0139 & 0.0136 & 0.0116 & 0.0129 & 0.0127 & \textcolor{red}{0.0094} \\
CIFAR100 & ResNet110 & 0.1299 & 0.0876 & 0.0461 & 0.0908 & 0.0542 & 0.0350 & 0.0387 & 0.0297 & 0.0382 & 0.0288 & 0.0399 & 0.0362 & 0.0375 & 0.0324 & 0.0369 & 0.0338 & 0.0359 & 0.0335 & \textcolor{blue}{0.0273} & 0.0340 & 0.0342 & \textcolor{red}{0.0150} \\
CIFAR100 & WideResNet32 & 0.1579 & 0.0801 & 0.0476 & 0.0923 & 0.0526 & 0.0375 & 0.0400 & \textcolor{red}{0.0314} & 0.0400 & \textcolor{blue}{0.0332} & 0.0402 & 0.0368 & 0.0361 & 0.0358 & 0.0505 & 0.0375 & 0.0394 & 0.0369 & 0.1388 & 0.0368 & 0.0342 & 0.0376 \\
CIFAR100 & DenseNet40 & 0.1666 & 0.0922 & 0.0407 & 0.0921 & 0.0526 & 0.0346 & 0.0369 & \textcolor{red}{0.0296} & 0.0373 & 0.0349 & 0.0387 & 0.0345 & 0.0341 & 0.0344 & 0.0464 & 0.0346 & 0.0354 & 0.0351 & 0.0744 & \textcolor{blue}{0.0335} & \textcolor{blue}{0.0335} & 0.0343 \\
CARS & ResNet50pre & 0.0757 & 0.0666 & 0.0815 & 0.0577 & 0.0770 & 0.0641 & 0.0770 & 0.0620 & 0.0717 & \textcolor{blue}{0.0344} & 0.0724 & 0.0679 & 0.0963 & 0.0665 & 0.0680 & 0.0618 & 0.0621 & 0.0669 & 0.0721 & 0.0626 & 0.0653 & \textcolor{red}{0.0074} \\
CARS & ResNet101pre & 0.0738 & 0.0858 & 0.0822 & 0.0773 & 0.0802 & 0.0573 & 0.0732 & 0.0661 & 0.0658 & \textcolor{blue}{0.0322} & 0.0694 & 0.0638 & 0.0930 & 0.0585 & 0.0610 & 0.0545 & 0.0591 & 0.0570 & 0.0755 & 0.0551 & 0.0586 & \textcolor{red}{0.0075} \\
CARS & ResNet101 & 0.0783 & 0.0671 & 0.0751 & 0.0583 & 0.0722 & 0.0645 & 0.0745 & 0.0658 & 0.0696 & \textcolor{blue}{0.0377} & 0.0734 & 0.0679 & 0.0732 & 0.0637 & 0.0660 & 0.0602 & 0.0680 & 0.0669 & 0.0746 & 0.0629 & 0.0655 & \textcolor{red}{0.0070} \\
BIRDS & ResNet50(NTS) & 0.0976 & 0.0942 & 0.0811 & 0.0685 & 0.0776 & 0.0585 & 0.0734 & 0.0667 & 0.0684 & \textcolor{blue}{0.0338} & 0.0766 & 0.0638 & 0.0805 & 0.0611 & 0.0650 & 0.0544 & 0.0634 & 0.0588 & 0.0941 & 0.0565 & 0.0613 & \textcolor{red}{0.0072} \\
ImageNet & ResNet152 & 0.0361 & 0.0907 & 0.0589 & 0.0863 & 0.0602 & 0.0211 & 0.0298 & 0.0331 & 0.0305 & \textcolor{blue}{0.0067} & 0.0310 & 0.0241 & 0.0238 & 0.0269 & 0.0263 & 0.0211 & 0.0221 & 0.0206 & 0.0357 & 0.0205 & 0.0209 & \textcolor{red}{0.0018} \\
ImageNet & DenseNet161 & 0.0344 & 0.1016 & 0.0618 & 0.0895 & 0.0629 & 0.0217 & 0.0295 & 0.0345 & 0.0313 & \textcolor{blue}{0.0064} & 0.0318 & 0.0251 & 0.0315 & 0.0283 & 0.0303 & 0.0209 & 0.0220 & 0.0210 & 0.0343 & 0.0211 & 0.0218 & \textcolor{red}{0.0016} \\
ImageNet & PNASNet5large & 0.0340 & 0.0697 & 0.0665 & 0.0651 & 0.0668 & 0.0483 & 0.0536 & 0.0350 & 0.0481 & \textcolor{blue}{0.0062} & 0.0484 & 0.0285 & 0.0327 & 0.0436 & 0.0331 & 0.0386 & 0.0273 & 0.0396 & 0.0355 & 0.0472 & 0.0481 & \textcolor{red}{0.0015} \\
ImageNet & SwinTransformer & 0.0246 & 0.0783 & 0.0700 & 0.0795 & 0.0688 & 0.0370 & 0.0460 & 0.0272 & 0.0364 & \textcolor{blue}{0.0057} & 0.0367 & 0.0297 & 0.0335 & 0.0329 & 0.0371 & 0.0297 & 0.0249 & 0.0307 & 0.0254 & 0.0358 & 0.0366 & \textcolor{red}{0.0013} \\
\hline
\multicolumn{2}{c}{\bf{Average Error}} & \bf{0.0725} & \bf{0.0687} & \bf{0.0537} & \bf{0.0649} & \bf{0.0559} & \bf{0.0385} & \bf{0.0440} & \bf{0.0392} & \bf{0.0414} & \bf{\textcolor{blue}{0.0255}} & \bf{0.0428} & \bf{0.0381} & \bf{0.0443} & \bf{0.0377} & \bf{0.0414} & \bf{0.0372} & \bf{0.0375} & \bf{0.0378} & \bf{0.0542} & \bf{0.0377} & \bf{0.0377} & \bf{\textcolor{red}{0.0143}} \\
\multicolumn{2}{c}{\bf{Average Relative Error}} & \bf{1.0000} & \bf{1.2699} & \bf{1.0259} & \bf{1.2016} & \bf{1.0566} & \bf{0.6919} & \bf{0.7988} & \bf{0.7119} & \bf{0.7444} & \bf{\textcolor{blue}{0.4726}} & \bf{0.7584} & \bf{0.6559} & \bf{0.7508} & \bf{0.6797} & \bf{0.7122} & \bf{0.7105} & \bf{0.6423} & \bf{0.6689} & \bf{0.8275} & \bf{0.6825} & \bf{0.6803} & \bf{\textcolor{red}{0.2233}} \\
\hline
\end{tabular}%
}
\end{table*}

\subsection{Implementation Details}
In our study, the hyperparameters were set as follows: $\epsilon$= $10^{-20}$, $M = 200$, and loss multiplier $r = 10^5$ for all datasets and tasks. Training was conducted with a maximum of 2000 epochs using the Adam optimizer with an initial learning rate of 0.005. A learning rate scheduler and early stopping were applied, monitoring training set $\mathrm{ECE}^{\mathrm{ew}}$ with patience of 20 and 160 epochs, respectively. The scheduler reduced the learning rate by a factor of 0.5 when the metric showed no improvement. For efficient training, a large batch size was used. All datasets had batch sizes equal to the training set size, except for the ImageNet experiments where GPU memory constraints led to batch sizes of 6000 and 3000 for hidden neurons of 20 and 50, respectively, in training MonotonicNet. Training was performed on Nvidia GPUs with PyTorch libraries. Due to different focuses of different types of calibration, such as top-label calibration only considering top-1 probability reliability disregarding other probabilities, this focus disparity results in low correlation between calibration metrics of different types. Considering this, we adopted the strategy from \cite{no.67}, using different model selectors for different calibration types to highlight the focus. As previous calibration studies focus primarily on top-label calibration, we selected dECE as the model selector for better top-label calibration, while $\mathrm{CWECE}_a$ was used for non-top-label calibration and NLL for its own evaluation. It is noteworthy that our learning strategy utilized a unified training setting across all datasets, without dataset- or task-specific tuning. \emph{Appendix \ref{sec:apdx-losscurves}} presents the training loss curves of the proposed method.

\subsection{Evaluations}
Regarding calibration evaluation, a variety of metrics have been employed across different studies. We systematically summarize the metrics at the top-label, classwise, and canonical levels to evaluate the performance of various models accordingly. Specifically, the top-label metrics include ECE metrics based on equal mass and equal width binnings, referred to as $\mathrm{ECE}^{\mathrm{em}}$ and $\mathrm{ECE}^{\mathrm{ew}}$, respectively, as well as the higher-order variant $\mathrm{ECE}_{r=2}$ and the debiased variant $\mathrm{dECE}$ \cite{no.71}. Other top-label metrics include $\mathrm{ACE}$, sweep binning-based calibration errors $\mathrm{ECE}_{r=1}^s$ and $\mathrm{ECE}_{r=2}^s$ \cite{no.50}, $\mathrm{MMCE}$ \cite{no.83}, $\mathrm{KDE}$-$\mathrm{ECE}$ \cite{no.70}, and $\mathrm{KS}$ error \cite{no.60}. At the classwise level, the metrics include the average CWECE ($\mathrm{CWECE}_{a}$) and total CWECE ($\mathrm{CWECE}_{s}$), along with the higher-order variant $\mathrm{CWECE}_{r=2}$, thresholded variant $t\mathrm{CWECE}$ \cite{no.1}, and k-means binning-based variant $t\mathrm{CWECE}^k$ \cite{no.1}. For the canonical level, the metrics used include DKDE-CE \cite{no.52} and SKCE \cite{no.79}. A summary and detailed discussion of these metrics are provided, respectively, in Section \ref{II.A} and in \emph{Appendix \ref{sec:apdx-evaluationsummary}} within a unified probabilistic framework.
MCE is not included in evaluation, as it estimates the error using a single bin with limited and variable sample sizes, disregarding most calibration information and making it prone to noise and binning settings \cite{no.101,no.205,no.206}. NLL and the Brier score, as PSRs, do not directly reflect calibration because they are influenced by discrimination performance \cite{no.3,no.90,no.130,no.125,no.51} and are decomposable into multiple factors beyond just calibration \cite{no.130,no.129,no.51}. 
Nonetheless, since NLL is commonly used as auxiliary indicators in literature, we report NLL as a reference.
The respective source code repositories for the above metrics are listed in \emph{Appendix \ref{subsec:apdx-evacodesource}}. These comprehensive metrics aim to provide a thorough comparison of different calibrators.

\begin{figure}[H]
\centering
\includegraphics[width=\columnwidth]{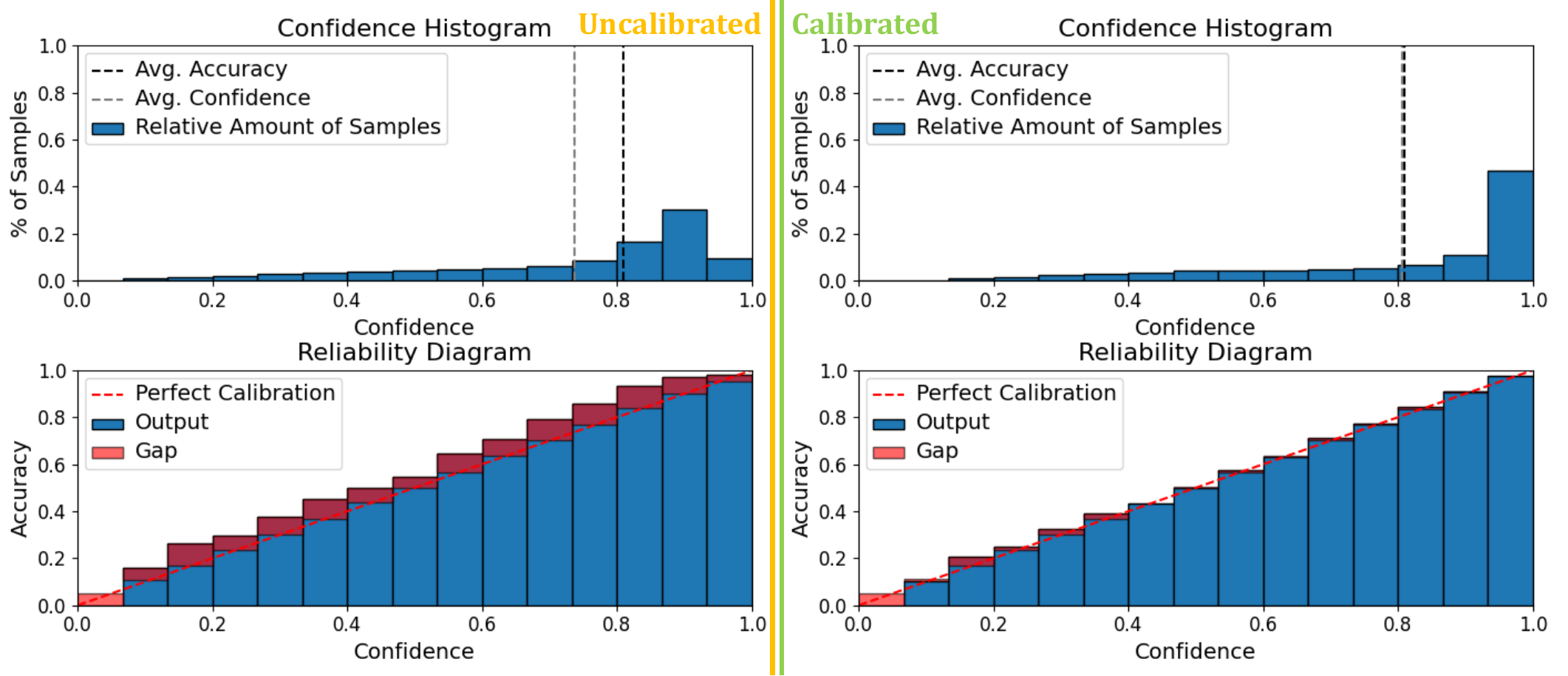}
\caption{Typical reliability diagram before and after calibration using our method (ImageNet-SwinTransformer experiment)}
\label{ourreliability}
\end{figure}

\subsection{Comparison Results}
For top-label calibration, Table \ref{tab:sweepece1} provides a detailed comparison using the metric $\mathrm{ECE}_{r=1}^s$, where `Uncal' denotes original uncalibrated predictions. The results show that, compared to 20 existing methods, our approach achieves the best calibration performance on 7 out of 15 calibration tasks (calibrating diverse networks on 6 datasets), and ranks second on 3 tasks. To summarize the performance across the 15 tasks, we calculated both the average calibration error (AE) and the average relative error (ARE), where relative error represents the ratio of calibrated error to the uncalibrated error for normalizing purpose. Our method demonstrates best average values with a significant margin. A visual depiction of relative error across various tasks is shown in Fig. \ref{fig:stackbarcombine}. Similar visualization result for corresponding absolute error is provided in \emph{Appendix \ref{sec:apdx-visulaizationcomparison}}. Quantitative results for other top-label calibration metrics, including $\mathrm{ECE}_{r=2}^s$, KDE-ECE, MMCE, KS error, $\mathrm{ECE}^{\mathrm{em}}$, $\mathrm{ECE}^{\mathrm{ew}}$, $\mathrm{ECE}_{r=2}$, dECE, and ACE, are provided in \emph{Appendix \ref{sec:apdx-topcalibrationnumber}} (Table \ref{tab:apdx-tab-cmp-sweepECE2} - \ref{tab:apdx-tab-cmp-eCE_2}). Similarly, our approach achieves the best performance across all average metrics, with the second-best varying among different existing methods, highlighting the advantage of our approach. \textbf{To summarize, compared to 20 prior methods, our approach consistently achieves state-of-the-art performance on all 10 top-label evaluators averaged over 15 calibration tasks}. We also provide a reliability diagram to visually assess calibration. Given the large number of tasks and methods, we take the ImageNet-SwinTransformer task as an illustrative example in Fig. \ref{ourreliability}. It clearly shows that after applying calibration with our method, the binwise gaps between predicted probabilities and accuracies are significantly reduced, especially for bins with abundant samples. Moreover, the overall difference between the mean probability and the overall accuracy for all test data is eliminated, as depicted by the alignment of dashed lines in the top two subfigures. The reliability diagram confirms the effectiveness of our method. Due to page limit, reliability diagrams for the other 20 methods are provided in \emph{Appendix \ref{sec:apdx-reliabilitydiagrams}} for comparison.

\begin{table*}[htbp]
\centering
\caption{DKDE-CE Metric for Canonical Calibration Comparison (Best in \textcolor{red}{Red}, Second-best in \textcolor{blue}{Blue})}
\label{tab:dkde-ce}
\setlength{\tabcolsep}{2pt}
\resizebox{\linewidth}{!}{%
\begin{tabular}{l l  c c c c c c c c c c c c c c c c c c c c c c c}
\hline
Dataset & Model & Uncal & HB & Iso & BBQ & ENIR & TS & VS & MS & Beta & ScaBin & Dir & GP & DIAG & DEC & IMax & SoftBin & Spline & EC & LECE & SCTL & TS+LECE & Ours &   \\
\hline
CIFAR10 & ResNet110 & 0.0216 & 0.0517 & \textcolor{red}{0.0075} & 0.0542 & \textcolor{blue}{0.0088} & 0.0143 & 0.0149 & 0.0614 & 0.0161 & 0.0102 & 0.0157 & 0.0133 & 0.0133 & 0.0200 & 0.0159 & 0.0152 & 0.0094 & 0.0146 & 0.0186 & 0.0146 & 0.0134 & 0.0089 & ($\times 10^{-1}$) \\
CIFAR10 & WideResNet32 & 0.0211 & 0.0482 & \textcolor{red}{0.0052} & 0.0504 & \textcolor{blue}{0.0059} & 0.0111 & 0.0106 & 0.0581 & 0.0111 & 0.0082 & 0.0114 & 0.0106 & 0.0111 & 0.0157 & 0.0161 & 0.0113 & 0.0085 & 0.0110 & 0.0099 & 0.0112 & 0.0111 & 0.0129 & ($\times 10^{-1}$) \\
CIFAR10 & DenseNet40 & 0.0208 & 0.0606 & \textcolor{red}{0.0081} & 0.0621 & 0.0092 & 0.0129 & 0.0140 & 0.0724 & 0.0149 & 0.0104 & 0.0137 & 0.0123 & 0.0128 & 0.0149 & 0.0181 & 0.0135 & 0.0120 & 0.0125 & 0.0142 & 0.0129 & 0.0137 & \textcolor{blue}{0.0091} & ($\times 10^{-1}$) \\
SVHN & ResNet152(SD) & \textcolor{blue}{0.0071} & 0.0085 & 0.0089 & 0.0084 & 0.0102 & 0.0115 & 0.0101 & \textcolor{red}{0.0054} & 0.0101 & 0.0484 & 0.0101 & 0.0118 & 0.0117 & 0.0094 & 0.0117 & 0.0176 & 0.0110 & 0.0122 & 0.0083 & 0.0120 & 0.0119 & 0.0105 & ($\times 10^{-1}$) \\
CIFAR100 & ResNet110 & 0.3131 & 1.5355 & \textcolor{blue}{0.1116} & 1.7831 & 0.1287 & 0.2143 & 0.2045 & 0.3300 & 0.2098 & 0.1839 & 0.2082 & 0.2382 & 0.2369 & 0.1479 & 0.3476 & 0.2177 & 0.1512 & 0.2186 & 0.1787 & 0.2171 & 0.2076 & \textcolor{red}{0.0167} & ($\times 10^{-1}$) \\
CIFAR100 & WideResNet32 & 0.2878 & 1.3526 & \textcolor{red}{0.1133} & 1.6085 & \textcolor{blue}{0.1217} & 0.1717 & 0.1648 & 0.3067 & 0.1715 & 0.2564 & 0.1692 & 0.1795 & 0.1811 & 0.1339 & 0.3609 & 0.1693 & 0.1502 & 0.1707 & 0.3460 & 0.1716 & 0.1670 & 0.1405 & ($\times 10^{-1}$) \\
CIFAR100 & DenseNet40 & 0.3669 & 1.7382 & \textcolor{red}{0.1796} & 1.9657 & 0.1953 & 0.2505 & 0.2484 & 0.3375 & 0.2460 & 0.4165 & 0.2463 & 0.2530 & 0.2462 & \textcolor{blue}{0.1822} & 0.5262 & 0.2430 & 0.2252 & 0.2416 & 0.2928 & 0.2476 & 0.2452 & 0.1958 & ($\times 10^{-1}$) \\
CARS & ResNet50pre & 0.1140 & 0.1834 & 0.1259 & 0.2232 & 0.1411 & 0.1115 & 0.1113 & 0.1295 & 0.1109 & 0.9363 & 0.1108 & 0.1149 & 0.1431 & \textcolor{blue}{0.0982} & 0.1187 & 0.1116 & 0.1171 & 0.1117 & 0.1159 & 0.1115 & 0.1129 & \textcolor{red}{0.0060} & ($\times 10^{-1}$) \\
CARS & ResNet101pre & \textcolor{blue}{0.1049} & 0.4636 & 0.1757 & 0.6462 & 0.1861 & 0.1675 & 0.1736 & 0.2185 & 0.1778 & 0.4937 & 0.1679 & 0.1745 & 0.1823 & 0.1277 & 0.1740 & 0.1679 & 0.1806 & 0.1675 & 0.1574 & 0.1677 & 0.1609 & \textcolor{red}{0.0066} & ($\times 10^{-1}$) \\
CARS & ResNet101 & 0.1738 & 0.1505 & 0.1195 & 0.1942 & 0.1417 & 0.1043 & 0.1115 & 0.1293 & 0.1043 & 0.9377 & 0.1029 & 0.1081 & 0.1049 & \textcolor{blue}{0.0914} & 0.1084 & 0.1054 & 0.1120 & 0.1042 & 0.1019 & 0.1045 & 0.1025 & \textcolor{red}{0.0052} & ($\times 10^{-1}$) \\
BIRDS & ResNet50(NTS) & 0.2291 & 0.4316 & 0.1774 & 0.5523 & 0.1855 & 0.1845 & 0.1781 & 0.2155 & 0.1825 & 0.4512 & 0.1926 & 0.1883 & 0.1987 & \textcolor{blue}{0.1521} & 0.2176 & 0.1838 & 0.1865 & 0.1847 & 0.1963 & 0.1840 & 0.1679 & \textcolor{red}{0.0050} & ($\times 10^{-1}$) \\
ImageNet & ResNet152 & 0.2716 & 1.2133 & 0.2754 & 1.6302 & 0.3108 & 0.2626 & 0.2529 & 0.3517 & 0.2568 & 0.7793 & 0.2556 & 0.2751 & 0.2675 & \textcolor{blue}{0.2058} & 0.3133 & 0.2626 & 0.2794 & 0.2629 & 0.2667 & 0.2630 & 0.2596 & \textcolor{red}{0.0005} & ($\times 10^{-1}$) \\
ImageNet & DenseNet161 & 0.2591 & 1.0466 & 0.2637 & 1.4499 & 0.2969 & 0.2531 & 0.2447 & 0.3399 & 0.2478 & 0.7938 & 0.2478 & 0.2658 & 0.2563 & \textcolor{blue}{0.1906} & 0.2903 & 0.2538 & 0.2640 & 0.2537 & 0.2544 & 0.2536 & 0.2499 & \textcolor{red}{0.0005} & ($\times 10^{-1}$) \\
ImageNet & PNASNet5large & 0.8512 & 0.7045 & 0.2187 & 0.9702 & 0.2409 & 0.2802 & 0.2554 & 0.2203 & 0.2861 & 0.5417 & 0.2881 & 0.2130 & 0.2186 & 0.2028 & 0.2000 & 0.4872 & \textcolor{blue}{0.1913} & 0.4470 & 0.8242 & 0.2902 & 0.2763 & \textcolor{red}{0.0030} & ($\times 10^{-1}$) \\
ImageNet & SwinTransformer & 0.6039 & 0.7675 & 0.2505 & 1.0535 & 0.2749 & 0.2407 & 0.2254 & 0.2520 & 0.2291 & 0.4452 & 0.2442 & 0.2335 & 0.2219 & \textcolor{blue}{0.1753} & 0.2147 & 0.3046 & 0.2193 & 0.2882 & 0.5937 & 0.2445 & 0.2367 & \textcolor{red}{0.0002} & ($\times 10^{-1}$) \\
\hline
\multicolumn{2}{c}{\bf{Average Error}} & \bf{0.2431} & \bf{0.6504} & \bf{0.1361} & \bf{0.8168} & \bf{0.1505} & \bf{0.1527} & \bf{0.1480} & \bf{0.2019} & \bf{0.1517} & \bf{0.4209} & \bf{0.1523} & \bf{0.1528} & \bf{0.1538} & \bf{\textcolor{blue}{0.1179}} & \bf{0.1956} & \bf{0.1710} & \bf{0.1412} & \bf{0.1667} & \bf{0.2253} & \bf{0.1537} & \bf{0.1491} & \bf{\textcolor{red}{0.0281}} & \bf{($\times 10^{-1}$)} \\
\multicolumn{2}{c}{\bf{Average Relative Error}} & \bf{1.0000} & \bf{2.7544} & \bf{\textcolor{blue}{0.6811}} & \bf{3.3412} & \bf{0.7617} & \bf{0.7866} & \bf{0.7700} & \bf{1.3810} & \bf{0.7833} & \bf{2.6701} & \bf{0.7777} & \bf{0.7992} & \bf{0.8156} & \bf{0.6865} & \bf{0.9800} & \bf{0.8719} & \bf{0.7440} & \bf{0.8104} & \bf{0.8938} & \bf{0.7941} & \bf{0.7766} & \bf{\textcolor{red}{0.2794}} & \bf{($\times 1$)} \\
\hline
\end{tabular}
}
\end{table*}

\begin{table*}[htbp]
\centering
\caption{Average Relative Error of Methods by Different Metrics (Best in \textcolor{red}{Red}, Second-best in \textcolor{blue}{Blue})}
\label{tab:overallare}
\setlength{\tabcolsep}{2pt}
\resizebox{0.9\linewidth}{!}{%
\begin{tabular}{l l  c c c c c c c c c c c c c c c c c c c c c c}
\hline
Metric & Uncal & HB & Iso & BBQ & ENIR & TS & VS & MS & Beta & ScaBin & Dir & GP & DIAG & DEC & IMax & SoftBin & Spline & EC & LECE & SCTL & TS+LECE & Ours \\
\hline
ECE$_{r=1}^s$ & 1.0000 & 0.7986 & 0.5966 & 0.8901 & 0.8216 & 0.3844 & 0.5038 & 0.5330 & 0.5648 & 0.6033 & 0.4509 & 0.2302 & 0.2339 & 0.6208 & 0.3039 & 0.5313 & \textcolor{blue}{0.2146} & 0.3591 & 0.7398 & 0.3968 & 0.3951 & \textcolor{red}{0.1729} \\
ECE$_{r=2}^s$ & 1.0000 & 0.7142 & 0.5678 & 0.8552 & 0.7801 & 0.4188 & 0.5513 & 0.5489 & 0.6096 & 0.6003 & 0.4601 & 0.2421 & 0.2360 & 0.5655 & 0.3121 & 0.5800 & \textcolor{blue}{0.2125} & 0.3539 & 0.7360 & 0.4608 & 0.4547 & \textcolor{red}{0.1704} \\
KDE-ECE & 1.0000 & 0.9613 & 0.6070 & 0.9817 & 0.8299 & 0.4813 & 0.5518 & 0.6881 & 0.6067 & 0.9840 & 0.5329 & 0.4067 & \textcolor{blue}{0.3828} & 0.7873 & 0.4520 & 0.6775 & 0.4182 & 0.4665 & 0.7821 & 0.5155 & 0.5308 & \textcolor{red}{0.3612} \\
MMCE & 1.0000 & 0.4501 & 0.5522 & 0.3609 & 0.7494 & 0.2793 & 0.4686 & 0.4881 & 0.5476 & 0.5280 & 0.4288 & \textcolor{blue}{0.1653} & 0.1804 & 0.5349 & 0.2329 & 0.4048 & 0.1778 & 0.2569 & 0.7133 & 0.2857 & 0.2875 & \textcolor{red}{0.1389} \\
KS error & 1.0000 & 0.6713 & 0.5935 & 0.4859 & 0.8156 & 0.2850 & 0.4919 & 0.5257 & 0.5274 & 0.5681 & 0.4392 & 0.1851 & \textcolor{blue}{0.1702} & 0.5387 & 0.2313 & 0.4313 & 0.1862 & 0.2313 & 0.7203 & 0.3150 & 0.3223 & \textcolor{red}{0.1439} \\
ECE$^{\mathrm{em}}$ & 1.0000 & 1.1414 & 0.5867 & 1.0284 & 0.8031 & 0.3907 & 0.4923 & 0.7026 & 0.5601 & 0.5984 & 0.4670 & \textcolor{blue}{0.2477} & 0.2534 & 0.6186 & 0.3316 & 0.5339 & 0.2925 & 0.3767 & 0.7483 & 0.4100 & 0.4022 & \textcolor{red}{0.2447} \\
ACE & 1.0000 & 1.4790 & 0.6222 & 1.6213 & 0.9064 & 0.5753 & 0.6213 & 0.9510 & 0.5786 & 1.0657 & 0.5547 & 0.4285 & 0.4184 & 0.5271 & 0.4408 & 0.5328 & \textcolor{blue}{0.4174} & 0.4741 & 0.8120 & 0.5805 & 0.5906 & \textcolor{red}{0.3655} \\
dECE & 1.0000 & 1.1634 & 0.5894 & 1.0607 & 0.8158 & 0.3621 & 0.4840 & 0.7026 & 0.5493 & 0.5764 & 0.4545 & \textcolor{blue}{0.2080} & 0.2139 & 0.6080 & 0.3030 & 0.5134 & 0.2491 & 0.3503 & 0.7389 & 0.3877 & 0.3743 & \textcolor{red}{0.1972} \\
ECE$^{\mathrm{ew}}$ & 1.0000 & 0.6826 & 0.5688 & 0.8484 & 0.7713 & 0.3874 & 0.4941 & 0.5688 & 0.5758 & 0.5930 & 0.4665 & \textcolor{blue}{0.2442} & 0.2704 & 0.5931 & 0.3283 & 0.5084 & 0.2667 & 0.3795 & 0.7613 & 0.3859 & 0.3809 & \textcolor{red}{0.2149} \\
ECE$_{r=2}$ & 1.0000 & 1.0805 & 0.5775 & 1.0800 & 0.8030 & 0.4847 & 0.5517 & 0.7377 & 0.5964 & 0.6235 & 0.5061 & \textcolor{blue}{0.3278} & 0.3380 & 0.5705 & 0.4340 & 0.5438 & 0.3666 & 0.4263 & 0.8044 & 0.4960 & 0.4914 & \textcolor{red}{0.3158} \\
CWECE$_s$ & 1.0000 & 0.7457 & 0.7117 & \textcolor{blue}{0.6814} & 0.7915 & 0.7541 & 0.7204 & 0.7583 & 0.7119 & 1.2943 & 0.7092 & 0.7520 & 0.7248 & 0.8204 & 0.7491 & 0.8763 & 0.7705 & 0.7702 & 0.8426 & 0.7660 & 0.7337 & \textcolor{red}{0.4708} \\
CWECE$_a$ & 1.0000 & 0.7418 & 0.7160 & \textcolor{blue}{0.6807} & 0.7939 & 0.7551 & 0.7285 & 0.7650 & 0.7152 & 1.3064 & 0.7132 & 0.7556 & 0.7262 & 0.8196 & 0.7532 & 0.8788 & 0.7726 & 0.7736 & 0.8436 & 0.7688 & 0.7406 & \textcolor{red}{0.4702} \\
CWECE$_{r=2}$ & 1.0000 & 0.8190 & 0.8111 & \textcolor{blue}{0.6937} & 0.8297 & 0.8391 & 0.8378 & 0.8602 & 0.8197 & 0.8592 & 0.8235 & 0.8299 & 0.8170 & 0.8167 & 0.8495 & 0.8550 & 0.8492 & 0.8400 & 0.9224 & 0.8397 & 0.8261 & \textcolor{red}{0.3161} \\
$t$CWECE & 1.0000 & 1.2699 & 1.0259 & 1.2016 & 1.0566 & 0.6919 & 0.7988 & 0.7119 & 0.7444 & \textcolor{blue}{0.4726} & 0.7584 & 0.6559 & 0.7508 & 0.6797 & 0.7122 & 0.7105 & 0.6423 & 0.6689 & 0.8275 & 0.6825 & 0.6803 & \textcolor{red}{0.2233} \\
$t$CWECE$^k$ & 1.0000 & 1.0275 & 0.8737 & 1.1964 & 0.9330 & 0.6517 & 0.7541 & 0.7083 & 0.7288 & \textcolor{blue}{0.5285} & 0.7048 & 0.6070 & 0.7187 & 0.6111 & 0.6418 & 0.6825 & 0.6039 & 0.6280 & 0.7942 & 0.6361 & 0.6238 & \textcolor{red}{0.2947} \\
DKDE-CE & 1.0000 & 2.7544 & \textcolor{blue}{0.6811} & 3.3412 & 0.7617 & 0.7866 & 0.7700 & 1.3810 & 0.7833 & 2.6701 & 0.7777 & 0.7992 & 0.8156 & 0.6865 & 0.9800 & 0.8719 & 0.7440 & 0.8104 & 0.8938 & 0.7941 & 0.7766 & \textcolor{red}{0.2794} \\
SKCE & 1.0000 & 1.1797 & 0.6272 & 1.2753 & 0.7782 & 0.5293 & 0.8633 & 1.1266 & 0.9277 & 6.3275 & 0.6406 & \textcolor{blue}{0.4403} & 0.6528 & 0.7272 & 0.5435 & 0.9979 & 0.5116 & 0.5530 & 0.6277 & 0.5416 & 0.4839 & \textcolor{red}{-0.1747} \\
\hline
\bf{Average} & \bf{1.0000} & \bf{1.0400} & \bf{0.6652} & \bf{1.0755} & \bf{0.8259} & \bf{0.5328} & \bf{0.6285} & \bf{0.7505} & \bf{0.6557} & \bf{1.1882} & \bf{0.5817} & \bf{\textcolor{blue}{0.4427}} & \bf{0.4649} & \bf{0.6545} & \bf{0.5058} & \bf{0.6547} & \bf{0.4527} & \bf{0.5129} & \bf{0.7828} & \bf{0.5449} & \bf{0.5350} & \bf{\textcolor{red}{0.2474}} \\
\hline
\end{tabular}
}
\end{table*}

Regarding classwise and canonical calibration, we evaluated $\mathrm{CWECE}{s}$, $\mathrm{CWECE}{a}$, $\mathrm{CWECE}_{r=2}$, $t\mathrm{CWECE}$, $t\mathrm{CWECE}^k$, DKDE-CE, and SKCE metrics. Tables \ref{tab:cwecethr} and \ref{tab:dkde-ce} exemplify the results of $t\mathrm{CWECE}$ and DKDE-CE, where our method achieves the best performance in 10 and 9 out of 15 tasks, respectively. Fig. \ref{fig:stackbarcombine} presents task-specific relative errors, with corresponding absolute errors provided in \emph{Appendix \ref{sec:apdx-visulaizationcomparison}}. Our method demonstrates a significant advantage in the average metrics across tasks. Detailed results for other metrics can be found in \emph{Appendix \ref{sec:apdx-classwisecalibrationnumber}} and \emph{\ref{sec:apdx-canonicalcalibrationnumber}} (Table \ref{tab:apdx-tab-cmp-cwecesum} - \ref{tab:apdx-tab-cmp-SKCE}). \textbf{Our approach consistently achieves the best ARE and AE calibration values across all 7 evaluation metrics.} Notably, in non-top-label calibration assessments, our method shows more clear improvements compared to previous methods (in contrast to top-label metrics), possibly because prior works have not adequately focused on non-top-label calibration. Table \ref{tab:overallare} reports ARE statistics across different metrics, while Fig. \ref{fig:RadarARE} provides a radar chart for a visual summary. Similar AE statistics and corresponding radar plots are presented in \emph{Appendix \ref{sec:apdx-overall_are_ae_number}} and \emph{\ref{sec:apdx-visulaizationcomparison}}, respectively. These results provide an overall comparison of prior methods and highlight the advantages of our approach. The NLL results are provided in \emph{Appendix \ref{sec:tab-cmp-NLL}}.

It is noteworthy that while our method is accuracy-preserving, this property is not guaranteed by existing approaches. Table \ref{tab:apdx-tab-accuracy-variation} in \emph{Appendix \ref{sec:apdx-tab-accuracy-variation}} reports the impact of various post-hoc calibrators on classification accuracy. Overall, except for TS, DIAG, SoftBin, EC and SCTL, which are accuracy-preserving, and DCE and LECE, which slightly improve average accuracy ($\leq$+0.09\%), most methods, including HB, Iso, BBQ, ENIR, VS, MS, Beta, ScaBin, GP, Imax, and Spline, negatively impact classification performance.

\begin{figure}[!h]
\centering
\includegraphics[width=\columnwidth]{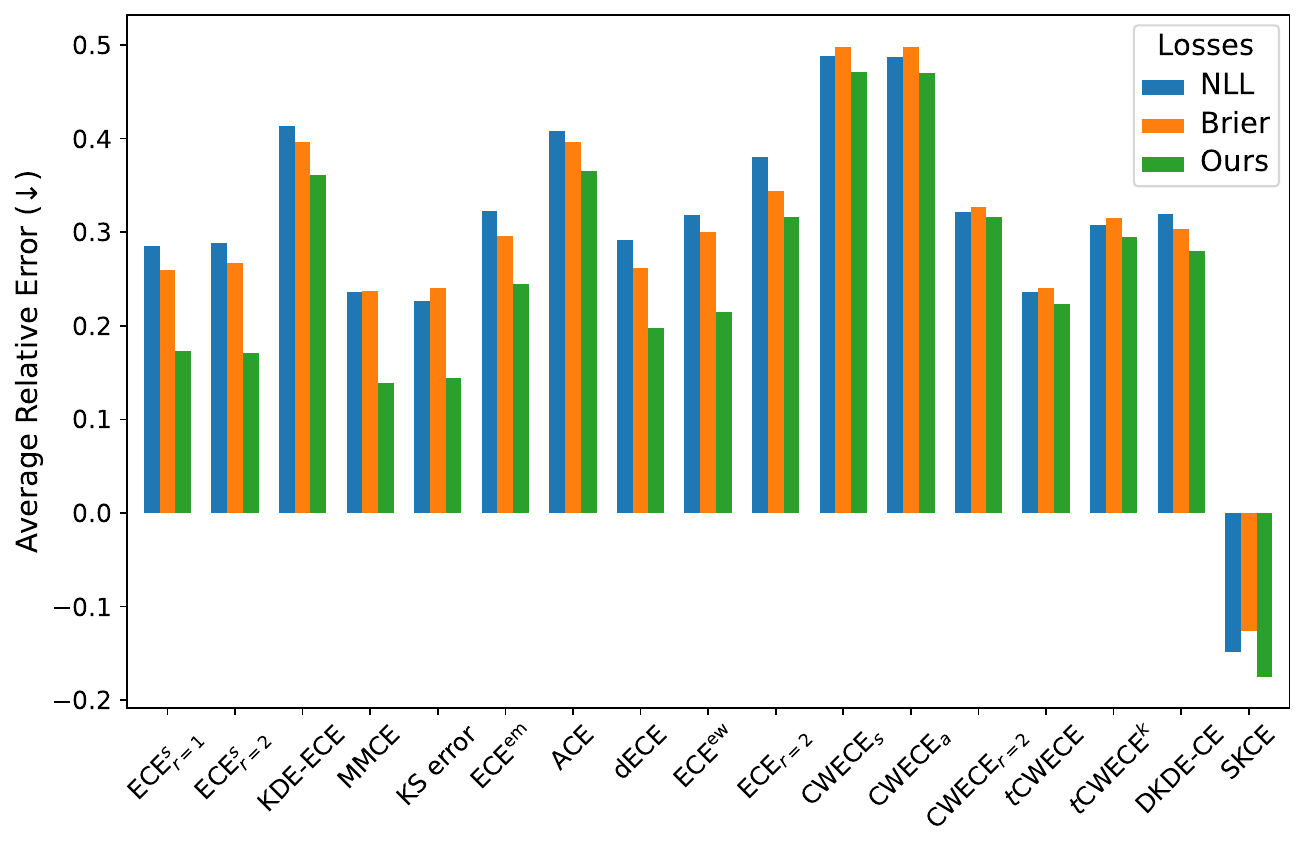}
\caption{Group bar chart comparing the ARE of our objective and PSRs across all metrics}
\label{fig:ComparingPSR}
\end{figure}

\subsection{Ablation Studies}

\subsubsection{Proper Scoring Rule Comparison} 
\label{sec:ablation-psr-comparison}
To eliminate the influence of the training scheme and fully validate the effectiveness of our proposed $h$-calibration training loss, we repeated the entire training process with all learning strategies unchanged, replacing only our loss function with the traditional PSRs. We compared the calibration results using PSRs to ours. Detailed quantitative comparisons of all calibration metrics can be found in \emph{Appendix \ref{sec:apdx-psrcomparison}} (Table \ref{tab:apdx-psrcomparison1} - \ref{tab:apdx-psrcomparison3}), and a visual summary of average relative error is provided in Fig. \ref{fig:ComparingPSR}. It shows that our method outperforms traditional PSRs, i.e., NLL loss (cross-entropy) and the Brier score (MSE), in (normalized) ARE across all 17 metrics. Regarding the (unnormalized) AE in \emph{Appendix \ref{sec:apdx-psrcomparison}}, our method also outperforms both NLL and Brier scores on 16 out of 17 metrics, ranking second only on DKDE-CE. In contrast, neither of the two PSRs showed consistent superiority over the other. These results provide further evidence for the effectiveness of the proposed $h$-calibration loss in learning calibration, supporting the theoretical analysis in Section \ref{method:sec5}, which establishes the advantage of our method over traditional PSRs from a theoretical standpoint.

\subsubsection{Hyperparameter Robustness}

Our method introduces three hyperparameters: constants $M$, $\epsilon$, and a trivial loss multiplier $r$. Unlike many previous methods that introduce variables lacking intuitive explanations (\textbf{limitation \#7}), we have clarified in the method section that our \emph{parameters have clear interpretations}: $M$ controls the approximation effectiveness of our loss function (approximation error decreasing exponentially as $M$ increases); $\epsilon$ reflects the constraining bound on calibration error; and $r$ ensures a reasonable loss range. In our implementation, we used consistent constants for these parameters across all tasks without task-specific tuning. The satisfactory calibration results obtained \emph{indicate the robustness} of our method. 

Here, we further investigate the impact of adjusting these three parameters on calibration. We doubled and halved the parameters from their default values, and for $\epsilon$, we further examined the impact of its exponential increase. We conducted experiments on four ImageNet tasks. Results on 17 metrics are presented in \emph{Appendix \ref{sec:apdx-robustnessanalysis}}, and Fig. \ref{robustnessfigure} shows the overall changes in ARE. The results show that variations in $M$ and the loss multiplier $r$ produce stable outcomes. This aligns with our expectations since the approximation error decreases exponentially with increasing $M$, and for $M$$\geq$$100$, the approximation error can become negligible, so further increasing $M$ should not significantly affect the results. Similarly, variations in $r$ are theoretically equivalent to adjusting the learning rate by a factor of two, which is expected to be inconsequential. When the constant $\epsilon$ increases exponentially, the results remain stable and robust for $\epsilon$~$\leq$$10^{-2}$. However, when increased to $10^{-1}$, the calibration errors show an upward trend. This is consistent with our expectations because an overly large error bound in the constraint prevents strict control of the calibration error.

Overall, these results align with our expectations and \emph{confirm the robustness} of the parameters within intuitive empirical ranges. Notably, the average performance (ARE \& AE) across the above tasks obtained with proper parameters 
fluctuations (excluding $\epsilon$=$10^{-1}$) remains top-ranking compared to previous methods, further verifying the robustness.

\begin{figure}[htbp]
\centering
\includegraphics[width=\columnwidth]{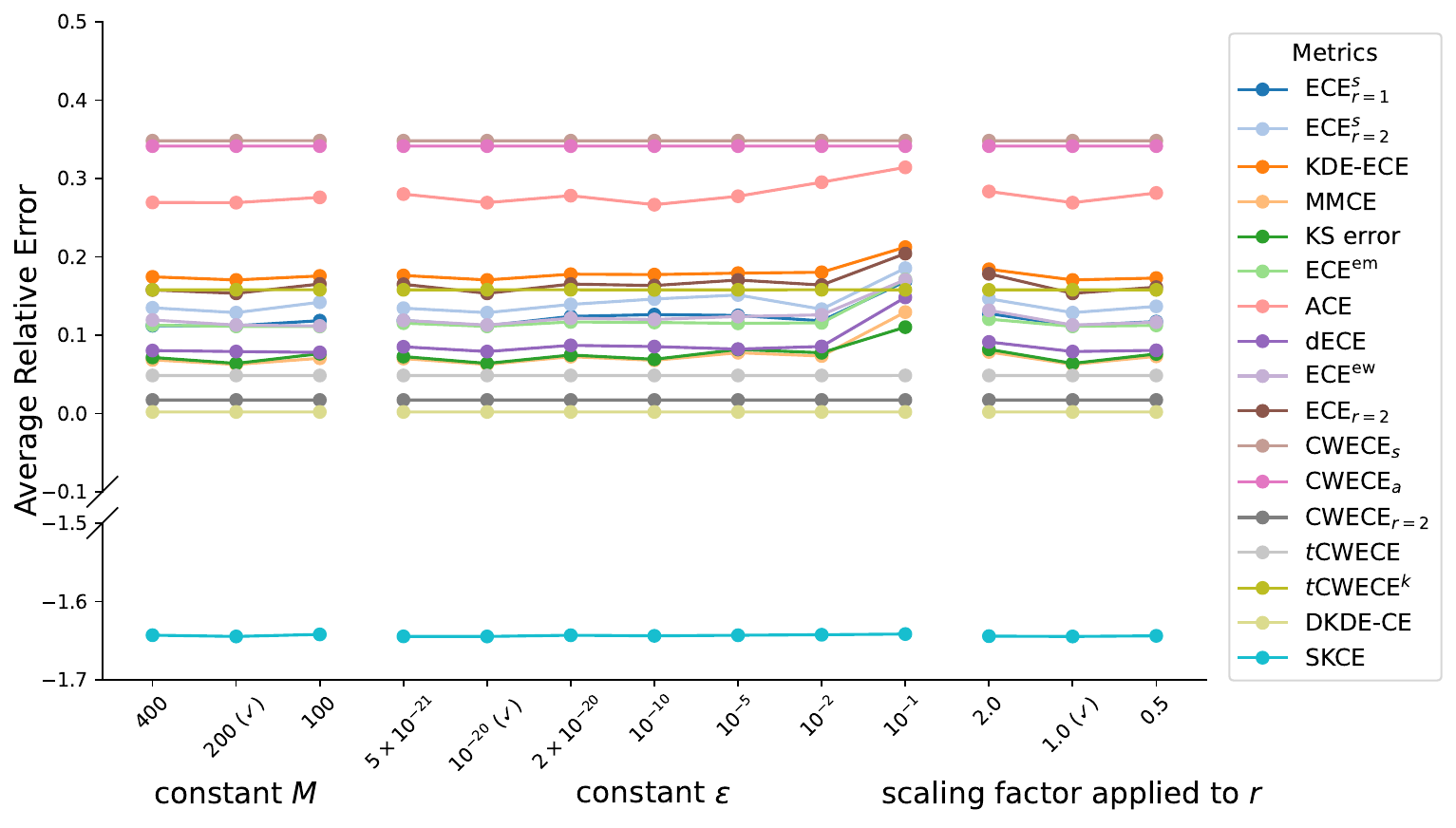}
\caption{ARE of various metrics under hyperparameter variations. Symbol ($\checkmark$) denotes the default values.}
\label{robustnessfigure}
\end{figure}

\section{Conclusion and Future Work}
In this study, we overcome ten common limitations in previous post-hoc calibration research by theoretically constructing a probabilistic learning framework called $h$-calibration. We show its sufficiency to classical calibration definitions, including canonical, top-label, and classwise calibration. Furthermore, we derive an approximately equivalent learning objective, which forms the basis for a simple and effective algorithm to learn canonical calibration with controllable error. Extensive experiments across 15 post-hoc recalibration tasks, involving 20 compared methods and 17 evaluation metrics, show the substantial advantages of our approach over traditional methods. Additionally, we discuss theoretically and experimentally the connections and advantages of our objective compared to traditional PSRs in learning class likelihoods. Our framework offers valuable tools for learning reliable likelihoods in related fields.

This study has limitations and potential areas for future exploration. Firstly, we focus primarily on the post-hoc calibration scenario and do not delve into calibration by modified training schemes. The reasons are threefold: (1) The post-hoc setting offers a more unified and standardized benchmark, allowing for more precise validation of our theoretical framework. (2) With a pretrained model, calibration and discriminative performance can be effectively decoupled, enabling a clearer evaluation of the net gains from different calibration strategies. (3) Post-hoc calibration is widely applicable, as many models prioritize classification performance during training and seek improved likelihood estimates during deployment without compromising accuracy. We acknowledge that adapting our method for training-time calibration may require adjustments, such as tuning hyperparameters to approximate PSRs, to maintain discriminative power. We leave this for future exploration.

Secondly, while our learning objective is computationally efficient (as shown in the pseudocode), following the design in \cite{no.67}, our method includes an automatic selection of the recalibration mapping. This process requires multiple training runs and increases computational cost. Nonetheless, we observed that this adaptive step is important for calibration performance compared to manually specified mapping. This issue is closely related to another open question in calibration research, separate from our learning objective design, which concerns the structure of optimal calibration mappings. We plan to explore this direction in future work. For further discussion on other future works, see \emph{Appendix \ref{sec:apdx-futurework}}.



\newpage
\onecolumn
{
\appendices
\section{Supporting Materials for Some Introduction Arguments} 
\label{sec:introductionevidence}
{
\mdseries
\subsection{Controversies on Calibration Effectiveness of Some Empirical Training Regularizers}
\label{sec:controversies_training_regularizer}
Although various training regularization techniques are employed to mitigate overfitting, increase entropy, or discourage overconfidence for improved calibration, some studies report that the effects of these techniques on calibration remain controversial. For instance, prior investigations suggest that the mixup does not necessarily enhance calibration and can, in some cases, deteriorate calibration \cite{no.90,no.3}. Empirical studies show that ensembles of deep neural networks, which are not directly related to probability calibration \cite{no.57}, are often not well calibrated \cite{no.2}, with overfitted ensemble models showcasing inadequate calibration performance \cite{no.24}. In experiments conducted by Fernando et al., focal loss was found not to yield a significant improvement in calibration \cite{no.48}. Furthermore, there are reports indicating that focal loss can induce under-confidence by generating predictions with higher entropy \cite{no.120,no.12}. The combination of different techniques, such as augmentations and ensemble methods, has been shown to yield models with diminished calibration \cite{no.119,no.61,no.3}. Additionally, it has been noted that regularization methods, including label smoothing, $L_p$ norm, focal loss, and mixup, render models less calibratable, i.e., they can hurt the final calibration performance when the following post-hoc recalibration is allowed \cite{no.53,no.3,no.47}.

\subsection{Unsuitability/Overfitting of Binning-Based Metrics for Calibration Training}
\label{sec:comparing_nll_ece_training}
\begin{wrapfigure}{l}{0.5\textwidth}
    \centering
    \begin{minipage}{0.49\linewidth}
        \centering
        \includegraphics[width=\linewidth]{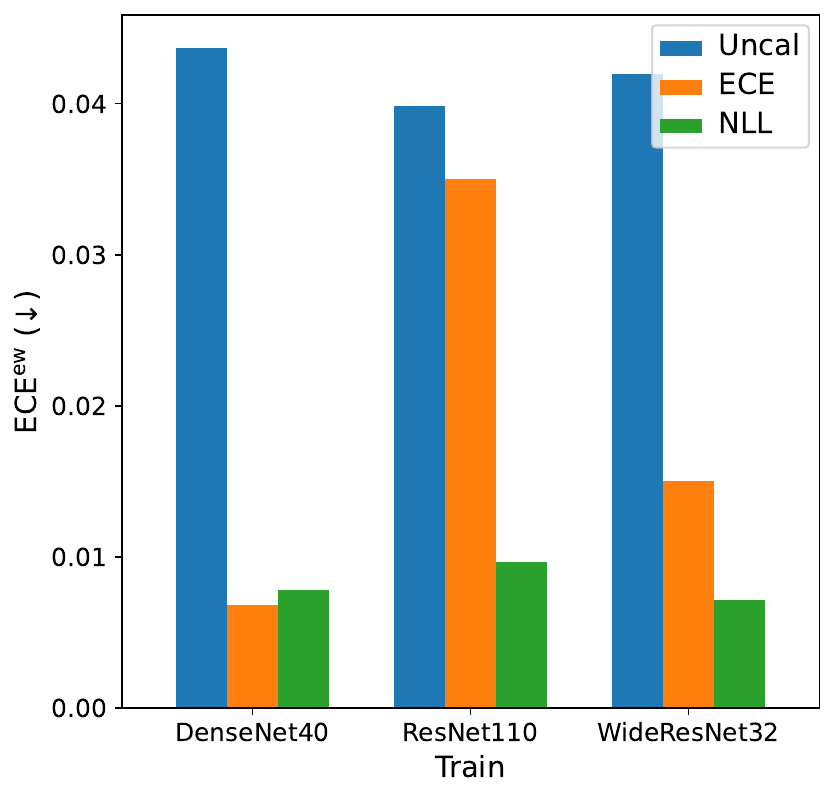}
    \end{minipage}%
    \hfill
    \begin{minipage}{0.49\linewidth}
        \centering
        \includegraphics[width=\linewidth]{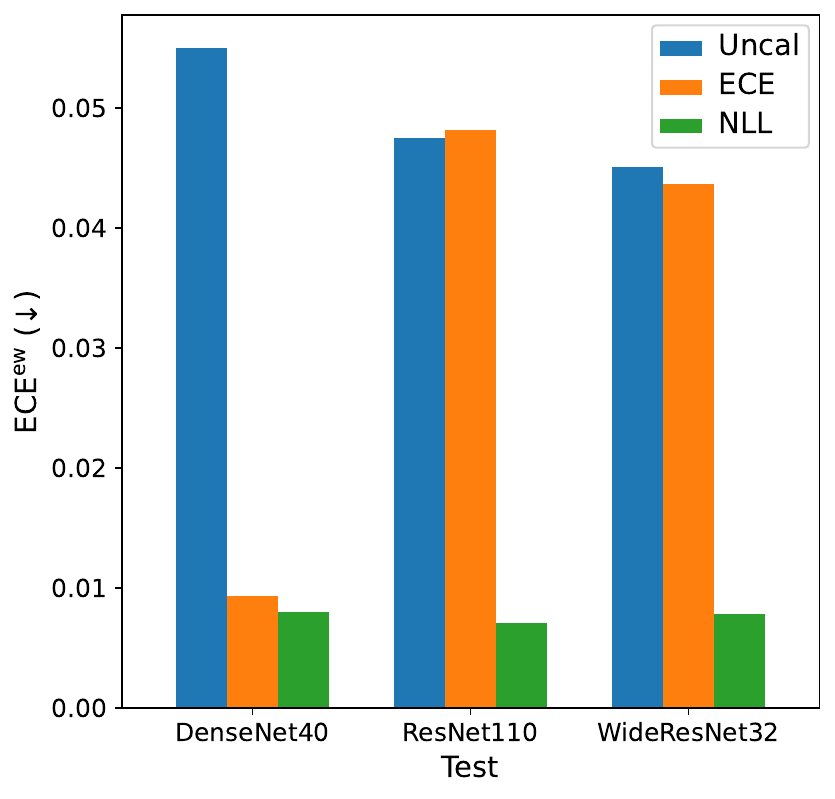}
    \end{minipage}
    \caption{Comparing training with NLL and binning-based ECE$^{\mathrm{ew}}$}
    \label{fig:comparing_nll_ece_training}
\end{wrapfigure}

We demonstrate the unsuitability and overfitting issues of binning-based metrics for calibration training, using the CIFAR10 calibration tasks from \cite{no.71,no.67} as an example. Specifically, we replace the default NLL loss in DIAG calibrator \cite{no.67} with the popular binning-based metric ECE$^{\mathrm{ew}}$ for training the post-hoc calibrator. For the ResNet110 and WideResNet32 tasks, we observe that, despite training with ECE$^{\mathrm{ew}}$, there is no significant improvement in the test set ECE$^{\mathrm{ew}}$ compared to the original uncalibrated (Uncal) performance. In the WideResNet32 task, clear overfitting is observed: the training set ECE$^{\mathrm{ew}}$ decreases noticeably, but there is no corresponding reduction in the test set ECE$^{\mathrm{ew}}$. In the ResNet110 task, both the training and test set ECE$^{\mathrm{ew}}$ remain high. This may be due to the non-differentiability of the binning operations, which results in an ineffective alignment between the initially recalibrated probabilities assigned to each bin and their observed frequencies within those bins, thereby hindering effective training convergence. These experimental results verify the unsuitability of using binning-based metrics as training objectives for calibration.

\subsection{Additional Summary Regarding Limitation \#4}
\label{sec:addtional_comment_limit4}
As outlined in the introduction section, limitation \#4 is not confined solely to methods focusing on equivalent formulation of perfect calibration. It also extends to intuitively designed or binning-based approaches, where numerous methods emphasize top-label or class-wise calibration instead of canonical calibration in their modeling and evaluation processes. For instance, previous studies align top-label probabilities with corresponding accuracies, as seen in the learning objectives of \cite{no.12,no.100,no.37,no.99,no.16,no.17,no.83,no.126}, and the SB-ECE method discussed in \cite{no.58}. Some approaches solely utilize top-label probability or accuracy to construct or interpret calibration methods, such as those proposed in \cite{no.66,no.62,no.80,no.13}, and the pTDE method in \cite{no.145} and S-AvUC in \cite{no.58}. Certain prior methods align class-wise accuracy with probability or design calibration methods based on class-wise probabilities, including methods in \cite{no.86,no.87,no.21,no.89,no.101,no.125,no.88,no.127,no.43,no.29,no.14-0}, or propose techniques that can selectively model top-label or class-wise calibration, as demonstrated in \cite{no.22} and \cite{no.60}. Furthermore, beyond modeling, model evaluations are often conducted only on non-canonical calibration metrics, such as top-label or class-wise calibration metrics, as evidenced by assessments in \cite{no.4,no.12,no.13,no.14-0,no.16,no.19,no.22,no.29,no.37,no.40,no.41,no.42,no.43,no.49,no.50,no.53,no.55,no.57,no.58,no.60,no.68,no.70,no.73,no.76,no.78,no.80,no.81,no.93,no.99,no.100,no.101,no.104,no.107,no.111,no.124,no.126,no.128,no.133,no.135,no.146,no.150}.

}

\section{More Detailed Discussion from Limitation \#6 to Limitation \#10} 
\label{sec:limitationdiscussion}
{
\mdseries
\textbf{Limitation \#6} underscores the reliance of some approaches/studies on numerous unverified assumptions. For instance, assumptions such as the adherence of learned representations to Gaussian distributions \cite{no.4} or Gaussian processes \cite{no.68,no.84,no.44}, the assignment of Gaussian priors to model parameters \cite{no.28,no.106,no.19}, or the presumption that data can be generated by Gaussian generative models \cite{no.49}. Additionally, assumptions extend to the conformity of class confidence or probability in model output (such as $\hat{p}_l(X)$, $\hat p_l(X)|Y=l$, $\hat{p}(X)|Y=l$ for class $l$ ) to Beta distributions \cite{no.50,no.85,no.151} or Dirichlet distributions \cite{no.71}, and the incorporation of a Beta distribution prior for binwise confidence \cite{no.87}. These empirical assumptions lack strict theoretical or empirical evidence, and certain assumptions may even be inherently contradictory. For example, it can be shown that the assumption that produced confidence follows a Beta distribution \cite{no.50,no.85,no.151} or Dirichlet distribution \cite{no.71} can conflict with the hypothesis in \cite{no.52} that the probability density of predicted probabilities within the simplex is Lipschitz continuous.

Many methods involve settings or hyperparameters that are non-universal or challenging to determine directly through theory or experience (\textbf{limitation \#7}). Examples include the configurations in implicit regularized data augmentation and ensemble techniques mentioned above; the binning settings in binning-based methods as indicated by \cite{no.5,no.21,no.50,no.51,no.52,no.70,no.31,no.34,no.1,no.22,no.60,no.76}; and the specifications of kernel functions in RKHS (reproducing kernel Hilbert space) from studies \cite{no.83,no.79,no.19} and in kernel smoothing from studies \cite{no.52,no.70}, as highlighted in \cite{no.79,no.22,no.50,no.1}. Moreover, the determination of weights for calibration regularizers, utilized as secondary optimization objectives in studies \cite{no.12,no.16,no.19,no.22,no.41,no.52,no.54,no.58,no.62,no.66,no.83,no.99,no.100,no.101,no.102,no.103,no.107,no.111,no.128,no.145,no.150}, serves as another example. It is worth noting that, despite some recent studies exploring optimization related to binning settings \cite{no.1,no.21,no.50,no.76}, there is currently no standard criterion for deciding which type of binning \cite{no.5}, for example, equal-mass or equal-width, or what number of bins, or what kind of membership functions in soft binning to use.

\textbf{Limitation \#8} pertains to the tradeoff between the probabilistic unit measure property and calibration. Many calibration approaches proposed for binary classification tasks, such as \cite{no.86,no.87,no.89,no.85,no.21,no.125,no.151,no.127}, necessitate the implementation of extension strategies to generalize calibration to multi-class scenarios.  For example, the techniques proposed in \cite{no.86,no.87,no.89,no.85} are extended by \cite{no.65} through the adoption of a one-vs-rest strategy. Some methods inherently employ the one-vs-rest strategy to calibrate probabilities for all classes, as seen in \cite{no.60,no.75,no.133,no.88,no.46}. Furthermore, certain methodologies individually calibrate all probability scalars (e.g., \cite{no.1} and \cite{no.70}) or their subsets (e.g., top-label probability in \cite{no.145}). The aforementioned binary-to-multiclass extension strategies, along with techniques for individual probability scalar calibration, compromise the unit measure property of probability (calibrated probability vector no longer sums to 1). Although this disruption can be addressed through additional normalization, as demonstrated in works such as \cite{no.65,no.133,no.46,no.88}, this extra normalization introduces a conundrum between probabilistic validity and calibration, leading to a situation where the probabilities are no longer guaranteed to be calibrated \cite{no.5,no.71,no.46,no.57,no.1}.

\textbf{Limitation \#9} concerns to the issue of non-accuracy preservation. Firstly, as mentioned above, methods employing a modified training scheme lack an accuracy-preserving property \cite{no.24,no.29,no.33,no.53} and incur significant additional computational overhead \cite{no.40,no.7,no.16,no.24,no.29}. Relevant methods include those proposed in \cite{no.3,no.4,no.6,no.12,no.13,no.14-0,no.16,no.18,no.19,no.22,no.24,no.28,no.33,no.39,no.41,no.42,no.43,no.44,no.48,no.49,no.52,no.53,no.54,no.55,no.58,no.61,no.62,no.63,no.64,no.66,no.72,no.73,no.83,no.84,no.90,no.91,no.93,no.95,no.99,no.100,no.101,no.102,no.103,no.104,no.106,no.107,no.111,no.120,no.129,no.130,no.139,no.81}. In contrast, post-hoc methods can mitigate perturbations to accuracy by enforcing monotonicity in the recalibration mapping of the sample-wise probability vector. However, many post-hoc approaches fail to ensure this monotonicity, compromising the original classification accuracy, often leading to a decrease. Examples include methods proposed in studies such as \cite{no.1,no.21,no.71,no.85,no.86,no.68,no.87,no.88,no.89,no.151,no.125,no.127,no.146,no.150}, and the vector and matrix scaling methods in \cite{no.109}. It is noteworthy that monotonicity in some binary classification calibration methods, e.g., \cite{no.85,no.151,no.127} and \cite{no.88}, were imposed on the transformations for class-wise probabilities rather than for sample-wise probabilities, failing to ensure accuracy preservation. Furthermore, methods adopting the one-vs-rest strategy (see above for relevant studies) that utilize multiple calibration mappings for different classes may disrupt the order of sample-wise classification probabilities, thus lacking accuracy preservation. Moreover, even when sample-wise monotonicity is ensured but not strictly, as in \cite{no.70}, there remains a risk of losing discriminability of logits or probabilities within certain intervals. In fact, methods based on hard binning for direct calibration inherently suffer from this issue (e.g., \cite{no.1,no.21,no.75,no.86,no.87,no.88,no.89,no.125,no.126,no.133}), where the calibrated values within the same bin lack discriminability \cite{no.5}.

\textbf{Limitation \#10} pertains to the issue of applicability. Some calibration methods are limited to specific models or require modifications to the original network structure or training procedures, thereby restricting their utility. For instance, \cite{no.12} and \cite{no.23} are applicable to models with block-wise network structures. \cite{no.15} is tailored for dynamic neural networks. \cite{no.18} and \cite{no.44} are designed for transformer-based network structures. Method \cite{no.19} is crafted for Bayesian neural networks (BNN). \cite{no.55} is specifically applicable to networks with batch normalization. Methods \cite{no.24}, \cite{no.139}, and \cite{no.145} are designed for models with ensembling structures such as Deep-Ensemble, Monte Carlo Dropout, and MIMO (Multi-Input and Multi-Output). Regarding modifications to the network structure or training strategies, Tao \emph{et al.} \cite{no.12} searches for the best-fitting combination of block predecessors to determine the network structure. Ye \emph{et al.} \cite{no.18} and Chen \emph{et al.} \cite{no.44} modify the standard attention structure in transformers. Zhong \emph{et al.} \cite{no.55} adjusts the parameter update strategy of batch normalization. Liu \emph{et al.} \cite{no.129} adds spectral normalization to hidden network weights and replaces the output layer with a Gaussian Process. Milios \emph{et al.} \cite{no.84} modifies the network by applying a Gaussian approximation in the logit space. Tian \emph{et al.} \cite{no.130} proposes geometric sensitivity decomposition to modify the standard last linear decision layer (softmax layer with preceding linear transformations). Tomani \emph{et al.} \cite{no.23} requires the selection of specific hidden layers for different backbone networks to learn calibration mappings. Galdran \emph{et al.} \cite{no.33} replaces the last layer of a network with a set of heads supervised with different loss functions. Xing \emph{et al.} \cite{no.63} introduces an additional confidence model to estimate the distance to prototypical class centers. Malmstr{\"o}m \emph{et al.} \cite{no.28} modifies the training process by proposing local linear approaches to estimate the posterior distribution of parameters and use it to generate the probability mass for calibration. Maddox \emph{et al.} \cite{no.106} modifies the training process by adopting stochastic weight averaging for the posterior distribution approximation of network weights to perform Bayesian model averaging. Wang \emph{et al.} \cite{no.54} adjusts the classification head by adding a category for modeling uncertainty and selecting specific hidden layers implementing SGLD (Stochastic Gradient Langevin Dynamics) sampling to generate features for confidence regularization.
}

\section{Summary of Calibration Methods}

\subsection{Summary of Calibration Methods by Modified Training Scheme}
\label{sec:trainingtime-calibration-summary}
{
\mdseries
Modified training schemes aim to enhance calibration during the training of classifiers and can be broadly categorized into four types: (a) augmentation or implicit regularization, e.g., \cite{no.73,no.49,no.97,no.72,no.93,no.111,no.94,no.95,no.118,no.43,no.24,no.104,no.109,no.12,no.18,no.44}, (b) model ensembling, e.g., \cite{no.77,no.139,no.114,no.113,no.106,no.28,no.115,no.20,no.33}, (c) regularization by explicit loss functions, e.g., \cite{no.64,no.83,no.94,no.62,no.58,no.99,no.102,no.100,no.101,no.53,no.84,no.103,no.107,no.13,no.4,no.52,no.66,no.22,no.60,no.48,no.14-0,no.54,no.63,no.11}, and (d) some hybrid methods, e.g., \cite{no.24,no.16,no.19,no.55,no.39,no.42}.

Regarding (a) augmentation or implicit regularization, studies in \cite{no.73,no.49,no.97} found that deep networks trained with mixup \cite{no.110} are better calibrated. Research by \cite{no.72,no.93,no.111,no.94} demonstrates that label smoothing \cite{no.112} serves as another regularization technique, reducing overconfidence. Other augmentation techniques, such as AugMix, CutMix, and AutoLabel, have been demonstrated to yield calibration benefits by Hendrycks \emph{et al.} \cite{no.95}, Yun \emph{et al.} \cite{no.118}, and Qin \emph{et al.} \cite{no.43}, respectively. Kim \emph{et al.} \cite{no.24} leverage various text augmentation techniques for calibrating transformer-based language models. Patel \emph{et al.} \cite{no.104} propose an adversarial data generation technique, i.e., OMADA, which yields calibration gains. Furthermore, Guo \emph{et al.} \cite{no.109} and Tao \emph{et al.} \cite{no.12} verify that weight decay and early stopping can alleviate miscalibration, respectively. Ye \emph{et al.} \cite{no.18} and Chen \emph{et al.} \cite{no.44} propose modified attention modules, i.e., LRSA and SGPA, respectively, for transformers, and such implicit regularizations show effectiveness in suppressing miscalibration. 

For (b) ensemble methods, various strategies, such as training multiple independent models with different weights \cite{no.77,no.139,no.114,no.113}, ensembles based on Bayesian networks (random weights) predictions \cite{no.106,no.28}, MIMO ensemble models utilizing subnetworks constructed through varying nerual connections \cite{no.115}, and ensembles with multiple heads generated through modified network structures \cite{no.20,no.33}, have been shown to enhance calibration.

Concerning (c) regularization by explicit loss functions, numerous methods have been proposed. For instance, replacing cross entropy (CE) loss with focal loss \cite{no.116}, as suggested by Mukhoti \emph{et al.} \cite{no.64}, has been shown to improve calibration, which can be interpreted as introducing a maximum-entropy regularizer. Kumar \emph{et al.} \cite{no.83} proposed MMCE, a kernel embedding-based measure, for constructing the equivalent form of perfect calibration for top-label probabilities, acting as a differentiable regularizer trained alongside standard CE loss. Pereyra \emph{et al.} \cite{no.94} relate label smoothing to maximum-entropy regularizer and propose a confidence penalty regularization added to CE loss. Krishnan and Tickoo \cite{no.62} introduce the AvUC regularization loss to learn models confident in accurate predictions and with higher uncertainty when likely to be inaccurate. Karandikar \emph{et al.} \cite{no.58} further propose a soft version for AvUC term and ECE measure as a differentiable auxiliary surrogate loss. Similarly, Bohdal \emph{et al.} propose a differentiable surrogate loss for ECE \cite{no.99}, utilizing an additional binning membership network to learn soft binning weights, combined with CE loss for model training. To discourage overconfident predictions, Seo \emph{et al.} \cite{no.102} design a VWCI loss function consisting of two cross-entropy loss terms with respect to the target and uniform distribution. Liang \emph{et al.} \cite{no.100} incorporate a DCA regularization term, which can be considered equivalent to ECE with a single bin, into the CE loss. Another similar auxiliary loss term named MDCA, equivalent to classwise ECE with a single bin, was proposed by Hebbalaguppe \emph{et al.} \cite{no.101}. Wang \emph{et al.} \cite{no.53} introduces an inverse focal loss that encourages overconfidence during the main training stage but benefits post-hoc recalibration by preserving sample hardness information. Milios \emph{et al.} \cite{no.84} apply a Gaussian prior approximation in the logit space, turning the original classification into a regression problem and discovering benefits for calibration. Joo \emph{et al.} \cite{no.103} explore the effect of applying $L_p$ norm regularization in function space (e.g., logits space) for calibration. Yun \emph{et al.} \cite{no.107} proposed a regularization term enforcing consistent predictions between different samples of the same label, reducing intra-class variations and mitigating overconfidence.  Tao \emph{et al.} \cite{no.13} propose a dual focal loss (DFL) aiming at reducing the size of the under-confidence region while preserving the advantages of focal loss in mitigating over-confidence. Chen \emph{et al.} \cite{no.4} propose a knowledge-transferring-based calibration method by estimating the importance weights in CE loss for samples of tail classes to implement long-tailed calibration. Popordanoska \emph{et al.} \cite{no.52} directly utilizes high-dimensional kernel smoothing to estimate class distribution given predicted probability, whose difference to predicted probability difines the calibration error and combined with CE to form KDE-XE loss. Moon \emph{et al.} \cite{no.66} propose a regularizer CRL regularization term, aligning accuracy and confidence by enforcing conﬁdence estimates whose ranking among samples are effective to distinguish correct from incorrect predictions. Yoon \cite{no.22} \emph{et al.} propose an ESD regularizer inspired by Kolmogorov-Smirnov error \cite{no.60}, assessing the alignment of top-label accuracy and confidence through the expected squared difference. Fernando \emph{et al.} \cite{no.48} propose DWB loss, rebalancing class weights based on class frequency and predicted probability of the ground truth class, interpretable as an entropy maximization term to penalize over-confident predictions. Benz and Rodriguez \cite{no.14-0} propose to align model confidence with the decision maker's confidence for calibration. Wang \emph{et al.} \cite{no.54} propose a $K+1$-way softmax formulation and an energy-based objective function, allowing the modeling of marginal data distribution using the extra dimension, which is beneficial for model calibration. Xing \emph{et al.} \cite{no.63} propose DBLE model, which bases its confidence estimation on distances in the representation space learned by an additional confidence model equipped with episodic training and prototypical loss for classification. Błasiok \emph{et al.} \cite{no.11} introduce structural risk minimization theory, advocating the integration of proper scoring rules with regularizations that assess the complexity of the network within specific constrained families, to discover well-calibrated networks.

Finally, some studies combine methods from the above three strategies to achieve calibration, as seen in \cite{no.24,no.16,no.19,no.55,no.39,no.42}.

}

\subsection{Summary of Post-hoc Calibration Methods}
\label{sec:posthoc-calibration-summary}
{
\mdseries
Post-hoc methods can be categorized into parametric and non-parametric methods: (a) the former use parametric models to design learning objectives, e.g., \cite{no.85,no.151,no.65,no.71,no.127,no.150,no.68}, while non-parametric methods can be further classified into five categories by learning objectives, including (b) objectives inspired by binning-based evaluation metrics, e.g., \cite{no.86,no.87,no.88,no.89,no.21,no.125,no.65,no.17,no.58,no.126,no.133}, (c) constructing equivalent forms for perfect calibration, e.g., \cite{no.60,no.70,no.145,no.52}, (d) other methods using proper scoring rules, e.g., \cite{no.20,no.23,no.40,no.57,no.67,no.78,no.80,no.124,no.135,no.146,no.109}, (e) other empirical methods, e.g., \cite{no.29,no.34,no.46,no.128,no.147,no.138,no.145}, and (f) hybrid strategies, e.g., \cite{no.70,no.75}. The specific methods are summarized as follows: 

For (a), Beta calibration \cite{no.85,no.151} was initially proposed for binary classification, assuming the confidence distribution under a target class follows a Beta distribution. It maximizes log-likelihood to learn model parameters and has been extended to multiclass scenarios using the one-vs-rest strategy \cite{no.65}. Dirichlet calibration \cite{no.71} generalizes Beta calibration to multi-classification with Dirichlet distributions. Bayesian isotonic calibration (Bayes-Iso) \cite{no.127} by Allikivi \emph{et al.} employs a prior over piecewise linear monotonic calibration maps and utilizes Monte Carlo sampling to approximate the posterior mean calibration map through likelihood maximization. Maronas \emph{et al.} \cite{no.150} proposed decoupled Bayesian neural networks (BNN) implemented with an MLP-based BNN trained with validation data to transform original probabilities, demonstrating improved calibration. Wenger \emph{et al.} \cite{no.68} proposed a recalibration approach based on a latent Gaussian process applied to classwise logits, inferred using variational inference.

For (b), various studies, including \cite{no.86,no.87,no.88,no.89,no.21,no.125}, minimizes binning-wise calibration error to align the accuracy and predicted probability of the positive class in binary classification, employing different construction of binning mappings. Specifically, Zadrozny \emph{et al.} \cite{no.86} uses histogram binning; Zadrozny \emph{et al.} \cite{no.88} constrains the monotonicity of binning mappings; Naeini \emph{et al.} \cite{no.87} applies Bayesian averaging to ensemble multiple calibration maps by histogram binning; Naeini \emph{et al.} \cite{no.89} extends \cite{no.88}  by relaxing the monotonic mapping with near-isotonic regression; Sun \emph{et al.} \cite{no.21} optimizes the number of bins for equal-mass binning by the MSE decomposition framework, balancing calibration and sharpness. ROC Binning \cite{no.125} enhances histogram binning by considering variations in the prevalence of the positive class within the dataset. These binary calibration methods can theoretically be extended to multi-class tasks using a one-vs-rest strategy, as shown in \cite{no.88} and \cite{no.65} for \cite{no.86}, \cite{no.87}, \cite{no.89}. Other methods \cite{no.17,no.58,no.126} minimize binning-wise calibration error to align top-label accuracy and confidence, differing in their binning construction techniques. Specifically, Clart{\'e} \emph{et al.} \cite{no.17} minimizes an expectation consistency term, similar to ECE estimate with a single bin; Frenkel \emph{et al.} \cite{no.126} performs a grid search for class-wise temperatures to minimize ECE; Karandikar \emph{et al.} \cite{no.58} uses soft binning ECE as the optimization objective by employing a soft membership function; Patel \emph{et al.} \cite{no.1} introduces I-Max binning, preserving label information under binning quantization to mitigate accuracy losses, and addresses sample-inefficiency by employing a shared class-wise binning strategy. Gupta \emph{et al.} \cite{no.133} proposes a multiclass-to-binary reduction framework, aligning top-label or classwise predicted probability with event frenquency through the utilization of top-label ECE and classwise ECE. It also suggests using high-dimensional multiclass binning, such as Sierpinski binning, Grid-style binning, and projection-based histogram binning, for canonical calibration in tasks with a small number of classes ($\leq 5$).

For (c), Gupta \emph{et al.} \cite{no.60} propose a calibration metric for classwise or top-r calibration based on the Kolmogorov-Smirnov test, comparing two empirical cumulative distributions. The recalibration mapping is obtained by optimizing the metric using a spline-fitting approach. Additionally, Zhang \emph{et al.} \cite{no.70} propose a kernel smoothing-based density estimator to obtain an equivalent form for perfect calibration. It was initially used only as an evaluation metric and later adopted as an optimization objective by subsequent research \cite{no.145}. While theoretically applicable to estimating canonical calibration error, this term is significantly impacted by the curse of dimensionality and therefore is practically employed exclusively for top-label calibration, as also noted in \cite{no.52}.

Methods in (d) include \cite{no.20,no.23,no.40,no.57,no.67,no.78,no.80,no.124,no.135,no.146}, and three scaling techniques introduced in \cite{no.109}. These methods differ primarily in their empirical design of the recalibration mapping. Specifically, Guo \emph{et al.} \cite{no.109} introduce temperature scaling (TS), vector scaling (VS), and matrix scaling (MS), applying different linear transformations to logits. Rahimi \emph{et al.} \cite{no.67} designs nonlinear accuracy-preserving mappings for logits, while \cite{no.146} directly applys MLP transformations. Ji \emph{et al.} \cite{no.80} suggests a piecewise scaling approach, utilizing different temperatures for confidence bins. Tomani \emph{et al.} \cite{no.135} employs a parameterized TS using an MLP to learn sample-specific temperature for logit scaling. Laves \emph{et al.} \cite{no.78} combines TS with dropout, where the linear layer for calculating logits in the pre-trained network was replaced by a linear layer with dropout. Some methods are designed for particular scenarios, such as those in \cite{no.20,no.23} and \cite{no.40} for out-of-distribution and distribution shift scenarios, \cite{no.57} for segmentation settings with mapping using spatial information, \cite{no.124} for binary classification with mapping ensuring flexibility, monotonicity, continuousness, and computational tractability (It can be proven that the monotonicity in binary classification does not extend to multi-class scenarios, i.e., the transform monotonicity of class-wise probability conflicts with the unit measure property for non-identity mappings). It is worth noting that transformations such as VS, MS and the MLP transformation in \cite{no.146} for logits, and the transformation in \cite{no.78} for features are not accuracy-preserving. 

For (e), Jung \emph{et al.} \cite{no.29} found the correlation between calibration error and the variance of class-wise training losses, proposing a adaptive weight for class-wise training loss to control the variance for calibration. Conde \emph{et al.} \cite{no.34} propose test-time augmentation techniques for analyzing predictions on various augmented images to enhance calibration in image classification. Valk and Kull \cite{no.46} introduce the Locally Equal Calibration Errors (LECE) assumption, proposing an intuitive approach by minimizing miscalibration defined by the average difference between predictions and labels in the neighborhood of the high-dimensional probabilistic predictions. Joy \emph{et al.} \cite{no.128} leverage a VAE to encapsulate latent features of the classifier, using the low-dimensional representation of VAE to learn sample-specific temperature for logit scaling. Additionally, there are learning objectives designed for specialized scenarios, such as calibration with noisy labels \cite{no.147}, unlabelled calibration datasets \cite{no.138}, and ensemble calibration with multiple models \cite{no.145}.

Finally, hybrid methods that combine different strategies include combination of ensemble temperature scaling and isotonic regression in \cite{no.70}, and the scaling-binning calibrator in \cite{no.75}.
}

\section{Discussion for Proposition \ref{pro1}}
\label{sec:apdx-propsition-hardlabeling}
\newtheorem*{proposition*}{Proposition}
\begin{proposition*}
For a feature representation $F_i$ in a network and its corresponding observation label $Y_i$, the true conditional probability $p(Y|F_i)$ cannot be guaranteed to be the one-hot vector of the target label $Y_i$.
\end{proposition*}

\begin{remarkqed}
This can be attributed to the influence of the network indicative bias. We illustrate this with an example. If we consider the output probabilities as $F_i$, the true conditional probability corresponding to these probabilities being equal to the one-hot vector of $Y_i$ is evidently incorrect. This is because such equivalence implies that network's output probabilities are almost error-free when approaching the one-hot vector, whereas such an ideal recognition performance is basically unattainable with the existence of any inductive bias.
\end{remarkqed}

\section{Explanation of Key Notations in the Manuscript}
\label{sec:apdx-notation}
{
\mdseries
\begin{longtable}{lp{12cm}}
\caption{Notation Table} \label{tab:long} \\
\hline
\endfirsthead

\multicolumn{2}{c}%
{\tablename\ \thetable{} -- continued from previous page} \\
\hline
\endhead

\hline \multicolumn{2}{r}{{continued on next page}} \\
\endfoot

\hline
\endlastfoot

\multicolumn{2}{l}{ \textbf{\underline{Random Variable (r.v.), Variable Space and Variable Mapping:}}} \\
$(F,Y)$ & feature-label pair.  \\
\rowcolor{gray!10} $\Omega_F$, $\Omega_Y$, $\Omega_F \times \Omega_Y$ & the space of $F$,  $Y$ and the product space of $F$ and $Y$, respectively.  \\
$\mathscr H(F)$ & a scalar or vector mapping of $F$, with different forms under different calibration types, detailed in Table \ref{tab1}.   \\
\rowcolor{gray!10} $\mathscr E$ & a scalar or vector valued indicator functions, with different forms under different calibration types, detailed in Table \ref{tab1}. \\
$S$ and $\mathcal S$ & $S$ denotes a scoring rule function, and $\mathcal{S}$ is the expected score of the scoring rule $S$. refer to Section \ref{sec:psranalysis} for a detailed definition. \\
\rowcolor{gray!10} $T_1$, $T_2$, $T_3$ and $\mathscr T$ & abbreviations for the functions $T_1(A, [R_1,R_2], Q_A^{\overline{R_1 R_2}})$, $T_2(A, [R_1,R_2], Q_A^{\overline{R_1 R_2}})$, $T_3(A, [R_1,R_2], Q_A^{\overline{R_1 R_2}})$, and $\mathscr T(A, [R_1,R_2], Q_A^{\overline{R_1 R_2}})$, respectively. \\
\rowcolor{gray!10} ~ & $T_1$, defined in Eq. \eqref{eq:t1_definition}, is the weighted sum of the $p_c(Y_i \notin A) \mathds 1_{\{Y_i \in A\}}$ for each sample $i$ in the set $Q_A^{\overline{R_1 R_2}}$, where samples within the non-boundary subset $Q_A^{\overset{\frown}{R_1 R_2}}$ are assigned a weight of 2, and samples within the boundary subsets $Q_A^{R_1}$ and $Q_A^{R_2}$ are given a weight of 1.\\
\rowcolor{gray!10} ~ & $T_2$, defined in Eq. \eqref{eq:t2_definition}, is the weighted sum of the $p_c(Y_i \in A) \mathds 1_{\{Y_i \notin A\}}$ for the same sample set and weight as $T_1$. \\
\rowcolor{gray!10} ~ & $T_3$, defined in Eq. \eqref{eq:t3_definition}, conveys the sum of weights, equaling twice the number of instances in $Q_A^{\overset{\frown}{R_1 R_2}}$ in addition to the numbers of instance in $Q_A^{R_1}$ and $Q_A^{R_2}$. \\
\rowcolor{gray!10} ~ & $\mathscr T$ indicates the error associated with $A$, $[R_1,R_2]$, and $Q_A^{\overline{R_1 R_2}}$, calculated as $\mathscr T = |T_1-T_2|/T_3$ in Eq. \eqref{eq:thm7regterm}. \\
$\widetilde T^A_1$, $\widetilde T^A_2$, $\widetilde T^A_3$ and $\widetilde {\mathscr T}$ & abbreviations for the functions $\widetilde T^A_1(A, [R_1,R_2], Q_A^{\overline{R_1 R_2}})$, $\widetilde T^A_2(A, [R_1,R_2], Q_A^{\overline{R_1 R_2}})$, $\widetilde T^A_3(A, [R_1,R_2], Q_A^{\overline{R_1 R_2}})$ and $\widetilde{\mathscr T}(A, [R_1,R_2], Q_A^{\overline{R_1 R_2}})$, which are reformulated forms of $T_1$, $T_2$, $T_3$ and $\mathscr T$, respectively. \\
~ & $\widetilde T^A_1$, defined in Eq. \eqref{eq:rewriteT1}, is the sum of the $p_c(Y_i \notin A) \mathds 1_{{Y_i \in A}}$ for each sample $i$ in the set $Q_A^{\overline{R_1 R_2}}$. \\
~ & $\widetilde T^A_2$, defined in Eq. \eqref{eq:rewriteT2}, is the sum of the $p_c(Y_i \in A) \mathds 1_{{Y_i \notin A}}$ for each sample $i$ in the set $Q_A^{\overline{R_1 R_2}}$. \\
~ & $\widetilde T^A_3$, defined in Eq. \eqref{eq:rewriteT3}, represents the number of samples in $Q_A^{\overline{R_1 R_2}}$. \\
~ & $\widetilde{\mathscr T}$ indicates the error associated with $A$, $[R_1,R_2]$, and $Q_A^{\overline{R_1 R_2}}$, calculated as $\widetilde {\mathscr T} = | \widetilde  T^A_1- \widetilde  T^A_2|/ \widetilde T^A_3$. \\
\rowcolor{gray!10} $\mathscr L (\mathscr R)$ and $\mathcal L$ & $\mathscr L (\mathscr R)$ represents the loss value corresponding to interval $\mathscr R$ (i.e., $[R_1,R_2]$), and $\mathcal L$ denotes the overall loss. \\
$w(\mathscr R)$ & weighting function for the interval $\mathscr R=[R_1,R_2]$. \\[6pt]
\multicolumn{2}{l}{\textbf{\underline{Indices:}}} \\
\rowcolor{gray!10} $i$ and $N$ & sample size of $N$ with sample index $i$ within the range $1 \leq i \leq N$. \\
$l$ and $L$ & class number of $L$ for r.v. $Y$ with class index $l$ within range $1 \leq l \leq L$. \\[6pt]
\multicolumn{2}{l}{\textbf{\underline{Probabilities and Measures:}}} \\
\rowcolor{gray!10} $p_\mu$ & ground truth probability, which in some contexts is succinctly referred to as $p$. \\
$p_c$, $^cp_i^l$ and ${^c}p_i^A$ & $p_c$ signifies the predicted calibrated probability, with $p_c(Y_i \in A | F_i)$ illustrating the predicted calibrated classification probability for sample $i$. the notations $^cp_i^l$ and ${^c}p_i^A$ are abbreviations for $p_c(Y_i=l|F_i)$ and $p_c(Y_i \in A|X_i)$, respectively. \\
\rowcolor{gray!10} $g$, $g_\theta$, $g^l$, and $g^A$ & $g$, defined from an optimization viewpoint, signifies the calibration mapping of r.v. $F$, and is sometimes denoted as $g_\theta$ to highlight its inclusion in a parameterized restricted function family with $\theta \in \Theta$. essentially, the prediction made by $g(F)$ equates to the calibrated probability, i.e., $g(F)=p_c(Y|F)$. the terms $g^l$ and $g^A$ refer $p_c(Y=l|F)$ and $p_c(Y \in A|F)$, respectively. \\
$p_F$ & the distribution of feature $F$.\\
\rowcolor{gray!10} $\lambda_{\mathrm{Lebesgue}}$ & Lebesgue measure. \\
$p_{\mathrm{psr}}$ and $p_{M}$ & the probabilistic predictors derived from the traditional proper scoring rule (psr) and the proposed approach, respectively. \\[6pt]
\multicolumn{2}{l}{\textbf{\underline{Sets and Events, Indicator Functions, Set Classes:}}} \\
\rowcolor{gray!10} $A$, $Y \in A$, $\mathds 1_{\{ Y \in A\}}$, $\mathds 1_A(Y)$ & $A$ refers to an event associated with r.v. $Y$, i.e., a subset of the space $\Omega_Y=\{1,2,...,L\}$. The occurrence of event $A$ is denoted by $Y \in A$, typically represented via the indicator functions $\mathds{1}_{\{ Y \in A\}}$ or $\mathds{1}_A(Y)$, which are one for occurrence and zero otherwise. \\
$\mathscr{F}_Y$ and $\mathscr{F}_F$ & $\mathscr{F}_Y$ refers to the $\sigma$-field of $Y$, encompassing all events, i.e., $\{\{1\}, \{2\}, ..., \{1, 2\}, ..., \{L-1, L\}, \{1, 2, 3\}, ..., \{1, 2, ..., L\}\}$.  $\mathscr{F}_F$ denotes the $\sigma$-field of the feature $F$. \\
\rowcolor{gray!10} $\mathscr{B}_{[0,1]}$ & the Borel $\sigma$-algebra on [0,1]. \\
$B_\delta(a)$ & the closed ball centered at $a$ with radius $\delta$, i.e., $[a - \delta, a + \delta]$. \\
\rowcolor{gray!10} $Q_A^{B_\delta(a)}$ & $Q_A^{B_\delta(a)}$ signifies a subset of sample whose predicted probability $p_c(Y_i \in A|F_i)$ falls within a closed interval $B_\delta(a)$, i.e., a subset of $\left\{ i|p_c(Y_i \in A|F_i) \in B_\delta(a),1\leq i \leq N \right\}$.\\
\rowcolor{gray!10} $Q_A^{\overline{R_1 R_2}}$, $Q_A^{R_1}$, $Q_A^{R_2}$, $Q_A^{\overset{\frown}{R_1 R_2}}$ & Similarly, $Q_A^{R_1}$, $Q_A^{R_2}$ and $Q_A^{\overset{\frown}{R_1 R_2}}$ are defined by substituting $B_\delta(a)$ with $[R_1,R_2]$, point set $\{R_1\}$, point set $\{R_2\}$ and open interval $(R_1,R_2)$, respectively, detailed in Eq. \eqref{eq:qr1} - Eq. \eqref{eq:qri_qr2_open}. \\
$\mathscr A$, $\mathbb A$ & $\mathscr A$ and $\mathbb A$ refer to event sets (i.e., set-based classes). Specifically, in the manuscript, $\mathscr A$ is specified as the set containing all atom events of r.v. $Y$, defined as $\mathscr A =\{A|A\subset \Omega_Y,|A|=1\}=\{\{1\},\{2\},...,\{L\}\}$. $\mathbb{A}$ is set the same as $\mathscr{A}$.\\
\rowcolor{gray!10} $\mathscr R$ and $\mathcal R_{\mathbb A}$ & $\mathscr R$ is the abbreviation for interval $[R_1,R_2]$. $\mathcal R_{\mathbb A}$ represents the set of $\mathscr R$ that are of interest for a given event set $\mathbb A$. \\
$Z^+$ & the set of positive integers. \\[6pt]
\multicolumn{2}{l}{\textbf{\underline{Constants:}}} \\
\rowcolor{gray!10} $h$ and $\epsilon$ & constant symbols denoting the error bound. \\
$M$ & a hyperparameter in our algorithm, representing the event count for constraint formation in the loss function, which controls the approximation error. \\
\rowcolor{gray!10} $r$ & a constant multiplier to increase the scale of the loss. \\[6pt]
\multicolumn{2}{l}{\textbf{\underline{Norms and Operations:}}} \\
$\mathbb E$ & the symbol for expectation, occasionally written as $\mathbb E_\mu$ to emphasize that the expectation or integral is with respect to the ground-truth probability measure $p_\mu$. \\
\rowcolor{gray!10} $|*|$ & indicates the cardinality of a set if $*$ stands for a set, or the absolute value when $*$ is a variable or number.\\
$\| * \|_\infty$ & infinity norm. \\
\rowcolor{gray!10} $L^2$ & $L^2$-norm. \\
$\| * \|_{M,\omega}$ &   norm for measuring the distance among vectors or matrices, defined in Eq. \eqref{eq:Mnorm}. \\[6pt]
\multicolumn{2}{l}{\textbf{\underline{Conditions:}}} \\
\rowcolor{gray!10} $h$-calibrated, bounded. & refer to Def. \ref{def:h-calibration}. \\
$\delta$-$\epsilon$ calibrated, bounded. & refer to Def. \ref{def:del-eps-calib}. \\
\rowcolor{gray!10} $h$-$\mathcal A$ calibrated, bounded. & refer to Def. \ref{def:h-a-calib}. \\
$\delta$-$\epsilon$-$\mathcal A$ calibrated, bounded. & refer to Def. \ref{def:del-eps-a-calib} \\[6pt]
\multicolumn{2}{l}{\textbf{\underline{Abbreviations:}}} \\
\rowcolor{gray!10} $a.s.$ & almost surely\\
\end{longtable}
}

\begin{landscape}
\section{Summary of Limitations and Our Solutions}
\label{sec:apdx-limitationsummary}

\begin{table}[H]
  \caption{Summary of limitations and how \textit{h}-calibration addresses each.}
  \small
  \centering
  \begin{tabular}{@{}p{4.5cm} p{8cm} p{11cm}@{}}
    \toprule
    \textbf{Limitation Types} &
    \textbf{Specific Limitations} &
    \textbf{How \textit{h}-calibration Addresses It} \\
    \midrule
    \multirow{5}{*}{Theoretical Gaps} &
    \textbf{Limitation \#1}: deficiency in statistical guarantees in relating learning objectives and common evaluators \newline (typical of category-1 calibrators) &
    We formulate a differentiable learning objective statistically equivalent to a calibration definition with controllable error (Def. \ref{def:h-calibration} of $h$-calibration), thereby ensuring a direct correspondence between training objectives and calibration goals (Secs. \ref{method:sec1}–\ref{method:sec4}). \\
    \cmidrule(l){2-3}
    & \textbf{Limitation \#2}: vulnerability to overfitting to necessary conditions \newline (typical of category-2 calibrators) &
    Our differentiable objective is derived as an \emph{equivalent} condition for well-calibration, rather than a merely necessary one. It aligns predicted probabilities with empirical frequencies without relying on manually defined sparse binning, thus reducing overfitting.
    \\
    \cmidrule(l){2-3}
    & \textbf{Limitation \#3}: inadequacy in presenting forms of real-world imperfect calibration \newline (typical of category-3 calibrators) &
    Our $h$-calibration (Def. \ref{def:h-calibration}) is compatible with real-world imperfect calibration with bounded error, with an equivalent learning objective derived in Secs. \ref{method:sec1}–\ref{method:sec4} that generalizes beyond perfect calibration. \\
    \cmidrule(l){2-3}
    & \textbf{Limitation \#4}: focusing on weak non-canonical calibration scenarios \newline (common in categories 1–3 and entails a trade-off with Limitation \#5) & We show in Thm. \ref{thm:1} that our learned $h$-calibration is canonical, rather than a weak, non-canonical form such as top-label or classwise-level calibration. \\
    \cmidrule(l){2-3}
    & \textbf{Limitation \#5}: computational challenges due to curse of dimensionality \newline (common in categories 1–3 and entails a trade-off with Limitation \#4) &
    The proposed probability alignment for calibration avoids estimating distributions over high-dimensional random variables, circumventing the prominent curse of dimensionality issue. \\
    \midrule
    \multirow{2}{*}{Methodological Dependencies} &
    \textbf{Limitation \#6}: reliance on unproven assumptions \newline (shared across categories)  & The learning objective in our method is derived from asymptotic statistical principles such as the law of large numbers, large deviation theory, etc., without relying on any parametric distributional assumptions. \\
    \cmidrule(l){2-3}
    & \textbf{Limitation \#7}: dependence on adjusting non-universal or non-intuitive parameters and settings \newline (shared across categories) &
    Our method involves only three hyperparameters, each with a clear and intuitive interpretation: $M$ controls the approximation effectiveness of our loss function (approximation error decreasing exponentially as $M$ increases); $\epsilon$ reflects the constraining bound on calibration error; and $r$ determines the loss range. \\
    \midrule
    \multirow{3}{*}{Practical Limitations} &
    \textbf{Limitation \#8}: conflicts with the unit measure property \newline (shared across categories) & Following prior work \cite{no.109}, we adopt a logit-mapping-based approach to learn the calibrator, explicitly enforcing the unit measure property during the learning process. \\
    \cmidrule(l){2-3}    
    & \textbf{Limitation \#9}: failure to preserve the original classification accuracy \newline (shared across categories) &
    Following prior work \cite{no.67,no.70}, we employ monotonic mapping as the logit transformation to preserve classification accuracy. \\
    \cmidrule(l){2-3}
    & \textbf{Limitation \#10}: limited applicability \newline (shared across categories) &
    Our method is applicable to any pretrained classifier following the typical post-hoc setting for calibrating the output probabilities of pretrained models. \\
    \bottomrule
  \end{tabular}
  \label{tab:hcalib-limitations}
\end{table}

\end{landscape}

\section{Proof of the Thm. \ref{thm:1}}
\label{sec:a1}
\newtheorem*{theorem*}{Theorem}
\begin{theorem*}
\label{thm:apdx1}
$h$-calibration is a sufficient condition for generalized canonical calibration with bounded error, i.e., $\| p_\mu(\mathscr E|\mathscr H(F))-\mathscr H(F) \|_\infty \leq h$, where $\mathscr H(F) = \big[ p_c( Y=1 | F),\dots,p_c( Y=L | F) \big]^\top$ and $\mathscr E = \big[ \mathds 1_{\{Y=1\}},\dots,\mathds 1_{\{Y=L\}} \big]^\top$.
\end{theorem*}
\begin{proof}
According to Def. \ref{def:h-calibration}, Eq. \eqref{eq:apdx1-1} represents $h$-calibration, while Eq. \eqref{eq:apdx1-2} denotes generalized calibration with a bounded error $h$.

\begin{align}
    & \bigg| p_\mu( Y \in A \big| F ) - p_c( Y \in A \big| F ) \bigg| \leq h; ~~\forall A \label{eq:apdx1-1} \\
    \Leftrightarrow & \bigg| \mathbb E_\mu\big[ \mathds 1_{\{Y \in A\}} \big| F\big] - \mathbb E_c\big[ \mathds 1_{\{Y \in A\}} \big| F\big] \bigg| \leq h; ~~\forall A \\
    \Rightarrow & \bigg| E_\mu\Big[ \mathbb E_\mu\big[ \mathds 1_{\{Y \in A\}} \big| F\big] - \mathbb E_c\big[ \mathds 1_{\{Y \in A\}} \big| F\big] \Big| \mathbb E_c\big[ \mathds 1_{\{Y \in A\}} \big| F\big]\Big] \bigg| \leq h; ~~\forall A \\
    \Leftrightarrow & \bigg| \mathbb E_\mu\Big[ \mathbb E_\mu \big[ \mathds 1_{\{Y \in A\}} \big| F\big] \Big| \mathbb E_c\big[ \mathds 1_{\{Y \in A\}} \big| F\big]\Big] - \mathbb E_c\big[ \mathds 1_{\{Y \in A\}} \big| F\big] \bigg| \leq h; ~~\forall A \\
    \Leftrightarrow & \bigg| \mathbb E_\mu\Big[ \mathds 1_{\{Y \in A\}} \Big| \mathbb E_c\big[ \mathds 1_{\{Y \in A\}} \big| F\big]\Big] - \mathbb E_c\big[ \mathds 1_{\{Y \in A\}} \big| F\big] \bigg| \leq h; ~~\forall A \\
    \Leftrightarrow & p_\mu( Y \in A | p_c(Y \in A | F) = s) \in [s-h,s+h]; ~~\forall A  \\
    \Rightarrow & \| p_\mu(\mathscr E|\mathscr H(F))-\mathscr H(F) \|_\infty \leq h \label{eq:apdx1-2}
\end{align}
\end{proof}

\section{Illustration of $h$-calibration and canonical calibration}
\label{sec:apdx-hcalibrationillustrate}
\begin{figure*}[htbp]
\centering
\includegraphics[width=\textwidth]{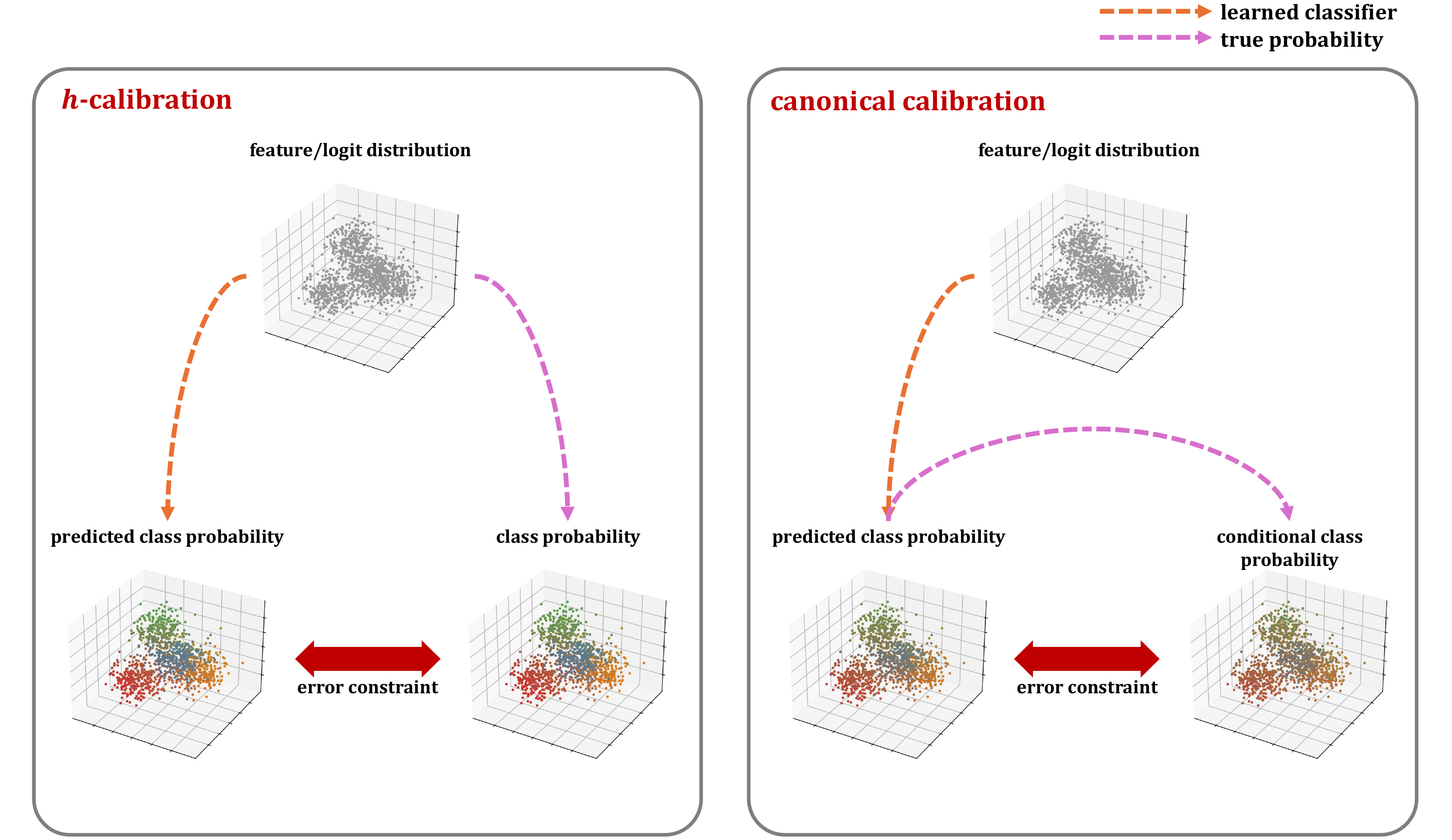}
\caption{Illustration comparing $h$-calibration and canonical calibration. $h$-calibration implies canonical calibration; the converse fails to hold. Example: On CIFAR-10, assigning each sample a uniform predictive distribution $[1/10,…,1/10]$ yields canonical calibration, despite having no discriminative power, and yet fails to satisfy $h$-calibration.}
\label{fig:h-calibration}
\end{figure*}

\section{Proof of the Thm. \ref{thm:2}}
\label{sec:a2}
\begin{theorem*}
\label{thm:apdx2}
For finite samples, a calibrated probability $p_c$ is $h$-calibrated if and only if $p_c$ is $\delta$-$\epsilon$ bounded.
\end{theorem*}
\begin{proof}~
\begin{enumerate}
\item[($\Rightarrow$):]
Given $\tau$-calibrated probability $p_c$, with the notations of Def. \ref{def:del-eps-calib} ($\delta$-$\epsilon$ boundedness), 
\begin{align}
    & \left| a - \frac{ \sum_{i \in Q_A^{B_\delta(a)}}\mathbb{E}_\mu[\mathds 1_A(Y_i)|F_i]}{|Q_A^{B_\delta(a)}|}  \right| \\
    = & \left| a - \frac{ \sum_{i \in Q_A^{B_\delta(a)}} p_c(Y_i\in A|F_i)}{|Q_A^{B_\delta(a)}|} + \frac{ \sum_{i \in Q_A^{B_\delta(a)}} p_c(Y_i\in A|F_i)}{|Q_A^{B_\delta(a)}|} - \frac{ \sum_{i \in Q_A^{B_\delta(a)}}\mathbb{E}_\mu[\mathds 1_A(Y_i)|F_i]}{|Q_A^{B_\delta(a)}|}  \right| \\
    = &\left| \frac{ \sum_{i \in Q_A^{B_\delta(a)}} \left[ a-p_c(Y_i\in A|F_i)\right]}{|Q_A^{B_\delta(a)}|} + \frac{ \sum_{i \in Q_A^{B_\delta(a)}} p_c(Y_i\in A|F_i)}{|Q_A^{B_\delta(a)}|} - \frac{ \sum_{i \in Q_A^{B_\delta(a)}} p_\mu(Y_i\in A|F_i)}{|Q_A^{B_\delta(a)}|} \right| \\
    \leq &\left| \frac{ \sum_{i \in Q_A^{B_\delta(a)}} \left[ a-p_c(Y_i\in A|F_i)\right]}{|Q_A^{B_\delta(a)}|} \right| + \left| \frac{ \sum_{i \in Q_A^{B_\delta(a)}} \left[ p_c(Y_i\in A|F_i) - p_\mu(Y_i\in A|F_i)  \right]  }{|Q_A^{B_\delta(a)}|} \right| \\
    \leq & \delta + h
\end{align}
Hence, $\delta$-$\epsilon$ boundedness is obtained by setting $h$ to $\epsilon$.
\item[($\Leftarrow$):]
Given $\delta$-$\epsilon$ bounded probability $p_c$, for any $\delta>0$, there exists an $a \in (0,1)$ such that $p_c(Y_i \in A|F_i) \in B_\delta(a)$. Then,
\begin{align}
    |p_{\mu}(Y_i \in A|F_i) - p_c(Y_i \in A|F_i)| 
    & = \left|\mathbb{E}_\mu\left[ \mathds 1_A(Y_i) | F_i \right] - a + a-p_c(Y_i \in A|F_i) \right| \\
    & \leq \left|\mathbb{E}_\mu\left[ \mathds 1_A(Y_i) | F_i \right] - a \right| + \left| a-p_c(Y_i \in A|F_i) \right| \\
    & \leq (\epsilon+\delta) + \delta
\end{align}
Since $\delta$ can be any small positive number, $p_c$ is shown to be $h$-calibrated by setting $h$ to $\epsilon$.
\end{enumerate}
\end{proof}

\section{Proof of the Thm. \ref{thm:3}}
\label{sec:a3}
\begin{theorem*}
\label{thm:apdx3}
A calibrated probability $p_c$ is $h$-$\mathcal A$ calibrated if and only if $p_c$ is $\delta$-$\epsilon$-$\mathcal A$ bounded. Both conditions are necessary but not sufficient for $h$-calibration. 
\end{theorem*}
\begin{proof} Similar to the proof in Appendix \ref{sec:a2}.
\begin{enumerate}
\item[($\Rightarrow$):]
With the notations of Def. \ref{def:del-eps-a-calib} ($\delta$-$\epsilon$-$\mathcal{A}$ boundedness), for $h$-calibrated probability $p_c$ and any $\mathcal{A}_i$, $1\leq i\leq N$,
\begin{align}
    & \left| a - \frac{ \sum_{i \in Q_{\mathcal{A}}^{B_\delta(a)}}\mathbb{E}_\mu[\mathds 1_{\mathcal{A}_i}(Y_i)|F_i]}{|Q_{\mathcal{A}}^{B_\delta(a)}|}  \right| \\
    & = \left| a - \frac{ \sum_{i \in Q_{\mathcal{A}}^{B_\delta(a)}} p_c(Y_i\in \mathcal{A}_i|F_i)}{|Q_{\mathcal{A}}^{B_\delta(a)}|} + \frac{ \sum_{i \in Q_\mathcal{A}^{B_\delta(a)}} p_c(Y_i\in \mathcal{A}_i|F_i)}{|Q_\mathcal{A}^{B_\delta(a)}|} - \frac{ \sum_{i \in Q_\mathcal{A}^{B_\delta(a)}}\mathbb{E}_\mu[\mathds 1_{\mathcal{A}_i}(Y_i)|F_i]}{|Q_{\mathcal{A}}^{B_\delta(a)}|}  \right| \\
    &= \left| \frac{ \sum_{i \in Q_\mathcal{A}^{B_\delta(a)}} \left[ a-p_c(Y_i\in \mathcal{A}_i|F_i)\right]}{|Q_\mathcal{A}^{B_\delta(a)}|} + \frac{ \sum_{i \in Q_\mathcal{A}^{B_\delta(a)}} p_c(Y_i\in \mathcal{A}_i|F_i)}{|Q_\mathcal{A}^{B_\delta(a)}|} - \frac{ \sum_{i \in Q_\mathcal{A}^{B_\delta(a)}} p_\mu(Y_i\in \mathcal{A}_i|F_i)}{|Q_\mathcal{A}^{B_\delta(a)}|} \right| \\
    &\leq \left| \frac{ \sum_{i \in Q_\mathcal{A}^{B_\delta(a)}} \left[ a-p_c(Y_i\in \mathcal{A}_i|F_i)\right]}{|Q_\mathcal{A}^{B_\delta(a)}|} \right| + \left| \frac{ \sum_{i \in Q_\mathcal{A}^{B_\delta(a)}} \left[ p_c(Y_i\in \mathcal{A}_i|F_i) - p_\mu(Y_i\in \mathcal{A}_i|F_i)  \right]  }{|Q_\mathcal{A}^{B_\delta(a)}|} \right| \\
    &\leq \delta + h.
\end{align}
It shows that $p_c$ is $\delta$-$\epsilon$-$\mathcal{A}$ bounded when we set $h$ to $\epsilon$. 
\item[($\Leftarrow$):]
Given $\delta$-$\epsilon$-$\mathcal{A}$ bounded probability $p_c$, for any $\delta>0$, there exists an $a \in (0,1)$ such that $p_c(Y_i \in A|F_i) \in B_\delta(a)$. Then,
\begin{align}
    |p_{\mu}(Y_i \in \mathcal A_i|F_i) - p_c(Y_i \in \mathcal A_i|F_i)| 
    & = \left|\mathbb{E}_\mu\left[ \mathds 1_{\mathcal A_i}(Y_i) | F_i \right] - a + a-p_c(Y_i \in {\mathcal A_i}|F_i) \right| \\
    & \leq \left|\mathbb{E}_\mu\left[ \mathds 1_{\mathcal A_i}(Y_i) | F_i \right] - a \right| + \left| a-p_c(Y_i \in {\mathcal A_i}|F_i) \right| \\
    & \leq (\epsilon+\delta) + \delta
\end{align}
Since $\delta$ can be any small positive number, $p_c$ is shown to be $h$-$\mathcal A$ calibrated by setting $h$ to $\epsilon$.
\end{enumerate}
It is easy to see that $h$-$\mathcal A$ calibration in Def. \ref{def:h-a-calib} is a necessary condition for $h$-calibration in Def. \ref{def:h-calibration}. Here, we can present a simple counterexample to illustrate that $h$-$\mathcal A$ calibration is not a sufficient condition for $h$-calibration.
Assuming there exists a $j$ and corresponding sets $\mathcal{A}_j$, $\mathcal{B}_j$ with $\mathcal{B}_j \cap \mathcal{A}_j = \varnothing$ such that
\begin{equation}
    p_{\mu}(Y_j \in \mathcal{A}_j|X_j) = p_c(Y_j \in \mathcal{A}_j|X_j)=0,~~
    p_{\mu}(Y_j \in \mathcal{B}_j|X_j) =0, ~~ 
     p_c(Y_j \in \mathcal{B}_j|X_j)=1,
\end{equation}
and for other $i\neq j$, $1 \leq i \leq N $ such that
\begin{equation}
    p_{\mu}(Y_i \in \mathcal{A}_i|F_i) = p_c(Y_i \in \mathcal{A}_i|F_i),
\end{equation}
it is clear that for any $0<h<1$, $p_c$ is $h$-$\mathcal{A}$ calibrated but not $h$-calibrated.
\end{proof}

\section{Proof of the Thm. \ref{thm:4}}
\label{sec:a4}
\begin{theorem*}
\label{thm:apdx4}
For the $\mathcal{A}_i$ specified in Eq. \eqref{eq:top-label-event} and Eq. \eqref{eq:classwise-event}, the corresponding $h$-$\mathcal{A}$ calibrations are sufficient for top-label and classwise calibrations, respectively, with uniform error bound $h$. That is, it holds that $| p_\mu (Y_i \in \mathcal A_i | p_c (Y_i \in \mathcal A_i | F_i)) - p_c (Y_i \in \mathcal A_i | F_i)| \leq h$, for the $\mathcal A_i$ defined in Eq. \eqref{eq:top-label-event} and Eq. \eqref{eq:classwise-event}, respectively.
\end{theorem*}
\begin{proof}
According to Def. \ref{def:h-a-calib} ($h$-$\mathcal A$ calibration), we have
\begin{align}
& ~ \big| p_\mu (Y_i \in \mathcal A_i \mid F_i) - p_c (Y_i \in \mathcal A_i \mid F_i) \big| \leq h  \\
\Rightarrow & \biggm| \mathbb E_\mu \Big[ 
\mathbb E_\mu \big[  \mathds 1_{\{Y_i \in \mathcal A_i\}} \mid F_i \big] \Big|\mathbb E_c \big[ \mathds 1_{\{Y_i \in \mathcal A_i\}} \mid F_i\big] \Big] - \mathbb E_c \big[ \mathds 1_{\{Y_i \in \mathcal A_i\}} \mid F_i\big] \biggm| \leq h \\
\Leftrightarrow & \biggm| \mathbb E_\mu \Big[ 
  \mathds 1_{\{Y_i \in \mathcal A_i\}} \Big|\mathbb E_c \big[ \mathds 1_{\{Y_i \in \mathcal A_i\}} \mid F_i\big] \Big] - \mathbb E_c \big[ \mathds 1_{\{Y_i \in \mathcal A_i\}} \mid F_i\big] \biggm| \leq h \\
\Leftrightarrow & ~ \big| p_\mu (Y_i \in \mathcal A_i \mid p_c (Y_i \in \mathcal A_i \mid F_i)) -  p_c (Y_i \in \mathcal A_i \mid F_i)\big| \leq h \label{eq:apdx4-1}
\end{align}
When $\mathcal{A}_i$ is defined as $\{l | \underset{l}{\operatorname {~argmax~}} p_c( Y_i = l|F_i)\}$ or as $\{l\}$ for fixed class $l$, as specified in Eq. \eqref{eq:top-label-event} and Eq. \eqref{eq:classwise-event}, inequality Eq. \eqref{eq:apdx4-1} corresponds to top-label and classwise calibration (for class $l$), respectively, with uniform error bound $h$.
\end{proof}

\section{Proof of the Thm. \ref{thm:5}}
\label{sec:a5}
\begin{theorem*}
\label{thm:apdx5}
With the notations in Def. \ref{def:del-eps-calib},
\begin{equation}
\label{eq:apdx5}
       \left| \frac{ \sum_{i \in Q_{A}^{B_\delta(a)}}\mathds 1_{A}(Y_i)}{|Q_{A}^{B_\delta(a)}|}
    - \frac{ \sum_{i \in Q_A^{B_\delta(a)}}\mathbb{E}_\mu[\mathds 1_A(Y_i)|F_i]}{|Q_A^{B_\delta(a)}|}  \right| 
       \xrightarrow[p_\mu~\&~L^2~\&~a.s.]{|Q_A^{B_\delta(a)}| \rightarrow \infty}   0.
\end{equation}
\end{theorem*}
\begin{proof}
There exists $C>0$, for any $i \in Q_A^{B_\delta(a)}$ such that
\begin{equation}
    \label{eq:apdx5-1}
    \mathbb{E}_\mu \left[\mathds 1_A(Y_i) - \mathbb{E}_\mu \left[ \mathds 1_A(Y_i) |F_i \right]\right] = 0,
\end{equation}
\begin{equation}
    \label{eq:apdx5-2}
    var(\mathds 1_A(Y_i) - \mathbb{E}_\mu \left[ \mathds 1_A(Y_i) |F_i \right]) \leq C < \infty.
\end{equation}
Given the independence between different sample points, we have
\begin{equation}
    \mathbb{E}_\mu \Big( \frac{1}{|Q_{A}^{B_\delta(a)}|} \sum_{i \in Q_{A}^{B_\delta(a)}}  \left( \mathds 1_{A}(Y_i) - \mathbb{E}_{\mu}\left[\mathds 1_{A}(Y_i)|F_i\right] \right) \Big)^2 = \frac{1}{|Q_{A}^{B_\delta(a)}|^2}  \sum_{i \in Q_{A}^{B_\delta(a)}}  \mathbb{E}_\mu \Big( \mathds 1_{A}(Y_i) - \mathbb{E}_{\mu}\left[\mathds 1_{A}(Y_i)|F_i\right] \Big)^2 \leq \frac{C}{|Q_A^{B_\delta(a)}|}.
\end{equation}
Then Chebyshev's inequality \cite{durrett2019probability} implies for any $\rho > 0$,
\begin{equation}
    p_\mu\Big( \Big| \frac{1}{|Q_{A}^{B_\delta(a)}|} \sum_{i \in Q_{A}^{B_\delta(a)}}  \left( \mathds 1_{A}(Y_i) - \mathbb{E}_{\mu}\left[\mathds 1_{A} Y_i)|F_i\right] \right) \Big| \geq \rho \Big) \leq \frac{C}{|Q_A^{B_\delta(a)}| \rho^2}.
\end{equation}
Hence Eq.\eqref{eq:apdx5} converges to zero in $p_\mu$ and $L^2$ as $|Q_A^{B_\delta(a)}| \rightarrow \infty$. Addtionally, Eq.\eqref{eq:apdx5-1} and Eq.\eqref{eq:apdx5-2} give us
\begin{equation}
    \sum_{i=1}^\infty \frac{\mathbb{E}_\mu \Big(  \mathds 1_{A}(Y_i) - \mathbb{E}_{\mu}\left[\mathds 1_{A}(Y_i)|F_i\right]  \Big)^2}{i^2} \leq \sum_{i=1}^\infty  \frac{C}{i^2}  < \infty.
\end{equation}
Kolmogorov's strong law of large numbers \cite{shao2003statisbook} implies
\begin{equation}
    \frac{ \sum_{i \in Q_{A}^{B_\delta(a)}}\Big(  \mathds 1_{A}(Y_i) - \mathbb{E}_{\mu}\left[\mathds 1_{A}(Y_i)|F_i\right]  \Big)}{|Q_{A}^{B_\delta(a)}|} \xrightarrow[a.s.]{|Q_A^{B_\delta(a)}| \rightarrow \infty} 0.
\end{equation}
\end{proof}

\section{Proof of the Thm. \ref{thm:6}}
\label{sec:a6}
\begin{theorem*}
With the notations in Def. \ref{def:del-eps-calib}, The difference term in Thm. \ref{thm:5} converges exponentially to zero in $p_\mu$ as $|Q_A^{B_\delta(a)}|\rightarrow \infty$, i.e., for any $\kappa>0$,
\begin{equation}
    p_\mu \left(
       \left| \frac{ \sum_{i \in Q_{A}^{B_\delta(a)}}\mathds 1_{A}(Y_i)}{|Q_{A}^{B_\delta(a)}|} - \frac{ \sum_{i \in Q_A^{B_\delta(a)}}\mathbb{E}_\mu[\mathds 1_A(Y_i)|F_i]}{|Q_A^{B_\delta(a)}|}   \right| > \kappa \right) 
\end{equation}
converges to zero exponentially as $|Q_A^{B_\delta(a)}|\rightarrow \infty$.
\end{theorem*}
\begin{proof}
    For independent and identically distributed (i.i.d.) random variables, $\mathds 1_A(Y_i) - \mathbb{E}_\mu [\mathds 1_A(Y_i)|F_i]$, by Eq.\eqref{eq:apdx5-1} and Eq.\eqref{eq:apdx5-2}, the following comulant generating function is finite, i.e.,
    \begin{equation}
        \varphi (\lambda) = \ln \mathbb E_\mu \left[ \exp\left(\lambda \left(\mathds 1_A(Y_i) - \mathbb{E}_\mu [\mathds 1_A(Y_i)|F_i] \right) \right) \right] < \infty.
    \end{equation}
    Thus the Cramer's theorem for large deviation \cite{dembo2009ldt} implies, for any $\kappa>0$
    \begin{equation}
        \lim_{|Q_A^{B_\delta(a)}| \rightarrow \infty} \frac{1}{|Q_A^{B_\delta(a)}|} \ln p_\mu\left( \frac{ \sum_{i \in _{A}^{B_\delta(a)}}\mathds 1_{A}(Y_i)}{|Q_{A}^{B_\delta(a)}|} - \frac{ \sum_{i \in Q_A^{B_\delta(a)}}\mathbb{E}_\mu[\mathds 1_A(Y_i)|F_i]}{|Q_A^{B_\delta(a)}|}    \geq \kappa \right) = - \gamma(\kappa),
        \label{eq:ldt_cramer}
    \end{equation}
    where the rate function $\gamma(\kappa)$ is the Fenchel-Legendre transform of $\varphi(\lambda)$, i.e.,
    \begin{equation}
    \label{eq:apdx6-1-1}
        \begin{aligned}
        \gamma(\kappa) &=  \varphi^*(\lambda)=  \sup_{\lambda \in R}\{\kappa \lambda - \varphi(\lambda)\} \\
        &= \sup_{\lambda \in R} \left\{ \kappa \lambda - \ln \mathbb E_\mu \left[  \exp  \left( \lambda \left(\mathds 1_A(Y_i)-\mathbb E_\mu[\mathds 1_A(Y_i)|F_i] \right) \right) \right]  \right\}.
        \end{aligned}
    \end{equation}
    The function
    \begin{equation}
        \kappa \lambda  - \ln \mathbb E_\mu \left[ \exp\left( \lambda \left( \mathds 1_A(Y_i) - \mathbb E_\mu[\mathds 1_A(Y_i)|F_i] \right) \right) \right]
    \end{equation}
    equals to zero and has positive derivative with respect to $\lambda$ when $\lambda=0$ (It is because the derivative of RHS is $\mathbb E_\mu \left[\mathds 1_A(Y_i) - \mathbb E_\mu[\mathds 1_A(Y_i)|F_i]\right] =0$ when $\lambda=0$), so $\gamma(\kappa)>0$ by the definition in Eq. \eqref{eq:apdx6-1-1}.
    In fact, Eq. \eqref{eq:ldt_cramer} also specifies the upper bound for the probability involved. Following large deviation theory, we explain below.
    Here we abbreviate $\sum\nolimits_{i \in Q_A^{B_\delta(a)}} \left( \mathds 1_A(Y_i)-\mathbb E_\mu\left[ \mathds 1_A(Y_i) |F_i \right] \right) $ as $S_{|Q_A^{B_\delta(a)}|}$ for readability reasons. Then we construct convex function $\exp(\lambda x)$ for any $\lambda >0$ and by using Chebyshev's inequality \cite{durrett2019probability}, we obtain 
    \begin{equation}
        \begin{aligned}
        p_\mu\left( S_{|Q_A^{B_\delta(a)}|} \geq \kappa |Q_A^{B_\delta(a)}| \right) 
        & \leq \frac{\mathbb E_\mu \exp(\lambda S_{|Q_A^{B_\delta(a)}|} )}{\exp(\lambda \kappa |Q_A^{B_\delta(a)}|)} = \frac{\prod_{i \in Q_A^{B_\delta(a)}} \mathbb E_\mu \exp(\lambda ( \mathds 1_A(Y_i)-\mathbb E_\mu [\mathds 1_A(Y_i)|F_i]))  }{\exp(\lambda \kappa |Q_A^{B_\delta(a)}|)} \\
        & = \exp \left( \left( \ln \mathbb E_\mu \left[ \exp \left( \lambda \left(\mathds 1_A(Y_i)- \mathbb E_\mu[\mathds 1_A(Y_i)|F_i]\right) \right) \right] - \lambda \kappa \right) {|Q_A^{B_\delta(a)}|} \right).
        \end{aligned}
    \end{equation}
    Since $\lambda$ is arbitrary, we derive
    \begin{gather}
        p_\mu\left( S_{|Q_A^{B_\delta(a)}|} \geq \kappa |Q_A^{B_\delta(a)}| \right) \leq \exp \left( -{|Q_A^{B_\delta(a)}|}  \cdot \sup_{\lambda >0} \left( \lambda \kappa - \ln \mathbb E_\mu \left[ \exp \left( \lambda \left(\mathds 1_A(Y_i)- \mathbb E_\mu[\mathds 1_A(Y_i)|F_i]\right) \right) \right] \right) \right) \\
        =  \exp \left( -{|Q_A^{B_\delta(a)}|}  \cdot \gamma(\kappa) \right).
    \end{gather}
    The above equality holds (turning $\lambda>0$ into $\lambda \in R$) because $\kappa \lambda  - \ln \mathbb E_\mu \left[ \exp\left( \lambda \left( \mathds 1_A(Y_i) - \mathbb E_\mu[\mathds 1_A(Y_i)n|F_i]\right) \right) \right]$ equals to zero when $\lambda=0$ and increases monotonically with respect to $\lambda$ when $\lambda \leq 0$. Specifically, monotonicity is ensured because $\kappa \lambda  - \ln \mathbb E_\mu \left[ \exp\left( \lambda \left( \mathds 1_A(Y_i) - \mathbb E_\mu[\mathds 1_A(Y_i)|F_i]\right) \right) \right]$ has positive derivative when $\lambda=0$ and $\ln \mathbb E_\mu \left[ \exp\left( \lambda \left( \mathds 1_A(Y_i) - \mathbb E_\mu[\mathds 1_A(Y_i)|F_i]\right) \right) \right]$ is convex with respect to $\lambda$ (The reader can check the convexity by using Hölder's inequality). Now we can see
    \begin{equation}
        p_\mu\left( \frac{ \sum_{i \in Q_{A}^{B_\delta(a)}}\mathds 1_{A}(Y_i)}{|Q_{A}^{B_\delta(a)}|} - \frac{ \sum_{i \in _A^{B_\delta(a)}}\mathbb{E}_\mu[\mathds 1_A(Y_i)|F_i]}{|Q_A^{B_\delta(a)}|}    \geq \kappa \right) \leq \exp\left( - \gamma(\kappa)|Q_A^{B_\delta(a)}| \right).
    \end{equation}
    Repeating the above procedures for i.i.d. random variables $\mathbb{E}_\mu [\mathds 1_A(Y_i)|F_i] - \mathds 1_A(Y_i) ]$ gives similar results
    \begin{equation}
        \lim_{|Q_A^{B_\delta(a)}| \rightarrow \infty} \frac{1}{|Q_A^{B_\delta(a)}|} \ln p_\mu\left( \frac{ \sum_{i \in Q_{A}^{B_\delta(a)}}\mathds 1_{A}(Y_i)}{|Q_{A}^{B_\delta(a)}|} - \frac{ \sum_{i \in Q_A^{B_\delta(a)}}\mathbb{E}_\mu[\mathds 1_A(Y_i)|F_i]}{|Q_A^{B_\delta(a)}|}    \leq -\kappa \right)= - \gamma(-\kappa) < 0,
    \end{equation}
    \begin{equation}
        p_\mu\left( \frac{ \sum_{i \in Q_{A}^{B_\delta(a)}}\mathds 1_{A}(Y_i)}{|Q_{A}^{B_\delta(a)}|} - \frac{ \sum_{i \in Q_A^{B_\delta(a)}}\mathbb{E}_\mu[\mathds 1_A(Y_i)|F_i]}{|Q_A^{B_\delta(a)}|}    \leq -\kappa \right) \leq \exp\left( - \gamma(-\kappa)|Q_A^{B_\delta(a)}| \right).
    \end{equation}
    Finally, we obtain
    \begin{equation}
        p_\mu\left( \left| \frac{ \sum_{i \in Q_{A}^{B_\delta(a)}}\mathds 1_{A}(Y_i)}{|Q_{A}^{B_\delta(a)}|} - \frac{ \sum_{i \in Q_A^{B_\delta(a)}}\mathbb{E}_\mu[\mathds 1_A(Y_i)|F_i]}{|Q_A^{B_\delta(a)}|} \right|  \geq \kappa \right) \leq \exp\left( - \gamma(-\kappa)|Q_A^{B_\delta(a)}| \right) + \exp\left( - \gamma(\kappa)|Q_A^{B_\delta(a)}| \right)
    \end{equation}
    with $\gamma(\kappa)$, $\gamma(-\kappa)>0$.
\end{proof}

\section{Proof of the Thm. \ref{thm:7} and the Probabilistic Interpretation}
\subsection{Proof of the Thm. \ref{thm:7}}
\label{sec:a7}
\begin{theorem*}
    With the notations in Def. \ref{def:del-eps-calib}, a necessary and sufficient condition for
    \begin{equation}
    \label{eq:apdx7-1}
        \left| a -  \frac{ \sum_{i \in Q_A^{B_\delta(a)}} \mathds 1_A(Y_i)}{|Q_A^{B_\delta(a)}|} \right| \leq \epsilon + \delta
    \end{equation}
    is that the following inequality holds:
    \begin{equation}
    \label{eq:apdx7-2}
    T(Q,R,A) \triangleq     
    \left|  \frac{    T_1   -    T_2        }{    T_3   } \right| \leq \epsilon
    \end{equation}
    for any $1 \leq R_1 < R_2 \leq 1$, $Q_A^{\overline{R_1R_2}}  \subseteq \{i| R_1 \leq {^cp}_i^A \leq R_2\} $ with $|Q_A^{\overline{R_1R_2}}|>0$, $A \in \mathscr F_Y$, where
    \begin{gather}
        T_1 = 2\sum\limits_{Q_A^{\overset{\frown}{R_1R_2}} \cap O_A} (1-{^cp_i^A}) + \sum\limits_{( Q_A^{R_1} \cup Q_A^{R_2} ) \cap O_A} (1-{^cp_i^A}), \\
        T_2 = 2\sum\limits_{ Q_A^{\overset{\frown}{R_1R_2}} \cap O_{A^{\mathbf C}}}{^cp_i^A} + \sum\limits_{(Q_A^{R_1} \cup Q_A^{R_2}) \cap O_{A^{\mathbf C}} } {^cp_i^A} , \\
        T_3 = 2 \sum\limits_{Q_A^{\overset{\frown}{R_1R_2}} } 1 + \sum\limits_{Q_A^{R_1} \cup Q_A^{R_2}} 1,
    \end{gather}    
    and 
    \begin{gather}
        {^cp_i^A  \triangleq p_c(Y_i \in A|F_i),  \label{eq:apdx7-3} }\\
        Q_A^{R_1} \triangleq  \{i| {^cp_i^A=R_1},i \in Q_A^{\overline{R_1R_2}}  \},  \\
        {Q_A^{R_2} \triangleq  \{i| {^cp_i^A=R_2},i \in Q_A^{\overline{R_1R_2}}  \}, }\\
        Q_A^{\overset{\frown}{R_1 R_2}} \triangleq  \{i|  R_1 < {^cp_i^A}< R_2,i \in Q_A^{\overline{R_1R_2}}  \},  \\
        {O_A \triangleq \{i|Y_i \in A\}, ~~O_{A^{\mathbf C}} \triangleq \{i|Y_i \in A^{\mathbf C}\}.      \label{eq:apdx7-4}}
    \end{gather}
\end{theorem*}
\begin{proof}~
\begin{enumerate}
    \item[($\Rightarrow$):]
    There exists a sample index set $\mathcal{I}$ such that
    \begin{equation}
        \alpha \triangleq \sum_{i \in Q_A^{B_\delta(a)}} \mathds 1_A(Y_i) =\sum_{i \in \mathcal{I}} \mathds 1_{ \big\{ p_c(Y_i \in A|F_i) \in B_\delta(a) \big\}\cap \big\{ Y_i \in A \big\}},
    \end{equation}
    \begin{equation}
        \beta \triangleq \sum_{i \in Q_A^{B_\delta(a)}} \mathds 1_{A^{\mathbf C}}(Y_i) =\sum_{i \in \mathcal{I}} \mathds 1_{ \big\{ p_c(Y_i \in A|F_i) \in B_\delta(a) \big\}\cap \big\{ Y_i \in A^{\mathbf C} \big\}}.
    \end{equation}
    Then, $|Q_A^{B_\delta(a)}| = \alpha + \beta$ and Eq.\eqref{eq:apdx7-1} becomes
    \begin{equation}
        a - (\epsilon + \delta) \leq \frac{\alpha}{\alpha+\beta} \leq a + (\epsilon +\delta),
    \end{equation}
    which can be simplified to
    \begin{equation}
        -(\epsilon + \delta)(\alpha + \beta) \leq (1-a) \alpha - a \beta  \leq (\epsilon +\delta)(\alpha+\beta).
    \end{equation}
    To ensure the differentiability of the objective function with respect to the calibrated probability $p_c$, we integrate the above formula and convert it into a differentiable form by removing the indicator function. Thus, for any $ 0 \leq R_1 < R_2 \leq 1$, 
    \begin{equation}
    \label{eq:apdx7-5}
        -\int_{R_1}^{R_2} \big[(\epsilon + \delta)(\alpha + \beta) \big] da \leq \int_{R_1}^{R_2} (1-a) \alpha da - \int_{R_1}^{R_2} a \beta da \leq \int_{R_1}^{R_2} \big[(\epsilon + \delta)(\alpha + \beta) \big] da,
    \end{equation}
    which can be simplified to
    \begin{equation}
    \label{eq:apdx7-6}
        \left| \frac{\int_{R_1}^{R_2}(1-a)\alpha da - \int_{R_1}^{R_2}a\beta da} {\int_{R_1}^{R_2}(\alpha+\beta)da} \right| \leq \epsilon+ \delta
    \end{equation}
    Calculating the integrals above derives (The reader can check it by selecting $\delta$ such that $2\delta < R_2 - R_1$)
    \begin{equation}
    \begin{aligned}
        \int_{R_1}^{R_2} \beta a da & = \sum_{  ^{R,\delta}\Psi_i^2 \cap \{ Y_i \in A^{\mathbf{C}} \} } 2 \delta p_c(Y_i \in A|F_i) \\
        & + \sum_{^{R,\delta}\Psi_i^1 \cap \{ Y_i \in A^{\mathbf{C}}\}} \left[ \big(p_c(Y_i\in A|F_i)+\delta\big)^2 - R_1^2 \right]/2 + \sum_{^{R,\delta}\Psi_i^3 \cap \{ Y_i \in A^{\mathbf{C}}\}} \left[ R_2^2 - \big(p_c(Y_i\in A|F_i)-\delta\big)^2  \right]/2,
    \end{aligned}
    \end{equation}
    where 
    \begin{gather}
        ^{R,\delta}\Psi^1_i \triangleq \{R_1 \leq p_c(Y_i \in A | F_i) \leq R_1+\delta \}, \\
        ^{R,\delta}\Psi^2_i \triangleq \{R_1 + \delta < p_c(Y_i \in A | F_i) < R_2 - \delta \}, \\
        ^{R,\delta}\Psi^3_i \triangleq \{R_2 - \delta \leq p_c(Y_i \in A | F_i) \leq R_2\}.
    \end{gather}
    Similarly, we can obtain
    \begin{equation}
        \begin{aligned}
            \int_{R_1}^{R_2} (1-a) \alpha da & = \sum_{^{R,\delta}\Psi_i^2 \cap {\{Y_i \in A\}}} 2\delta(1-p_c(Y_i \in A|F_i)) \\
            & + \sum_{^{R,\delta}\Psi_i^1 \cap \{ Y_i \in A \}} \big(2-p_c(Y_i \in A|F_i)-R_1 - \delta\big) \big( p_c(Y_i \in A|F_i) - R_1 +\delta \big)/2 \\
            & + \sum_{^{R,\delta}\Psi_i^3 \cap \{ Y_i \in A \}} \big(2-p_c(Y_i \in A|F_i)-R_2 + \delta\big) \big( R_2 - p_c(Y_i \in A|F_i) +\delta \big)/2,
        \end{aligned}
    \end{equation}
    and
    \begin{equation}
    \int_{R_1}^{R_2}(\alpha +\beta)da = \sum_{^{R,\delta}\Psi_i^1} \left(p_c(Y_i \in A|F_i)-R_1+\delta \right)  + \sum_{^{R,\delta}\Psi_i^2} 2\delta + \sum_{^{R,\delta}\Psi_i^3}\left( R_2 - p_c(Y_i \in A|F_i) +\delta  \right).
    \end{equation}
    Computing the following limits gives
    \begin{align}
            \lim_{\delta \rightarrow0} \frac{1}{\delta}\int_{R_1}^{R_2} \beta a da & = \sum_{\{p_c(Y_i \in A|F_i)=R_1\} \cap \{Y_i \in A^{\mathbf C}\}}p_c(Y_i \in A|F_i) {\label{eq:apdx7-7} }\\
            & +\sum_{\{R_1<p_c(Y_i \in A|F_i)<R_2\} \cap \{Y_i \in A^{\mathbf C}\}}2p_c(Y_i \in A|F_i) + \sum_{\{p_c(Y_i \in A|F_i)=R_2\} \cap \{Y_i \in A^{\mathbf C}\}}p_c(Y_i \in A|F_i), \nonumber \\
            \lim_{\delta \rightarrow0} \frac{1}{\delta}\int_{R_1}^{R_2} (1-\alpha) a da & = \sum_{\{p_c(Y_i \in A|F_i)=R_1\} \cap \{Y_i \in A\}}\big(1-p_c(Y_i \in A|F_i)\big) {\label{eq:apdx7-8} }\\
            & +\sum_{\{R_1<p_c(Y_i \in A|F_i)<R_2\} \cap \{Y_i \in A\}}2\big( 1-p_c(Y_i \in A|F_i) \big) \nonumber \\
            & + \sum_{\{p_c(Y_i \in A|F_i)=R_2\} \cap \{Y_i \in A\}}\big(1-p_c(Y_i \in A|F_i) \big), \nonumber \\
            \lim_{\delta \rightarrow 0}\frac{1}{\delta}\int_{R_1}^{R_2} (\alpha+\beta) da & =\sum_{\{p_c(Y_i \in A|F_i)=R_1\}} 1 +\sum_{\{R_1<p_c(Y_i \in A|F_i)<R_2\} }2 + \sum_{\{p_c(Y_i \in A|F_i)=R_2\}}1 {\label{eq:apdx7-9} }
    \end{align}
    With the notations from Eq.\eqref{eq:apdx7-3} to Eq.\eqref{eq:apdx7-4}, combining Eq.\eqref{eq:apdx7-5}, Eq.\eqref{eq:apdx7-6}, Eq.\eqref{eq:apdx7-7}, Eq.\eqref{eq:apdx7-8} and Eq.\eqref{eq:apdx7-9} obtains
    \begin{equation}
    \resizebox{0.95\hsize}{!}{$
        \begin{aligned}
        T(Q,R,A) & \triangleq \left|  \frac{ \left[ \sum\limits_{  i \in (Q_A^{R_1}\cup Q_A^{R_2}) \cap O_A} (1-{^cp_i^A}) + 2\sum\limits_{i \in Q_A^{\overset{\frown}{R_1R_2}} \cap O_A}(1-{^cp_i^A})\right] -       \left[ \sum\limits_{i \in (Q_A^{R_1}\cup Q_A^{R_2}) \cap O_{A^{\mathbf C}} } {^cp_i^A} + 2\sum\limits_{i \in Q_A^{\overset{\frown}{R_1R_2}} \cap O_{A^{\mathbf C}}} {^cp_i^A}  \right]  }{ \sum\limits_{i \in Q_A^{R_1} } 1 + \sum\limits_{ i \in Q_A^{\overset{\frown}{R_1R_2}} } 2 + \sum\limits_{i \in Q_A^{R_2} } 1 } \right|   \\
        & =  \lim_{\delta \rightarrow0} \left| \frac{\int_{R_1}^{R_2}(1-a)\alpha da - \int_{R_1}^{R_2}a\beta da} {\int_{R_1}^{R_2}(\alpha+\beta)da} \right| \leq \lim_{\delta \rightarrow 0} (\epsilon+ \delta) = \epsilon
        \end{aligned}$
        }
    \end{equation}
    \item[($\Leftarrow$):]
    For any interval $\left[a-\delta, a+\delta\right] \subseteq [0,1]$, let $R_1 = a - \delta$ and $R_2 = a+ \delta$.
    \begin{enumerate}
        \item[(1)] If $Q_A^{R_1} = \varnothing$ and $Q_A^{R_2} = \varnothing$, Eq.\eqref{eq:apdx7-2} becomes
            \begin{equation}
            \label{eq:apdx7-10}
                T(Q,R,A) = \left|  \frac{  2\sum\limits_{i \in Q_A^{\overline{R_1R_2}} \cap O_A} (1-{^cp_i^A}) - 2\sum\limits_{i \in Q_A^{\overline{R_1R_2}} \cap O_{A^{\mathbf C}}} {^cp_i^A} }{ 2\sum\limits_{i \in Q_A^{\overline{R_1R_2}} } 1 } \right|  \leq \epsilon.
            \end{equation}
        \item[(2)] If $ Q_A^{R_1} \neq \varnothing$ and $Q_A^{R_2} = \varnothing$, since $Q_A$ is a finite set, there exists $\varsigma>0$ such that $\{^cp_i^A| {^cp_i^A \in [R_1,R_1+\varsigma]},i \in Q_A^{\overline{R_1R_2}}  \} = \{R_1\}$ and $\{^cp_i^A| {^cp_i^A \in [R_2 - \varsigma ,R_2]},i \in Q_A^{\overline{R_1R_2}}  \} = \varnothing$. Applying Eq.\eqref{eq:apdx7-2} obtains
    \begin{equation}
    \label{eq:apdx7-11}
        \left|  \frac{ 2\sum\limits_{i \in Q_A^{R_1} \cap O_A} (1-{^cp_i^A}) -  2\sum\limits_{i \in Q_A^{R_1} \cap O_{A^{\mathbf C}} } {^cp_i^A} }{2\sum\limits_{i \in Q_A^{R_1} } 1 } \right| \leq \epsilon,
    \end{equation}
    and 
    \begin{equation}
    \label{eq:apdx7-12}
        \left|  \frac{ \left[ 2\sum\limits_{i \in Q_A^{\overset{\frown}{R_1R_2}} \cap O_A} (1-{^cp_i^A}) + 2\sum\limits_{i \in Q_A^{R_2} \cap O_A} (1-{^cp}_i^A)  \right] -  \left[ 2\sum\limits_{i \in Q_A^{\overset{\frown}{R_1R_2}} \cap O_{A^{\mathbf C}}} {^cp_i^A}  + 2\sum\limits_{i \in Q_A^{R_2} \cap O_{A^{\mathbf C}} } {^cp_i^A} \right] }{ 2\sum\limits_{i \in Q_A^{\overset{\frown}{R_1R_2}} } 1 + 2\sum\limits_{i \in Q_A^{R_2}} 1 } \right| \leq \epsilon.
    \end{equation}
    Eq.\eqref{eq:apdx7-10} can also be derived by combining Eq.\eqref{eq:apdx7-11} and Eq.\eqref{eq:apdx7-12}
    \end{enumerate}
    Following similar procedures as above, we can also get Eq.\eqref{eq:apdx7-10} when $Q_A^{R_1} = \varnothing$, $Q_A^{R_2} \neq \varnothing$ or $Q_A^{R_1} \neq \varnothing$, $Q_A^{R_2} \neq \varnothing$. Below we will show Eq.\eqref{eq:apdx7-1} can be derived given Eq.\eqref{eq:apdx7-10}. Firstly, some terms are abbreviated for clarity as follows: 
    \begin{equation}
        \rho_A \triangleq \frac{\sum_{i \in Q_A^{B_\delta(a)}} \mathds 1_A(Y_i) }{|Q_A^{B_\delta(a)}|} = \frac{\left|Q_A^{\overline{R_1R_2}} \cap O_A\right|}{\left|Q_A^{\overline{R_1R_2}}\right|},
    \end{equation}
    \begin{equation}
        \phi_A \triangleq \frac{1}{\left|Q_A^{\overline{R_1R_2}} \cap O_A\right|} \sum\limits_{i \in Q_A^{\overline{R_1R_2}} \cap O_A} {^cp_i^A},
    \end{equation}
    \begin{equation}
        \phi_{A^\mathbf C} \triangleq \frac{1}{\left|Q_A^{\overline{R_1R_2}} \cap O_{A^\mathbf C}\right|} \sum\limits_{i \in Q_A^{\overline{R_1R_2}} \cap O_{A^\mathbf C}} {^cp_i^A}.
    \end{equation}
    It is obvious that
    \begin{equation}
    \label{eq:apdx7-13}
        a-\delta = R_1 \leq \phi_A,\phi_{A^\mathbf C} \leq R_2 = a + \delta.
    \end{equation}
    Eq.\eqref{eq:apdx7-10} can thus be rewritten as
    \begin{equation}
    \label{eq:apdx7-14}
        \left| \rho_A \cdot (1-\phi_A) - (1-\rho_A) \cdot \phi_{A^\mathbf C} \right| \leq \epsilon
    \end{equation}
    By simple transformation, we have
    \begin{equation}
        a- \delta -\epsilon=R_1 - \epsilon \leq \frac{\phi_{A^\mathbf C} -\epsilon }{1 - \phi_A + \phi_{A^\mathbf C}} \leq \rho_A \leq \frac{\phi_{A^\mathbf C}+\epsilon}{1-\phi_A + \phi_{A^\mathbf C}} \leq R_2+\epsilon = a + \delta + \epsilon,
    \end{equation}
    where the middle two inequalities are given by Eq.\eqref{eq:apdx7-14}, and the other two inequalities can be obtained by the range of $\phi_A$ and $\phi_{A^\mathbf C}$ specified in Eq.\eqref{eq:apdx7-13}. Therefore, we derive
    \begin{equation}
        \left| a -  \frac{ \sum_{i \in Q_A^{B_\delta(a)}} \mathds 1_A(Y_i)}{|Q_A^{B_\delta(a)}|} \right|  = \left| a -  \rho_A \right| \leq \epsilon + \delta.
    \end{equation}
\end{enumerate}

\end{proof}

\subsection{The Probabilistic Interpretation}
\label{sec:apdx-probinterp}

Here we present a probabilistic explanation for the aforementioned theorems, ignoring the bounds involved. Thms. \ref{thm:5} and \ref{thm:6} convert the constraint in Def. \ref{def:del-eps-calib} into Eq. \eqref{eq:rm-exp}, where the LHS is expected to approach zero. Equivalently for any $\mathcal I \subseteq Z^+$,
\noindent

\begin{equation}
    a \approx  \frac{ \frac{1}{|\mathcal I|}  \sum_{i \in \mathcal I}\mathds 1_{\{ {^cp}_i^A \in B_\delta(a), Y_i \in A \}} }{\frac{1}{|\mathcal I|} \sum_{i \in \mathcal I} [ \mathds 1_{\{ {^cp}_i^A \in B_\delta(a), Y_i \in A  \}} +  \mathds 1_{\{ {^cp}_i^A \in B_\delta(a), Y_i \notin A  \}} ] } , 
\label{eq:approxingstats}
\end{equation}

where $p_c(Y \in A|F) \mathds 1_{\{Y \in A\}}$ and $p_c(Y \in A|F) \mathds 1_{\{Y \notin A\}}$ can be considered as the r.v. mapping $\left(\Omega_F,\mathscr{F}_F, p_F\right)$ to $([0,1],{\mathscr B_{[0,1]}},\lambda_{\text{Leb}})$. $\mathscr{F}_F$ and $\mathscr B_{[0,1]}$ are the $\sigma$-field of $F$ and  Borel $\sigma$-algebra on $[0,1]$, respectively. $p_F$ and $\lambda_{\text{Leb}}$ represent the feature distribution and Lebesgue measure. Let $\tau_A^A$ and $\tau_{A^{\mathbf C}}^A$ denote the densities of $p_c(Y \in A|F) \mathds 1_{\{Y \in A\}}$ and $p_c(Y \in A|F) \mathds 1_{\{Y \notin A\}}$ in space $([0,1],{\mathscr B_{[0,1]}},\lambda_{\text{Leb}})$, respectively. Then, for Eq. \eqref{eq:approxingstats}, the numerator
$\frac{1}{|\mathcal I|}{\sum_{i \in \mathcal I}\mathds 1_{\{ p_c(Y_i \in A|F_i) \in B_\delta(a), Y_i \in A \}}}$
and the second term in the denominator 
$    \frac{1}{|\mathcal I|}{ \sum_{i \in \mathcal I} \mathds 1_{\{ p_c(Y_i \in A | F_i) \in B_\delta(a), Y_i \notin A  \}}}$
can be seen respectively as empirical estimations of $\tau_A^A(a)$ and $\tau_{A^{\mathbf C}}^A(a)$, accordingly. Therefore, Eq. \eqref{eq:approxingstats} essentially states
\begin{equation}
    a  \approx \tau_A^A(a) / [\tau_A^A(a)+\tau_{A^{\mathbf C}}^A(a) ],
\end{equation}
i.e., $a \cdot \tau_{A^{\mathbf C}}^A(a) \approx (1-a) \cdot \tau_A^A(a)$. Correspondingly, Thm. \ref{thm:7} essentially imposes the following constraint (with error bound estimated)

\begin{equation*}
    T_2 / T_3 \approx \int \mathds 1_{ \{R_1 \leq a \leq R_2 \}} \cdot a \cdot \tau_{A^{\mathbf C}}^A(a) d\lambda_{\mathrm{Leb}}(a)
\end{equation*}
\begin{equation}
    \approx \int \mathds 1_{ \{R_1 \leq a \leq R_2 \}} (1-a) \cdot \tau_A^A(a) d\lambda_{\mathrm{Leb}}(a) \approx T_1 / T_3,
\end{equation}
i.e., the integration of r.v. $p_c(Y \in A|F) \mathds 1_{\{Y \notin A\}}$ on set $\{R_1 \leq p_c(Y \in A|F) \mathds 1_{\{Y \notin A\}} \leq R_2\}$ approximately equals to the integration of $1-p_c(Y \in A|F) \mathds 1_{\{Y \in A\}}$ on set $ \{R_1 \leq  p_c(Y \in A|F) \mathds 1_{\{Y \in A\}} \leq R_2\}$.

\section{Summarizing Existing Evaluation Metrics under a Unified Probabilistic Framework}
\label{sec:apdx-evaluationsummary}
{\mdseries
In this section, we summarize and compare diverse computational evaluation metrics within a unified probabilistic framework regarding measuring the deviation between $p_\mu (\mathscr E|\mathscr H(F))$ and $\mathscr H(F)$  in Table \ref{tab1}.
\subsection{Top-label Calibration Evaluator}
\subsubsection{\textbf{ECE}}
The Expected Calibration Error (ECE) is the most popular metric for calibration evaluation. The fundamental idea of ECE is to approximate 
\begin{equation}
\label{ece_theoretical_definition}
    \sqrt[r]{\mathbb E \Big[ \Big| p_\mu (\mathscr E|\mathscr H(F))-\mathscr H(F) \Big|^r \Big]},
\end{equation}
where $\mathscr H(F) = \max\limits_l p_c( Y=l| F)$ and $\mathscr E = \mathds 1_{\{ Y = \underset{l}{\operatorname {argmax~}} p_c( Y=l| F) \}}$. This approximation necessitates the selection of an appropriate discrete binning strategy to estimate $p_\mu (\mathscr E|\mathscr H(F))$. Different selections of $r$ for the normed space and binning strategy lead to different approximations. Common choices for $r$ include 1 or 2. Prevalent binning methods are set to equal-width bins (ew) within the interval $[0,1]$, or to equal-mass binning, where each bin contains an equal number of samples according to the distribution of $\mathscr H(F)$. More sophisticated binning methods, such as I-Max binning \cite{no.1} and sweep binning \cite{no.50}, aim to preserve label information during the binning process \cite{no.1} and select the maximum number of bins while preserving monotonicity in the binwise accuracy \cite{no.50}, respectively. The most common configuration for ECE in literature involves setting $r=1$ and using equal width binning, with the calculation formula specified as:
\begin{gather}
    \text{ECE}_{r=1}^{\text{ew}} = \sum_{m=1}^M \frac{|B_m|}{N} |A_m-C_m|; \\
    \text{where}~
    A_m =\frac{1}{|B_m|} \sum_{i \in B_m} \mathds 1_{ \{ Y_i = \underset{l}{\operatorname {argmax~}} p_c( Y=l| F_i)\}};~ 
    C_m = \frac{1}{|B_m|} \sum_{i \in B_m} \max_{l} p_c(Y=l|F_i)
\end{gather}
and $B_m$ contains all the samples with their confidence score, $\max_l p_c(Y=l | F_i)$, falls within the interval $\Big[\frac{m}{M},\frac{m+1}{M} \Big)$.

Other variants of ECE include a debiased variant \cite{no.75} and thresholding variants \cite{no.76,no.3}. The debiased variant aims to mitigate the variation in estimation errors as the sample size varies, while thresholding variants calculate error solely for predictions above a certain confidence threshold to wash out the influence of vast tiny predictons. 

In this study, we evaluate calibration using equal width binning $\mathrm{ECE}^{\mathrm{ew}}$ and its higher-order variant $\mathrm{ECE}_{r=2}$. We also compute the equal mass binning variant $\mathrm{ECE}^{\mathrm{em}}$, as well as the debiased equal mass binning variant ($\mathrm{dECE}$). Notably, we do not apply a debiased ECE estimator for equal-width binning, as it does not guarantee sufficient samples per bin, which lead to computational instability or even errors. Additionally, we compute the sweep binning versions, $\mathrm{ECE}_{r=1}^s$ and $\mathrm{ECE}_{r=2}^s$. Since IMax binning was originally designed for training a calibrator rather than for calibration evaluation, we exclude it from our evaluation metrics.

\subsubsection{\textbf{ACE \& MCE}}
In comparison to the calculation of ECE, the Average Calibration Error (ACE) is derived by directly averaging the binning-wise errors, rather than using a weighted average. On the other hand, the Maximum Calibration Error (MCE) quantifies the largest error across all bins. 
\begin{equation}
    \text{ACE}^{\text{ew}} = \sum_{m=1}^M \frac{1}{M} |A_m-C_m|;~~
    \text{MCE}^{\text{ew}} = \max_{1\leq m\leq M} |A_m-C_m|;
\end{equation}
Similarly, that there are variations of ACE and MCE depending on the binning strategy employed. Particularly, when opting for equal mass binning, the ACE is equivalent to the ECE. Thus, ACE by default refers to equal-width binning. 

In this study, ACE is included in evaluation metrics. However, since MCE estimates the error based on a single bin with limited and variable sample sizes, a majority of the predictive information is overlooked, making MCE highly sensitive to noise and binning configuration \cite{no.205,no.206,no.101}. Consequently, MCE is not included in our evaluation metrics.

\subsubsection{\textbf{KS error}}
Gupta \emph{et al.} \cite{no.60} propose the KS error (Kolmogorov-Smirnov calibration error) whose definition can be written as 
\begin{equation}
    \text{KS} = \max_{0\leq\sigma\leq1} \int_0^\sigma \big| p_\mu(\mathscr E|\mathscr H(F)) - \mathscr H(F) \big|p_{\mathscr H(F)}(d\sigma).
\end{equation}
The corresponding discrete approximation is given by 
\begin{gather}
    \text{KS} = \max_{1\leq i \leq N} 
    \big|  H_i - G_i \big|;\\
    \text{where}~H_i = \frac{1}{N} \sum_{j \in D_i} \mathds 1_{ \{Y_j = \underset{l}{\operatorname {argmax}}~ p_c( Y=l| F_j)\}};~
    G_i = \frac{1}{N} \sum_{j \in D_i} \max_{l} p_c(Y=l|F_j)
\end{gather}
and $G_i = \frac{1}{N} \sum_{j \in D_i} \max_{l} p_c(Y=l|F_j)$, with $D_i$ containing all the samples whose confidence score is less than $\max_{l} p_c(Y=l|F_i)$.

\subsubsection{\textbf{KDE-ECE}}
Zhang \emph{et al.} \cite{no.70} introduce the KDE-ECE (Kernel Density Estimation-based Expected Calibration Error estimator), which employs kernel smoothing techniques to estimate the distribution density, thereby estimating the expectation for the $r$-th power of Eq. \eqref{ece_theoretical_definition}. This method serves as an alternative to the traditional binning approach in approximating Eq. \eqref{ece_theoretical_definition} and it is defined as 
\begin{gather}
    \text{KDE-ECE} = \int \| z - \tilde{\pi}(z)  \|_r^r \tilde{p}_{\mathscr H(F)}(dz); \\
    \text{where}~~ \tilde{p}_{\mathscr H(F)}(z) = \frac{1}{Nh}\sum_{i=1}^N K_h(z - \mathscr H(F_i)) = \frac{1}{Nh}\sum_{i=1}^N K_h\big(z - \max_l p_c(Y=l | F_i) \big),\\
    \tilde{\pi}(z) = \frac{\sum_{i=1}^N \mathscr E_i K_h(z-\mathscr H(F_i)) }{\sum_{i=1}^N K_h(z-\mathscr H(F_i))} = 
\frac{\sum\limits_{i=1}^N \mathds 1_{\{ 
Y_i = \underset{l}{\operatorname {argmax}}~ p_c( Y=l| F_i)
\} }  K_h (z - \underset{l}{\max}~ p_c(Y=l | F_i)) }{\sum\limits_{i=1}^N K_h (z - \underset{l}{\max}~ p_c(Y=l | F_i)) },
\end{gather}
and $K_h$ denotes the selected smoothing kernel function $K$ along with its bandwidth hyperparameter $h$. Although the KDE-ECE approximation can, conceptually, be extended for calibration evaluation in a high-dimensional canonical context, its practical application is restricted by the curse of dimensionality, which introduces substantial estimation errors in high-dimensional density estimation and integration. Consequently, the code implementation by \cite{no.70} was limited to top-label calibration evaluation \cite{no.52}.

\subsubsection{\textbf{MMCE}}
Kumar \emph{et al.} \cite{no.83} propose the MMCE (Maximum Mean Calibration Error) metric. Let $\mathcal H$ denote the Reproducible Kernel Hilbert Space (RKHS) induced by a universal kernel $k:[0,1]\times[0,1] \rightarrow R$, and let the corresponding reproducing kernel feature map be denoted by $\phi: [0,1]\rightarrow \mathcal H$. The definition of MMCE can be expressed as $\text{MMCE} = \Big\| \mathbb E \big[ (\mathscr E -\mathscr H(F))\phi (\mathscr H(F))  \big] \Big\|_{\mathcal H}$, where the definitions of $\mathscr E$ and $\mathscr H(F)$ are provided in Table \ref{tab1} of the main text under the top-label calibration setting. That is, $\mathscr H(F) = \max_y p_c( Y=l| F)$ and $\mathds 1_ {\{ Y = \underset{l}{\operatorname {argmax}} ~p_c( Y=l| F) \}}$. Given the properties of Hilbert space, by defining $\mathcal F$ as the unit ball in space $\mathcal H$, the meaning of MMCE can be interpreted as $\sup_{f \in \mathcal F} \mathbb E \Big[ \big( p_\mu(\mathscr E | \mathscr H(F))-\mathscr H(F)\big)  f(\mathscr H(F)) \Big]$. The equivalence proof can be summarized as follows.
\begin{align}
    \text{MMCE} & = \Big\| \mathbb E \big[ (\mathscr E -\mathscr H(F))\phi (\mathscr H(F))  \big] \Big\|_{\mathcal H} = \sup_{f \in \mathcal F} \mathbb E \big[ (\mathscr E -\mathscr H(F)) \langle \phi(\mathscr H(F)),f  \rangle  \big] \\
    & = \sup_{f \in \mathcal F} \mathbb E \big[ (\mathscr E -\mathscr H(F)) f(\mathscr H(F))  \big] 
    = \sup_{f \in \mathcal F} \mathbb E \Big[ \mathbb E \big[ (\mathscr E -\mathscr H(F)) f(\mathscr H(F))  \big| \mathscr H(F)  \big] \Big] \\
    & = \sup_{f \in \mathcal F} \mathbb E \Big[ \mathbb E \big[ (\mathscr E -\mathscr H(F))  \big| \mathscr H(F) \big] f(\mathscr H(F)) \Big] 
    = \sup_{f \in \mathcal F} \mathbb E \Big[ \big ( p_\mu(\mathscr E | \mathscr H(F))-\mathscr H(F)\big)  f(\mathscr H(F)) \Big].
\end{align}
Therefore, this metric can also be considered as an indicator assessing the difference between $\mathscr E$ and $\mathscr H(F)$. In terms of computational approximation, the squared MMCE can be written as
\begin{align}
    \text{MMCE}^2 & = \Big\langle \mathbb E \big[ (\mathscr E -\mathscr H(F))\phi (\mathscr H(F))  \big], \mathbb E \big[ (\mathscr E -\mathscr H(F))\phi (\mathscr H(F))  \big] \Big\rangle_{\mathcal H} \\
    & = \mathbb E \Big[ \Big \langle (\mathscr E -\mathscr H(F))\phi (\mathscr H(F)), (\overline {\mathscr E} -\overline{\mathscr H(F)})\phi (\overline{\mathscr H(F)}) \Big \rangle_{\mathcal H} \Big] \\
    & = \mathbb E \Big[ (\mathscr E -\mathscr H(F)) \Big \langle \phi (\mathscr H(F)), \phi (\overline{\mathscr H(F)}) \Big \rangle_{\mathcal H} (\overline {\mathscr E} -\overline{\mathscr H(F)}) \Big] \\
    & = \mathbb E \Big[ (\mathscr E -\mathscr H(F)) k(\mathscr H(F),\overline{\mathscr H(F)}) (\overline {\mathscr E} -\overline{\mathscr H(F)}) \Big],
\end{align}
where $\overline{\mathscr E}$ and $\overline{\mathscr H (F)}$ symbolize independent copies of $\mathscr E$ and $\mathscr H(F)$, respectively. Therefore, the formula can be approximated as 
\begin{equation}
    \text{MMCE}^2 = \frac{1} {N^2}\sum_{1\leq i,j\leq N}  \big( \mathscr E_i - \mathscr H(F_i)\big) k\big(\mathscr H(F_i), \mathscr H(F_j) \big) \big( \mathscr E_j - \mathscr H(F_j)\big),
\end{equation}
where $\mathscr E_i = \mathds 1_{\{Y_i = \underset{l}{\operatorname {argmax}} ~p_c( Y=l| F_i) \}}$ and $\mathscr H(F_i) = \max_l p_c( Y=l| F_i)$. The evaluation depend on the kernel functions $k$ selected.

\subsection{Classwise Calibration Evaluator}
The evaluation of classwise calibration is primarily performed through the CWECE (Classwise Expected Calibration Error ), whose essential idea is to estimate 
\begin{equation}
 \sqrt[r]{
\sum_{l=1}^L\mathbb E \Big[ \Big| p_\mu (\mathscr E|\mathscr H(F))-\mathscr H(F) \Big|^r \Big]},
\end{equation}
where $\mathscr H(F) = p_c( Y=l| F)$ and $\mathscr E = \mathds 1_{\{Y = l\}}$. As with ECE, employing different configurations yields a different variants of CWECE. Options includes the selecting of discrete binning to approximate $p_\mu (\mathscr E|\mathscr H(F))$ and setting value of $r$ (typically 1 or 2). Different ECE approximation configurations can be directly applied to CWECE, with the most common practice being the selection of equal width binning (dividing $[0,1]$ into $M$ equal intervals) and $r=1$, resulting in
\begin{gather}
\label{eq:cwece}
    \text{CWECE}^{\text{ew}}_{r=1} = \frac{1}{L} \sum_{l=1}^L \sum_{m=1}^M \frac{|B_{l,m}|}{N} |A_{l,m} - C_{l,m}|; \\
    \text{where}~
    A_{l,m} = \frac{1}{|B_{l,m}|} \sum_{i \in B_{l,m}} \mathds 1_{\{ Y_i=l \}} \text{~and~}
    C_{l,m} = \frac{1}{|B_{l,m}|} \sum_{i \in B_{l,m}} p_c(Y=l|F_i),
\end{gather}
and $B_{l,m}$ contains all the samples whose confidence score $p_c(Y=l | F_i) \in \big[\frac{m}{M},\frac{m+1}{M} \big)$.

In the calculation of CWECE, some existing studies omit the denominator $1/L$ from Eq. \eqref{eq:cwece}. We refer to this form as the total CWECE (abbreviated as $\mathrm{CWECE}_s$), whereas the form that retains the denominator is called the average CWECE (abbreviated as $\mathrm{CWECE}_{a}$ or $\mathrm{CWECE}$ by default) in this study. There are other variants of CWECE. Study \cite{no.1} used a thresholded variant ($t\mathrm{CWECE}$), which computes CWECE for instances where the classwise predicted probability exceeds a given threshold, aiming to mitigate the influence of numerous small class probabilities in classwise setting (particularly significant when the number of classes is large). Study \cite{no.1} also introduced a k-means clustering strategy to derive bin edges with balanced bin widths and sample sizes for estimating CWECE, resulting in the k-means binning-based variant $t\mathrm{CWECE}^{k}$. In this study, we incorporate these variants into our evaluation metrics.

\subsection{Canonical calibration evaluator}
\subsubsection{\textbf{SKCE}}
Widmann \emph{et al.} \cite{no.79} extended MMCE \cite{no.83} to the canonical context by introducing SKCE (Squared Kernel Calibration Error). Let $\mathcal H$ denote the Reproducible Kernel Hilbert Space (RKHS) induced by a matrix-valued kernel $k: \Delta^L \times \Delta^L \rightarrow {\mathbb R}^{L \times L}$, and let $\mathcal K_{\zeta}v = k(\cdot,\zeta)v \in \mathcal H$ for any $\zeta \in \Delta^L$ and $v \in \mathbb R^L$. Following the notation in Table \ref{tab1}, for a canonical setting, $\mathscr H(F)$ and $\mathscr E$ correspond to $\mathscr H(F) = \big[ p_c( Y=1 | F),...,p_c( Y=L | F) \big]^\top$ and $\mathscr E = \big[ \mathds 1_{\{Y=1\}},...,\mathds 1_{\{Y=L\}} \big]^\top$, respectively. Defining $\mathcal F$ as the unit ball in space $\mathcal H$, employing an approach akin to MMCE's derivation, it follows that
\begin{align}
    \text{SKCE} & \triangleq \sup_{f \in \mathcal F} \mathbb E \big[ \big\langle
(p_c(\mathscr E | \mathscr H(F)) -\mathscr H(F)), f(\mathscr H(F)) \big\rangle_{\mathbb R^L}  \big] = \sup_{f \in \mathcal F} \mathbb E \big[ \big\langle (\mathscr E -\mathscr H(F)), f(\mathscr H(F)) \big\rangle_{\mathbb R^L}  \big] \\
&  = \sup_{f \in \mathcal F} \mathbb E \big[ \big\langle \mathcal K_{\mathscr H(F)}(\mathscr E - \mathscr H(F)), f \big\rangle_{\mathcal H} \big] = \big\langle 
\mathbb E \big[ \mathcal K_{\mathscr H(F)}(\mathscr E - \mathscr H(F))\big], \mathbb E \big[ \mathcal K_{\mathscr H(F)}(\mathscr E - \mathscr H(F))\big] \big\rangle_{\mathcal H} \\
& = \mathbb E \big[ \big\langle \mathcal K_{\mathscr H(F)}(\mathscr E - \mathscr H(F)), \mathcal K_{\overline{\mathscr H(F)}}(\overline{\mathscr E} - \overline{\mathscr H (F)}) \big\rangle_{\mathcal H} \big] \\
& = \mathbb E \Big[ \Big\langle\mathscr E - \mathscr H(F), k\big(\mathscr H(F),\overline{\mathscr H(F)}\big)\big( \overline{\mathscr E} - \overline{\mathscr H(F)}\big) \Big\rangle_{\mathbb R^L} \Big],
\end{align}
where $\overline{\mathscr E}$ and $\overline{\mathscr H (F)}$ indicates independent copies of $\mathscr E$ and $\mathscr H(F)$, respectively. Hence, an approximation formula can be expressed as
\begin{equation}
    \widehat{\text{SKCE}} = \frac{1}{N^2} \sum_{1\leq i,j \leq N} h_{i,j} = \frac{1} {N^2}\sum_{1\leq i,j\leq N} \big( \mathscr E_i - \mathscr H(F_i)\big)^\top k\big(\mathscr H(F_i), \mathscr H(F_j) \big) \big( \mathscr E_j - \mathscr H(F_j)\big),
\end{equation}
where $\mathscr E_i = \big[ \mathds 1_{\{Y_i=1\}},...,\mathds 1_{\{Y_i=L\}} \big]^\top$ and $\mathscr H(F_i) = \big[ p_c( Y_i=1 | F),...,p_c( Y_i=L | F) \big]^\top$. The evaluation also depend on the choice of kernel function $k$. Additionally, Widmann \emph{et al.} \cite{no.79} also provide two other estimators as approximations of $\widehat{\text{SKCE}}$: an unbiased quadratic estimator and an unbiased linear estimator, expressed as ${\binom{N}{2}}^{-1}\sum_{1\leq i<j\leq N} h_{i,j}$ and $\lfloor N/2 \rfloor^{-1} \sum_{i=1}^{\lfloor N/2 \rfloor} h_{2i-1,2i}$, respectively.

It's noteworthy that the definition of the kernel function, as outlined above, deviates in its form from that presented in the open-source code released by Widmann \emph{et al.} This discrepancy arises because Widmann \emph{et al.} have adapted the kernel function's definition to support both classification and regression tasks \cite{no.191}. In the adapted framework, the kernel function is conceptualized as a mapping, $\tilde{k}: (\mathcal{P} \times \mathcal{Y}) \times (\mathcal{P} \times \mathcal{Y}) \rightarrow \mathbb{R}$. Although the definitions vary in their form, in the context of classification tasks, the kernel $\tilde{k}$ maintains a one-to-one correspondence with the kernel $k$, with the relation $[k(p,p')]_{y,y'} = \tilde{k}((p,y),(p',y'))$. Therefore, despite the differences in descriptions, the kernels are functionally equivalent.

\subsubsection{\textbf{DKDE-CE}}
Popordanoska \emph{et al.} \cite{no.52} introduced the DKDE-CE (Dirichlet Kernel Based Calibration Error Estimator), extending the kernel smoothing technique used in the top-label KDE-ECE \cite{no.70} for canonical calibration. This extension employs the Dirichlet kernel for the kernel smoothing process. Specifically, for the canonical setting, $\mathscr H(F)$ and $\mathscr E$ are defined as $\mathscr H(F) = \big[ p_c( Y=1 | F),...,p_c( Y=L | F) \big]^\top$ and $\mathscr E = \big[ \mathds 1_{\{Y=1\}},...,\mathds 1_{\{Y=L\}} \big]^\top$, respectively, in Table \ref{tab1}. The definition of $\text{DKDE-CE}$ is given as 
\begin{gather}
\label{eq:DKDE-CE}
    \text{DKDE-CE} = \int \| z - \tilde{\pi}(z)  \|_r^r \tilde{p}_{\mathscr H(F)}(dz);\\
    \tilde{\pi}(z) = \frac{\sum_{i=1}^N \mathscr E_i K_{\text Dir}(z;\mathscr H(F_i)) }{\sum_{i=1}^N K_{\text Dir}(z;\mathscr H(F_i))};~~\text{where}~~K_{\text{Dir}}(a;b) = \frac{\Gamma(L+\sum_{l=1}^L b_l/h)}{\prod_{l=1}^L \Gamma(1+b_l/h)}
    \prod_{l=1}^L a_l^{b_l/h}.
\end{gather}
where $h$ is the bandwidth hyperparameter in the Dirichlet kernel. The integral in Eq. \eqref{eq:DKDE-CE} can be further approximated, yielding 
\begin{equation}
    \widehat{\text{DKDE-CE}} = \frac{1}{N} \sum_{j=1}^N \left\| \mathscr H(F_j) - \frac{\sum_{i\neq j}^N \mathscr E_i K_{\text Dir}(\mathscr H(F_j);\mathscr H(F_i)) }{\sum_{i \neq j}^N K_{\text Dir}(\mathscr H(F_j);\mathscr H(F_i))} \right\|_r^r,
\end{equation}

In this study, both SKCE and DKDE-CE are used as evaluation metrics. For SKCE, we use the default unbiased implementation provided in the source code \cite{no.79}, and it is worth noting that the unbiased version can yield negative SKCE values. For DKDE-CE, we also utilize the source code implementation \cite{no.52}, with the default parameters of order $r=2$ and bandwidth $h=1$.

\subsection{Comparison between $h$-calibration and kernel-based methods}

We now further clarify the theoretical and computational connections and differences between our approach and these kernel-based methods \cite{no.70,no.83,no.79,no.52}.

\emph{Theoretically}, both our method and kernel-based approaches define differentiable calibration error objectives via integral transforms. However, our method offers several advantages:
(a) The optimization target of $h$-calibration (Eq. \eqref{eq:h-calibration-v1}) is a sufficient condition (see Thm. \ref{thm:1}) for canonical calibration, which is the focus of most kernel-based methods. However, the reverse does not hold. For example, on the CIFAR-10 dataset, assigning every sample a uniform predictive distribution $[1/10,...,1/10]$ , which clearly lacks discriminativity, satisfies canonical calibration but not $h$-calibration.
(b) Our method constructs asymptotically equivalent formulations for learning with uniformly bounded error (Thms. \ref{thm:2}, \ref{thm:6} and \ref{thm:7}), while kernel-based metrics depend on empirically chosen kernels with unclear interpretation and lack bounded guarantees.

\emph{Computationally}, both our approach and kernel methods reduce average discrepancies between predicted probabilities and frequencies, but differ in alignment structure:
(c) Kernel methods align high-dimensional representations, which are prone to overfitting and affected by the curse of dimensionality, making them unsuitable for canonical calibration. In contrast, our method derives constraints that align the one-dimensional predicted probability and frequency of randomized events, with provable approximation bounds (Thm. \ref{thm:6}), thereby mitigating overfitting and avoiding high-dimensional instability. Indeed, this overfitting risk explains why, among the related works, studies \cite{no.70} and \cite{no.79} introduce kernel-based calibration error only as an evaluation metric, rather than using it for post-hoc recalibration. Studies \cite{no.83} and \cite{no.52} incorporate kernel-based calibration error as a regularization term added to cross-entropy during training, again rather than directly applying it for post-hoc calibration. As such, direct experimental comparisons with these approaches are not applicable.

It is worth noting that one of our compared approaches, Spline \cite{no.60}, which can be shown to be a kernel-based measure \cite{gretton2012kernel}, and is used in a post-hoc recalibration setup. We included it as a representative baseline and observe that our method consistently outperforms it across 15 tasks and 17 evaluation metrics in our experiments, clearly highlighting the advantage of our method.

}

\section{Specification of Weighting Function and Calibration Mapping}
{\mdseries
\subsection{Weighting Function}
\label{sec:apdx-weighitingfunction}
As explained in Eq. \eqref{eq:finalloss} in the main text, we note the need to weight $\mathscr L(\mathscr R)$ with $w(\mathscr R)$ to address the bias that regularizing interval $\mathscr R$ tends to focus on low-probabilities as the number of classes increases in a multi-class setting. Here, we detail how we utilize k-means clustering to mitigate this bias. Specifically, we extract the centroids of the elements $\mathbf{q}$ within each $\mathscr{R}$ and cluster them into $C$ clusters using k-means. We then set $w(\mathscr R) = \frac{1}{CN_{\mathscr R}}$, where $N_{\mathscr R}$ represents the number of samples in the cluster containing the centroid of $\mathscr R$. This weighting strategy ensures that $\mathscr L(\mathscr R)$ is averaged separately for highly densely and sparsely distributed $\mathscr R$, thereby adaptively mitigating the regularizing interval bias caused by class number variation. In practice, the number of clusters $C$ is set to 15.

\subsection{Monotonic Calibration Mapping}
\label{sec:apdx-monotonicmapping}
As discussed in Section \ref{method:sec4}, the ground-truth transformation from uncalibrated to calibrated logits is inherently unknown. Following prior studies \cite{no.70,no.67}, we employ specific types of learnable monotonic functions as calibration mappings to preserve classification accuracy. We employ three mapping families: ensemble linear mapping, piecewise linear mapping, and nonlinear mapping.

\begin{description}[style=unboxed,leftmargin=0cm]
  \item[Ensemble Linear Mapping:]
  Inspired by ensemble temperature scaling \cite{no.70}, which extends temperature scaling by applying triple linear scalings to logits and producing three sets of softmax probabilities for weighted averaging, we extend this idea in our study. In the original work \cite{no.70}, only one temperature parameter was learnable, while the other two were fixed at 1 and infinity. Here, we extend the mapping to use $m$ learnable temperature parameters.
  \item[Piecewise Linear Mapping:]
  We introduce a simple, continuous piecewise linear mapping as a bridge between linear and nonlinear mappings. The mapping divides the input range $[-100, 0]$ into $z$ equal segments, each defined by a linear function with learnable slopes. The logit is normalized by subtracting the maximum logit before transformation, ensuring it falls within the effective range of $[-100, 0]$.
  \item[Nonlinear Mapping:]
  We adopt the nonlinear monotonic network used in \cite{no.67}, specifically MonotonicNet \cite{no.207}, as the learnable mapping.
\end{description}

Hyperparameters for mapping families are set as follows: $m$ takes values from $\{16, 32, 64, 128\}$ for linear mapping; $z$ is chosen from $\{1, 10, 100, 500\}$ for piecewise linear mapping; MonotonicNet uses 2 hidden layers with the number of neurons selected from $\{2, 10, 20, 50\}$. Inspired by prior work using cross-validation to select suitable mapping of hyperparamter \cite{no.67,no.71,no.84,no.46,no.88}, we adopt a simpler approach by choosing the optimal mapping that achieve the best training set calibration performance. Notably, our mapping selection process is simpler and involves significantly fewer candidate mappings than those used in prior studies \cite{no.67}.
}

\section{Explanation of Eq. \eqref{eq:equiobj} and the Proofs of Props. \ref{prop:betterthanpsr} and \ref{prop:controllableerror}}
{\mdseries
\subsection{Explanation of Eq. \eqref{eq:equiobj}}
\label{sec:apdx-explain-obj}
Algorithm \ref{algo1} systematically extracts many subsets, with each subset $\mathscr D$ containing $M$ elements, from $\big\{ {^cp}_i^A \coloneqq p_c(Y_i \in A|F_i) | 1\leq i \leq N, A \in \{\{1\},...,\{L\}\} \big\}$. For each subset $\mathscr D$, the corresponding error in Eq. \eqref{eq:subseterror} is calculated (assuming $\epsilon=0$).
\begin{equation}
\label{eq:subseterror}
    \left|  \frac{    \sum\limits_{(i,A)\in \mathscr D} (1-{^cp_i^A}) 1_{\{Y_i \in A\}}   -    
    \sum\limits_{ (i,A)\in \mathscr D}{^cp_i^A}1_{\{Y_i \in A^{\mathbf C}\}}}{    |\mathscr D| } \right|
    =
    \left| \frac{1}{|\mathscr D|} \sum\limits_{\mathscr D}  \mathds 1_{\{Y_i \in A\}}   
    -
    \frac{1}{|\mathscr D|} \sum\limits_{ \mathscr D}{^cp_i^A} \right|
\end{equation}
Following this, a weighted average of errors across different $\mathscr D$ is computed to produce the learning loss. This procedure can be seen as equivalent to computing the following learning criteria for a specific $\omega$.
\begin{equation}
    \min\limits_{g} \Big\|  \left[ g^l(f_i) \right]_{1\leq i\leq N, 1 \leq l \leq  L} - \left[\mathds 1_{\{Y_i = l\}} \right]_{1\leq i\leq N, 1 \leq l \leq L} \Big\|_{M,\omega}
\end{equation}

\subsection{Proof of the Prop. \ref{prop:betterthanpsr}}
\label{sec:apdx-proof-proposition}
\begin{proposition*}
    For any $\alpha>0$, 
    \begin{equation}
        \Big\| \big[ p_\mu \big] - \big [p_M \big] \Big\|_{M,\omega} \leq \Big\| \big[ p_\mu \big] - \big[ p_{\mathrm{psr}}\big] \Big\|_{M,\omega} + \alpha
    \end{equation}
    holds with high probability (failure probability below $\frac{2}{\alpha\sqrt{M}}$), where $p_{\mu}$ refers to the ground truth classification probability. 
\end{proposition*}
\begin{proof}
For the ease of notation, we abbreviate $\left[\mathds 1_{\{Y_i = l\}} \right]_{ 1\leq i\leq N,1 \leq l\leq L}$ as $[ Y ]$. By the definitions in Eq. \eqref{eq:equiobj} and Eq. \eqref{eq:discretepsr}, it naturally follows that
\begin{equation}
\label{eq:psrcomparison}
    \Big \| \big[ p_{M} \big] - \big[ Y \big]  \Big\|_{M,\omega}  \leq \Big \| \big[ p_{\text{psr}} \big] - \big[ Y \big]  \Big\|_{M,\omega}.
\end{equation}
We begin by establishing a number of inequalities:
\begin{align}
\Big\| \big [p_\mu \big] - \big[ p_M \big] \Big\|_{M,\omega} 
& \leq 
\Big\| \big [ p_\mu \big] - \big[ Y \big] \Big\|_{M,\omega} 
 + \Big\| \big [ Y \big] - \big[ p_M \big] \Big\|_{M,\omega} 
\overset{{\footnotesize \circled{1}}}{\leq}
\Big\| \big [ p_\mu \big] - \big[ Y \big] \Big\|_{M,\omega} 
+\Big\| \big[Y\big] - \big[p_{\text{psr}}\big] \Big\|_{M,\omega}  \label{eq:ineqstart}\\
& \leq
\Big\| \big [ p_\mu \big] - \big[ Y \big] \Big\|_{M,\omega}  
+ \Big(\Big\| \big[ Y\big] - \big[ p_\mu \big] \Big\|_{M,\omega}  
+ \Big\| \big[p_\mu \big] -\big[ p_{\text{psr}} \big] \Big\|_{M,\omega}  \Big)\\
& = \Big\| \big[p_\mu \big] -\big[ p_{\text{psr}} \big] \Big\|_{M,\omega}  
+ 2\Big\| \big[ Y\big] - \big[ p_\mu \big] \Big\|_{M,\omega}  
\label{eq:ineqend}
\end{align}
Generally, for a restricted mapping family $\Theta$, Eq. \eqref{eq:psrcomparison}, or corresponding {\footnotesize $\circled{1}$}, becomes strict inequality, i.e, $\big \| \big[ p_{M} \big] - \big[ Y \big]  \big\|_{M,\omega} < \big \| \big[ p_{\text{psr}} \big] - \big[ Y \big]  \big\|_{M,\omega} $, leading to the conclusion that 
\begin{equation}
\label{eq:comparewithpsr}
    \big\| \big [p_\mu \big] - \big[ p_M \big] \big\|_{M,\omega}  < \big\| \big[p_\mu \big] -\big[ p_{\text{psr}} \big] \big\|_{M,\omega} 
+ 2\big\| \big[ Y\big] - \big[ p_\mu \big] \big\|_{M,\omega}.
\end{equation}

Subsequently, we derive the following series of inequalities, where the establishment of inequality {\footnotesize $\circled{2}$} is obtained by Chebyshev's inequality, inequality {\footnotesize \circled{3}} by Jensen's inequality, and equality {\footnotesize \circled{4}} through independence assumption, respectively.
\begin{align}
    p\Big(  \Big\| [Y] - [p_\mu] \Big\|_{M,\omega} \geq \alpha \Big)
    & \overset{\footnotesize \circled{2}}{\leq}
    \frac{1}{\alpha}\mathbb E \Big\| [Y] - [p_\mu] \Big\|_{M,\omega}
    = \frac{1}{\alpha} 
    \sum_{\mathscr D \in \mathbb D} \omega(\mathscr D)
    \mathbb E \Bigg| \frac{\sum_{(i,l) \in \mathscr D}\big[ Y\big]_{i,l}}{M}  - 
    \frac{\sum_{(i,l) \in \mathscr D}\big[p_\mu\big]_{i,l}}{M} \Bigg| \label{eq:ineqstart1}\\
    & \overset{\footnotesize \circled{3}}{\leq} \frac{1}{\alpha}
    \sum_{\mathscr D \in \mathbb D}  \omega(\mathscr D)
    \sqrt{ \mathbb E
    \Bigg| \frac{\sum_{(i,l) \in \mathscr D} \Big(\big[ Y\big]_{i,l} - \big[p_\mu\big]_{i,l} \Big)}{M} \Bigg|^2} \\
    & = \frac{1}{\alpha} \sum_{\mathscr D \in \mathbb D} \omega(\mathscr D)
    \sqrt{ \frac{1}{M^2} \mathbb E \bigg|
    \sum\limits_{i}\sum\limits_{l \in \mathscr D_i}
    \Big(\big[ Y\big]_{i,l} - \big[p_\mu\big]_{i,l} \Big)
    \bigg|^2} \\
    & \overset{\footnotesize \circled{4}}{=}
    \frac{1}{\alpha} 
    \sum_{\mathscr D \in \mathbb D}  \omega(\mathscr D)
    \sqrt{ \frac{1}{M^2}
    \sum\limits_{i} \mathbb E
    \bigg| \sum\limits_{l \in \mathscr D_i}
    \Big(\big[ Y\big]_{i,l} - \big[p_\mu\big]_{i,l} \Big) \bigg|^2
    } \\
    & = \frac{1}{\alpha} 
    \sum_{\mathscr D \in \mathbb D} \omega(\mathscr D)
    \sqrt{ \frac{1}{M^2}
    \sum\limits_{i} \mathbb E \Big|
    \mathds 1_{\{Y_i \in \mathscr D_i\}} - p_\mu(Y_i \in \mathscr D_i \mid F_i)
    \Big|^2} \\
    & \leq \frac{1}{\alpha} 
    \sum_{\mathscr D \in \mathbb D} \omega(\mathscr D)
    \sqrt{
    \frac{1}{M^2} \sum\limits_{i}
    \mathds 1_{\{\mathscr D_i \neq \varnothing\}}}
    \leq 
    \frac{1}{\alpha} 
    \sum_{\mathscr D \in \mathbb D} \omega(\mathscr D)
    \sqrt{ \frac{1}{M^2} \cdot M}
    \overset{\footnotesize \circled{5}}{=}
    \frac{1}{\alpha \sqrt{M}},
    \label{eq:ineqend1}
\end{align}
where $\mathscr  D_i = \{(i',l) \in \mathscr D \mid i'=i\}$, leading to $\sum\nolimits_i \mathds 1_{\{ \mathscr D_i \neq \varnothing\}} \leq M$. Equality {\footnotesize \circled{5}} holds under the condition $\sum_{\mathscr D \in \mathbb D} \omega(\mathscr D) =1$, which can be assumed without loss of generality.
The above inequalities suggest that, for sufficiently large $M$, each of the above terms can be made significantly small. 

By integrating the conclusion of Eq. \eqref{eq:comparewithpsr} with that from Eq. \eqref{eq:ineqstart1}-\eqref{eq:ineqend1}, we deduce
\begin{equation}
    \mathbb E \Big\| \big [p_\mu \big] - \big[ p_M \big] \Big\|_M  < 
    \mathbb E \Big\| \big[p_\mu \big] -\big[ p_{\text{psr}} \big] \Big\|_M 
    + \frac{2}{\sqrt{M}}
\end{equation}
and
\begin{equation}
    p\Big(\Big\| \big[ p_\mu \big] - \big [p_M \big] \Big\|_M - \Big\| \big[ p_\mu \big] - \big[ p_{\text{psr}}\big] \Big\|_M < \alpha \Big)
    \geq p\Big(
    2 \Big\|  \big[ Y \big] - \big[ p_\mu \big] \Big\| \leq \alpha
    \Big)
    \geq 
    1-\frac{2}{\alpha\sqrt{M}}
\end{equation}
\end{proof}

\subsection{Proof of the Prop. \ref{prop:controllableerror}}
\label{sec:apdx-proof-proposition1}
\begin{proposition*}
    For any $\alpha>0$, 
    \begin{equation}
        \Big\| \big[ p_\mu \big] - \big [p_M \big] \Big\|_{M,\omega}  \leq  \Xi + \alpha
    \end{equation}
    holds with high probability (failure probability below $\frac{1}{\alpha\sqrt{M}}$), where $\Xi$  reflects the learning loss.
\end{proposition*}
\begin{proof}
We begin by presenting
\begin{equation}
\label{eq:controlerror}
    \Big\| \big[ p_\mu \big] - \big [p_M \big] \Big\|_{M,\omega} 
    \leq 
    \Big\| \big[ p_\mu \big] - \big[ Y] \Big\|_{M,\omega} + \Big\| \big[ Y \big] - \big[ p_M] \Big\|_{M,\omega}
\end{equation}
As outlined in Section \ref{method:sec5} and the Appendix \ref{sec:apdx-explain-obj} regarding the explanation of Eq. \eqref{eq:equiobj}, the second term on the right side corresponds to the learning loss, which attains its optimal value $\Xi$ after optimization within the mapping family $\Theta$. 
According to the Eq. \eqref{eq:ineqstart1} - \eqref{eq:ineqend1}, we have
\begin{equation}
\label{eq:controlerror1}
    p\Big(\Big\| \big[ p_\mu \big] - \big [Y \big] \Big\|_{M,\omega} < \alpha \Big) \geq 
    1-\frac{1}{\alpha\sqrt{M}}.
\end{equation}
By integrating Eq. \eqref{eq:controlerror} with Eq. \eqref{eq:controlerror1}, we arrive at
\begin{equation}
    p\Big(\Big\| \big[ p_\mu \big] - \big [p_M \big] \Big\|_{M,\omega} 
    <
    \Xi + \alpha
    \Big) \geq
    1 - \frac{1}{\alpha \sqrt{M}}.
\end{equation}
\end{proof}

}

\section{Dataset Summary, Evaluation Code Sources}
\label{sec:apdx-datasetsummary}

\subsection{Dataset Summary}
\label{subsec:apdx-dataset}
\begin{table}[H]
\caption{Statistics of the Experimental Datasets}
\label{tab:dataset}
\begin{center}
\begin{tabular}{cccc}
\hline
Dataset   & \#Classes & \makecell{Training set size \\ (logit-label pairs)} & \makecell{Test set size\\(logit-label pairs)} \\ \hline
CIFAR-10  & 10        & 5000                 & 10000         \\
SVHN      & 10        & 6000                 & 26032         \\
CIFAR-100 & 100       & 5000                 & 10000         \\
CARS      & 196       & 4020                 & 4020          \\
BIRDS     & 200       & 2897                 & 2897          \\
ImageNet  & 1000      & 25000                & 25000         \\ \hline
\end{tabular}
\end{center}
\end{table}

In the established benchmark \cite{no.71,no.67}, 14 calibration tasks were included for the datasets, each for calibrating a pretrained network classifier. 
The networks employed by \cite{no.71} include ResNet110 \cite{no.197}, WideResNet32\cite{no.199}, and DenseNet40\cite{no.200} for CIFAR10 and CIFAR100; ResNet152 SD \cite{no.198} for SVHN; and DenseNet161 \cite{no.200} and ResNet152 \cite{no.197} for ImageNet. Additionally, ResNet50 NTSNet \cite{no.202} and PNASNet5 Large \cite{no.201} were employed by \cite{no.67} as calibration tasks for the BIRDS and ImageNet datasets, respectively. Study \cite{no.67} also provided tasks involving ResNet classifiers trained on the CARS dataset, where `pre' indicates the networks initialized with ImageNet weights.

\subsection{Evaluation Code Sources}
\label{subsec:apdx-evacodesource}
Regarding evaluator implementation, 
the metrics 
$\mathrm{ECE}_{r=1}^s$ \cite{no.50}, $\mathrm{ECE}_{r=2}^s$ \cite{no.50}, KDE-ECE \cite{no.70}, KS error \cite{no.60}, SKCE \cite{no.79}, DKDE-CE \cite{no.52}, $t\mathrm{CWECE}$ \cite{no.1}, and $t\mathrm{CWECE}^k$ \cite{no.1} were computed using the source code 
provided by the respective studies. The metrics $\mathrm{ECE}^{\mathrm{ew}}$ and MMCE were sourced from \cite{no.65}, $\mathrm{ECE}^{\mathrm{em}}$ from \cite{no.64}, and dECE, $\mathrm{ECE}_{r=2}$, and $\mathrm{CWECE}_{r=2}$ from \cite{no.75}. The metrics $\mathrm{ECE}^{\mathrm{ew}}$, $\mathrm{CWECE}_{s}$, and $\mathrm{CWECE}_{a}$ were implemented using code from \cite{no.67}. All metrics involving binning used the default value of 15 bins, following previous studies \cite{no.3,no.6,no.12,no.13,no.23,no.24,no.34,no.40,no.42,no.46,no.50,no.51,no.52,no.53,no.54,no.55,no.58,no.60,no.61,no.64,no.67,no.70,no.78,no.90,no.91,no.93,no.95,no.101,no.109,no.111,no.129,no.147,no.150}. Given the linearity between $\mathrm{CWECE}_{s}$ and $\mathrm{CWECE}_{a}$ under same binning, we reduced the bin number for $\mathrm{CWECE}_{s}$ by one to ensure metric diversity.

\section{Loss Curves of the Proposed Method}
\label{sec:apdx-losscurves}
In some tasks, the loss initially increases before decreasing. This may be attributed to the fact that the pretrained model was not originally trained with calibration objectives, resulting in poorly calibrated predictions at the start of post-hoc recalibration. During the early training stages, these poorly calibrated predictions can lead to large fluctuations and a temporary increase in loss. However, as training progresses and the model becomes better calibrated, the loss gradually decreases and stabilizes.

\begin{figure}[htb]
    \centering
    \begin{minipage}{0.32\textwidth}
        \includegraphics[width=\linewidth]{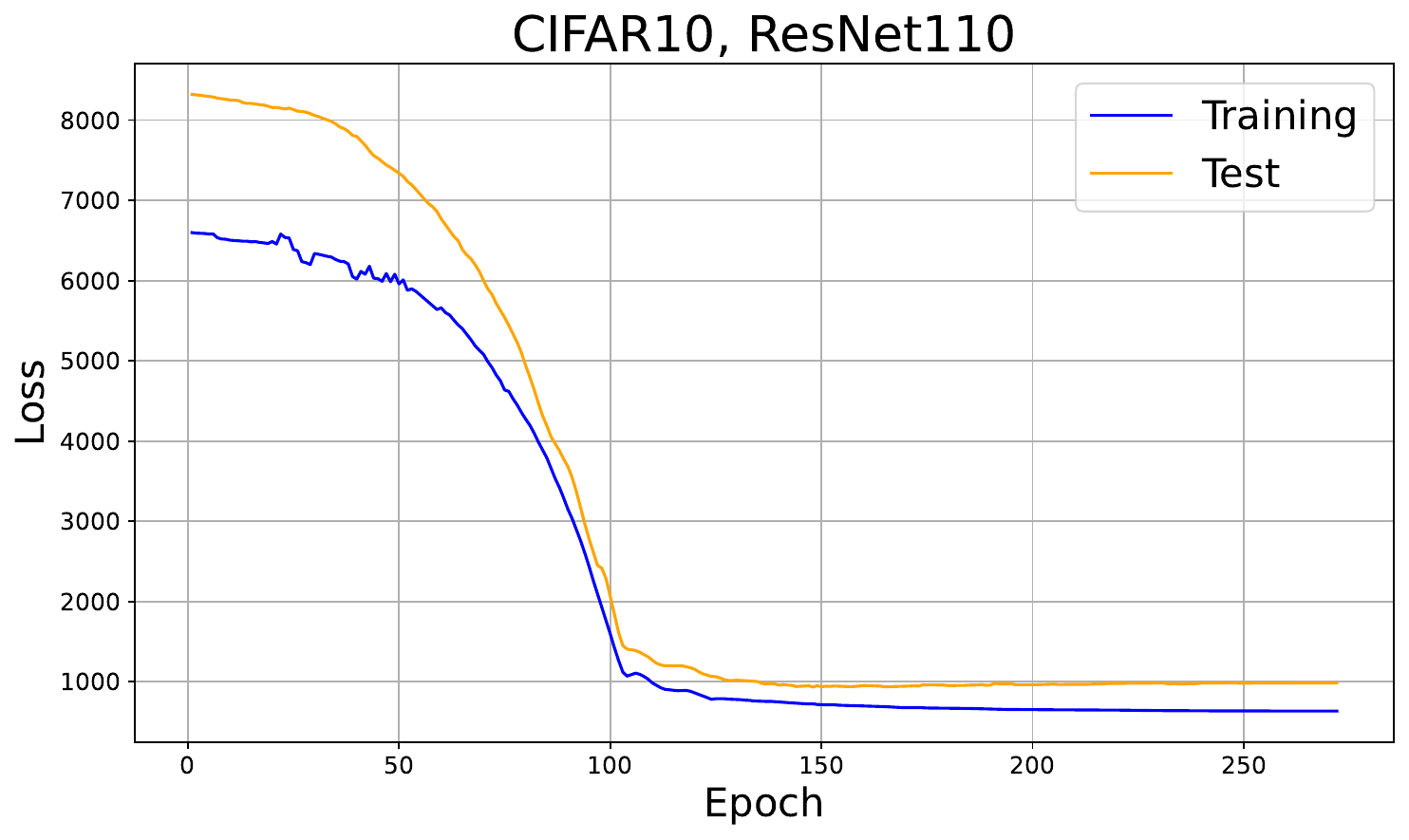}
    \end{minipage}
    \begin{minipage}{0.32\textwidth}
        \includegraphics[width=\linewidth]{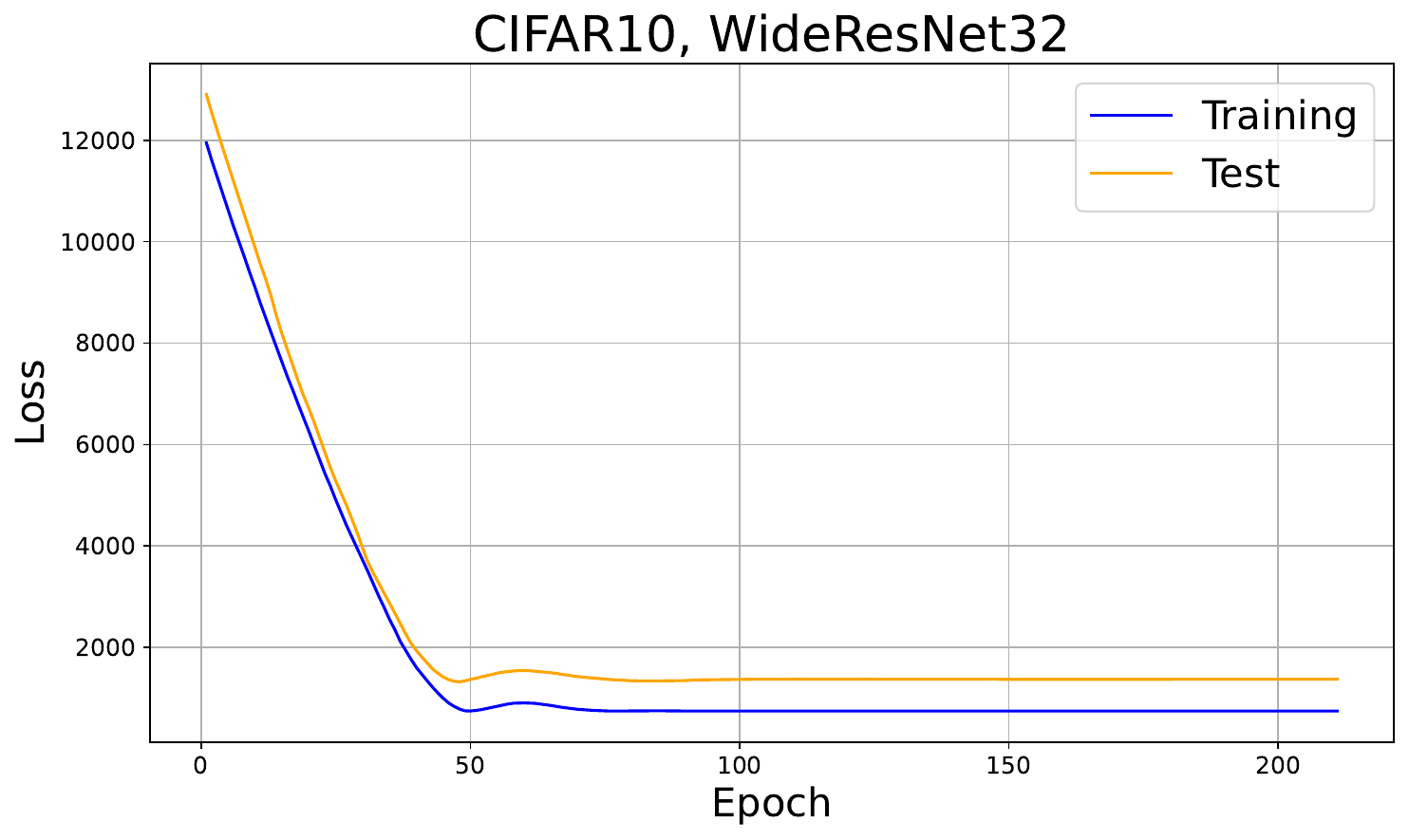}
    \end{minipage}
    \begin{minipage}{0.32\textwidth}
        \includegraphics[width=\linewidth]{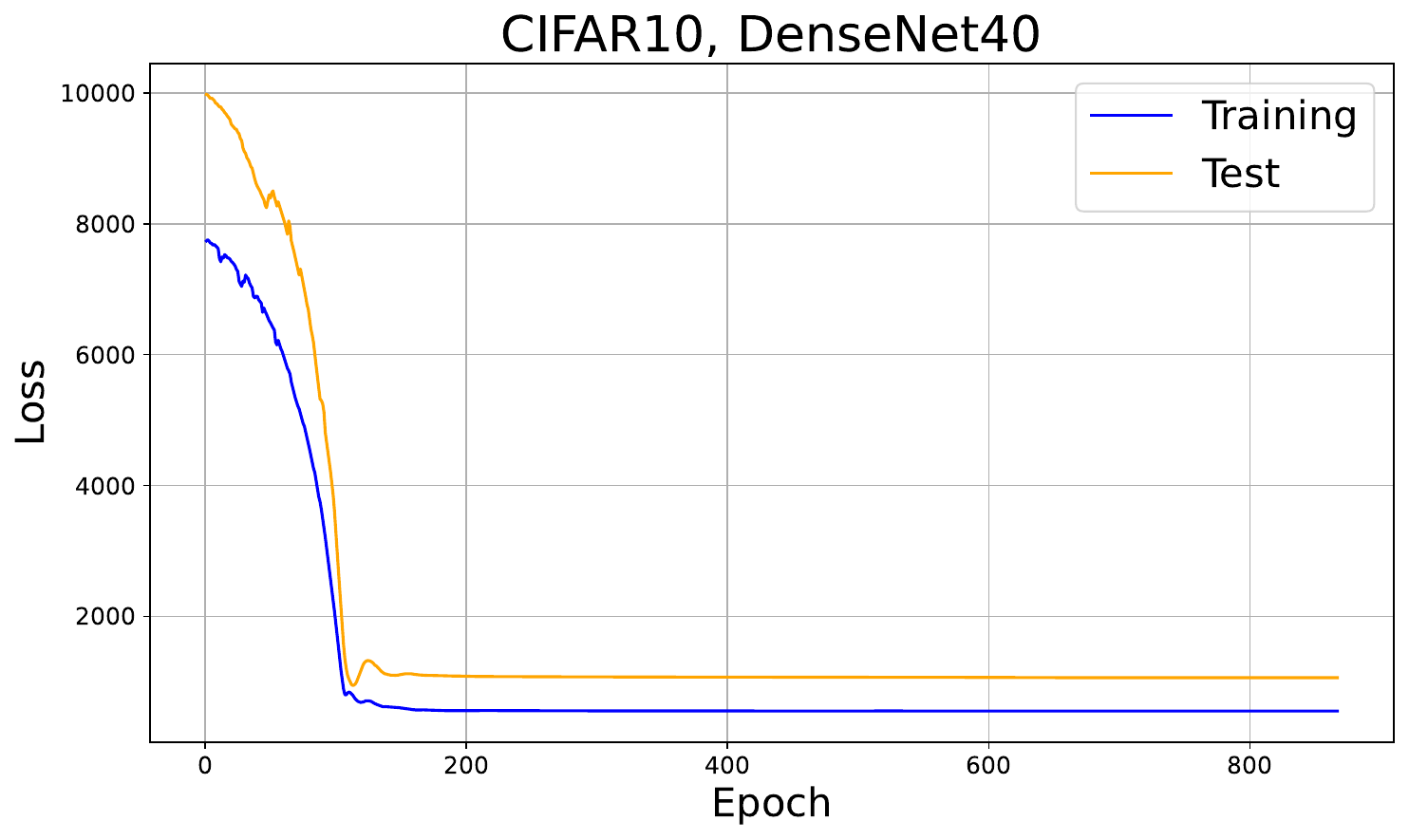}
    \end{minipage}  
    \begin{minipage}{0.32\textwidth}
        \includegraphics[width=\linewidth]{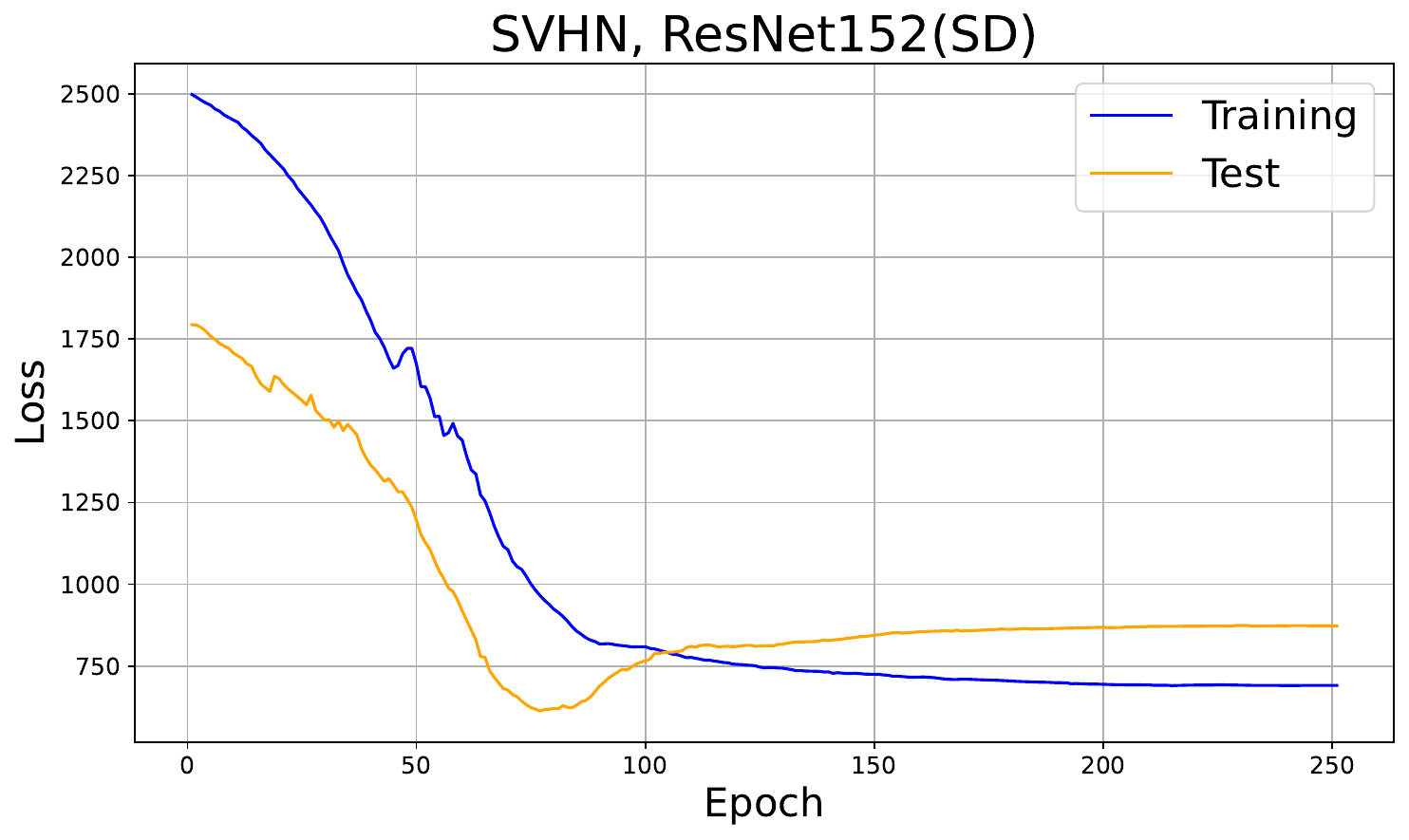}
    \end{minipage}   
    \begin{minipage}{0.32\textwidth}
        \includegraphics[width=\linewidth]{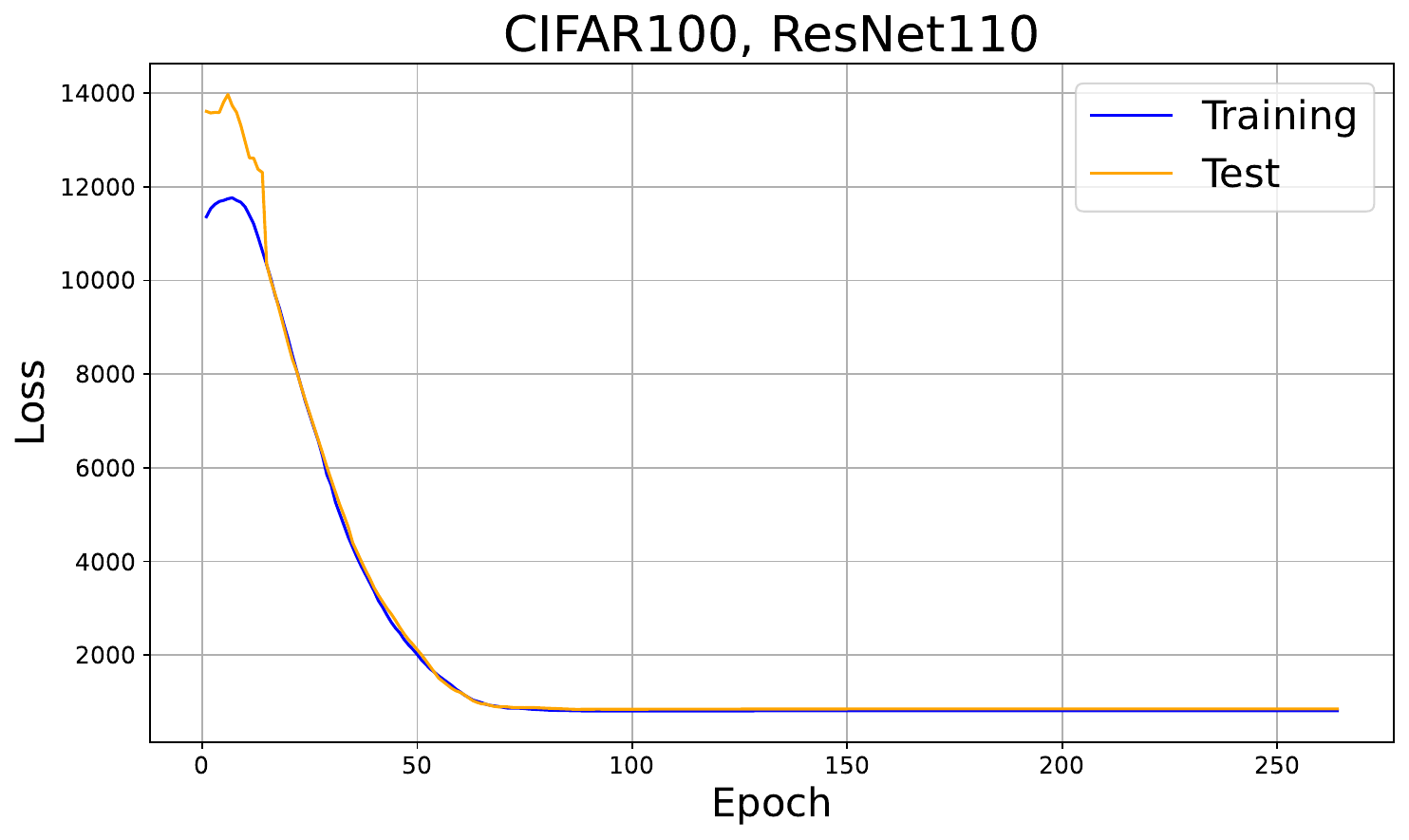}
    \end{minipage}  
    \begin{minipage}{0.32\textwidth}
        \includegraphics[width=\linewidth]{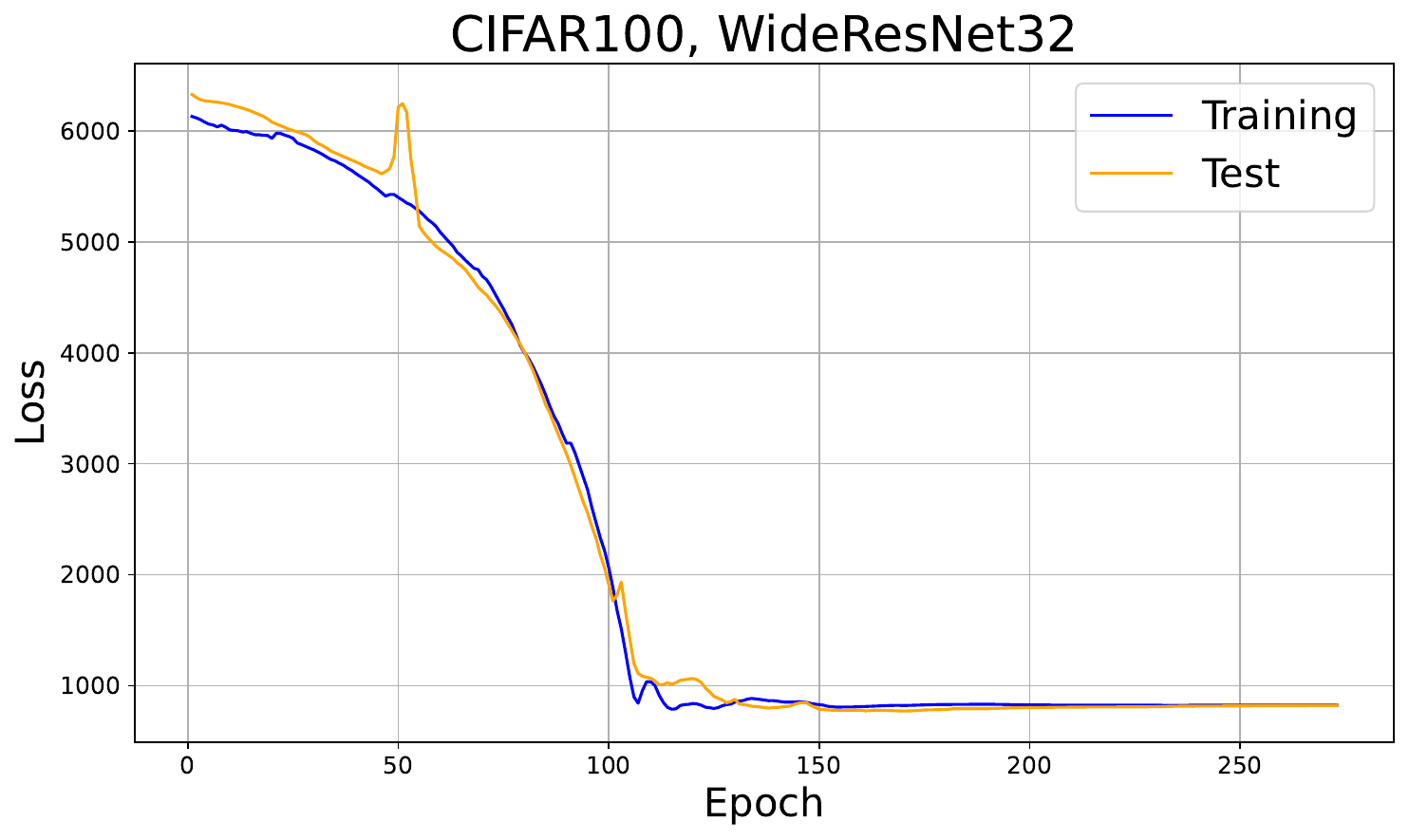}
    \end{minipage}   
    \begin{minipage}{0.32\textwidth}
        \includegraphics[width=\linewidth]{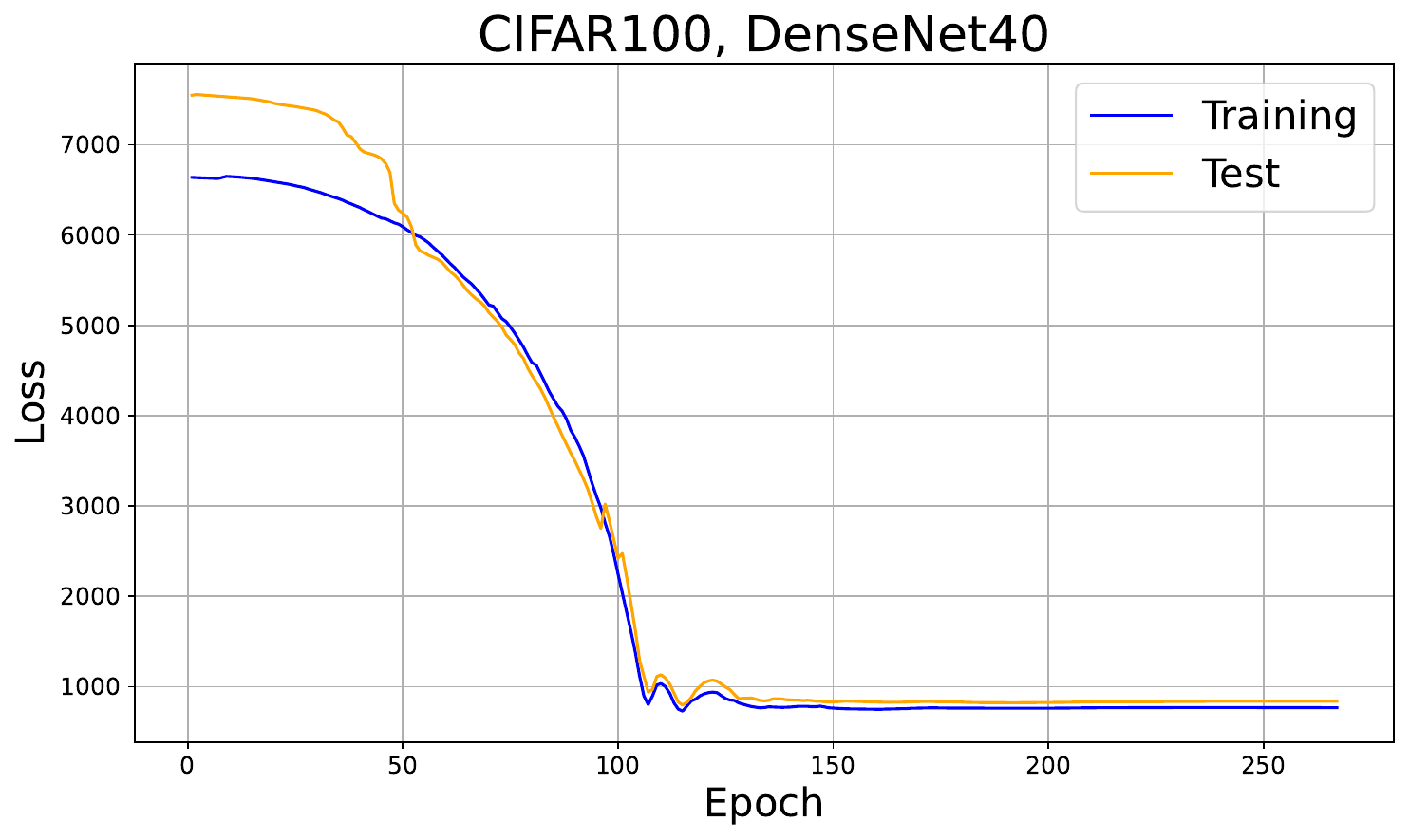}
    \end{minipage}   
    \begin{minipage}{0.32\textwidth}
        \includegraphics[width=\linewidth]{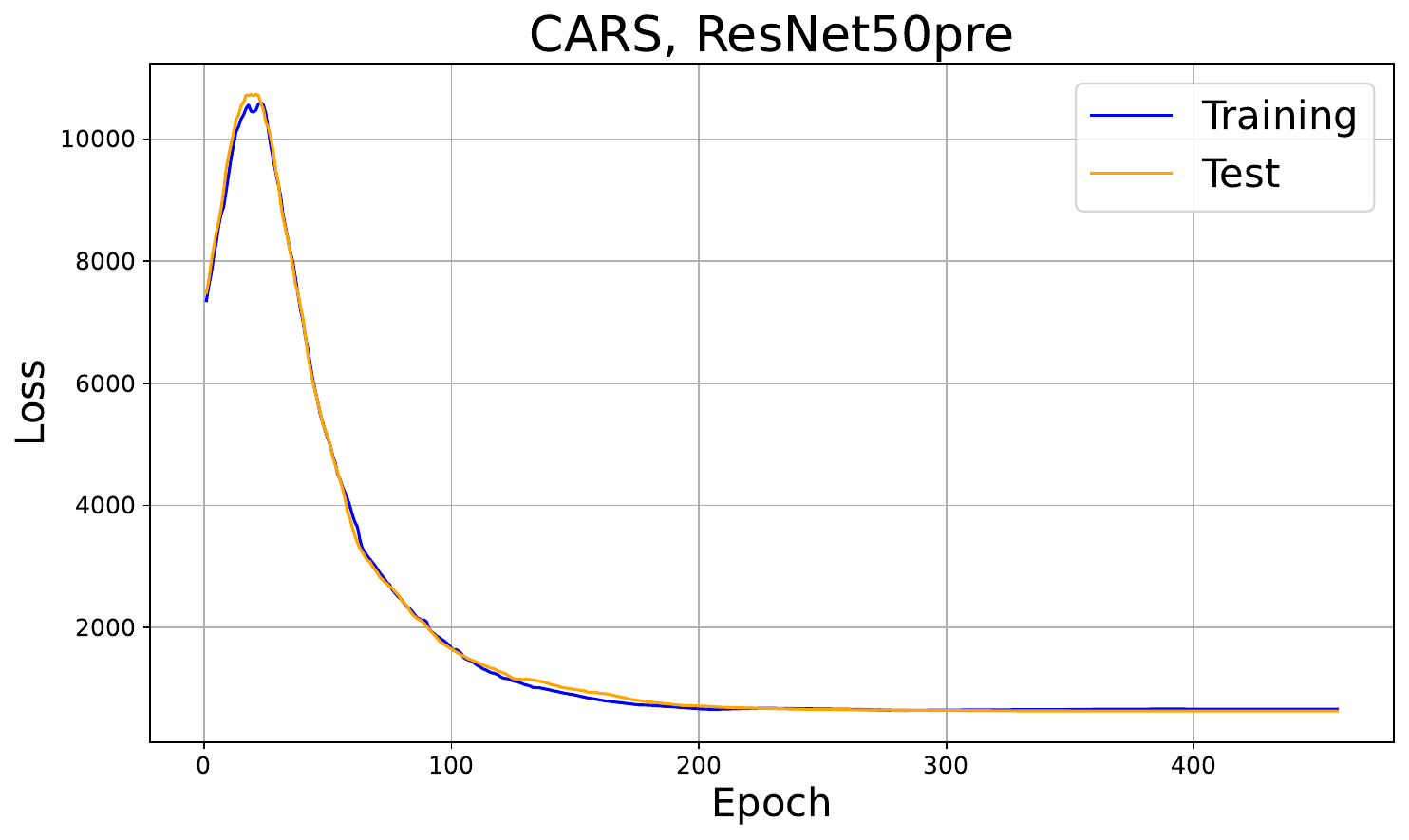}
    \end{minipage}     
    \begin{minipage}{0.32\textwidth}
        \includegraphics[width=\linewidth]{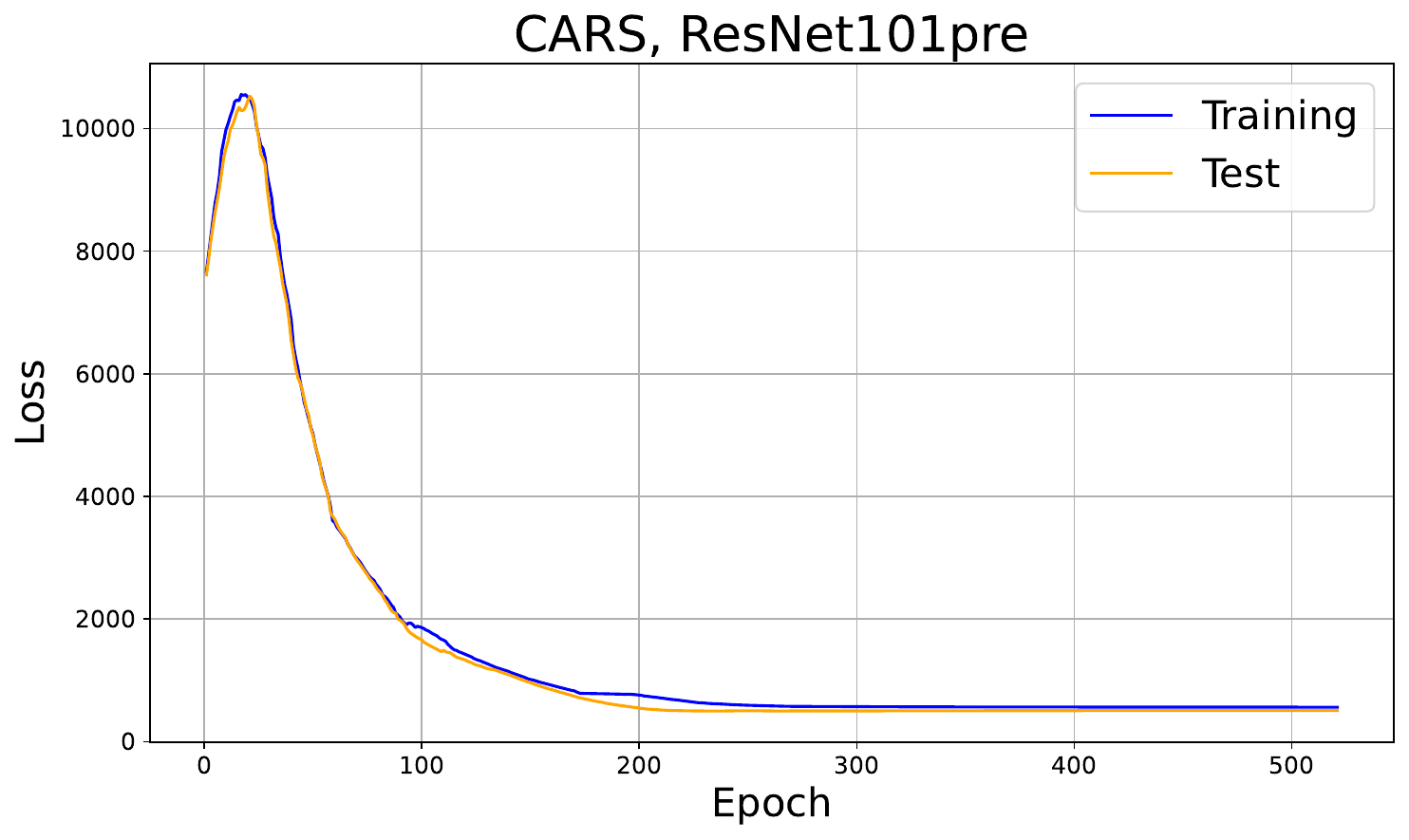}
    \end{minipage}      
    \begin{minipage}{0.32\textwidth}
        \includegraphics[width=\linewidth]{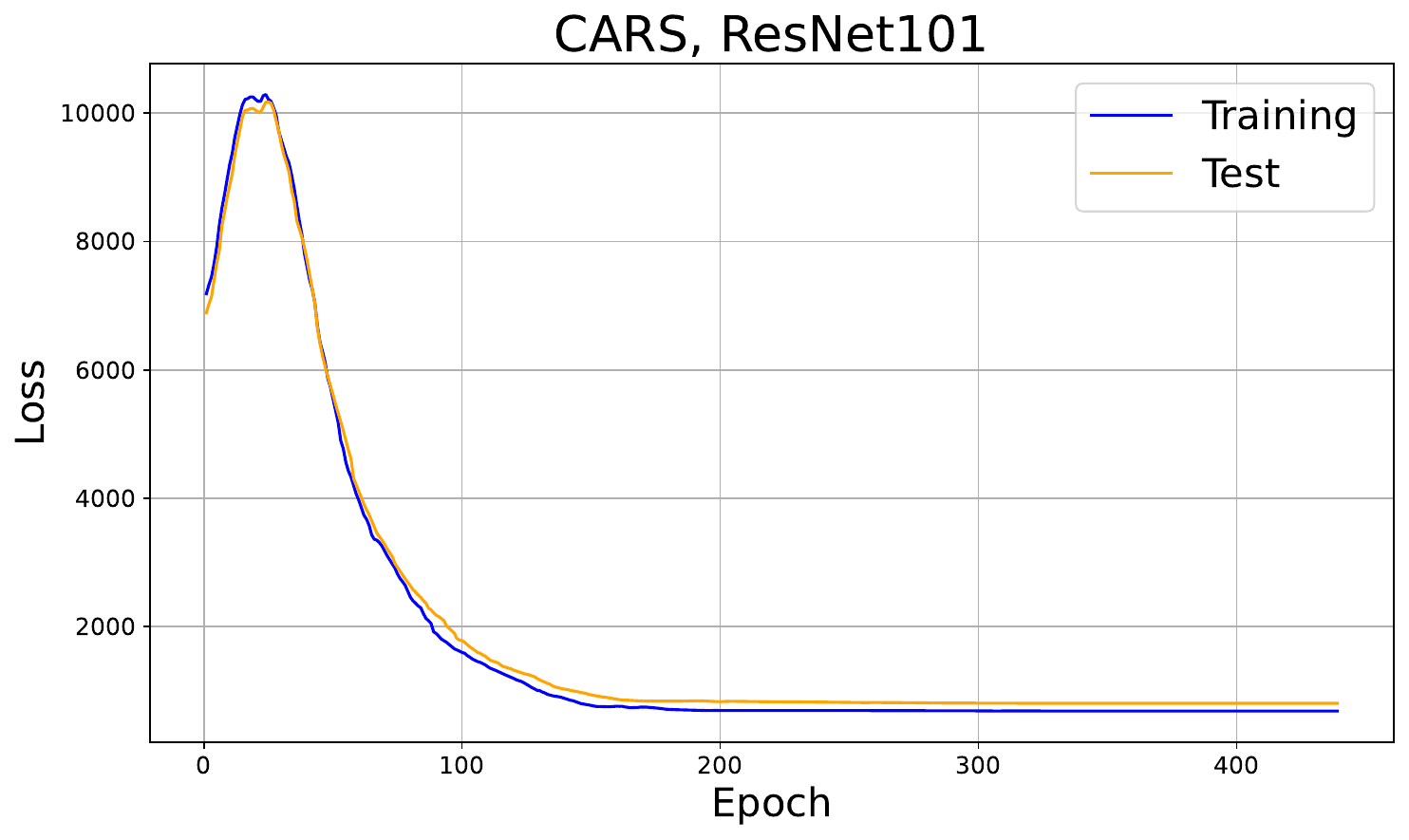}
    \end{minipage}   
    \begin{minipage}{0.32\textwidth}
        \includegraphics[width=\linewidth]{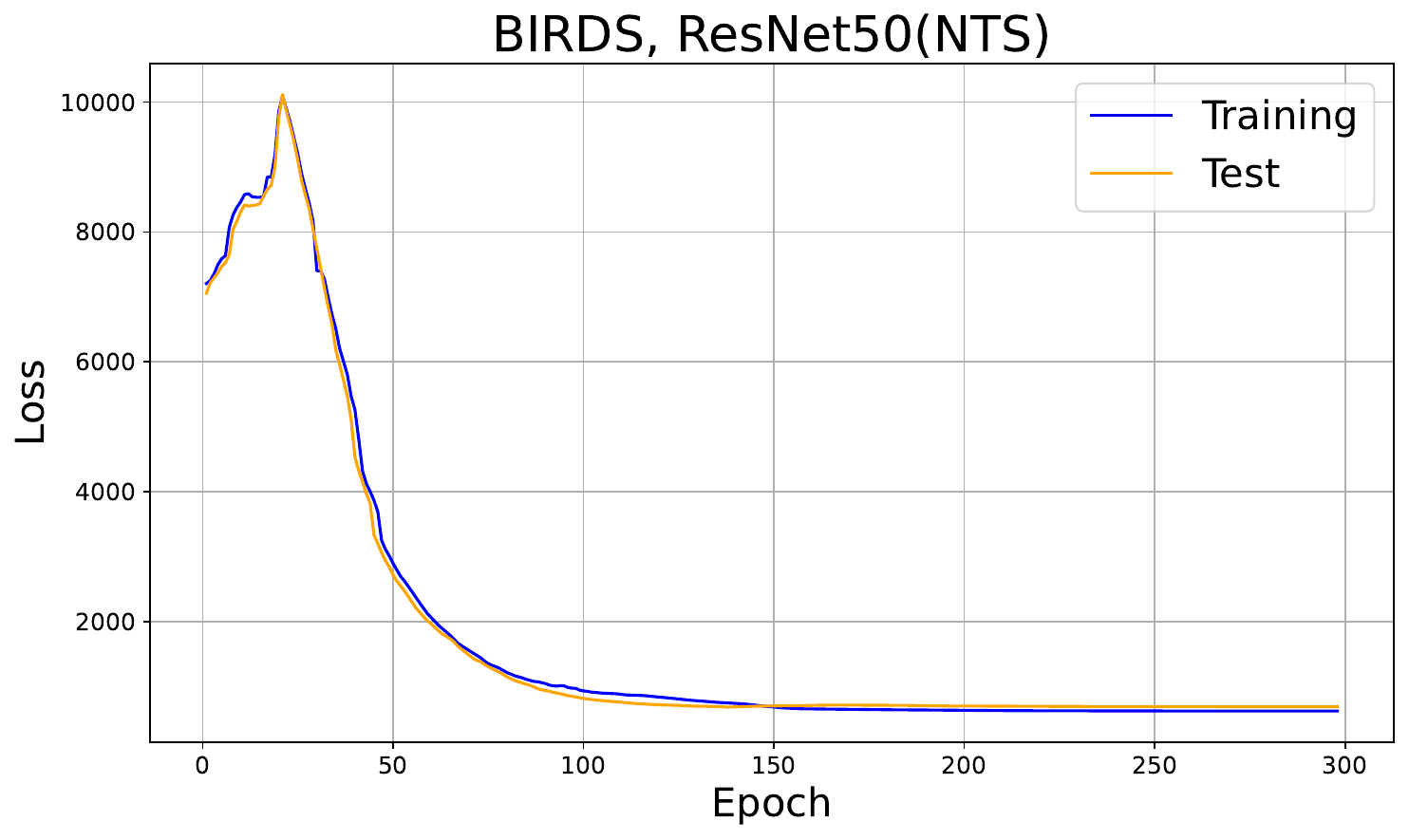}
    \end{minipage}       
    \begin{minipage}{0.32\textwidth}
        \includegraphics[width=\linewidth]{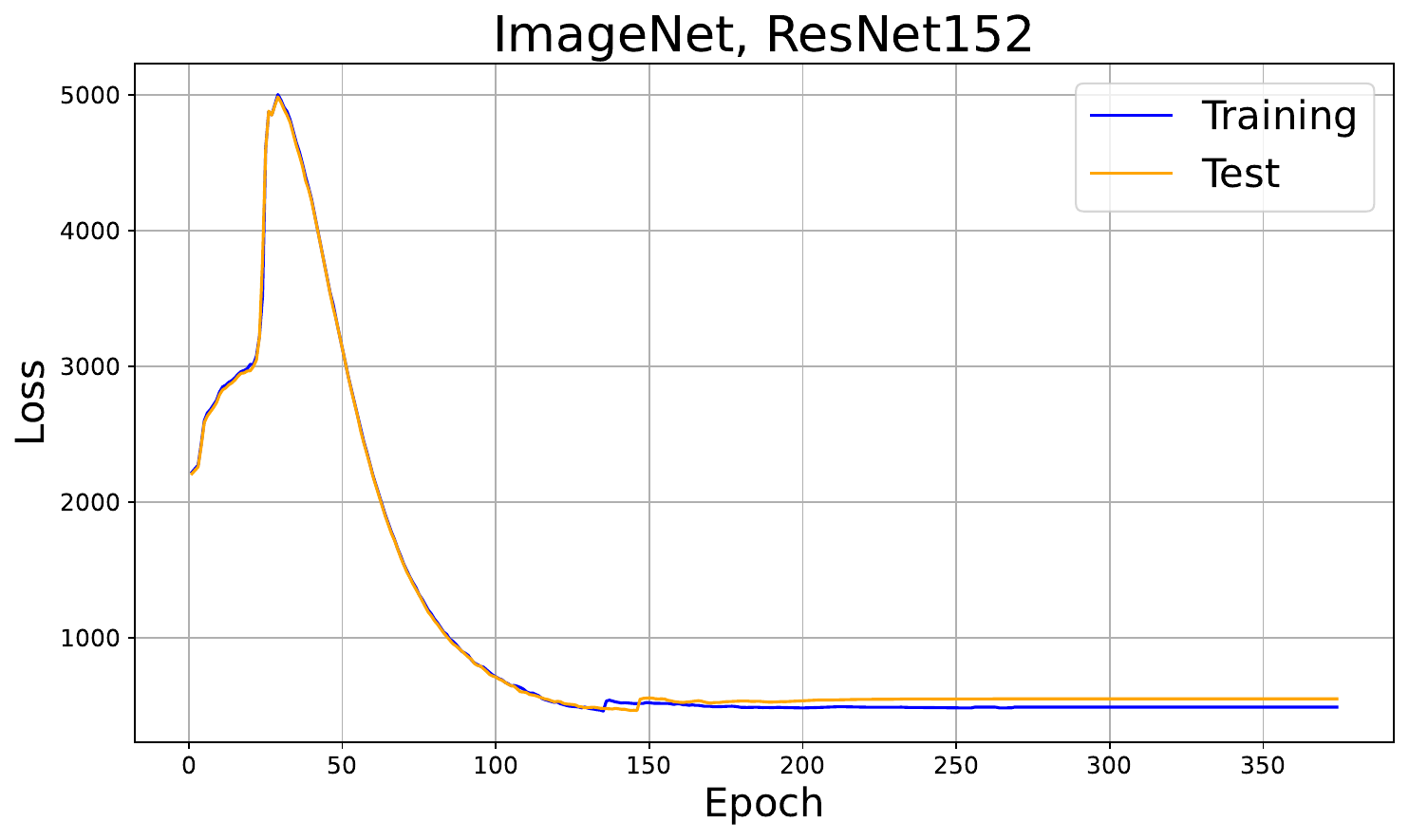}
    \end{minipage}       
    \caption{Training and test loss curves of the proposed method - I}
\end{figure}

\begin{figure}[H]
    \centering        
    \begin{minipage}{0.32\textwidth}
        \includegraphics[width=\linewidth]{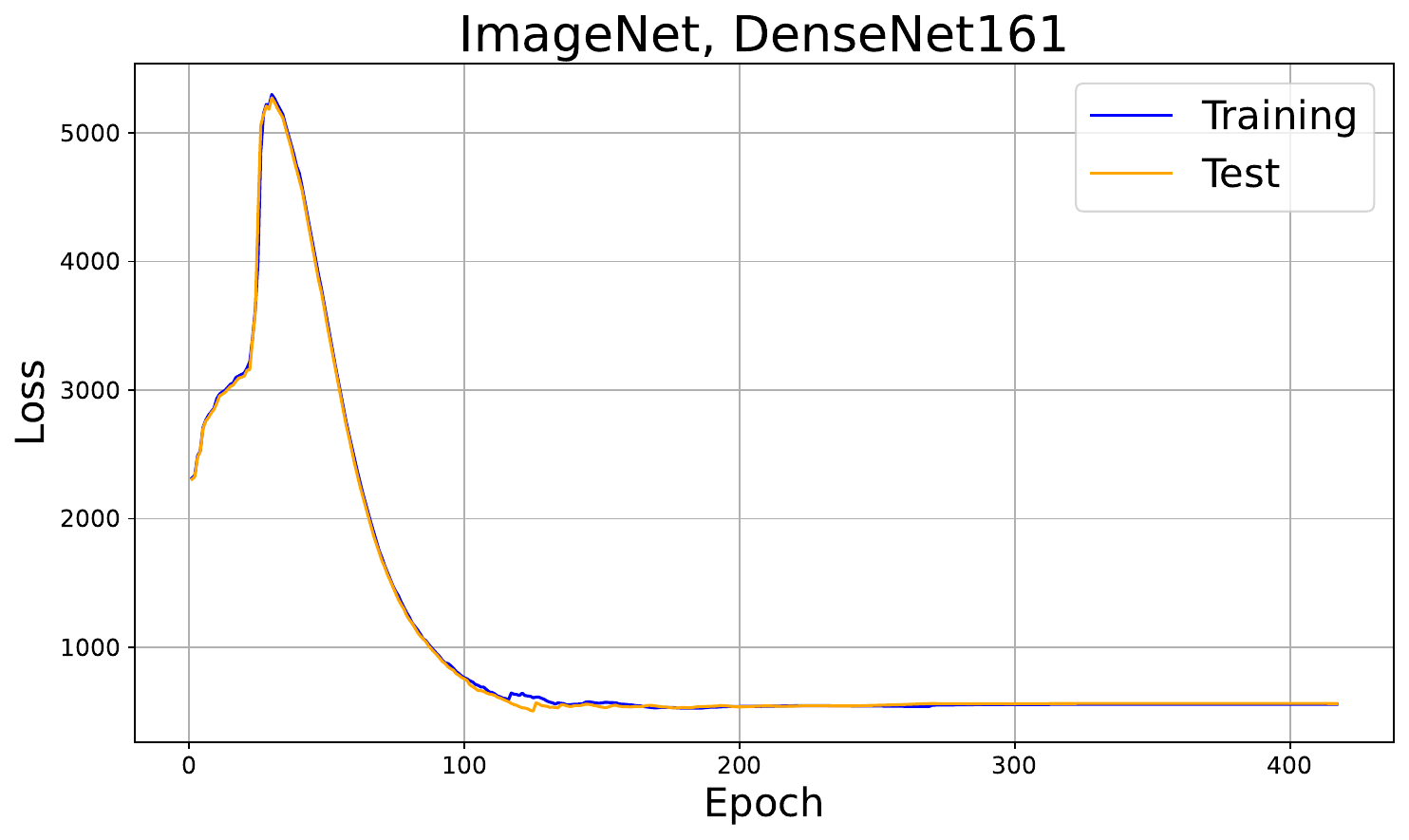}
    \end{minipage}   
    \begin{minipage}{0.32\textwidth}
        \includegraphics[width=\linewidth]{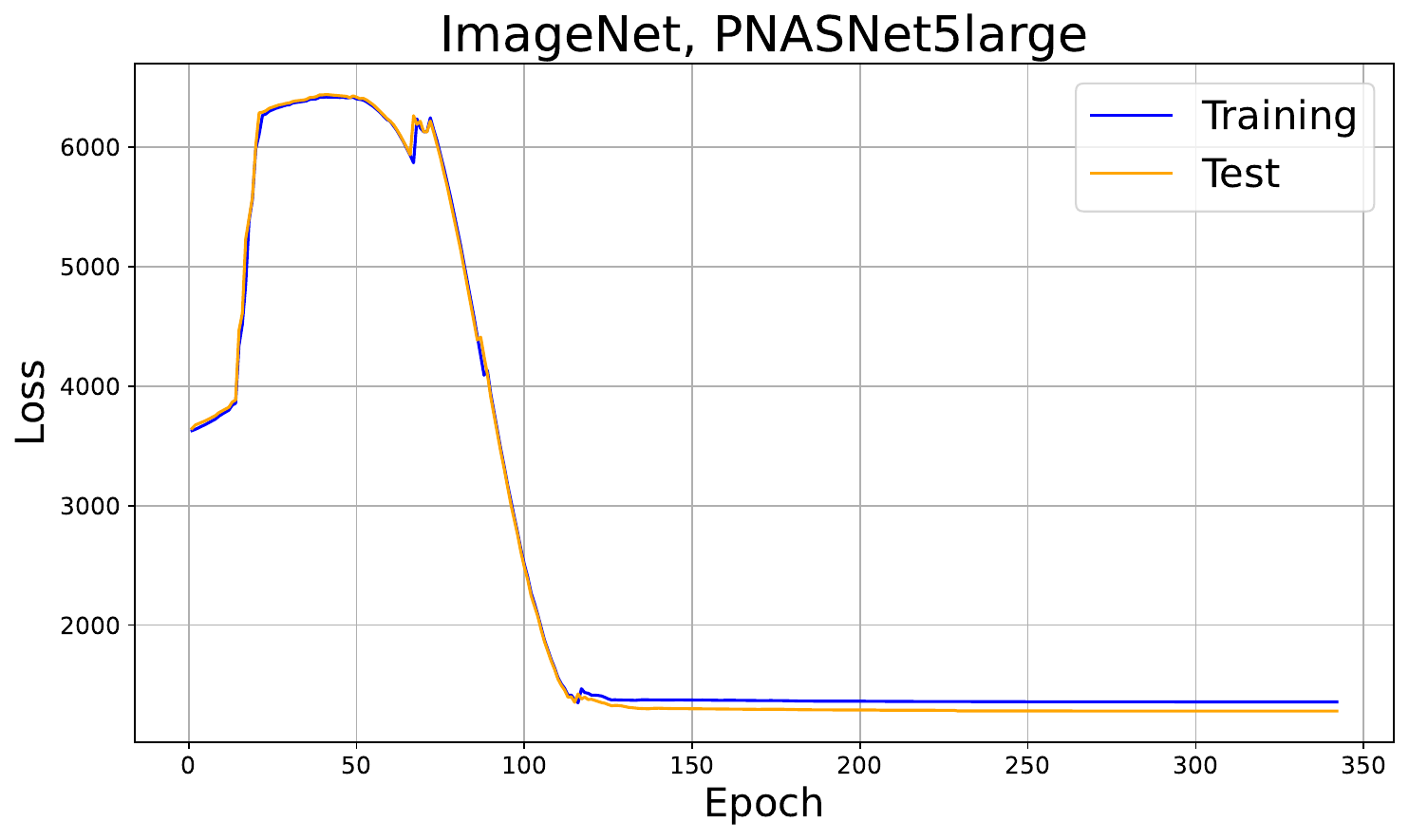}
    \end{minipage}       
    \begin{minipage}{0.32\textwidth}
        \includegraphics[width=\linewidth]{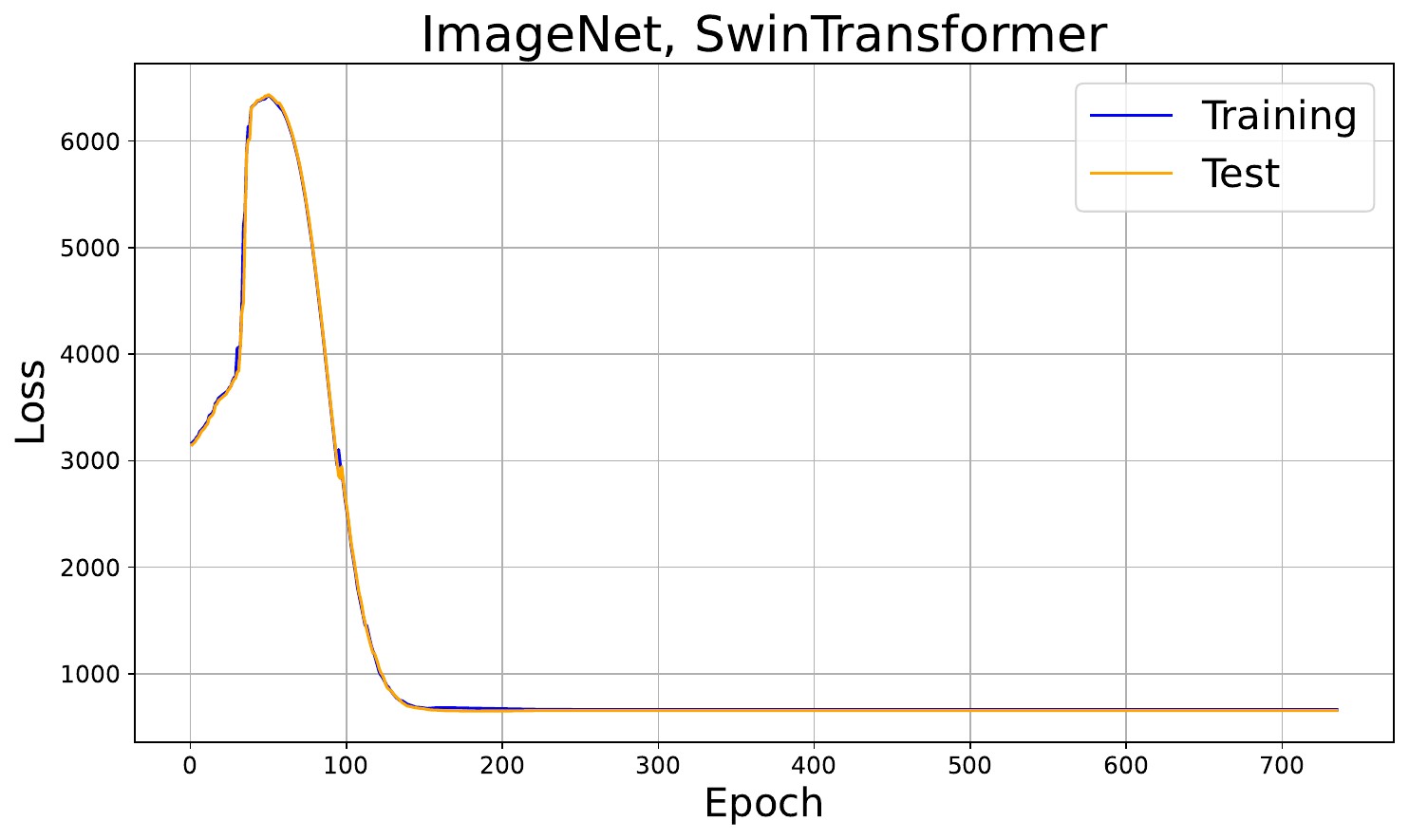}
    \end{minipage}           
    \caption{Training and test loss curves of the proposed method - II}
\end{figure}

\section{Visualized Performance Comparison with Existing Methods}
\label{sec:apdx-visulaizationcomparison}
{\mdseries

\noindent

\begin{landscape}
\begin{figure}[ht]
\centering
\includegraphics[width=1.25\textwidth]{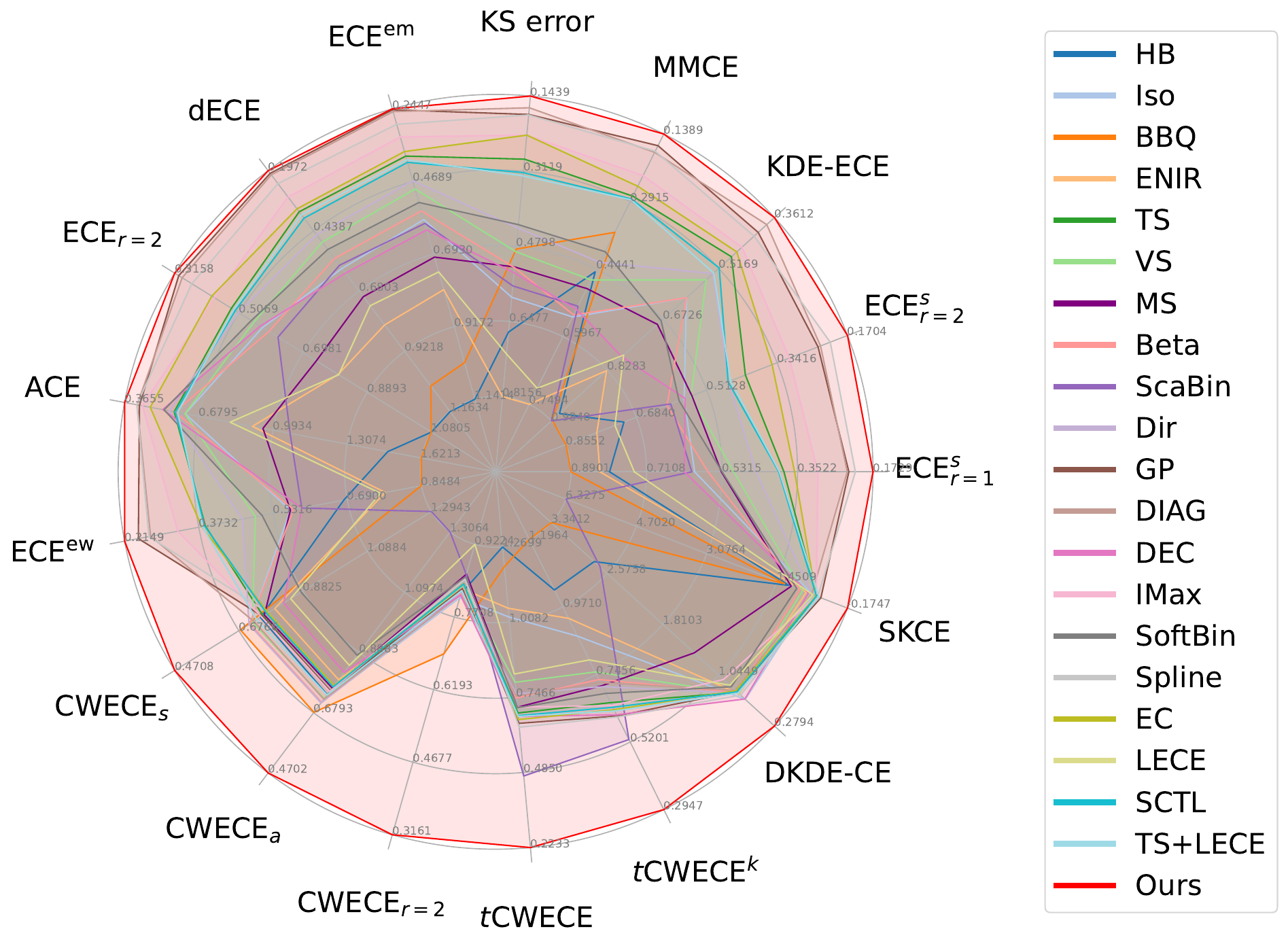}
\caption{Overall comparison by average relative calibration error (ARE) across all metrics}
\label{fig:image1}
\end{figure}

\begin{figure}[ht]
\centering
\includegraphics[width=1.25\textwidth]{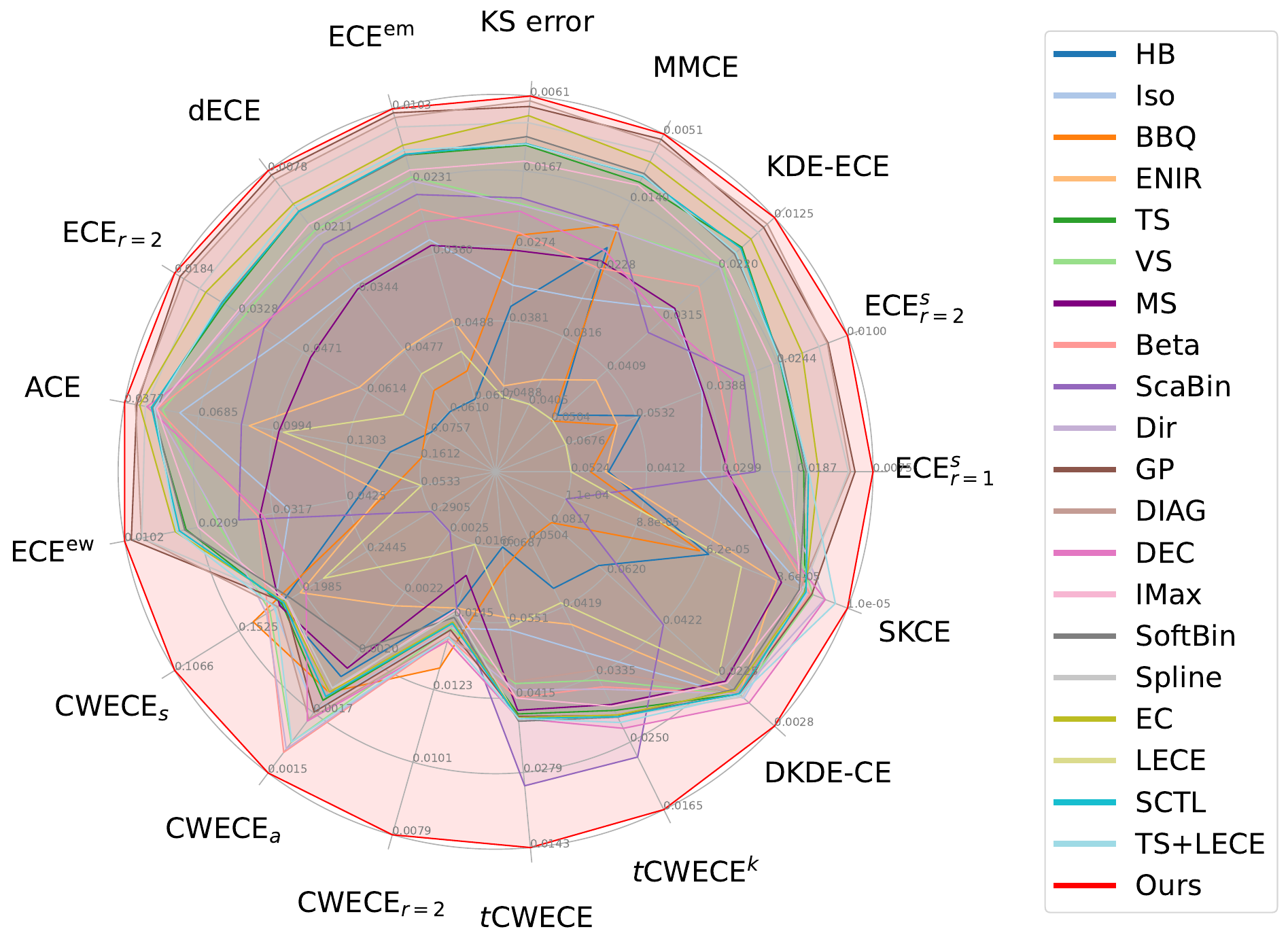}
\caption{Overall comparison by average calibration error (AE) across all metrics}
\label{fig:image2}
\end{figure}
\end{landscape}

\begin{figure}[htb]
    \centering
    \begin{minipage}{0.48\textwidth}
        \includegraphics[width=\linewidth]{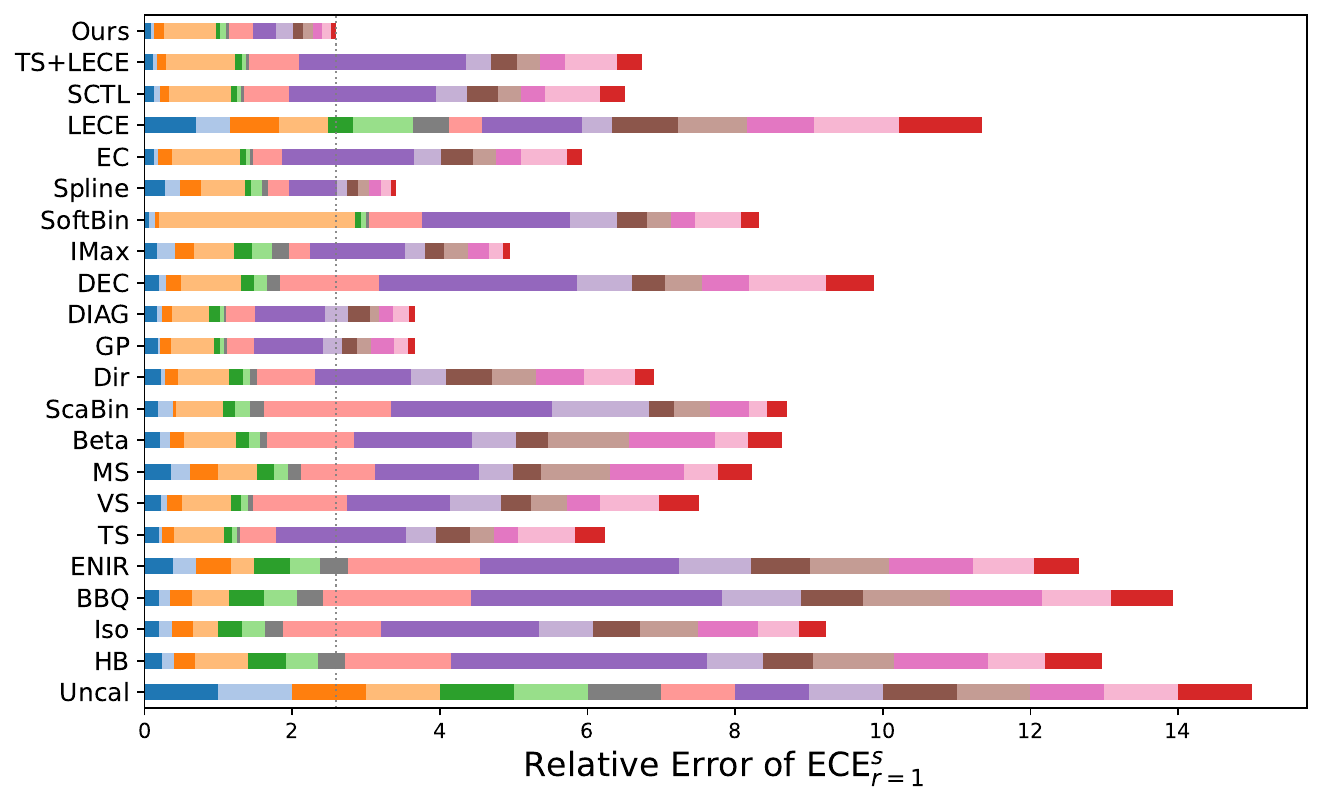}
    \end{minipage}
    \begin{minipage}{0.48\textwidth}
        \includegraphics[width=\linewidth]{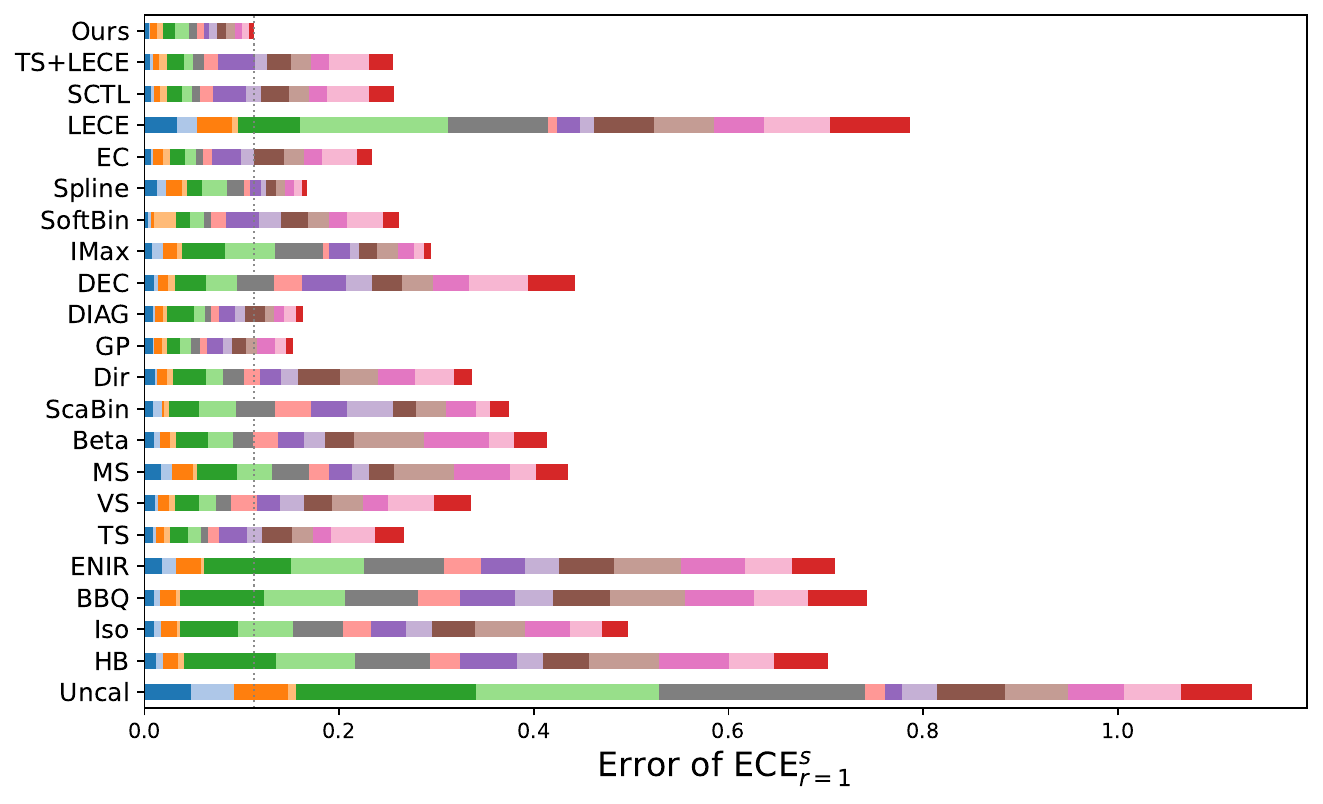}
    \end{minipage}
    \\[5pt]
    \begin{minipage}{0.48\textwidth}
        \includegraphics[width=\linewidth]{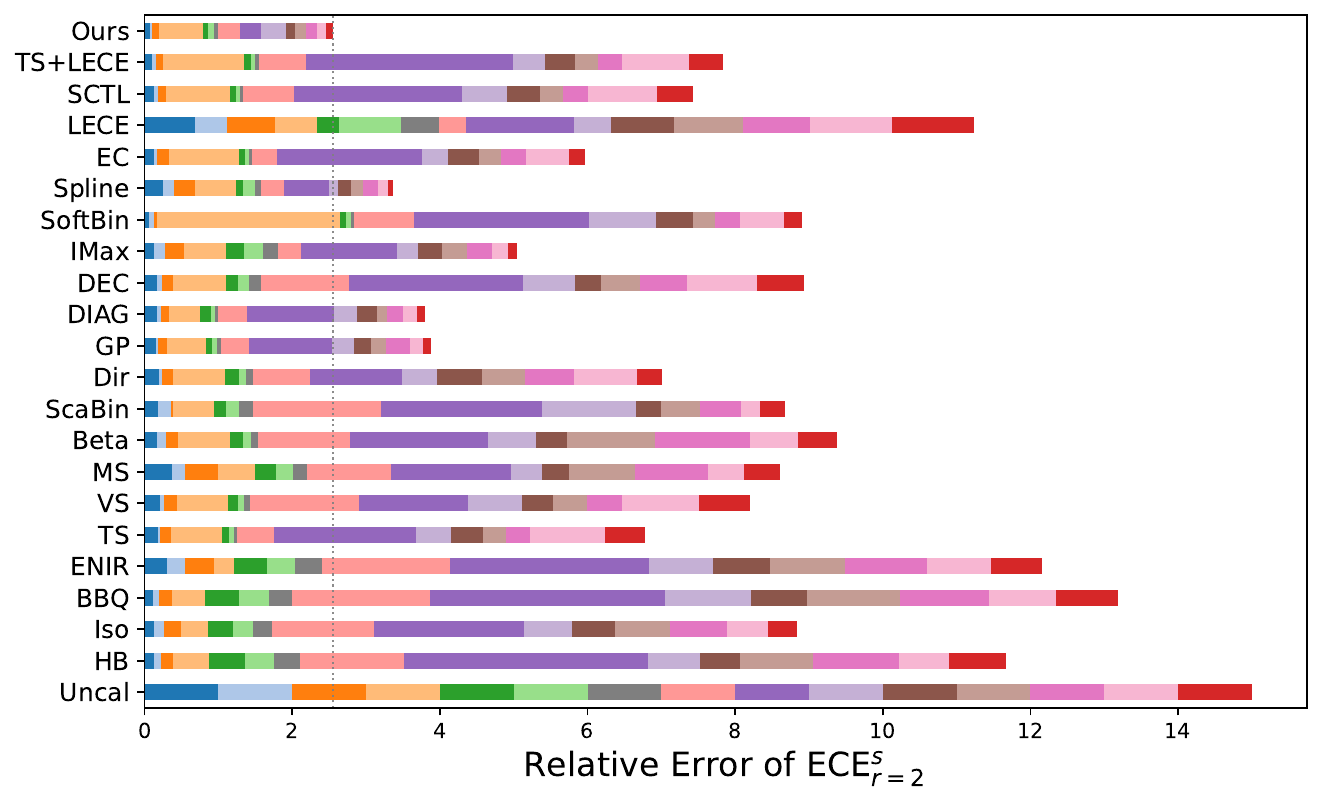}
    \end{minipage}
    \begin{minipage}{0.48\textwidth}
        \includegraphics[width=\linewidth]{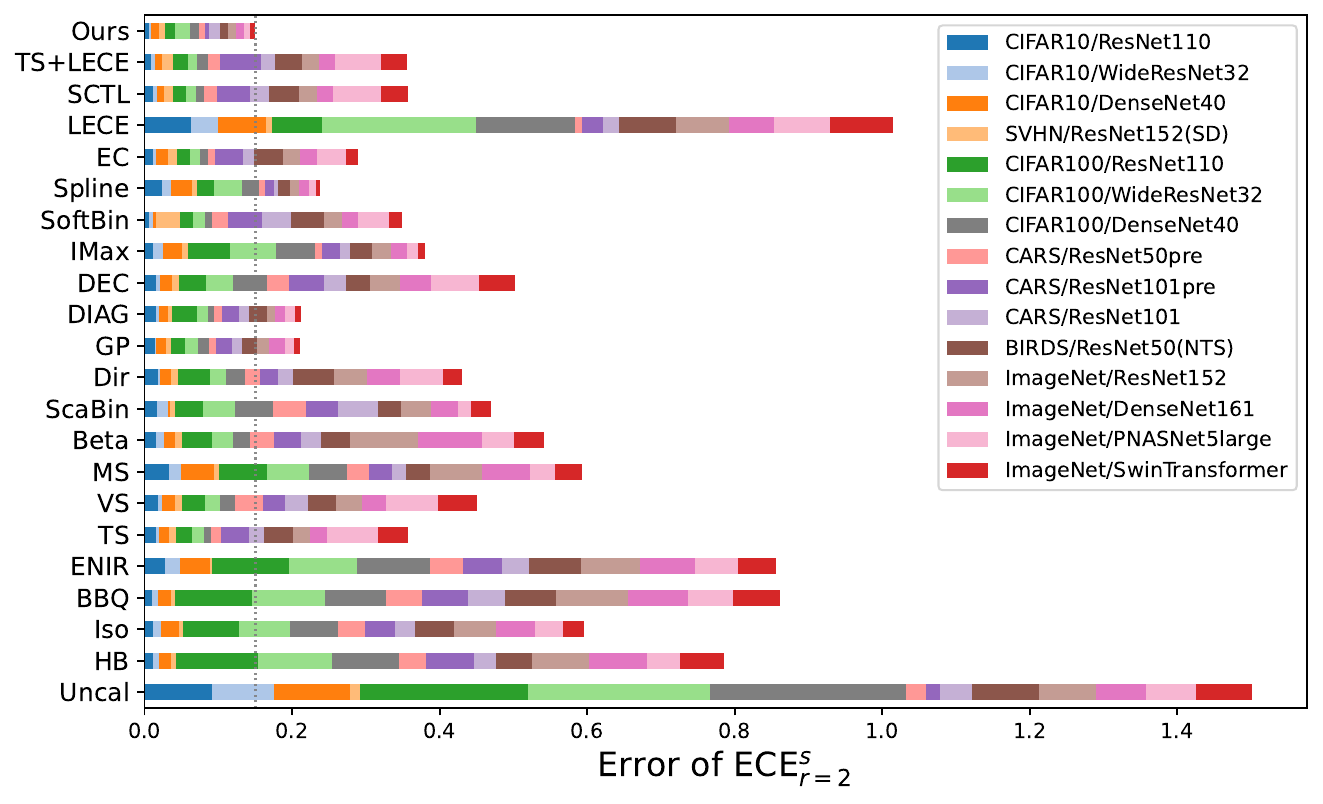}
    \end{minipage}
    \\[5pt]
    \begin{minipage}{0.48\textwidth}
        \includegraphics[width=\linewidth]{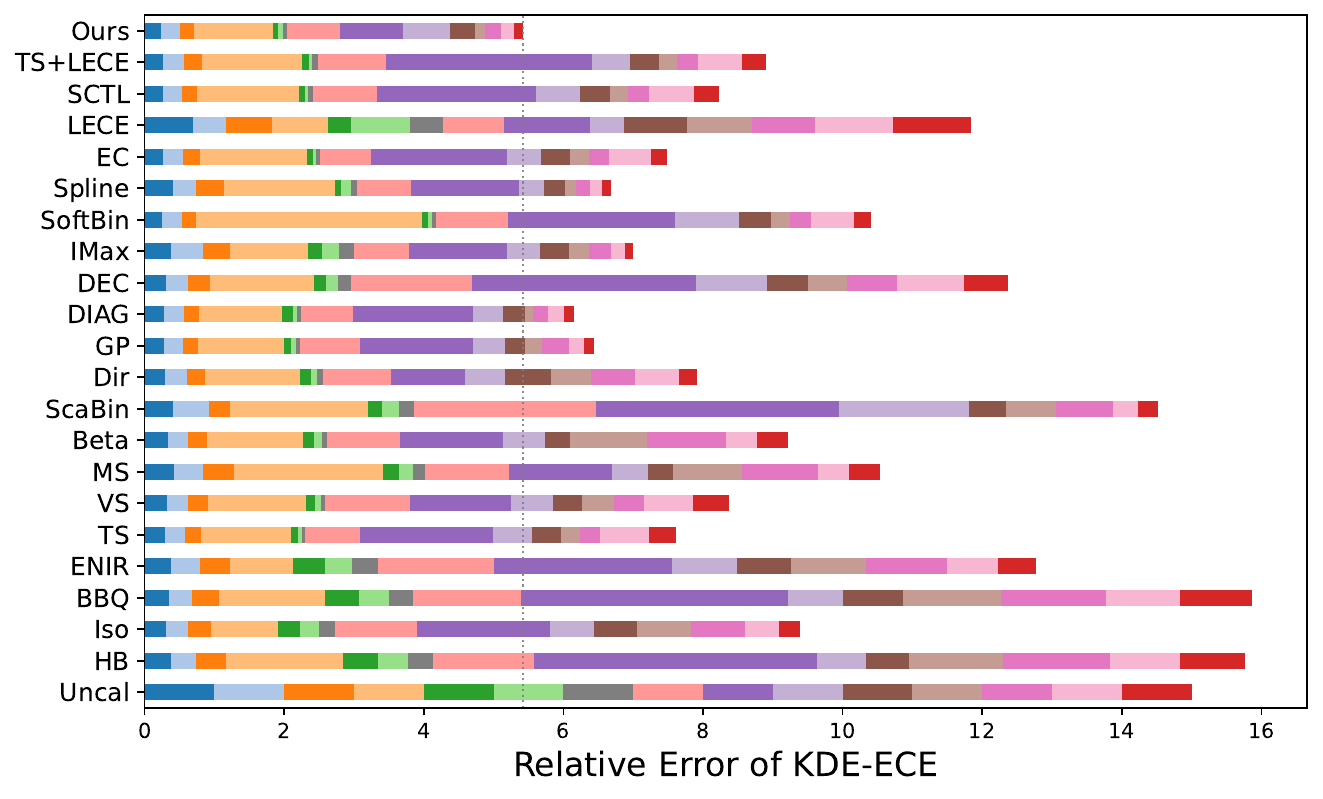}
    \end{minipage}
    \begin{minipage}{0.48\textwidth}
        \includegraphics[width=\linewidth]{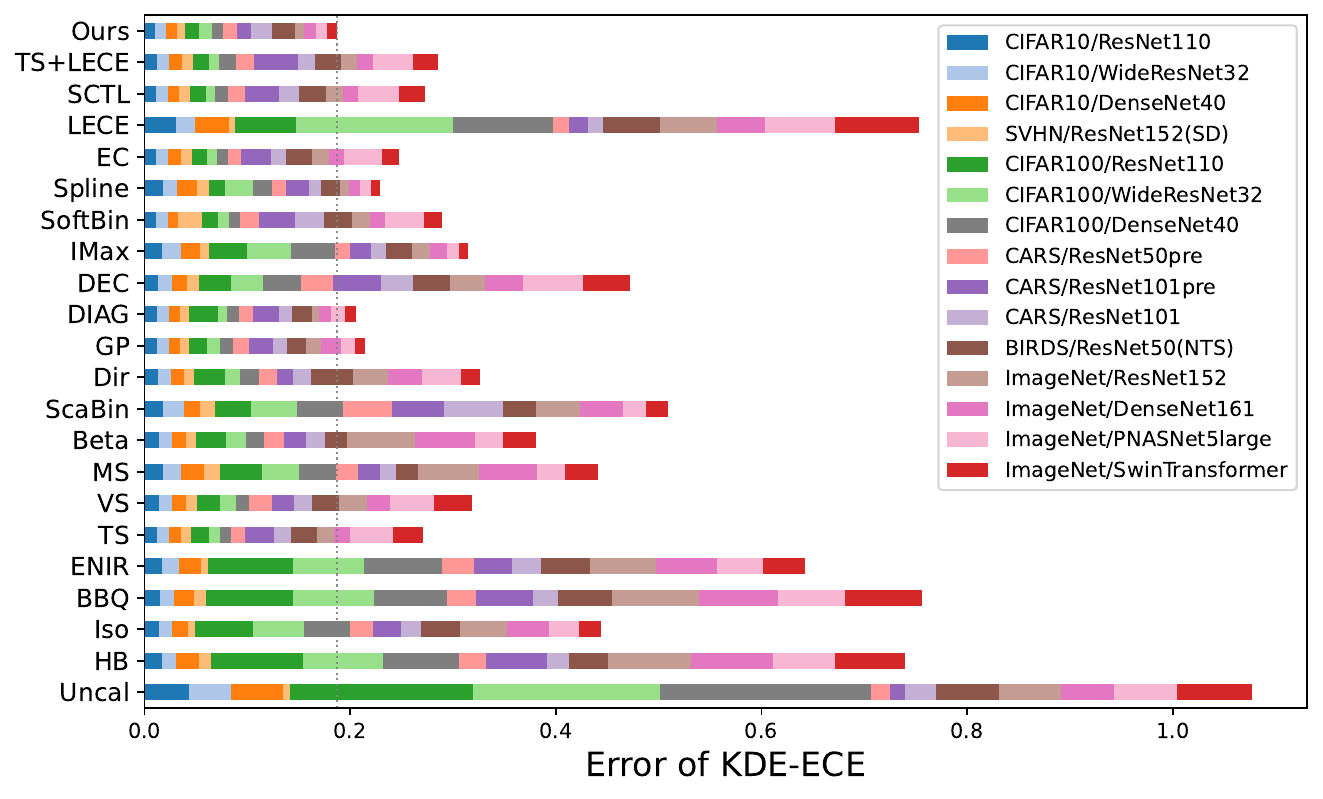}
    \end{minipage}
    \\[5pt]    
    \begin{minipage}{0.48\textwidth}
        \includegraphics[width=\linewidth]{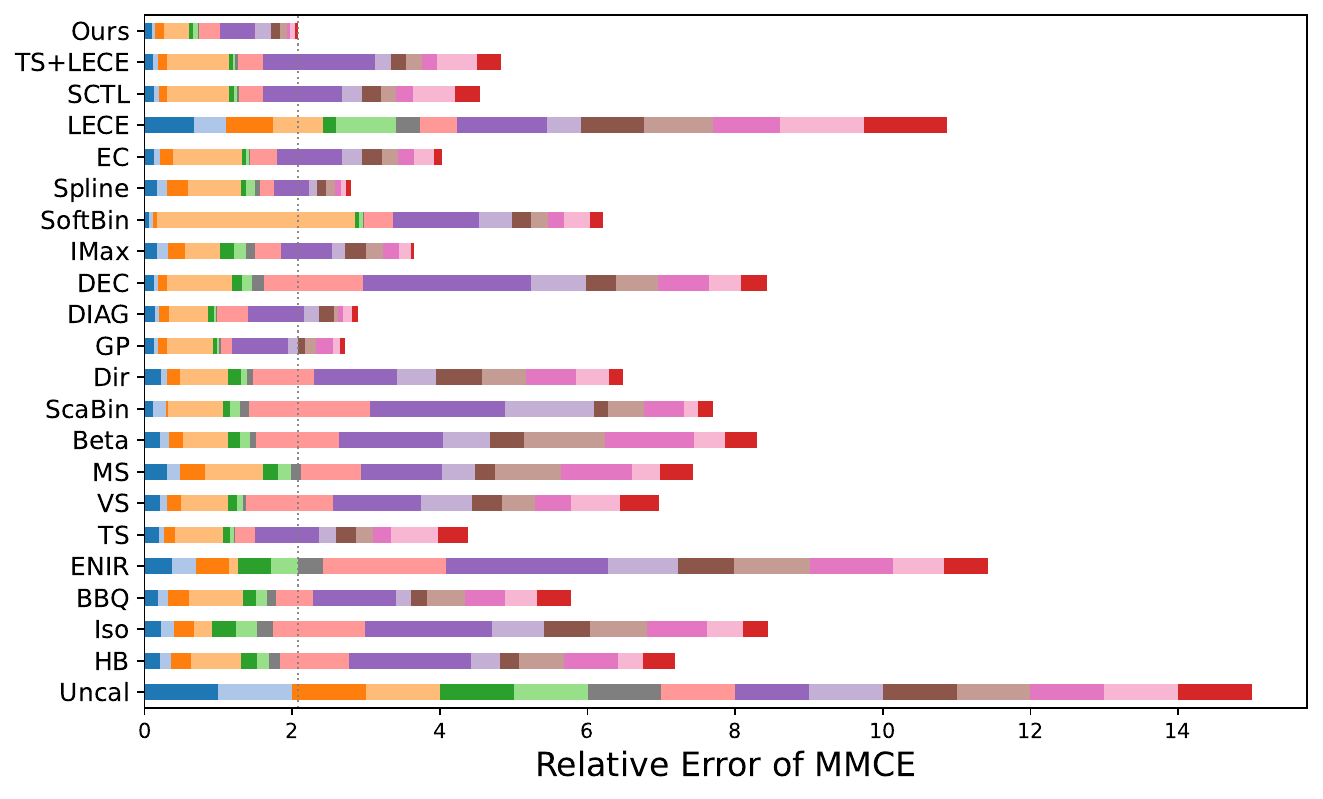}
    \end{minipage}
    \begin{minipage}{0.48\textwidth}
        \includegraphics[width=\linewidth]{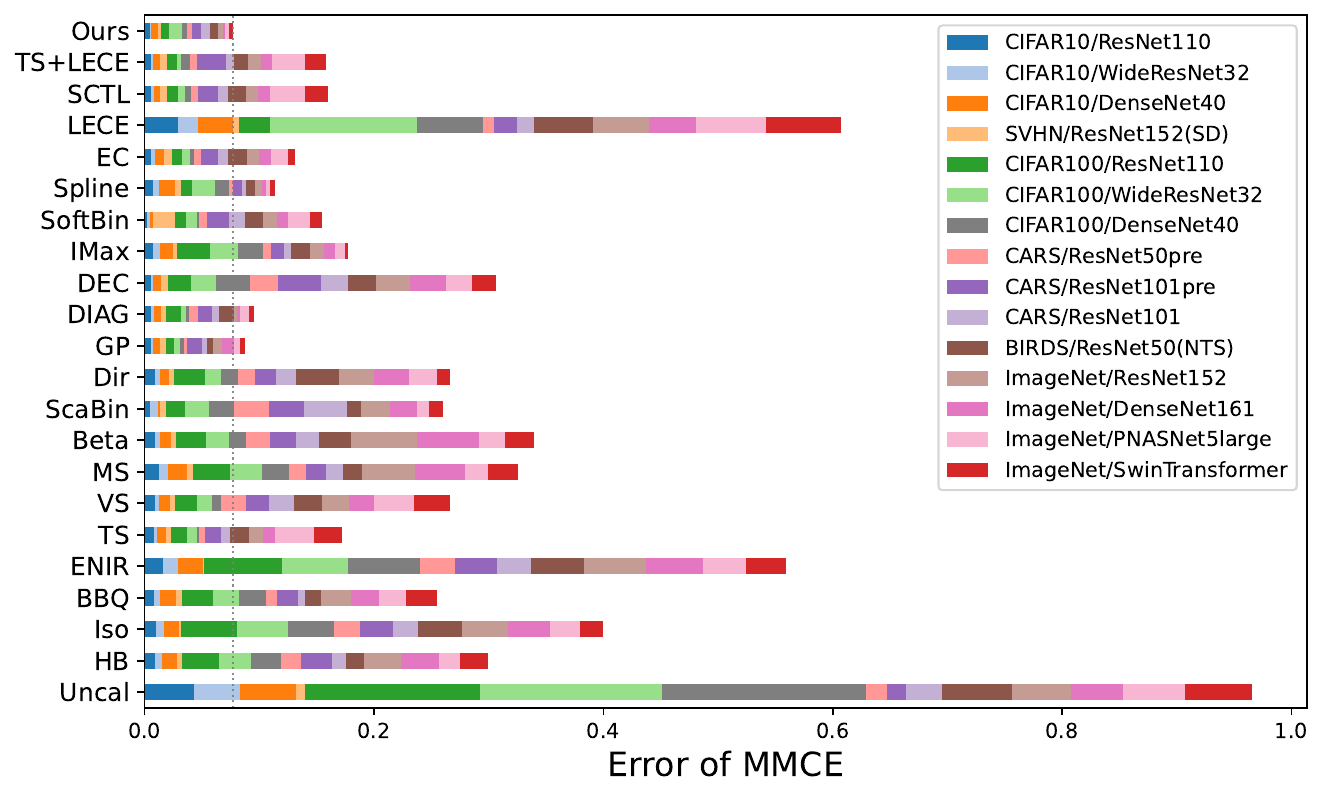}
    \end{minipage}        
    \caption{Metric-specific relative/absolute calibration errors across all tasks – Part I}
\end{figure}

\begin{figure}[htb]
    \centering 
    \begin{minipage}{0.48\textwidth}
        \includegraphics[width=\linewidth]{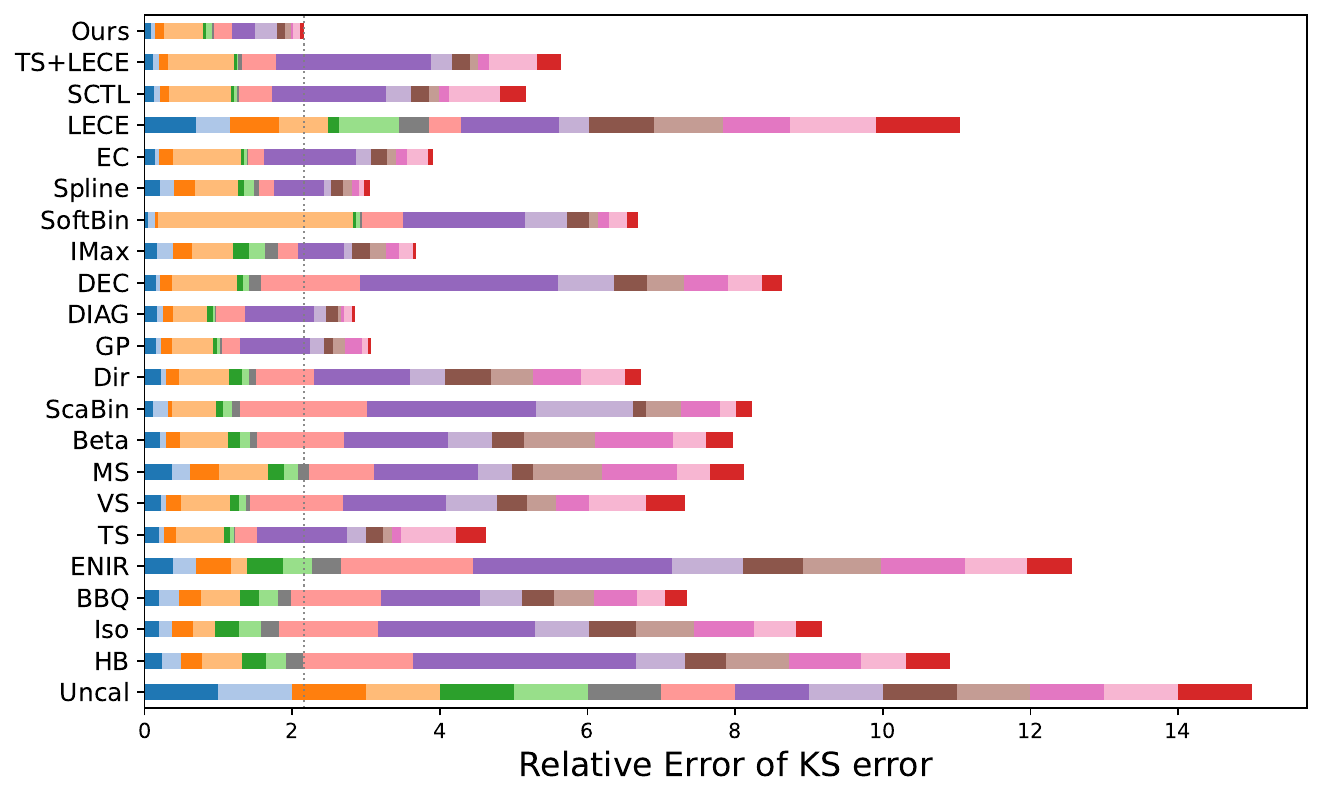}
    \end{minipage}
    \begin{minipage}{0.48\textwidth}
        \includegraphics[width=\linewidth]{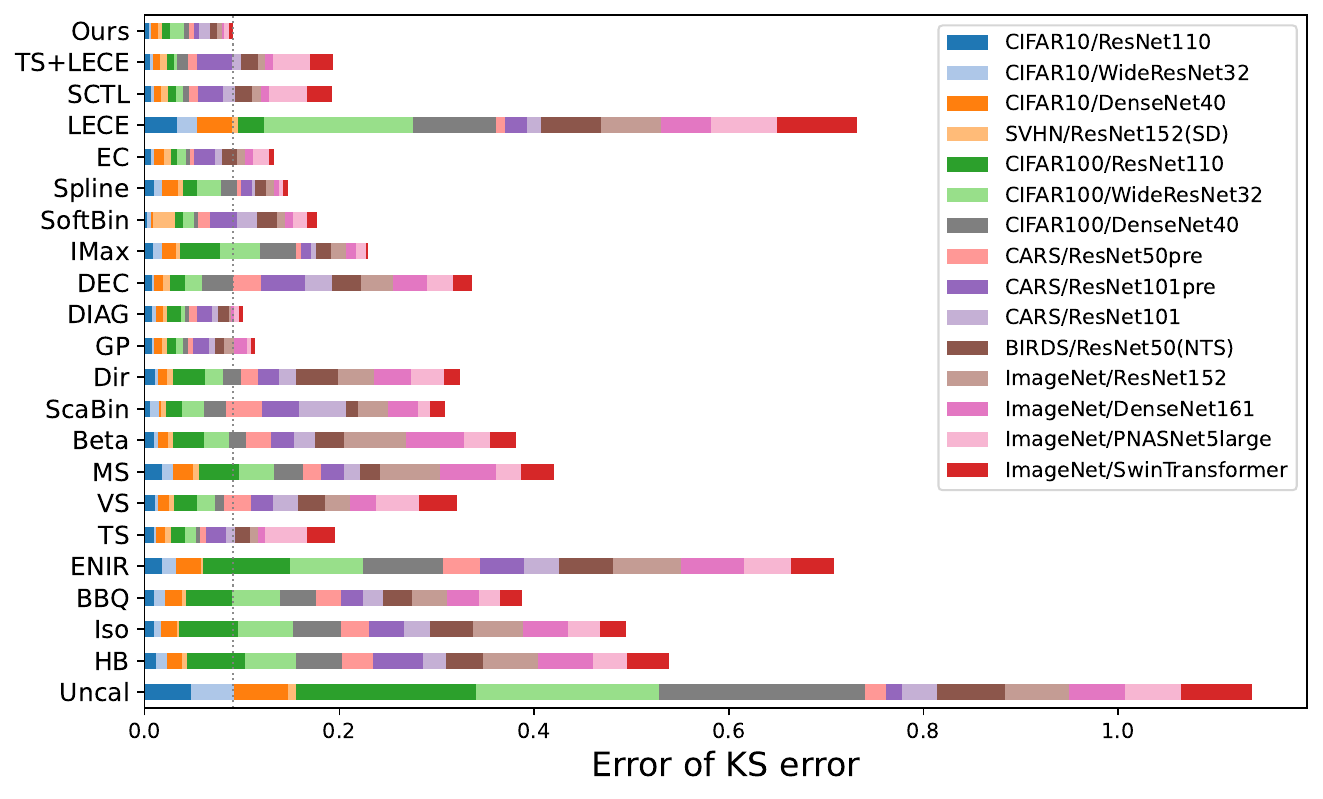}
    \end{minipage}        
    \\[5pt]    
    \begin{minipage}{0.48\textwidth}
        \includegraphics[width=\linewidth]{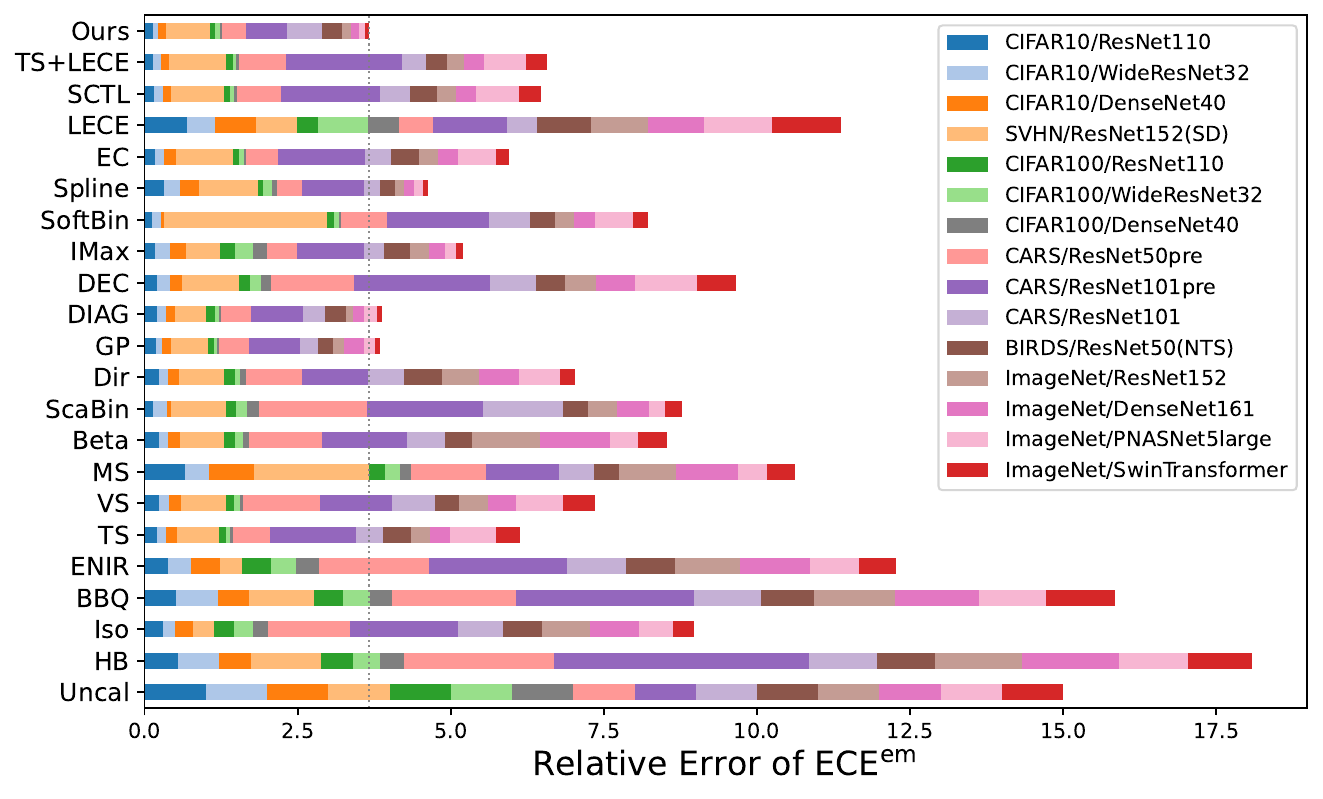}
    \end{minipage}
    \begin{minipage}{0.48\textwidth}
        \includegraphics[width=\linewidth]{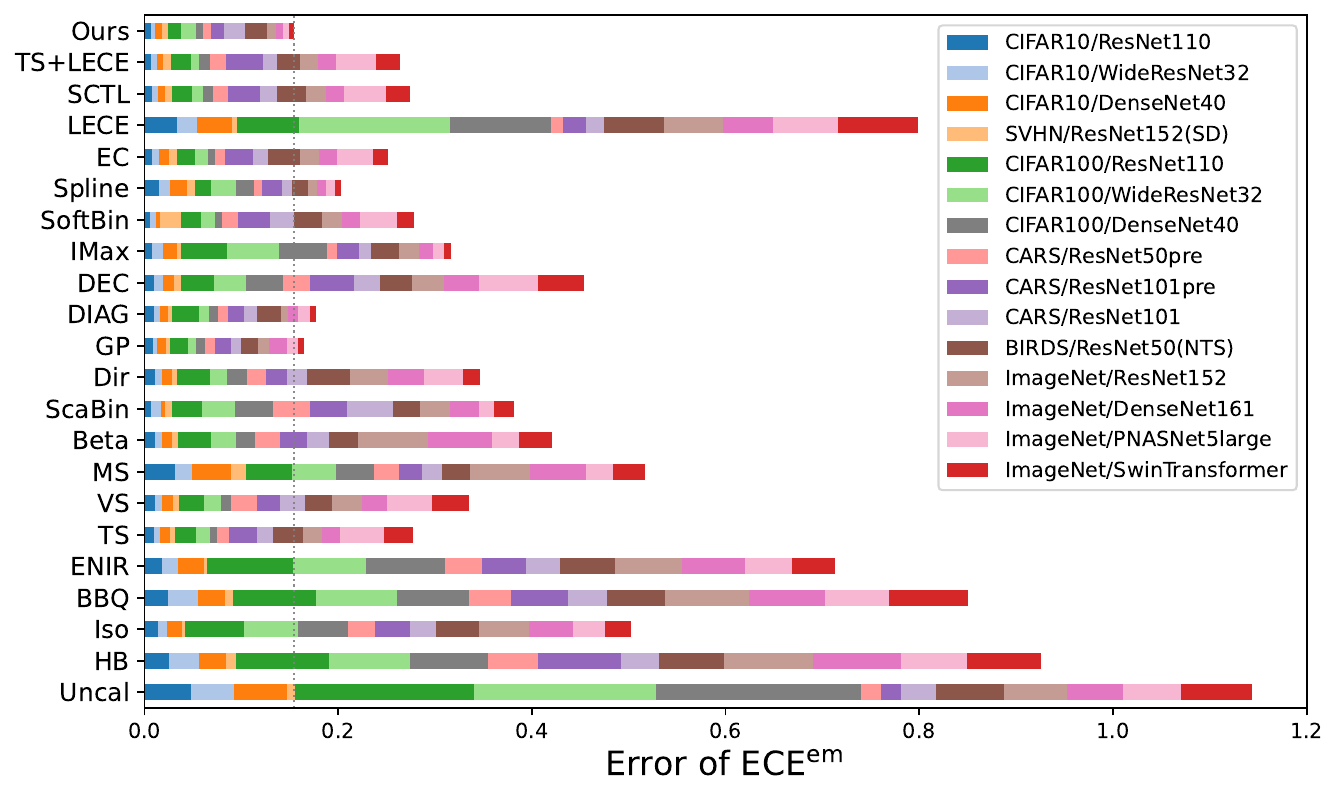}
    \end{minipage}          
    \\[5pt]   
    \begin{minipage}{0.48\textwidth}
        \includegraphics[width=\linewidth]{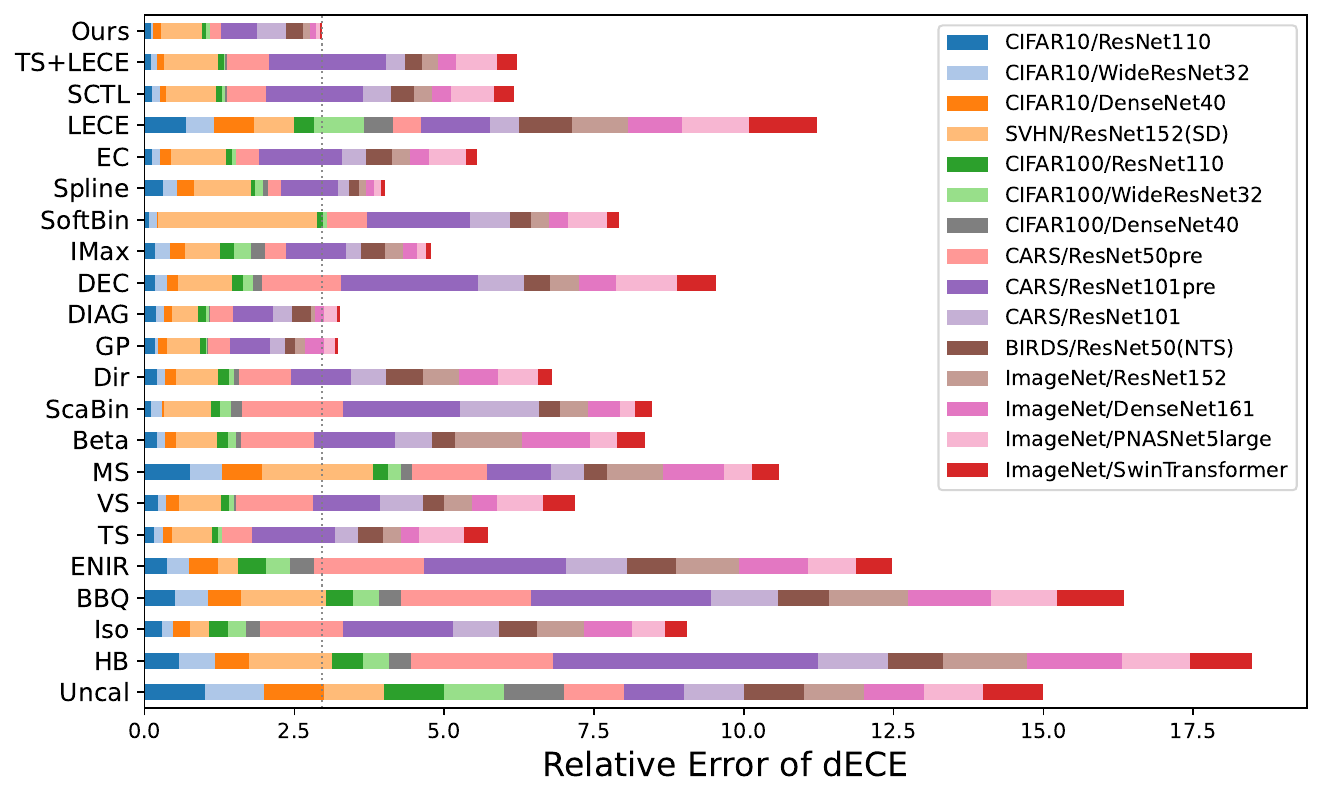}
    \end{minipage}
    \begin{minipage}{0.48\textwidth}
        \includegraphics[width=\linewidth]{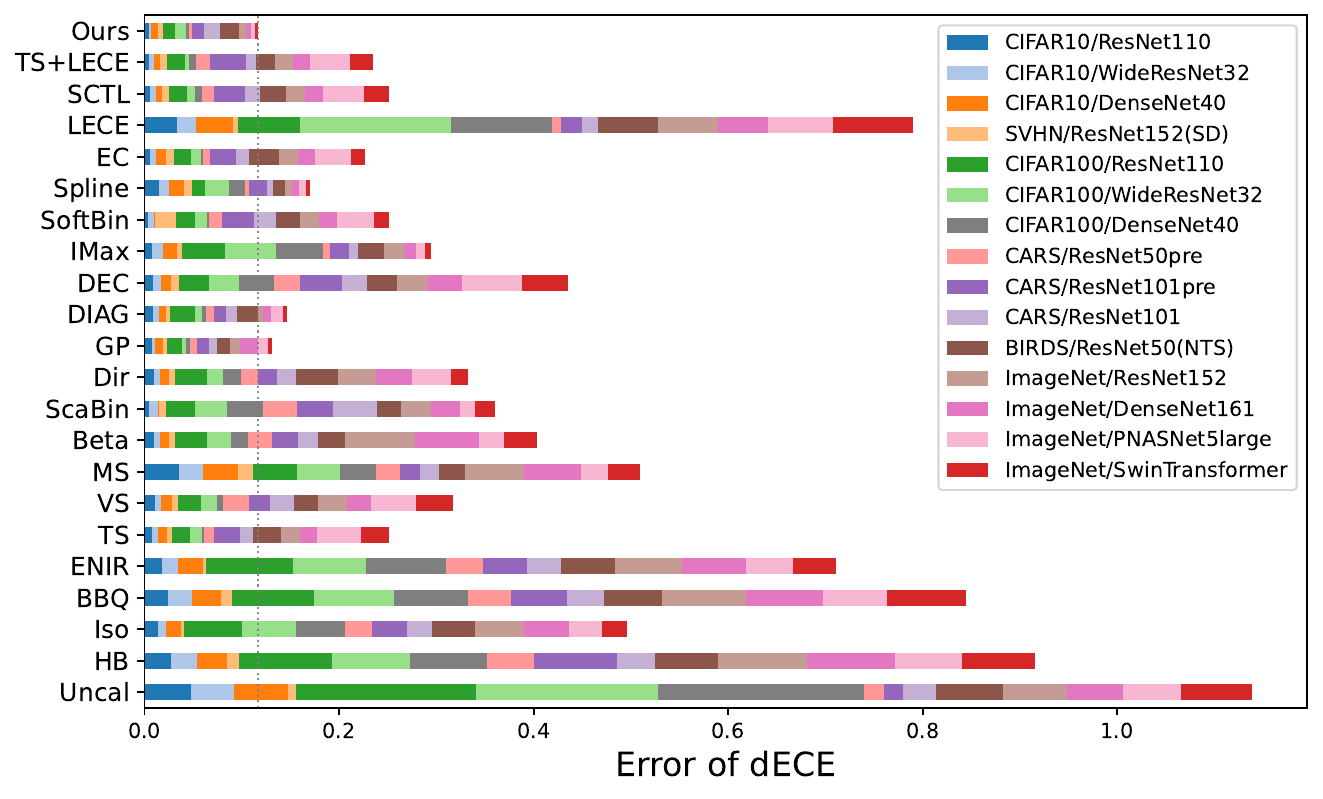}
    \end{minipage}
    \\[5pt]    
    \begin{minipage}{0.48\textwidth}
        \includegraphics[width=\linewidth]{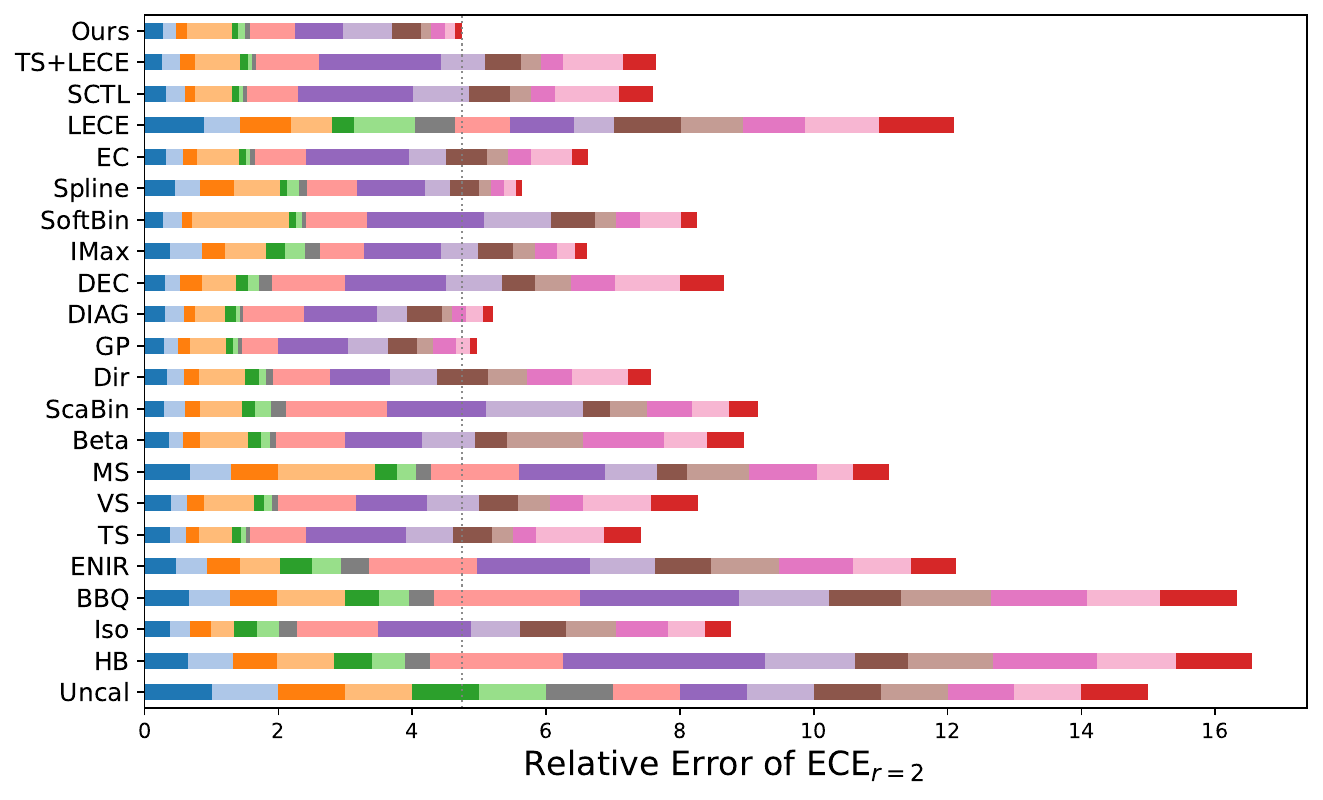}
    \end{minipage}
    \begin{minipage}{0.48\textwidth}
        \includegraphics[width=\linewidth]{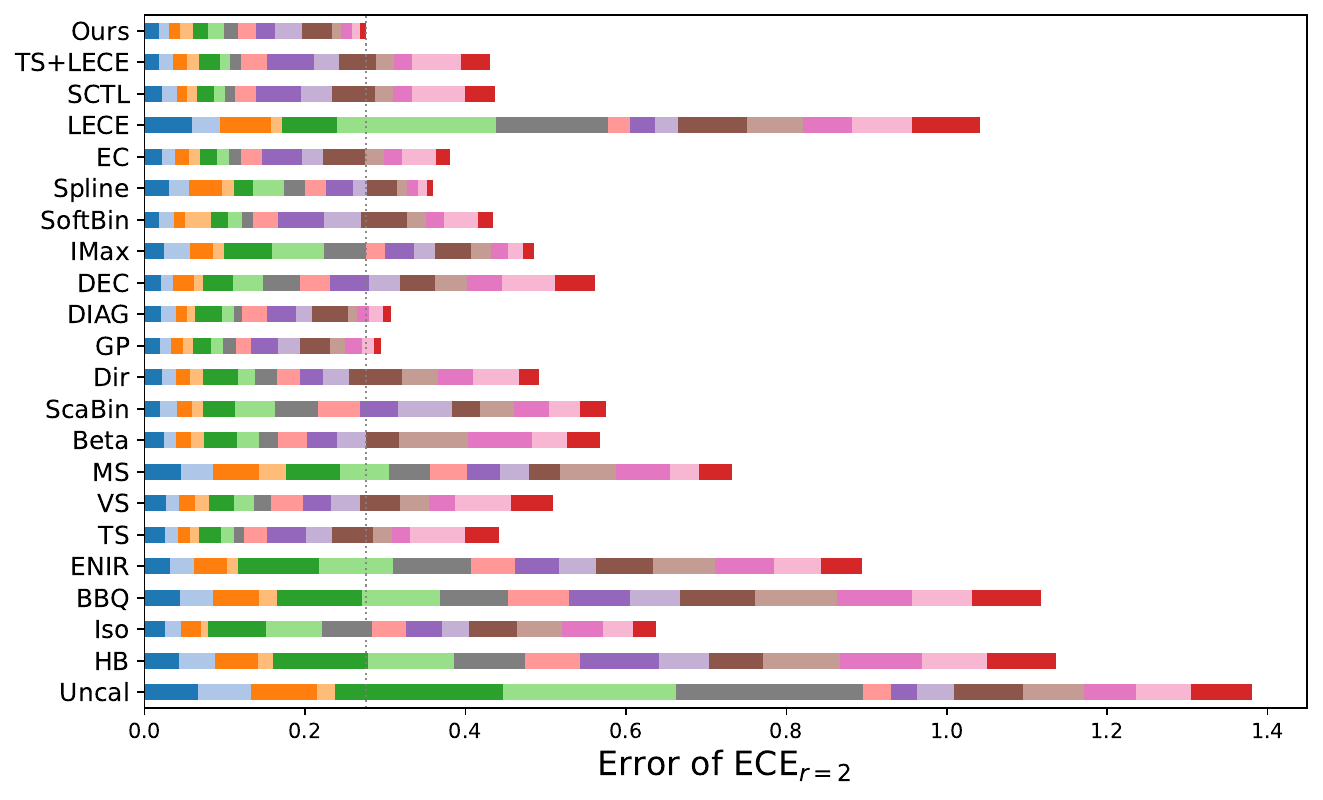}
    \end{minipage}    
    \caption{Metric-specific relative/absolute calibration errors across all tasks – Part II}
\end{figure}

\begin{figure}[htb]
    \centering 
    \begin{minipage}{0.48\textwidth}
        \includegraphics[width=\linewidth]{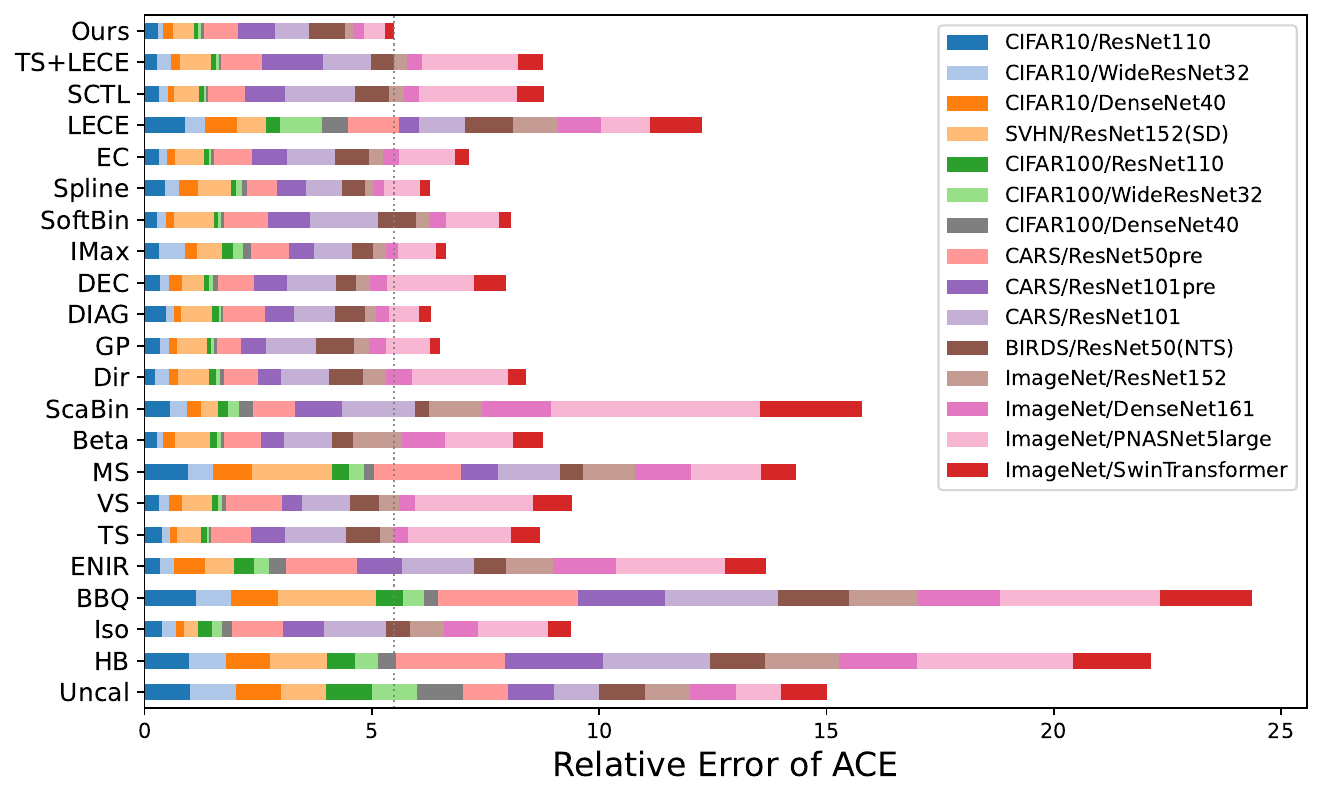}
    \end{minipage}
    \begin{minipage}{0.48\textwidth}
        \includegraphics[width=\linewidth]{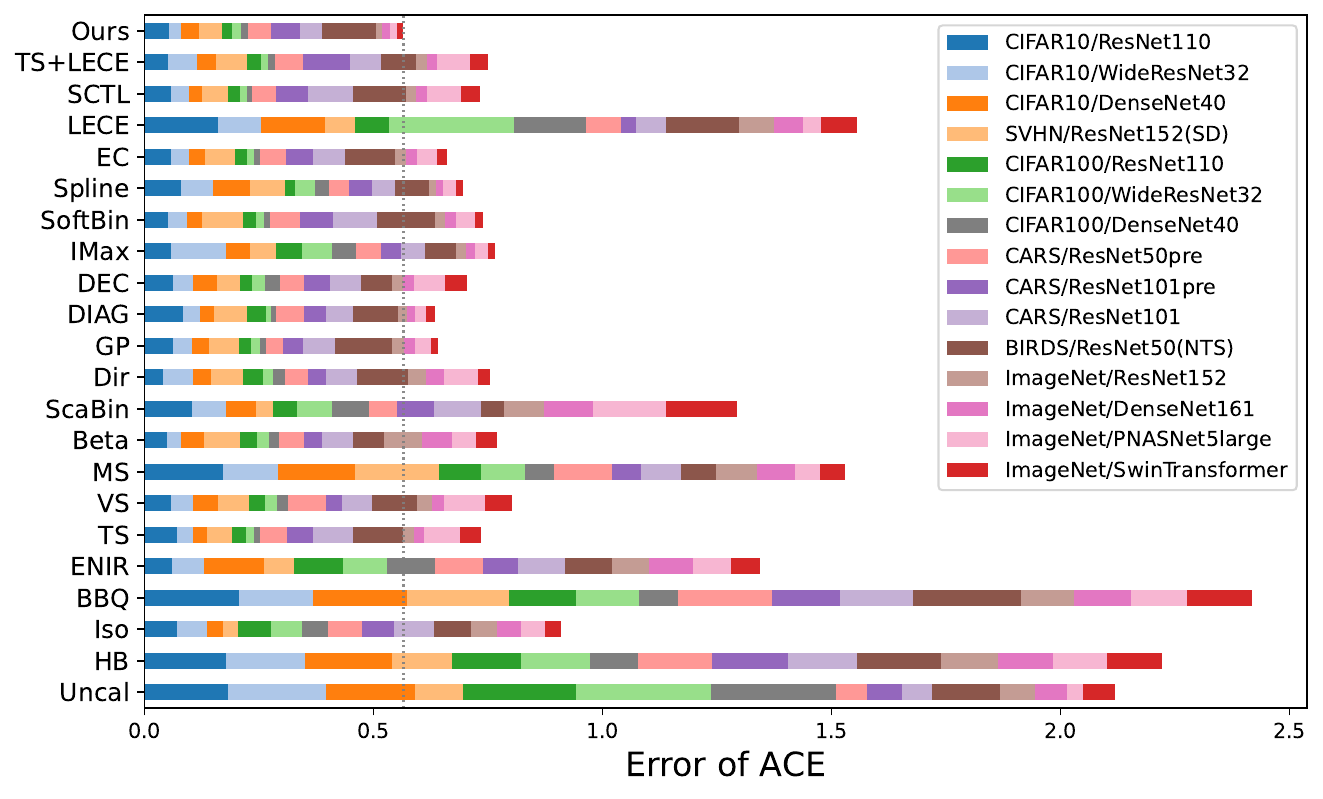}
    \end{minipage}        
    \\[5pt]    
    \begin{minipage}{0.48\textwidth}
        \includegraphics[width=\linewidth]{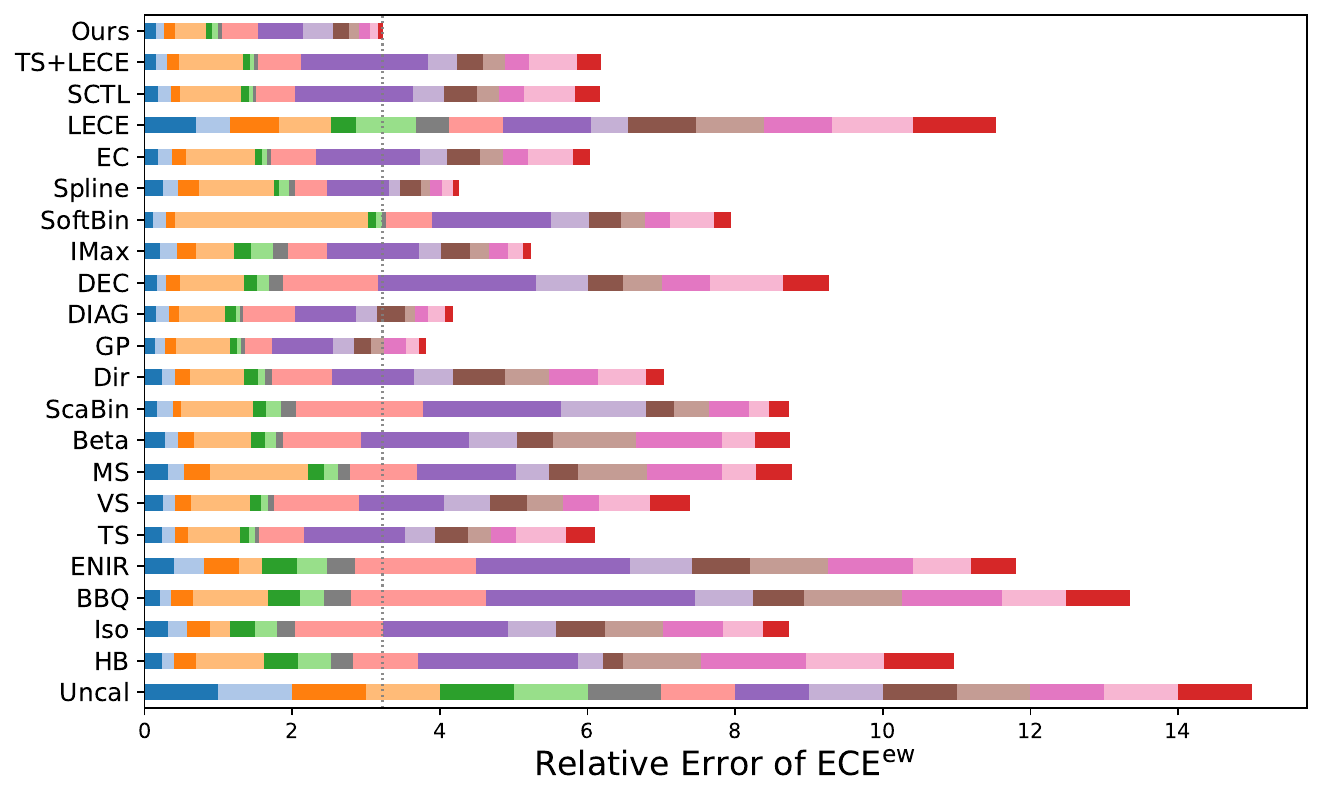}
    \end{minipage}
    \begin{minipage}{0.48\textwidth}
        \includegraphics[width=\linewidth]{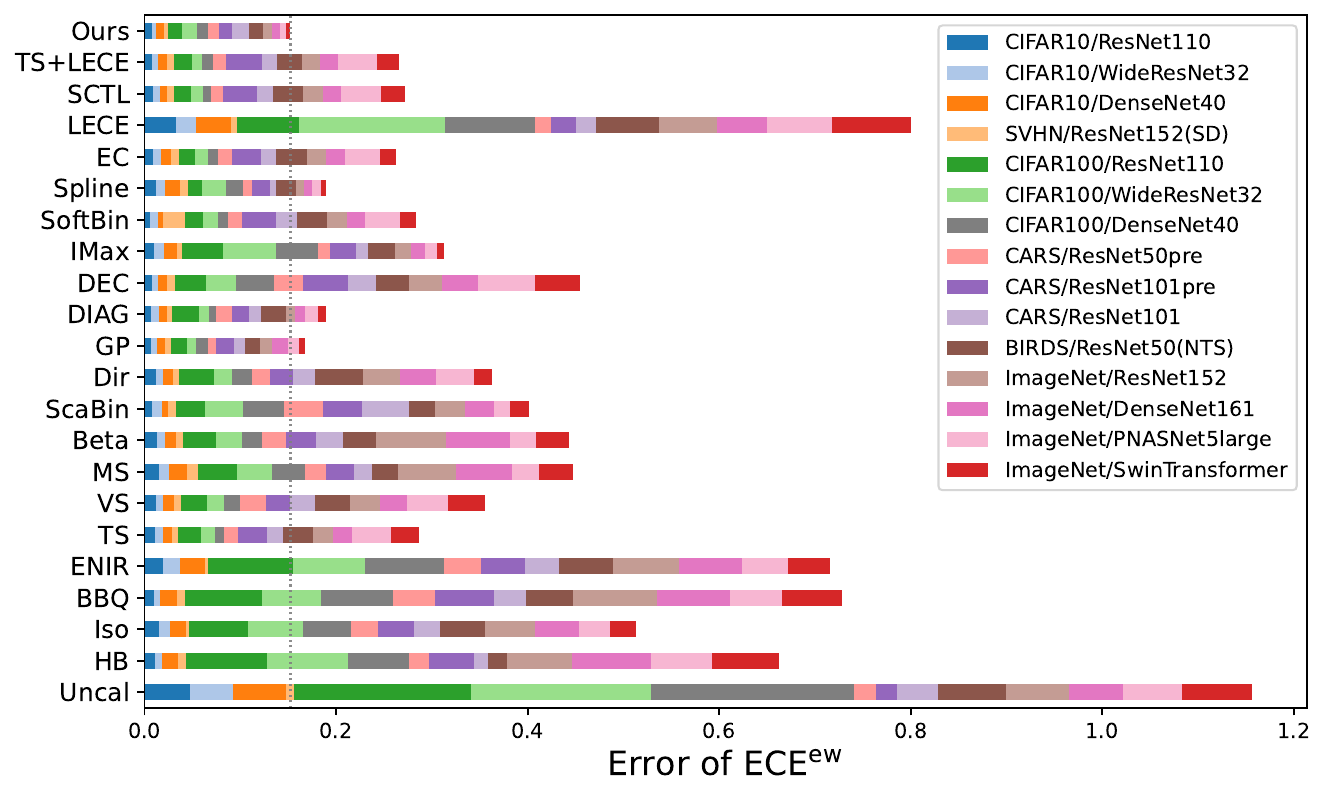}
    \end{minipage}      
    \\[5pt]  
    \begin{minipage}{0.48\textwidth}
        \includegraphics[width=\linewidth]{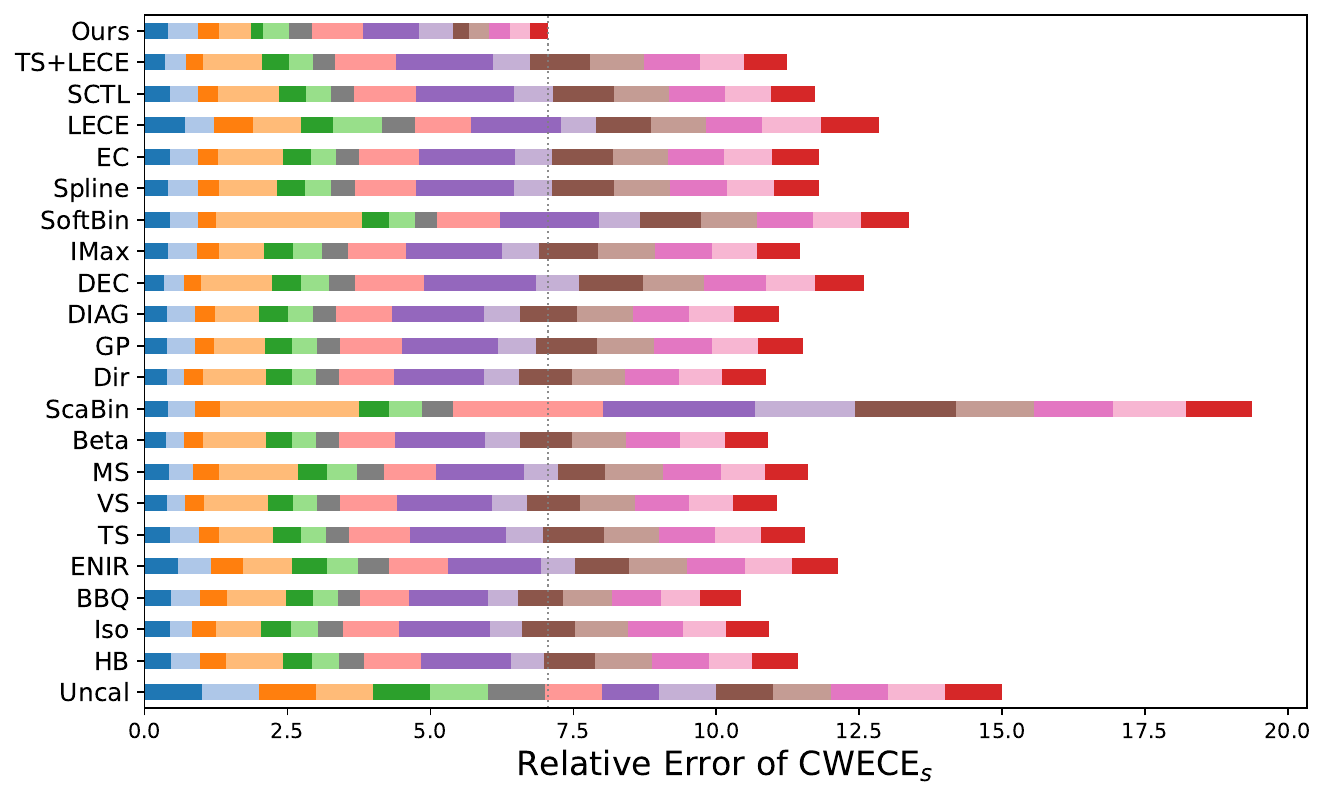}
    \end{minipage}
    \begin{minipage}{0.48\textwidth}
        \includegraphics[width=\linewidth]{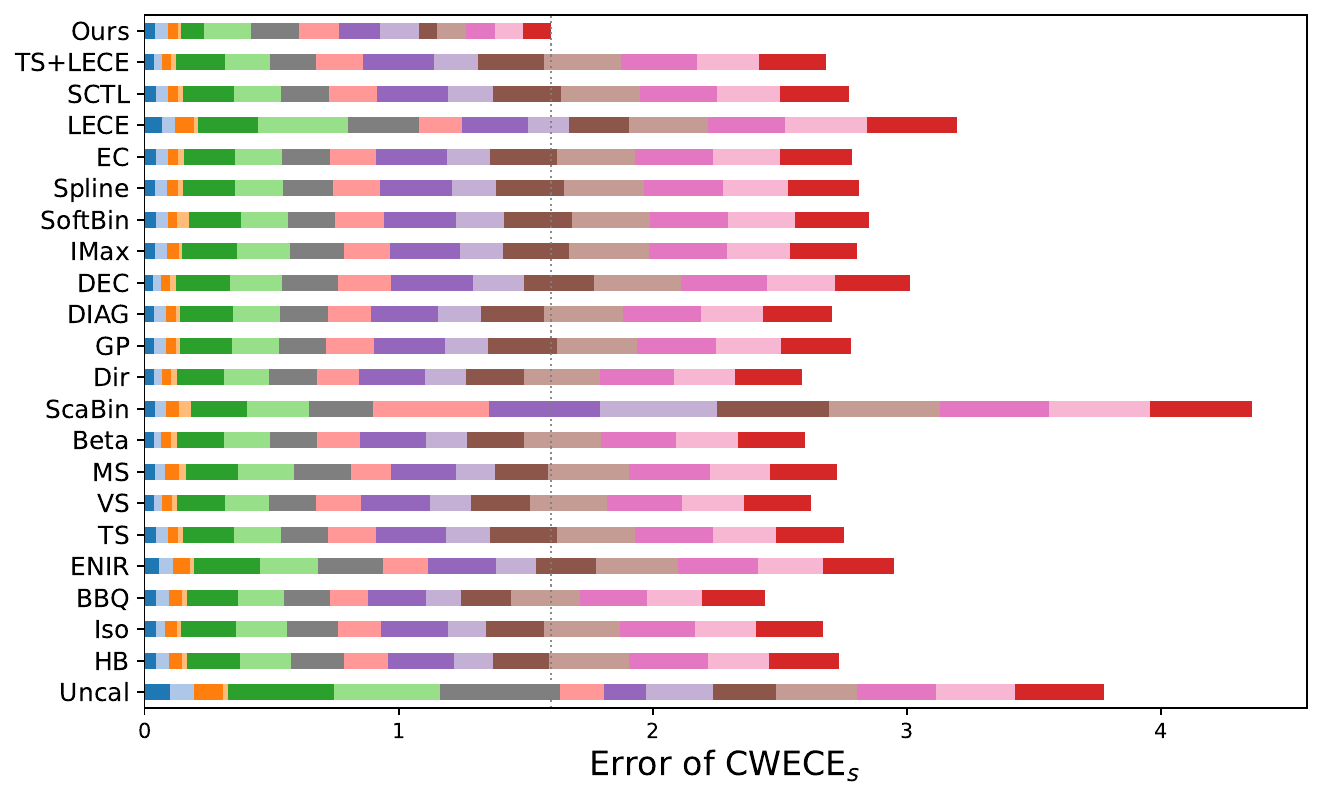}
    \end{minipage}
    \\[5pt]    
    \begin{minipage}{0.48\textwidth}
        \includegraphics[width=\linewidth]{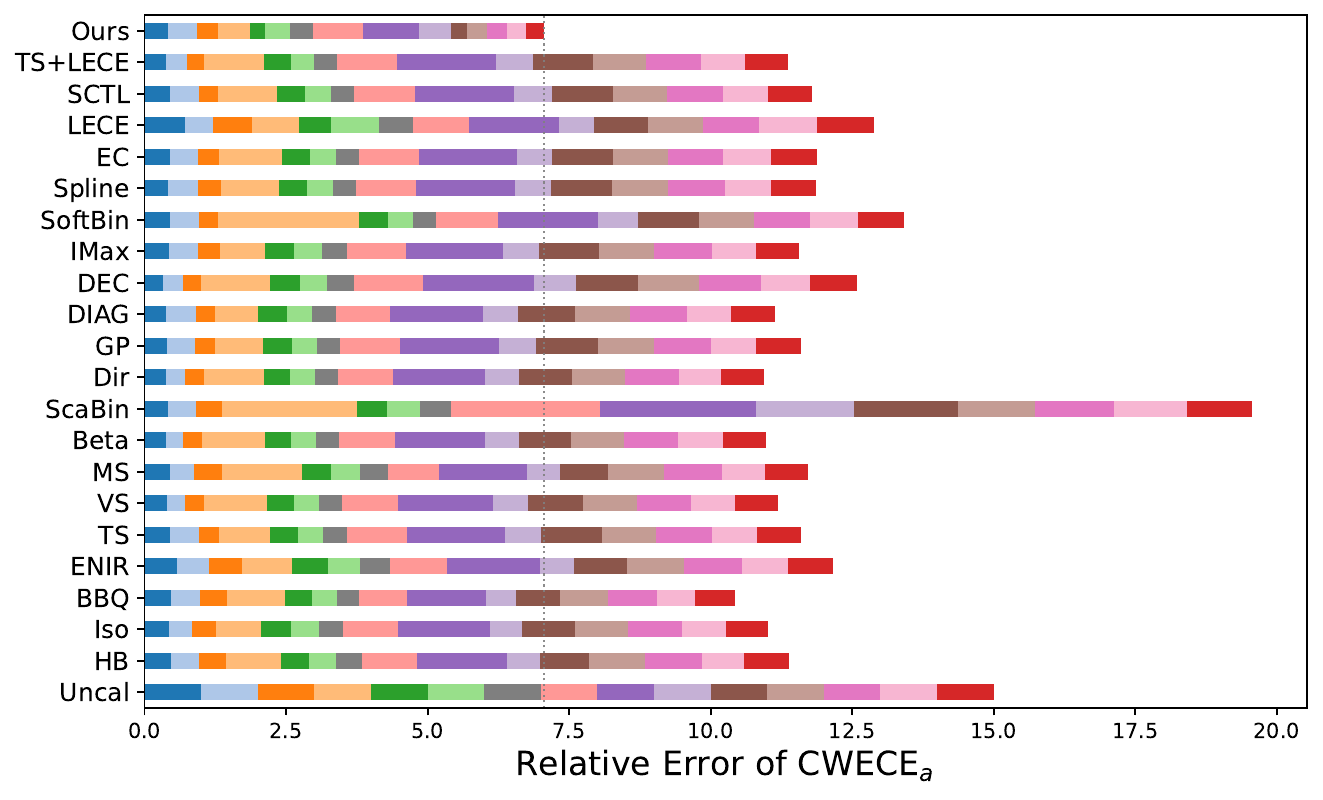}
    \end{minipage}
    \begin{minipage}{0.48\textwidth}
        \includegraphics[width=\linewidth]{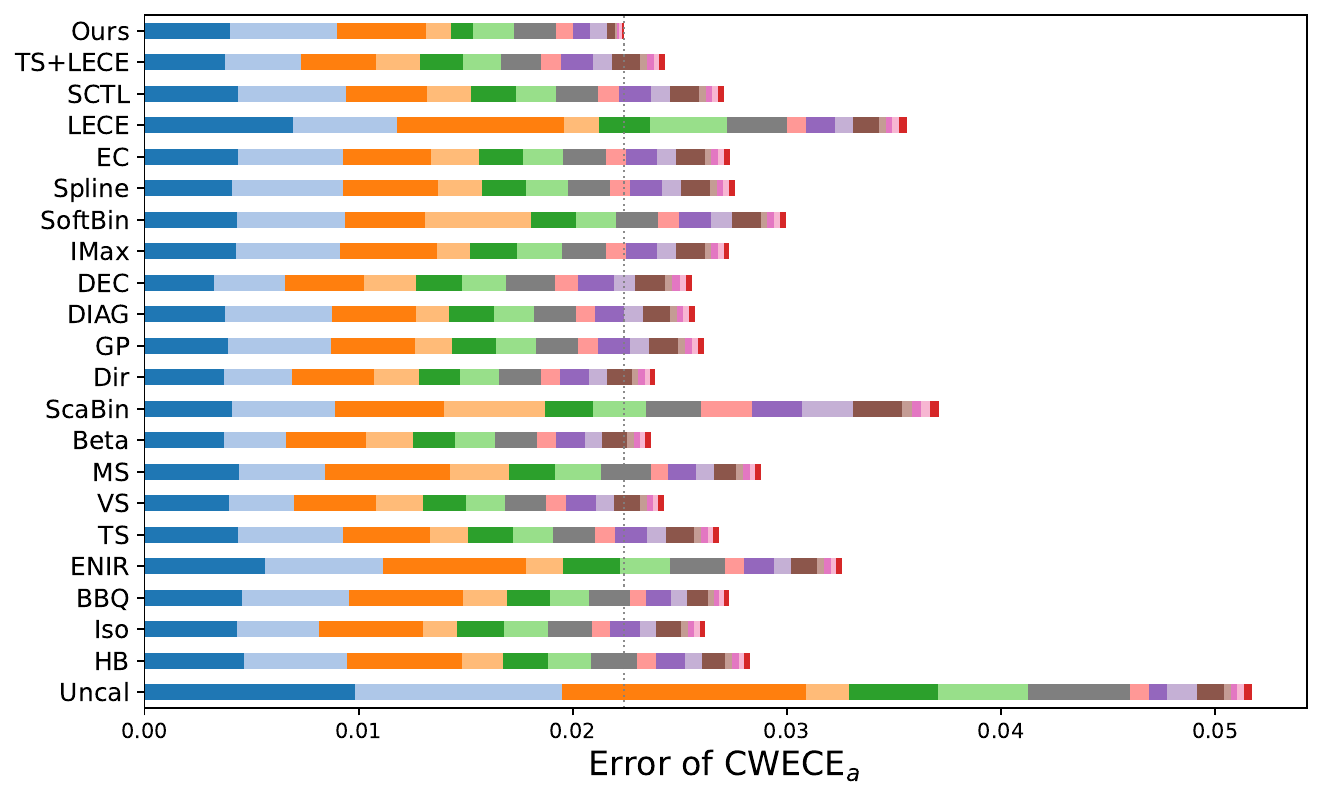}
    \end{minipage}    
    \caption{Metric-specific relative/absolute calibration errors across all tasks – Part III}
\end{figure}

\begin{figure}[htb]
    \centering
    \begin{minipage}{0.48\textwidth}
        \includegraphics[width=\linewidth]{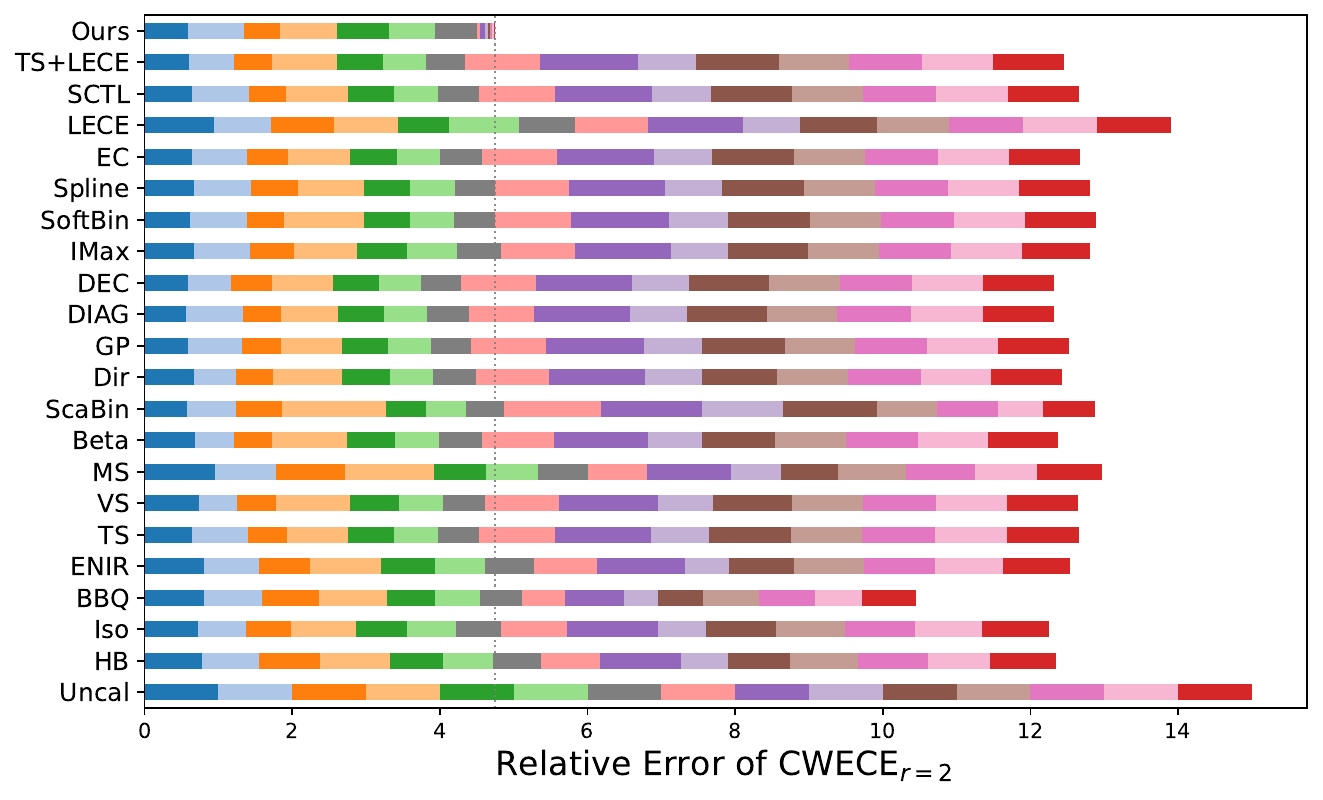}
    \end{minipage}
    \begin{minipage}{0.48\textwidth}
        \includegraphics[width=\linewidth]{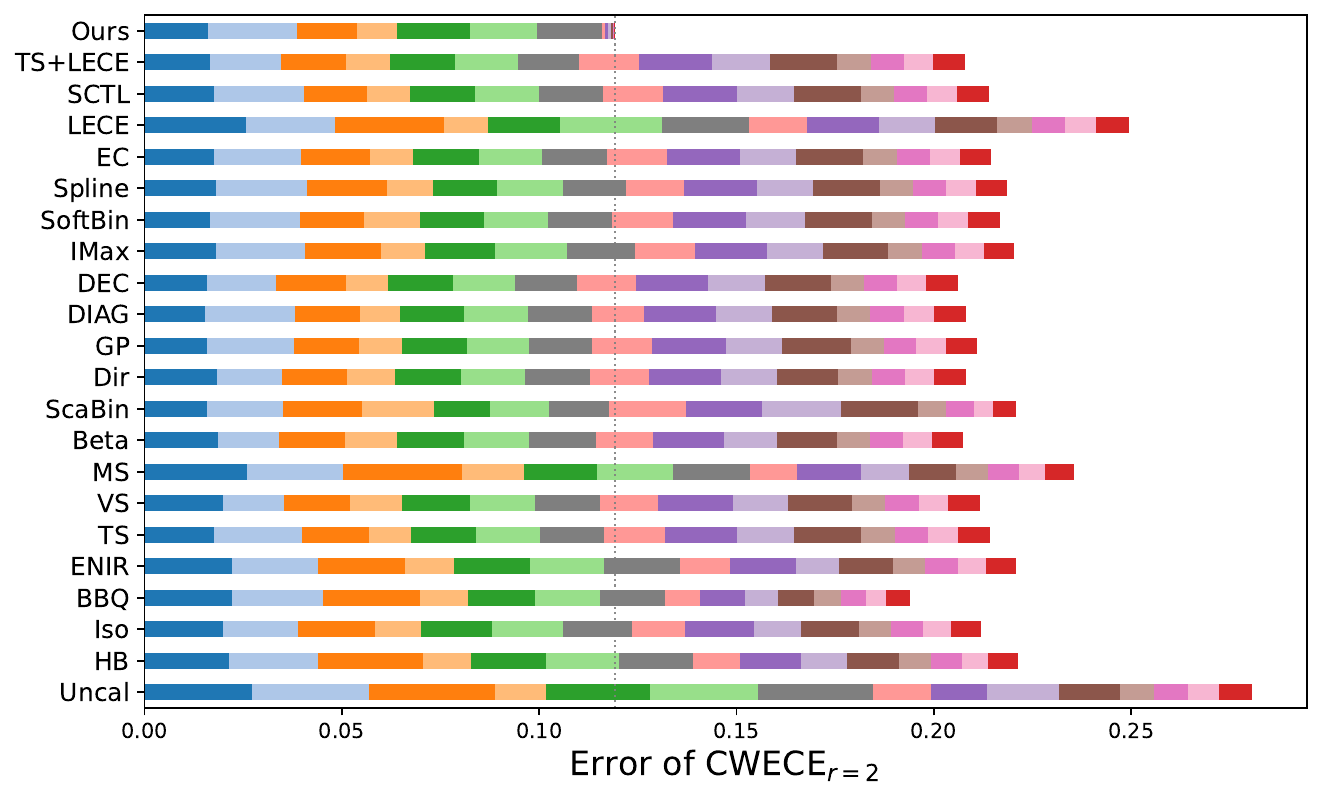}
    \end{minipage}        
    \\[5pt]    
    \begin{minipage}{0.48\textwidth}
        \includegraphics[width=\linewidth]{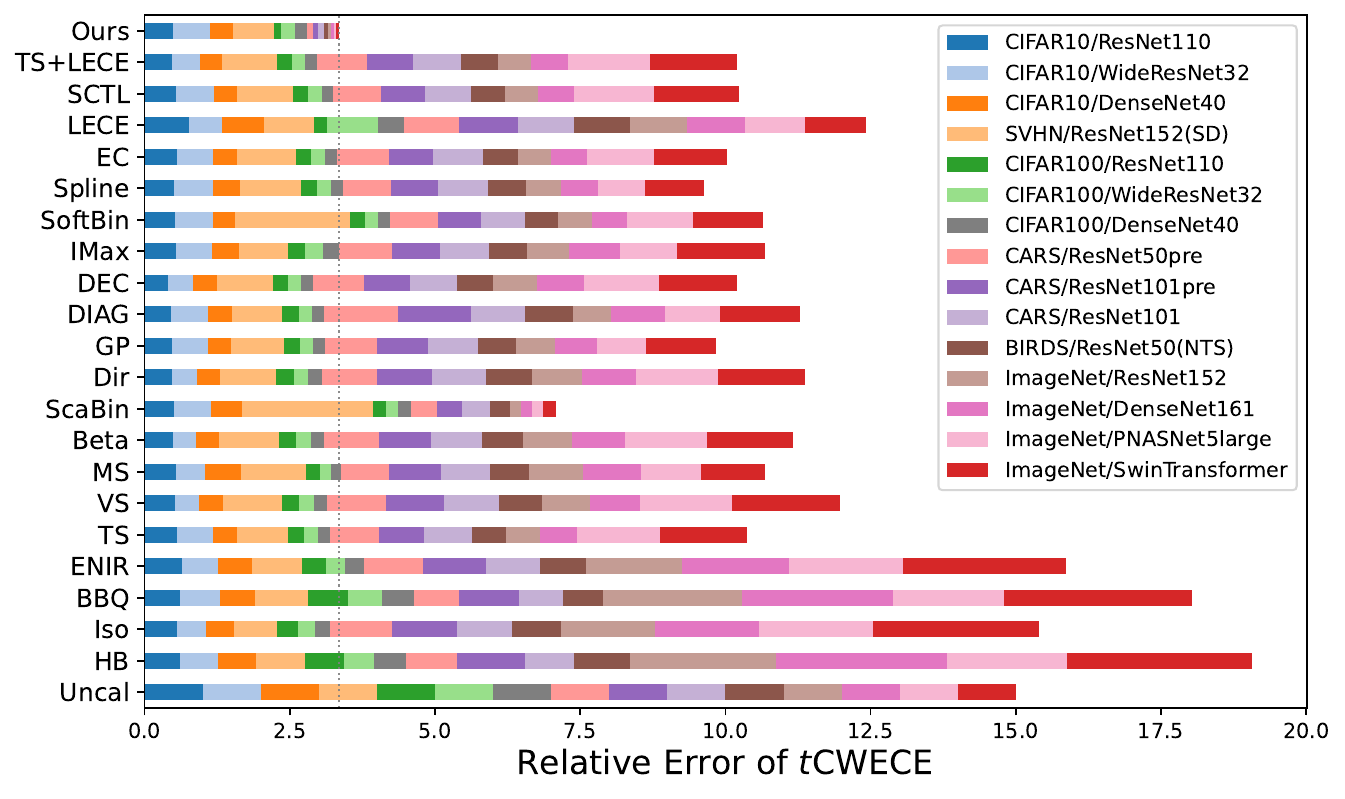}
    \end{minipage}
    \begin{minipage}{0.48\textwidth}
        \includegraphics[width=\linewidth]{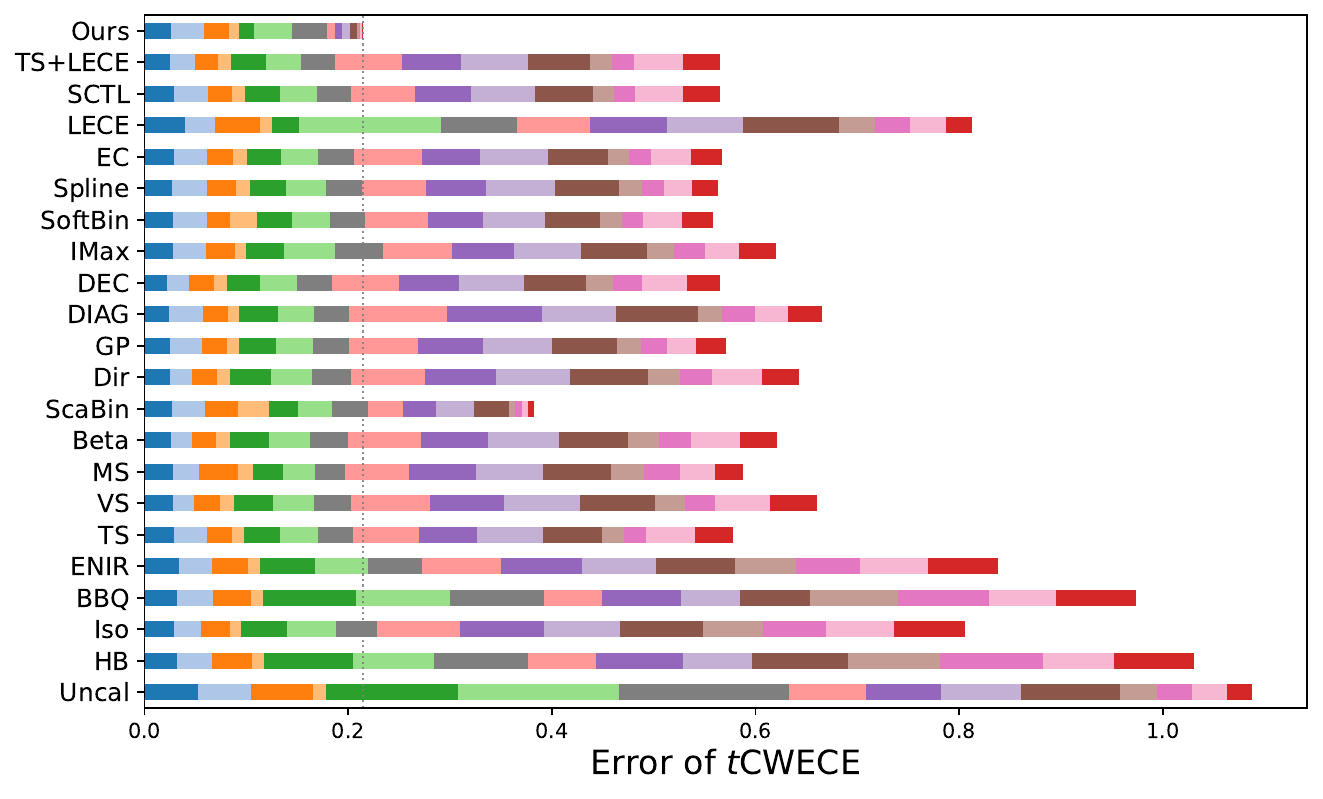}
    \end{minipage}       
    \\[5pt]    
    \begin{minipage}{0.48\textwidth}
        \includegraphics[width=\linewidth]{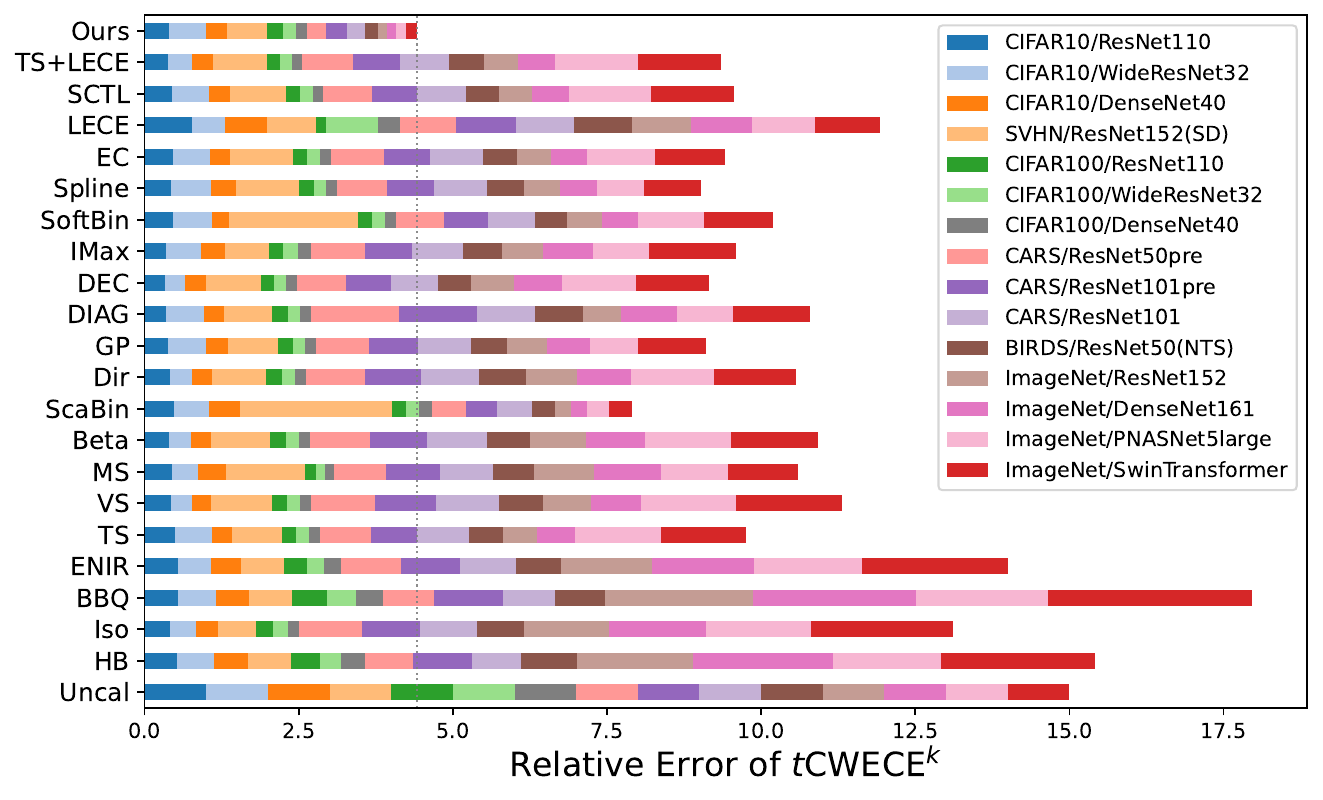}
    \end{minipage}
    \begin{minipage}{0.48\textwidth}
        \includegraphics[width=\linewidth]{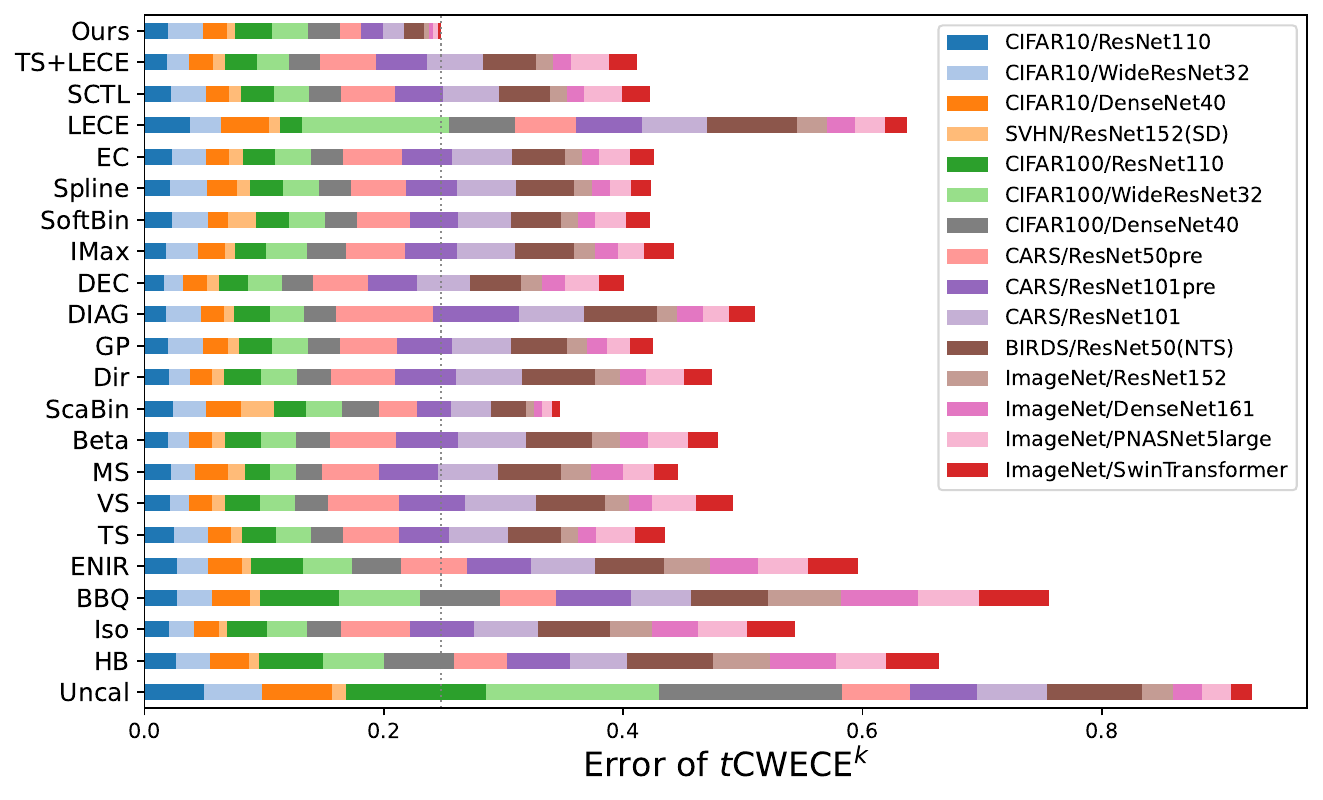}
    \end{minipage}
    \\[5pt]    
    \begin{minipage}{0.48\textwidth}
        \includegraphics[width=\linewidth]{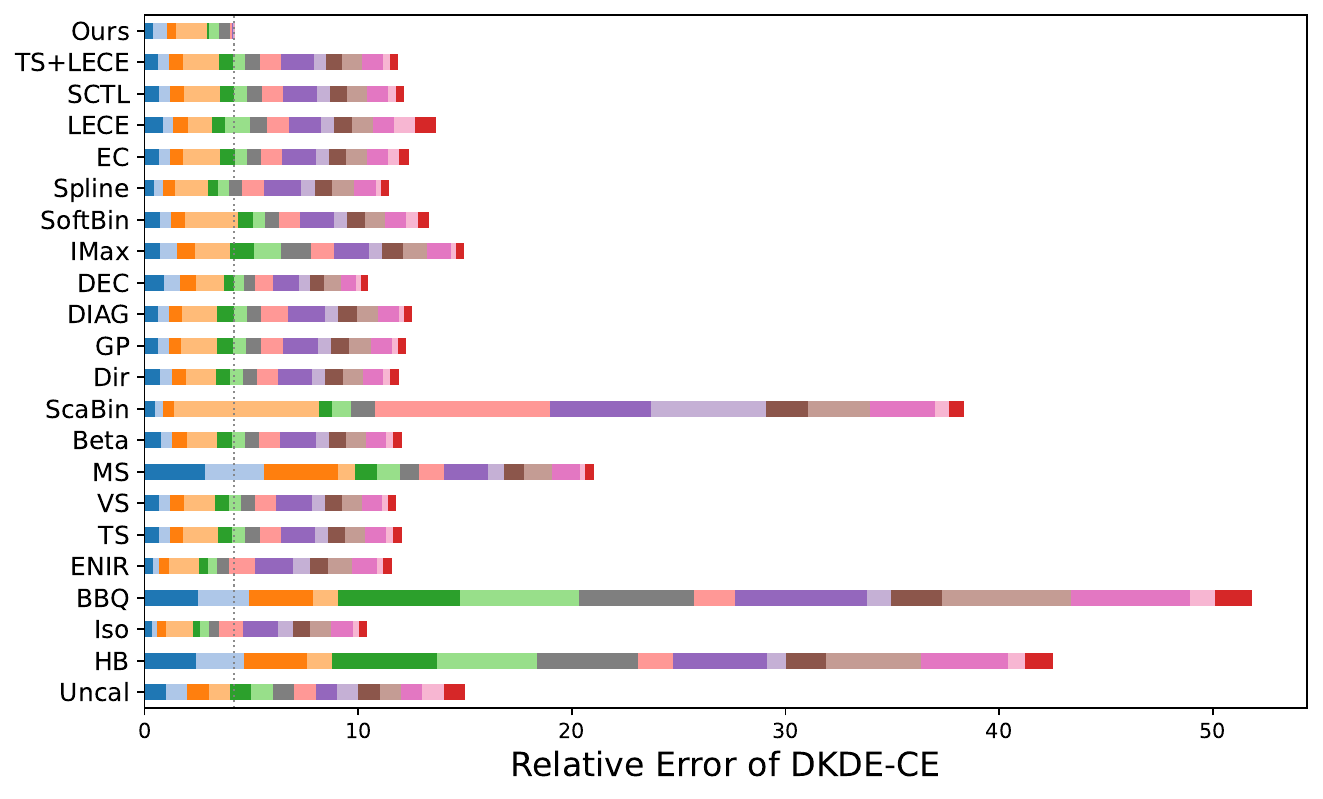}
    \end{minipage}
    \begin{minipage}{0.48\textwidth}
        \includegraphics[width=\linewidth]{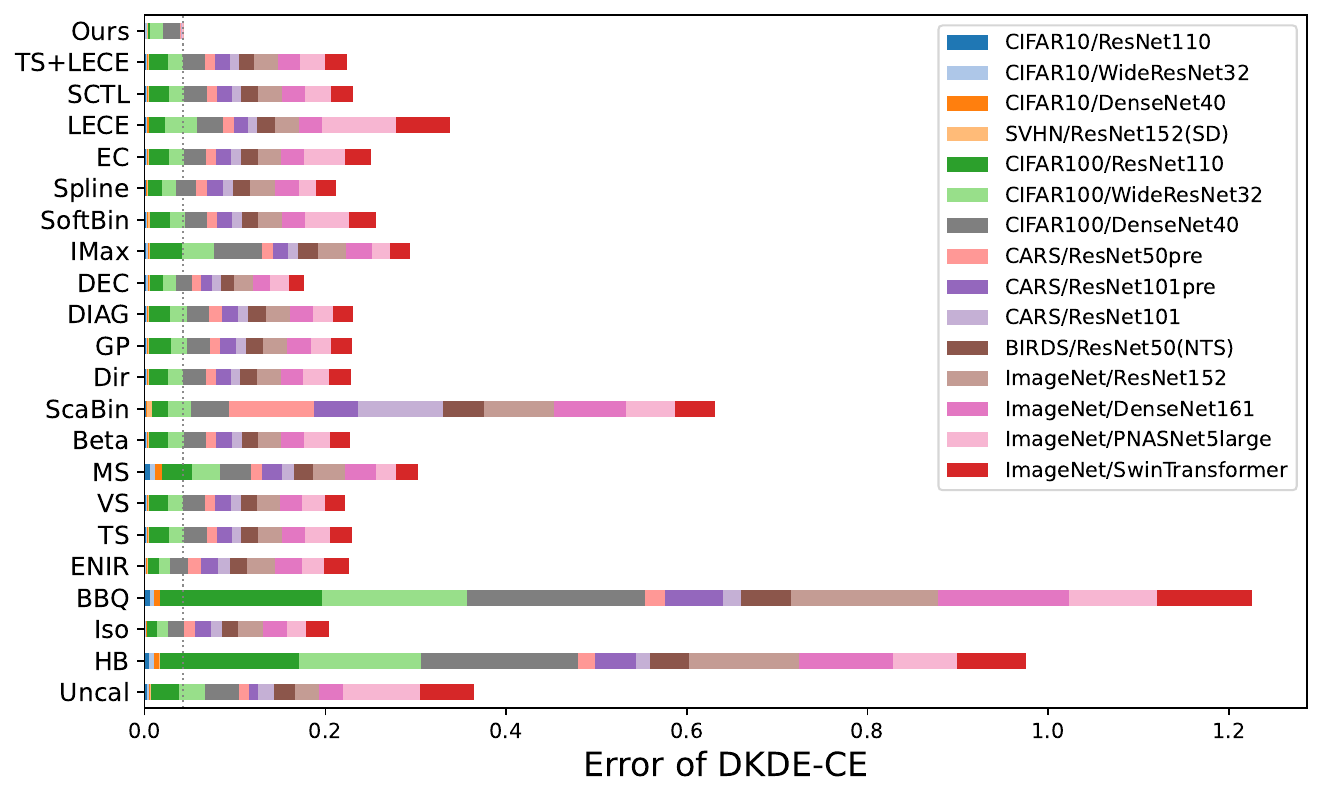}
    \end{minipage}    
    \caption{Metric-specific relative/absolute calibration errors across all tasks – Part IV}
\end{figure}

\clearpage
\vspace*{-2cm}
\begin{figure}[!t]
    \centering
    \begin{minipage}{0.48\textwidth}
        \includegraphics[width=\linewidth]{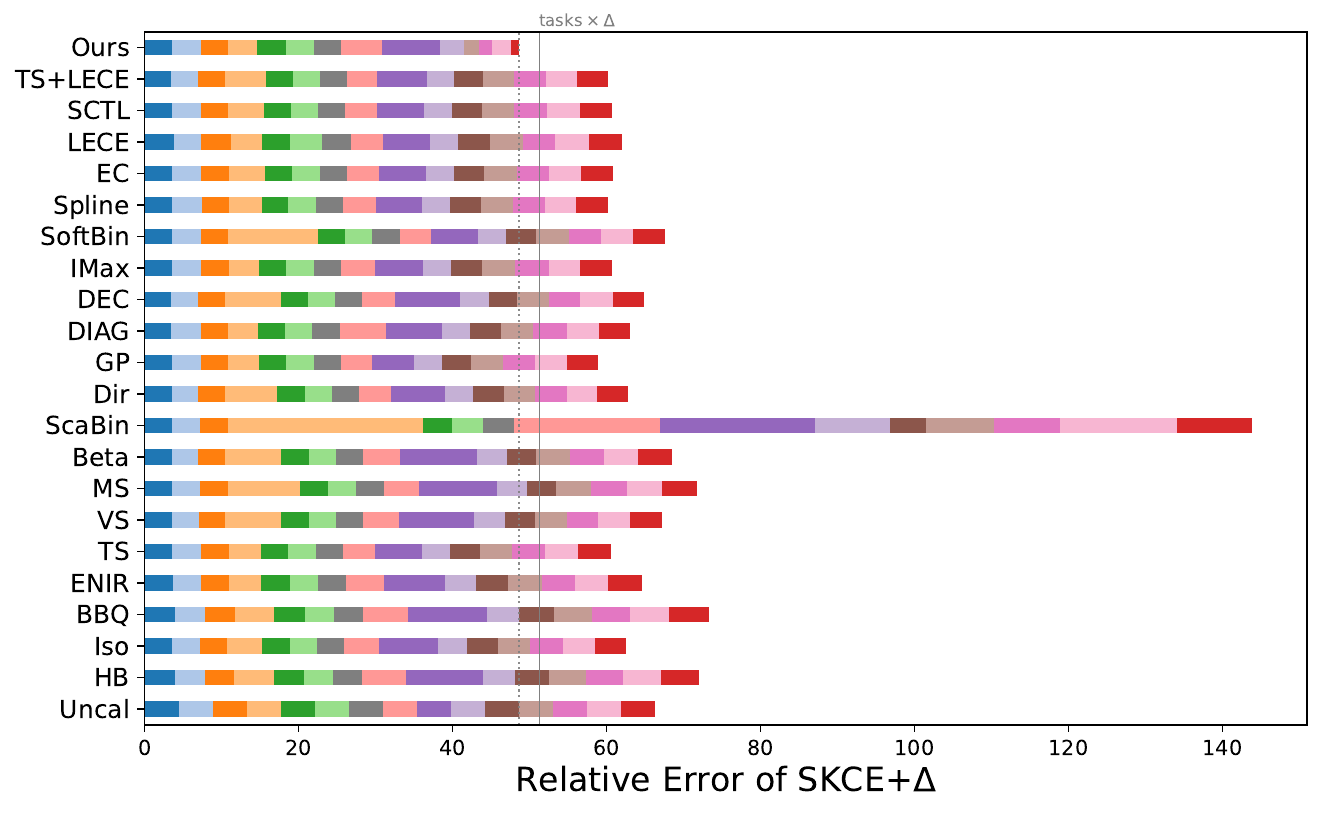}
    \end{minipage}
    \begin{minipage}{0.48\textwidth}
        \includegraphics[width=\linewidth]{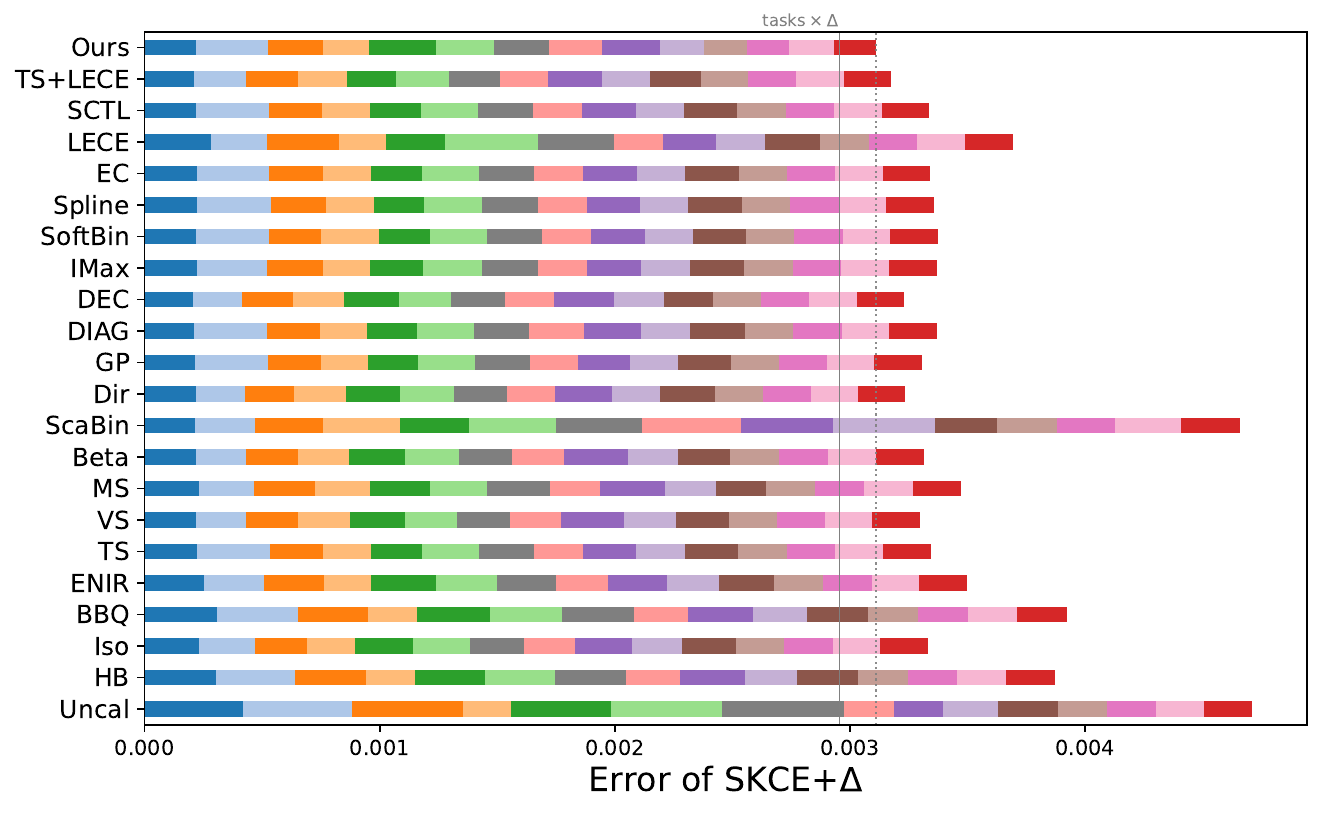}
    \end{minipage}                
    \caption{Metric-specific relative/absolute calibration errors across all tasks – Part V}
\end{figure}
}

\begin{landscape}
\section{Overall Comparison by ARE and AE across All Metrics for Different Methods}
\label{sec:apdx-overall_are_ae_number}
{\mdseries

\begin{table}[H]
\centering
\caption{Overall Comparison by Average Relative Calibration Error (ARE) (Best in \textcolor{red}{Red}, Second-best in \textcolor{blue}{Blue})}
\setlength{\tabcolsep}{2pt}
\resizebox{\linewidth}{!}{
\begin{tabular}{l l  c c c c c c c c c c c c c c c c c c c c c c}
\hline
Metric & Uncal & HB & Iso & BBQ & ENIR & TS & VS & MS & Beta & ScaBin & Dir & GP & DIAG & DEC & IMax & SoftBin & Spline & EC & LECE & SCTL & TS+LECE & Ours \\
\hline
ECE$_{r=1}^s$ & 1.0000 & 0.7986 & 0.5966 & 0.8901 & 0.8216 & 0.3844 & 0.5038 & 0.5330 & 0.5648 & 0.6033 & 0.4509 & 0.2302 & 0.2339 & 0.6208 & 0.3039 & 0.5313 & \textcolor{blue}{0.2146} & 0.3591 & 0.7398 & 0.3968 & 0.3951 & \textcolor{red}{0.1729} \\
ECE$_{r=2}^s$ & 1.0000 & 0.7142 & 0.5678 & 0.8552 & 0.7801 & 0.4188 & 0.5513 & 0.5489 & 0.6096 & 0.6003 & 0.4601 & 0.2421 & 0.2360 & 0.5655 & 0.3121 & 0.5800 & \textcolor{blue}{0.2125} & 0.3539 & 0.7360 & 0.4608 & 0.4547 & \textcolor{red}{0.1704} \\
KDE-ECE & 1.0000 & 0.9613 & 0.6070 & 0.9817 & 0.8299 & 0.4813 & 0.5518 & 0.6881 & 0.6067 & 0.9840 & 0.5329 & 0.4067 & \textcolor{blue}{0.3828} & 0.7873 & 0.4520 & 0.6775 & 0.4182 & 0.4665 & 0.7821 & 0.5155 & 0.5308 & \textcolor{red}{0.3612} \\
MMCE & 1.0000 & 0.4501 & 0.5522 & 0.3609 & 0.7494 & 0.2793 & 0.4686 & 0.4881 & 0.5476 & 0.5280 & 0.4288 & \textcolor{blue}{0.1653} & 0.1804 & 0.5349 & 0.2329 & 0.4048 & 0.1778 & 0.2569 & 0.7133 & 0.2857 & 0.2875 & \textcolor{red}{0.1389} \\
KS error & 1.0000 & 0.6713 & 0.5935 & 0.4859 & 0.8156 & 0.2850 & 0.4919 & 0.5257 & 0.5274 & 0.5681 & 0.4392 & 0.1851 & \textcolor{blue}{0.1702} & 0.5387 & 0.2313 & 0.4313 & 0.1862 & 0.2313 & 0.7203 & 0.3150 & 0.3223 & \textcolor{red}{0.1439} \\
ECE$^{\mathrm{em}}$ & 1.0000 & 1.1414 & 0.5867 & 1.0284 & 0.8031 & 0.3907 & 0.4923 & 0.7026 & 0.5601 & 0.5984 & 0.4670 & \textcolor{blue}{0.2477} & 0.2534 & 0.6186 & 0.3316 & 0.5339 & 0.2925 & 0.3767 & 0.7483 & 0.4100 & 0.4022 & \textcolor{red}{0.2447} \\
ACE & 1.0000 & 1.4790 & 0.6222 & 1.6213 & 0.9064 & 0.5753 & 0.6213 & 0.9510 & 0.5786 & 1.0657 & 0.5547 & 0.4285 & 0.4184 & 0.5271 & 0.4408 & 0.5328 & \textcolor{blue}{0.4174} & 0.4741 & 0.8120 & 0.5805 & 0.5906 & \textcolor{red}{0.3655} \\
\hline
\end{tabular}
}
\end{table}

\begin{table}[H]
\centering
\caption{Overall Comparison by Average Calibration Error (AE) (Best in \textcolor{red}{Red}, Second-best in \textcolor{blue}{Blue})}
\setlength{\tabcolsep}{2pt}
\resizebox{\linewidth}{!}{
\begin{tabular}{l l  c c c c c c c c c c c c c c c c c c c c c c c}
\hline
Metric & Uncal & HB & Iso & BBQ & ENIR & TS & VS & MS & Beta & ScaBin & Dir & GP & DIAG & DEC & IMax & SoftBin & Spline & EC & LECE & SCTL & TS+LECE & Ours &   \\
\hline
ECE$_{r=1}^s$ & 0.0759 & 0.0468 & 0.0331 & 0.0495 & 0.0473 & 0.0178 & 0.0224 & 0.0290 & 0.0276 & 0.0250 & 0.0224 & \textcolor{blue}{0.0101} & 0.0108 & 0.0295 & 0.0196 & 0.0174 & 0.0111 & 0.0156 & 0.0524 & 0.0171 & 0.0170 & \textcolor{red}{0.0075} & ($\times 1$) \\
ECE$_{r=2}^s$ & 0.1001 & 0.0524 & 0.0397 & 0.0574 & 0.0571 & 0.0238 & 0.0300 & 0.0395 & 0.0361 & 0.0313 & 0.0287 & \textcolor{blue}{0.0140} & 0.0141 & 0.0335 & 0.0253 & 0.0233 & 0.0159 & 0.0193 & 0.0676 & 0.0238 & 0.0237 & \textcolor{red}{0.0100} & ($\times 1$) \\
KDE-ECE & 0.0718 & 0.0493 & 0.0296 & 0.0504 & 0.0428 & 0.0180 & 0.0213 & 0.0294 & 0.0254 & 0.0339 & 0.0218 & 0.0143 & \textcolor{blue}{0.0137} & 0.0315 & 0.0209 & 0.0193 & 0.0153 & 0.0165 & 0.0502 & 0.0182 & 0.0190 & \textcolor{red}{0.0125} & ($\times 1$) \\
MMCE & 0.0644 & 0.0200 & 0.0266 & 0.0170 & 0.0373 & 0.0114 & 0.0178 & 0.0217 & 0.0226 & 0.0173 & 0.0177 & \textcolor{blue}{0.0058} & 0.0063 & 0.0204 & 0.0118 & 0.0103 & 0.0076 & 0.0088 & 0.0405 & 0.0107 & 0.0106 & \textcolor{red}{0.0051} & ($\times 1$) \\
KS error & 0.0758 & 0.0359 & 0.0329 & 0.0258 & 0.0472 & 0.0131 & 0.0214 & 0.0280 & 0.0254 & 0.0205 & 0.0216 & 0.0075 & \textcolor{blue}{0.0068} & 0.0224 & 0.0153 & 0.0118 & 0.0098 & 0.0088 & 0.0488 & 0.0128 & 0.0129 & \textcolor{red}{0.0061} & ($\times 1$) \\
ECE$^{\mathrm{em}}$ & 0.0762 & 0.0617 & 0.0335 & 0.0567 & 0.0475 & 0.0185 & 0.0223 & 0.0345 & 0.0281 & 0.0255 & 0.0231 & \textcolor{blue}{0.0110} & 0.0118 & 0.0303 & 0.0211 & 0.0185 & 0.0135 & 0.0167 & 0.0532 & 0.0183 & 0.0176 & \textcolor{red}{0.0103} & ($\times 1$) \\
\hline
\end{tabular}
}
\end{table}

}
\end{landscape}

\begin{landscape}
\section{Calibration Results}
\subsection{Top-label Calibration Results}
\label{sec:apdx-topcalibrationnumber}
{\mdseries

\subsubsection{ECE$^s_{r=1}$}
\label{sec:tab-cmp-sweepECE1}
\begin{table}[H]
\centering
\caption{ECE$_{r=1}^s$ across All Tasks (Best in \textcolor{red}{Red}, Second-best in \textcolor{blue}{Blue})}
\setlength{\tabcolsep}{2pt}
\resizebox{\linewidth}{!}{
\begin{tabular}{l l  c c c c c c c c c c c c c c c c c c c c c c}
\hline
Dataset & Model & Uncal & HB & Iso & BBQ & ENIR & TS & VS & MS & Beta & ScaBin & Dir & GP & DIAG & DEC & IMax & SoftBin & Spline & EC & LECE & SCTL & TS+LECE & Ours \\
\hline
CIFAR10 & ResNet110 & 0.0475 & 0.0111 & 0.0093 & 0.0095 & 0.0183 & 0.0091 & 0.0107 & 0.0170 & 0.0096 & 0.0087 & 0.0106 & 0.0083 & 0.0081 & 0.0091 & 0.0078 & \textcolor{red}{0.0030} & 0.0131 & 0.0063 & 0.0331 & 0.0061 & 0.0054 & \textcolor{blue}{0.0043} \\
CIFAR10 & WideResNet32 & 0.0448 & 0.0072 & 0.0077 & 0.0062 & 0.0139 & 0.0021 & 0.0034 & 0.0114 & 0.0061 & 0.0090 & 0.0025 & \textcolor{red}{0.0013} & 0.0030 & 0.0045 & 0.0108 & 0.0036 & 0.0089 & 0.0020 & 0.0205 & 0.0032 & 0.0027 & \textcolor{blue}{0.0015} \\
CIFAR10 & DenseNet40 & 0.0549 & 0.0160 & 0.0161 & 0.0167 & 0.0261 & 0.0090 & 0.0112 & 0.0211 & 0.0105 & \textcolor{red}{0.0020} & 0.0098 & 0.0082 & 0.0075 & 0.0108 & 0.0147 & \textcolor{blue}{0.0029} & 0.0161 & 0.0104 & 0.0366 & 0.0067 & 0.0067 & 0.0073 \\
SVHN & ResNet152(SD) & 0.0084 & 0.0061 & \textcolor{blue}{0.0029} & 0.0043 & \textcolor{red}{0.0027} & 0.0057 & 0.0057 & 0.0045 & 0.0060 & 0.0054 & 0.0058 & 0.0050 & 0.0042 & 0.0069 & 0.0045 & 0.0224 & 0.0050 & 0.0078 & 0.0056 & 0.0071 & 0.0078 & 0.0059 \\
CIFAR100 & ResNet110 & 0.1848 & 0.0950 & 0.0605 & 0.0863 & 0.0889 & 0.0187 & 0.0250 & 0.0413 & 0.0334 & 0.0308 & 0.0346 & \textcolor{blue}{0.0140} & 0.0285 & 0.0314 & 0.0446 & 0.0152 & 0.0160 & 0.0154 & 0.0637 & 0.0154 & 0.0182 & \textcolor{red}{0.0118} \\
CIFAR100 & WideResNet32 & 0.1878 & 0.0808 & 0.0564 & 0.0831 & 0.0755 & 0.0134 & 0.0177 & 0.0360 & 0.0257 & 0.0385 & 0.0177 & 0.0107 & 0.0103 & 0.0326 & 0.0517 & 0.0136 & 0.0259 & 0.0107 & 0.1523 & \textcolor{blue}{0.0101} & \textcolor{red}{0.0088} & 0.0149 \\
\hline
\end{tabular}
}
\end{table}

\subsubsection{ECE$^s_{r=2}$} 
\label{sec:tab-cmp-sweepECE2}
\begin{table}[H]
\centering
\caption{ECE$_{r=2}^s$ across All Tasks (Best in \textcolor{red}{Red}, Second-best in \textcolor{blue}{Blue})}
\label{tab:apdx-tab-cmp-sweepECE2}
\setlength{\tabcolsep}{2pt}
\resizebox{\linewidth}{!}{
\begin{tabular}{l l  c c c c c c c c c c c c c c c c c c c c c c}
\hline
Dataset & Model & Uncal & HB & Iso & BBQ & ENIR & TS & VS & MS & Beta & ScaBin & Dir & GP & DIAG & DEC & IMax & SoftBin & Spline & EC & LECE & SCTL & TS+LECE & Ours \\
\hline
CIFAR10 & ResNet110 & 0.0908 & 0.0119 & 0.0114 & 0.0099 & 0.0269 & 0.0158 & 0.0184 & 0.0331 & 0.0149 & 0.0163 & 0.0178 & 0.0142 & 0.0154 & 0.0149 & 0.0116 & \textcolor{red}{0.0058} & 0.0231 & 0.0117 & 0.0625 & 0.0115 & 0.0088 & \textcolor{blue}{0.0065} \\
CIFAR10 & WideResNet32 & 0.0850 & 0.0078 & 0.0113 & 0.0075 & 0.0209 & 0.0031 & 0.0050 & 0.0156 & 0.0112 & 0.0155 & 0.0034 & \textcolor{red}{0.0018} & 0.0040 & 0.0061 & 0.0128 & 0.0054 & 0.0124 & 0.0032 & 0.0367 & 0.0046 & 0.0046 & \textcolor{blue}{0.0022} \\
CIFAR10 & DenseNet40 & 0.1025 & 0.0162 & 0.0244 & 0.0179 & 0.0404 & 0.0143 & 0.0180 & 0.0451 & 0.0156 & \textcolor{red}{0.0027} & 0.0147 & 0.0130 & 0.0119 & 0.0157 & 0.0267 & \textcolor{blue}{0.0042} & 0.0287 & 0.0164 & 0.0659 & 0.0105 & 0.0102 & 0.0105 \\
SVHN & ResNet152(SD) & 0.0132 & 0.0064 & \textcolor{blue}{0.0048} & 0.0059 & \textcolor{red}{0.0036} & 0.0092 & 0.0091 & 0.0067 & 0.0093 & 0.0072 & 0.0094 & 0.0069 & 0.0054 & 0.0093 & 0.0073 & 0.0327 & 0.0073 & 0.0126 & 0.0076 & 0.0115 & 0.0145 & 0.0078 \\
CIFAR100 & ResNet110 & 0.2284 & 0.1117 & 0.0757 & 0.1048 & 0.1043 & 0.0221 & 0.0312 & 0.0649 & 0.0407 & 0.0380 & 0.0439 & 0.0190 & 0.0345 & 0.0375 & 0.0574 & \textcolor{blue}{0.0176} & 0.0223 & 0.0179 & 0.0684 & 0.0180 & 0.0207 & \textcolor{red}{0.0147} \\
CIFAR100 & WideResNet32 & 0.2471 & 0.1000 & 0.0694 & 0.0986 & 0.0922 & 0.0162 & 0.0204 & 0.0576 & 0.0278 & 0.0433 & 0.0206 & 0.0177 & 0.0140 & 0.0369 & 0.0619 & 0.0163 & 0.0384 & 0.0137 & 0.2077 & \textcolor{blue}{0.0134} & \textcolor{red}{0.0125} & 0.0203 \\
\hline
\end{tabular}
}
\end{table}

\subsubsection{KDE-ECE}
\label{sec:tab-cmp-kdeECEt1}
\begin{table}[H]
\centering
\caption{KDE-ECE across All Tasks (Best in \textcolor{red}{Red}, Second-best in \textcolor{blue}{Blue})}
\setlength{\tabcolsep}{2pt}
\resizebox{\linewidth}{!}{
\begin{tabular}{l l  c c c c c c c c c c c c c c c c c c c c c c}
\hline
Dataset & Model & Uncal & HB & Iso & BBQ & ENIR & TS & VS & MS & Beta & ScaBin & Dir & GP & DIAG & DEC & IMax & SoftBin & Spline & EC & LECE & SCTL & TS+LECE & Ours \\
\hline
CIFAR10 & ResNet110 & 0.0434 & 0.0166 & 0.0135 & 0.0150 & 0.0165 & 0.0124 & 0.0140 & 0.0181 & 0.0142 & 0.0178 & 0.0127 & 0.0118 & 0.0122 & 0.0131 & 0.0165 & \textcolor{blue}{0.0110} & 0.0175 & 0.0114 & 0.0301 & 0.0113 & 0.0116 & \textcolor{red}{0.0100} \\
CIFAR10 & WideResNet32 & 0.0406 & 0.0144 & 0.0127 & 0.0134 & 0.0169 & 0.0117 & 0.0122 & 0.0167 & 0.0121 & 0.0207 & 0.0126 & \textcolor{blue}{0.0114} & 0.0115 & 0.0132 & 0.0187 & \textcolor{blue}{0.0114} & 0.0138 & 0.0117 & 0.0191 & \textcolor{blue}{0.0114} & 0.0122 & \textcolor{red}{0.0112} \\
CIFAR10 & DenseNet40 & 0.0501 & 0.0216 & 0.0161 & 0.0197 & 0.0215 & 0.0114 & 0.0141 & 0.0224 & 0.0136 & 0.0153 & 0.0131 & 0.0108 & 0.0108 & 0.0153 & 0.0189 & \textcolor{red}{0.0099} & 0.0196 & 0.0120 & 0.0330 & 0.0105 & 0.0128 & \textcolor{blue}{0.0101} \\
SVHN & ResNet152(SD) & 0.0074 & 0.0124 & 0.0071 & 0.0112 & \textcolor{blue}{0.0066} & 0.0095 & 0.0104 & 0.0158 & 0.0101 & 0.0146 & 0.0100 & 0.0091 & 0.0088 & 0.0110 & 0.0083 & 0.0239 & 0.0117 & 0.0113 & \textcolor{red}{0.0060} & 0.0108 & 0.0106 & 0.0083 \\
CIFAR100 & ResNet110 & 0.1782 & 0.0894 & 0.0556 & 0.0850 & 0.0830 & 0.0171 & 0.0228 & 0.0409 & 0.0291 & 0.0346 & 0.0298 & 0.0175 & 0.0276 & 0.0318 & 0.0371 & 0.0150 & 0.0153 & \textcolor{blue}{0.0147} & 0.0594 & 0.0153 & 0.0158 & \textcolor{red}{0.0129} \\
CIFAR100 & WideResNet32 & 0.1816 & 0.0775 & 0.0499 & 0.0783 & 0.0685 & 0.0111 & 0.0148 & 0.0363 & 0.0195 & 0.0452 & 0.0148 & 0.0125 & \textcolor{blue}{0.0094} & 0.0309 & 0.0431 & 0.0112 & 0.0271 & 0.0096 & 0.1523 & \textcolor{red}{0.0089} & 0.0096 & 0.0133 \\
\hline
\end{tabular}
}
\end{table}

\subsubsection{MMCE}
\label{sec:tab-cmp-MMCE}
\begin{table}[H]
\centering
\caption{MMCE across All Tasks (Best in \textcolor{red}{Red}, Second-best in \textcolor{blue}{Blue})}
\setlength{\tabcolsep}{2pt}
\resizebox{\linewidth}{!}{
\begin{tabular}{l l  c c c c c c c c c c c c c c c c c c c c c c}
\hline
Dataset & Model & Uncal & HB & Iso & BBQ & ENIR & TS & VS & MS & Beta & ScaBin & Dir & GP & DIAG & DEC & IMax & SoftBin & Spline & EC & LECE & SCTL & TS+LECE & Ours \\
\hline
CIFAR10 & ResNet110 & 0.0428 & 0.0090 & 0.0096 & 0.0079 & 0.0159 & 0.0081 & 0.0090 & 0.0127 & 0.0091 & 0.0046 & 0.0092 & 0.0056 & 0.0058 & 0.0055 & 0.0072 & \textcolor{red}{0.0023} & 0.0074 & 0.0057 & 0.0287 & 0.0056 & 0.0051 & \textcolor{blue}{0.0043} \\
CIFAR10 & WideResNet32 & 0.0406 & 0.0059 & 0.0069 & 0.0053 & 0.0130 & 0.0030 & 0.0036 & 0.0074 & 0.0045 & 0.0071 & 0.0037 & \textcolor{blue}{0.0018} & 0.0026 & 0.0021 & 0.0058 & 0.0027 & 0.0055 & 0.0031 & 0.0174 & 0.0026 & 0.0023 & \textcolor{red}{0.0014} \\
CIFAR10 & DenseNet40 & 0.0487 & 0.0130 & 0.0135 & 0.0138 & 0.0218 & 0.0073 & 0.0091 & 0.0164 & 0.0092 & \textcolor{red}{0.0019} & 0.0082 & 0.0064 & 0.0061 & 0.0062 & 0.0113 & \textcolor{blue}{0.0020} & 0.0132 & 0.0084 & 0.0312 & 0.0054 & 0.0059 & 0.0061 \\
SVHN & ResNet152(SD) & 0.0074 & 0.0050 & \textcolor{blue}{0.0018} & 0.0054 & \textcolor{red}{0.0009} & 0.0048 & 0.0047 & 0.0058 & 0.0045 & 0.0054 & 0.0048 & 0.0046 & 0.0039 & 0.0065 & 0.0035 & 0.0197 & 0.0053 & 0.0068 & 0.0050 & 0.0062 & 0.0062 & 0.0025 \\
CIFAR100 & ResNet110 & 0.1527 & 0.0320 & 0.0488 & 0.0269 & 0.0683 & 0.0134 & 0.0195 & 0.0322 & 0.0260 & 0.0157 & 0.0269 & \textcolor{blue}{0.0070} & 0.0133 & 0.0200 & 0.0292 & 0.0091 & 0.0101 & 0.0083 & 0.0268 & 0.0097 & 0.0086 & \textcolor{red}{0.0066} \\
CIFAR100 & WideResNet32 & 0.1587 & 0.0274 & 0.0447 & 0.0232 & 0.0579 & 0.0092 & 0.0130 & 0.0282 & 0.0201 & 0.0210 & 0.0135 & 0.0051 & \textcolor{blue}{0.0040} & 0.0215 & 0.0246 & 0.0094 & 0.0196 & 0.0071 & 0.1281 & 0.0060 & \textcolor{red}{0.0038} & 0.0112 \\
\hline
\end{tabular}
}
\end{table}

\subsubsection{KS Error}
\label{sec:tab-cmp-sKSCE}
\begin{table}[H]
\centering
\caption{KS Error across All Tasks (Best in \textcolor{red}{Red}, Second-best in \textcolor{blue}{Blue})}
\setlength{\tabcolsep}{2pt}
\resizebox{\linewidth}{!}{
\begin{tabular}{l l  c c c c c c c c c c c c c c c c c c c c c c}
\hline
Dataset & Model & Uncal & HB & Iso & BBQ & ENIR & TS & VS & MS & Beta & ScaBin & Dir & GP & DIAG & DEC & IMax & SoftBin & Spline & EC & LECE & SCTL & TS+LECE & Ours \\
\hline
CIFAR10 & ResNet110 & 0.0475 & 0.0112 & 0.0093 & 0.0095 & 0.0183 & 0.0092 & 0.0103 & 0.0173 & 0.0096 & 0.0055 & 0.0107 & 0.0076 & 0.0080 & 0.0074 & 0.0082 & \textcolor{red}{0.0024} & 0.0100 & 0.0064 & 0.0331 & 0.0063 & 0.0055 & \textcolor{blue}{0.0043} \\
CIFAR10 & WideResNet32 & 0.0448 & 0.0114 & 0.0077 & 0.0118 & 0.0139 & 0.0029 & 0.0032 & 0.0114 & 0.0038 & 0.0090 & 0.0028 & \textcolor{blue}{0.0025} & 0.0035 & \textcolor{blue}{0.0025} & 0.0092 & 0.0039 & 0.0082 & 0.0029 & 0.0205 & 0.0037 & 0.0034 & \textcolor{red}{0.0020} \\
CIFAR10 & DenseNet40 & 0.0549 & 0.0160 & 0.0161 & 0.0168 & 0.0261 & 0.0090 & 0.0112 & 0.0212 & 0.0106 & \textcolor{red}{0.0027} & 0.0098 & 0.0082 & 0.0076 & 0.0091 & 0.0145 & \textcolor{red}{0.0027} & 0.0159 & 0.0104 & 0.0366 & \textcolor{blue}{0.0067} & \textcolor{blue}{0.0067} & 0.0073 \\
SVHN & ResNet152(SD) & 0.0085 & 0.0046 & \textcolor{blue}{0.0025} & 0.0045 & \textcolor{red}{0.0018} & 0.0055 & 0.0056 & 0.0056 & 0.0055 & 0.0052 & 0.0058 & 0.0048 & 0.0040 & 0.0074 & 0.0047 & 0.0224 & 0.0050 & 0.0078 & 0.0057 & 0.0071 & 0.0076 & 0.0045 \\
CIFAR100 & ResNet110 & 0.1848 & 0.0597 & 0.0605 & 0.0468 & 0.0889 & 0.0149 & 0.0239 & 0.0409 & 0.0317 & 0.0163 & 0.0333 & 0.0095 & 0.0142 & 0.0154 & 0.0404 & 0.0077 & 0.0146 & \textcolor{red}{0.0060} & 0.0267 & 0.0088 & \textcolor{blue}{0.0069} & 0.0075 \\
CIFAR100 & WideResNet32 & 0.1878 & 0.0530 & 0.0564 & 0.0494 & 0.0755 & 0.0116 & 0.0177 & 0.0361 & 0.0257 & 0.0223 & 0.0177 & 0.0072 & \textcolor{blue}{0.0046} & 0.0171 & 0.0415 & 0.0118 & 0.0250 & 0.0085 & 0.1529 & 0.0066 & \textcolor{red}{0.0033} & 0.0146 \\
\hline
\end{tabular}
}
\end{table}

\subsubsection{ECE$^{\text{em}}$}
\label{sec:tab-cmp-adaECE}
\begin{table}[H]
\centering
\caption{ECE$^{\mathrm{em}}$ across All Tasks (Best in \textcolor{red}{Red}, Second-best in \textcolor{blue}{Blue})}
\setlength{\tabcolsep}{2pt}
\resizebox{\linewidth}{!}{
\begin{tabular}{l l  c c c c c c c c c c c c c c c c c c c c c c}
\hline
Dataset & Model & Uncal & HB & Iso & BBQ & ENIR & TS & VS & MS & Beta & ScaBin & Dir & GP & DIAG & DEC & IMax & SoftBin & Spline & EC & LECE & SCTL & TS+LECE & Ours \\
\hline
CIFAR10 & ResNet110 & 0.0475 & 0.0256 & 0.0140 & 0.0246 & 0.0183 & 0.0094 & 0.0110 & 0.0311 & 0.0109 & 0.0068 & 0.0109 & 0.0088 & 0.0095 & 0.0097 & 0.0078 & \textcolor{red}{0.0056} & 0.0149 & 0.0076 & 0.0331 & 0.0076 & \textcolor{blue}{0.0065} & \textcolor{blue}{0.0065} \\
CIFAR10 & WideResNet32 & 0.0447 & 0.0302 & 0.0088 & 0.0306 & 0.0166 & 0.0069 & 0.0072 & 0.0179 & 0.0068 & 0.0096 & 0.0070 & \textcolor{blue}{0.0041} & 0.0067 & 0.0092 & 0.0116 & 0.0065 & 0.0115 & 0.0067 & 0.0205 & 0.0065 & 0.0059 & \textcolor{red}{0.0037} \\
CIFAR10 & DenseNet40 & 0.0549 & 0.0285 & 0.0161 & 0.0273 & 0.0262 & 0.0094 & 0.0113 & 0.0402 & 0.0105 & \textcolor{blue}{0.0041} & 0.0099 & 0.0087 & 0.0079 & 0.0110 & 0.0138 & \textcolor{red}{0.0033} & 0.0172 & 0.0106 & 0.0365 & 0.0071 & 0.0069 & 0.0074 \\
SVHN & ResNet152(SD) & 0.0084 & 0.0097 & \textcolor{red}{0.0030} & 0.0091 & \textcolor{red}{0.0030} & 0.0058 & 0.0061 & 0.0159 & 0.0060 & 0.0075 & 0.0061 & 0.0050 & \textcolor{blue}{0.0043} & 0.0078 & 0.0047 & 0.0224 & 0.0081 & 0.0079 & 0.0056 & 0.0072 & 0.0079 & 0.0060 \\
CIFAR100 & ResNet110 & 0.1848 & 0.0967 & 0.0605 & 0.0858 & 0.0889 & 0.0211 & 0.0255 & 0.0470 & 0.0344 & 0.0310 & 0.0338 & 0.0183 & 0.0276 & 0.0338 & 0.0469 & 0.0202 & \textcolor{blue}{0.0165} & 0.0192 & 0.0637 & 0.0201 & 0.0201 & \textcolor{red}{0.0143} \\
CIFAR100 & WideResNet32 & 0.1878 & 0.0830 & 0.0565 & 0.0828 & 0.0755 & 0.0145 & 0.0178 & 0.0455 & 0.0257 & 0.0343 & 0.0176 & \textcolor{red}{0.0082} & 0.0107 & 0.0332 & 0.0538 & 0.0147 & 0.0262 & 0.0129 & 0.1555 & 0.0117 & \textcolor{blue}{0.0086} & 0.0149 \\
\hline
\end{tabular}
}
\end{table}

\subsubsection{ACE}
\label{sec:tab-cmp-aECE}
\begin{table}[H]
\centering
\caption{ACE across All Tasks (Best in \textcolor{red}{Red}, Second-best in \textcolor{blue}{Blue})}
\setlength{\tabcolsep}{2pt}
\resizebox{\linewidth}{!}{
\begin{tabular}{l l  c c c c c c c c c c c c c c c c c c c c c c}
\hline
Dataset & Model & Uncal & HB & Iso & BBQ & ENIR & TS & VS & MS & Beta & ScaBin & Dir & GP & DIAG & DEC & IMax & SoftBin & Spline & EC & LECE & SCTL & TS+LECE & Ours \\
\hline
CIFAR10 & ResNet110 & 0.1820 & 0.1777 & 0.0702 & 0.2050 & 0.0595 & 0.0705 & 0.0579 & 0.1718 & \textcolor{blue}{0.0478} & 0.1034 & \textcolor{red}{0.0396} & 0.0622 & 0.0837 & 0.0626 & 0.0570 & 0.0511 & 0.0795 & 0.0569 & 0.1604 & 0.0579 & 0.0507 & 0.0527 \\
CIFAR10 & WideResNet32 & 0.2131 & 0.1718 & 0.0654 & 0.1634 & 0.0694 & 0.0350 & 0.0482 & 0.1205 & \textcolor{blue}{0.0318} & 0.0752 & 0.0655 & 0.0421 & 0.0378 & 0.0421 & 0.1201 & 0.0421 & 0.0697 & 0.0397 & 0.0930 & 0.0397 & 0.0634 & \textcolor{red}{0.0260} \\
CIFAR10 & DenseNet40 & 0.1961 & 0.1895 & 0.0357 & 0.2039 & 0.1308 & \textcolor{blue}{0.0309} & 0.0547 & 0.1680 & 0.0493 & 0.0642 & 0.0390 & 0.0358 & \textcolor{blue}{0.0309} & 0.0539 & 0.0519 & 0.0314 & 0.0813 & 0.0348 & 0.1393 & \textcolor{red}{0.0276} & 0.0408 & 0.0397 \\
SVHN & ResNet152(SD) & 0.1038 & 0.1311 & \textcolor{red}{0.0319} & 0.2236 & 0.0672 & 0.0553 & 0.0681 & 0.1829 & 0.0800 & \textcolor{blue}{0.0376} & 0.0710 & 0.0663 & 0.0707 & 0.0504 & 0.0588 & 0.0909 & 0.0756 & 0.0659 & 0.0674 & 0.0567 & 0.0693 & 0.0495 \\
CIFAR100 & ResNet110 & 0.2466 & 0.1526 & 0.0734 & 0.1471 & 0.1072 & 0.0292 & 0.0329 & 0.0914 & 0.0368 & 0.0532 & 0.0425 & 0.0250 & 0.0410 & 0.0258 & 0.0557 & 0.0268 & \textcolor{blue}{0.0228} & 0.0254 & 0.0730 & 0.0256 & 0.0294 & \textcolor{red}{0.0220} \\
CIFAR100 & WideResNet32 & 0.2945 & 0.1497 & 0.0662 & 0.1360 & 0.0960 & 0.0182 & 0.0266 & 0.0961 & 0.0266 & 0.0750 & 0.0224 & 0.0206 & \textcolor{red}{0.0122} & 0.0286 & 0.0645 & 0.0183 & 0.0426 & 0.0159 & 0.2739 & \textcolor{blue}{0.0158} & 0.0160 & 0.0206 \\
\hline
\end{tabular}
}
\end{table}

\subsubsection{dECE}
\label{sec:tab-cmp-deadaeCE}
\begin{table}[H]
\centering
\caption{dECE across All Tasks (Best in \textcolor{red}{Red}, Second-best in \textcolor{blue}{Blue})}
\setlength{\tabcolsep}{2pt}
\resizebox{\linewidth}{!}{
\begin{tabular}{l l  c c c c c c c c c c c c c c c c c c c c c c}
\hline
Dataset & Model & Uncal & HB & Iso & BBQ & ENIR & TS & VS & MS & Beta & ScaBin & Dir & GP & DIAG & DEC & IMax & SoftBin & Spline & EC & LECE & SCTL & TS+LECE & Ours \\
\hline
CIFAR10 & ResNet110 & 0.0474 & 0.0270 & 0.0138 & 0.0243 & 0.0181 & 0.0078 & 0.0102 & 0.0357 & 0.0099 & 0.0049 & 0.0100 & 0.0080 & 0.0089 & 0.0084 & 0.0080 & \textcolor{red}{0.0034} & 0.0144 & 0.0060 & 0.0330 & 0.0059 & \textcolor{blue}{0.0048} & 0.0049 \\
CIFAR10 & WideResNet32 & 0.0447 & 0.0272 & 0.0080 & 0.0247 & 0.0162 & 0.0061 & 0.0066 & 0.0240 & 0.0058 & 0.0085 & 0.0060 & \textcolor{blue}{0.0026} & 0.0059 & 0.0084 & 0.0110 & 0.0059 & 0.0107 & 0.0058 & 0.0203 & 0.0057 & 0.0049 & \textcolor{red}{0.0018} \\
CIFAR10 & DenseNet40 & 0.0549 & 0.0309 & 0.0160 & 0.0294 & 0.0260 & 0.0088 & 0.0111 & 0.0363 & 0.0097 & \textcolor{blue}{0.0016} & 0.0095 & 0.0080 & 0.0071 & 0.0104 & 0.0143 & \textcolor{red}{0.0008} & 0.0158 & 0.0102 & 0.0371 & 0.0060 & 0.0060 & 0.0066 \\
SVHN & ResNet152(SD) & 0.0084 & 0.0117 & \textcolor{red}{0.0027} & 0.0120 & \textcolor{blue}{0.0028} & 0.0055 & 0.0060 & 0.0156 & 0.0058 & 0.0066 & 0.0060 & 0.0047 & 0.0037 & 0.0077 & 0.0048 & 0.0223 & 0.0079 & 0.0078 & 0.0056 & 0.0070 & 0.0076 & 0.0058 \\
CIFAR100 & ResNet110 & 0.1848 & 0.0958 & 0.0596 & 0.0837 & 0.0889 & 0.0188 & 0.0239 & 0.0455 & 0.0332 & 0.0297 & 0.0326 & 0.0155 & 0.0258 & 0.0316 & 0.0445 & 0.0191 & \textcolor{blue}{0.0133} & 0.0178 & 0.0636 & 0.0188 & 0.0182 & \textcolor{red}{0.0116} \\
CIFAR100 & WideResNet32 & 0.1878 & 0.0805 & 0.0556 & 0.0820 & 0.0755 & 0.0121 & 0.0164 & 0.0432 & 0.0244 & 0.0332 & 0.0162 & \textcolor{red}{0.0033} & 0.0071 & 0.0309 & 0.0522 & 0.0123 & 0.0248 & 0.0105 & 0.1549 & 0.0088 & \textcolor{blue}{0.0042} & 0.0122 \\
\hline
\end{tabular}
}
\end{table}

\subsubsection{ECE$^{\text{ew}}$}
\label{sec:tab-cmp-ece}
\begin{table}[H]
\centering
\caption{ECE$^{\mathrm{ew}}$ across All Tasks (Best in \textcolor{red}{Red}, Second-best in \textcolor{blue}{Blue})}
\setlength{\tabcolsep}{2pt}
\resizebox{\linewidth}{!}{
\begin{tabular}{l l  c c c c c c c c c c c c c c c c c c c c c c}
\hline
Dataset & Model & Uncal & HB & Iso & BBQ & ENIR & TS & VS & MS & Beta & ScaBin & Dir & GP & DIAG & DEC & IMax & SoftBin & Spline & EC & LECE & SCTL & TS+LECE & Ours \\
\hline
CIFAR10 & ResNet110 & 0.0475 & 0.0113 & 0.0147 & 0.0099 & 0.0188 & 0.0113 & 0.0120 & 0.0149 & 0.0131 & 0.0080 & 0.0113 & \textcolor{blue}{0.0070} & 0.0071 & 0.0081 & 0.0097 & \textcolor{red}{0.0055} & 0.0119 & 0.0086 & 0.0331 & 0.0085 & 0.0073 & 0.0075 \\
CIFAR10 & WideResNet32 & 0.0451 & 0.0073 & 0.0119 & 0.0065 & 0.0184 & 0.0080 & 0.0074 & 0.0101 & 0.0078 & 0.0099 & 0.0079 & 0.0056 & 0.0079 & \textcolor{blue}{0.0055} & 0.0108 & 0.0079 & 0.0090 & 0.0082 & 0.0205 & 0.0078 & 0.0069 & \textcolor{red}{0.0048} \\
CIFAR10 & DenseNet40 & 0.0550 & 0.0161 & 0.0168 & 0.0169 & 0.0263 & 0.0095 & 0.0114 & 0.0192 & 0.0120 & \textcolor{red}{0.0061} & 0.0108 & 0.0087 & 0.0080 & 0.0102 & 0.0136 & \textcolor{blue}{0.0063} & 0.0158 & 0.0105 & 0.0366 & 0.0071 & 0.0090 & 0.0085 \\
SVHN & ResNet152(SD) & 0.0086 & 0.0079 & \textcolor{red}{0.0024} & 0.0087 & \textcolor{blue}{0.0026} & 0.0061 & 0.0070 & 0.0114 & 0.0067 & 0.0084 & 0.0063 & 0.0063 & 0.0053 & 0.0075 & 0.0045 & 0.0226 & 0.0087 & 0.0081 & 0.0061 & 0.0071 & 0.0074 & 0.0036 \\
CIFAR100 & ResNet110 & 0.1848 & 0.0853 & 0.0620 & 0.0805 & 0.0889 & 0.0238 & 0.0273 & 0.0410 & 0.0351 & 0.0306 & 0.0366 & 0.0165 & 0.0280 & 0.0323 & 0.0431 & 0.0190 & \textcolor{blue}{0.0148} & 0.0176 & 0.0647 & 0.0184 & 0.0184 & \textcolor{red}{0.0146} \\
CIFAR100 & WideResNet32 & 0.1878 & 0.0843 & 0.0577 & 0.0614 & 0.0755 & 0.0148 & 0.0182 & 0.0360 & 0.0268 & 0.0398 & 0.0180 & \textcolor{red}{0.0098} & \textcolor{blue}{0.0105} & 0.0318 & 0.0554 & 0.0150 & 0.0244 & 0.0128 & 0.1523 & 0.0123 & 0.0109 & 0.0156 \\
\hline
\end{tabular}
}
\end{table}

\subsubsection{ECE$_{r=2}$}
\label{sec:tab-cmp-eCE_2}
\begin{table}[H]
\centering
\caption{ECE$_{r=2}$ across All Tasks (Best in \textcolor{red}{Red}, Second-best in \textcolor{blue}{Blue})}
\label{tab:apdx-tab-cmp-eCE_2}
\setlength{\tabcolsep}{2pt}
\resizebox{\linewidth}{!}{
\begin{tabular}{l l  c c c c c c c c c c c c c c c c c c c c c c}
\hline
Dataset & Model & Uncal & HB & Iso & BBQ & ENIR & TS & VS & MS & Beta & ScaBin & Dir & GP & DIAG & DEC & IMax & SoftBin & Spline & EC & LECE & SCTL & TS+LECE & Ours \\
\hline
CIFAR10 & ResNet110 & 0.0665 & 0.0428 & 0.0247 & 0.0444 & 0.0314 & 0.0250 & 0.0265 & 0.0455 & 0.0243 & 0.0197 & 0.0219 & 0.0193 & 0.0207 & 0.0207 & 0.0247 & \textcolor{blue}{0.0178} & 0.0302 & 0.0211 & 0.0591 & 0.0216 & \textcolor{red}{0.0175} & 0.0184 \\
CIFAR10 & WideResNet32 & 0.0660 & 0.0446 & 0.0206 & 0.0406 & 0.0305 & 0.0161 & 0.0159 & 0.0399 & 0.0142 & 0.0206 & 0.0174 & \textcolor{blue}{0.0140} & 0.0181 & 0.0147 & 0.0318 & 0.0193 & 0.0249 & 0.0167 & 0.0352 & 0.0184 & 0.0180 & \textcolor{red}{0.0124} \\
CIFAR10 & DenseNet40 & 0.0823 & 0.0538 & 0.0254 & 0.0569 & 0.0408 & 0.0159 & 0.0205 & 0.0574 & 0.0198 & 0.0186 & 0.0177 & 0.0150 & 0.0136 & 0.0262 & 0.0291 & \textcolor{blue}{0.0125} & 0.0413 & 0.0169 & 0.0633 & \textcolor{red}{0.0124} & 0.0174 & 0.0138 \\
SVHN & ResNet152(SD) & 0.0226 & 0.0193 & \textcolor{red}{0.0076} & 0.0232 & 0.0133 & 0.0113 & 0.0167 & 0.0330 & 0.0163 & 0.0142 & 0.0157 & 0.0119 & \textcolor{blue}{0.0101} & 0.0115 & 0.0135 & 0.0328 & 0.0155 & 0.0143 & 0.0137 & 0.0124 & 0.0154 & 0.0154 \\
CIFAR100 & ResNet110 & 0.2091 & 0.1183 & 0.0732 & 0.1054 & 0.1018 & 0.0268 & 0.0323 & 0.0676 & 0.0401 & 0.0398 & 0.0436 & 0.0222 & 0.0336 & 0.0376 & 0.0598 & 0.0220 & 0.0233 & \textcolor{blue}{0.0213} & 0.0686 & 0.0216 & 0.0251 & \textcolor{red}{0.0184} \\
CIFAR100 & WideResNet32 & 0.2158 & 0.1076 & 0.0694 & 0.0977 & 0.0922 & 0.0162 & 0.0245 & 0.0616 & 0.0284 & 0.0499 & 0.0215 & 0.0152 & 0.0155 & 0.0367 & 0.0643 & 0.0164 & 0.0391 & 0.0144 & 0.1987 & \textcolor{blue}{0.0141} & \textcolor{red}{0.0126} & 0.0203 \\
\hline
\end{tabular}
}
\end{table}

}
\end{landscape}

\begin{landscape}
\subsection{Classwise Calibration Results}
\label{sec:apdx-classwisecalibrationnumber}

{\mdseries

\subsubsection{CWECE$_s$}
\label{sec:tab-cmp-CwECEsum}
\begin{table}[H]
\centering
\caption{CWECE$_s$ across All Tasks (Best in \textcolor{red}{Red}, Second-best in \textcolor{blue}{Blue})}
\label{tab:apdx-tab-cmp-cwecesum}
\setlength{\tabcolsep}{2pt}
\resizebox{\linewidth}{!}{
\begin{tabular}{l l  c c c c c c c c c c c c c c c c c c c c c c}
\hline
Dataset & Model & Uncal & HB & Iso & BBQ & ENIR & TS & VS & MS & Beta & ScaBin & Dir & GP & DIAG & DEC & IMax & SoftBin & Spline & EC & LECE & SCTL & TS+LECE & Ours \\
\hline
CIFAR10 & ResNet110 & 0.098 & 0.046 & 0.043 & 0.045 & 0.058 & 0.043 & 0.039 & 0.042 & 0.037 & 0.040 & 0.038 & 0.038 & 0.038 & \textcolor{red}{0.034} & 0.041 & 0.043 & 0.039 & 0.043 & 0.070 & 0.043 & \textcolor{blue}{0.036} & 0.041 \\
CIFAR10 & WideResNet32 & 0.095 & 0.048 & 0.038 & 0.049 & 0.054 & 0.048 & \textcolor{blue}{0.030} & 0.040 & \textcolor{red}{0.029} & 0.046 & \textcolor{red}{0.029} & 0.048 & 0.048 & 0.032 & 0.047 & 0.048 & 0.050 & 0.048 & 0.048 & 0.048 & 0.034 & 0.050 \\
CIFAR10 & DenseNet40 & 0.114 & 0.052 & 0.047 & 0.054 & 0.065 & 0.040 & 0.038 & 0.052 & 0.038 & 0.050 & 0.037 & 0.038 & 0.039 & \textcolor{blue}{0.035} & 0.045 & 0.036 & 0.043 & 0.040 & 0.077 & 0.039 & \textcolor{red}{0.034} & 0.042 \\
SVHN & ResNet152(SD) & 0.019 & 0.019 & \textcolor{blue}{0.015} & 0.020 & 0.016 & 0.018 & 0.022 & 0.027 & 0.022 & 0.047 & 0.022 & 0.017 & \textcolor{blue}{0.015} & 0.024 & \textcolor{blue}{0.015} & 0.049 & 0.020 & 0.022 & 0.016 & 0.021 & 0.020 & \textcolor{red}{0.011} \\
CIFAR100 & ResNet110 & 0.416 & 0.211 & 0.216 & 0.197 & 0.261 & 0.201 & 0.188 & 0.207 & 0.188 & 0.222 & \textcolor{blue}{0.186} & 0.203 & 0.209 & 0.210 & 0.215 & 0.202 & 0.202 & 0.202 & 0.233 & 0.201 & 0.193 & \textcolor{red}{0.089} \\
CIFAR100 & WideResNet32 & 0.420 & 0.198 & 0.202 & 0.182 & 0.228 & 0.186 & \textcolor{red}{0.174} & 0.220 & 0.179 & 0.241 & 0.179 & 0.184 & 0.186 & 0.204 & 0.210 & 0.186 & 0.191 & 0.186 & 0.356 & 0.185 & \textcolor{blue}{0.175} & 0.188 \\
CIFAR100 & DenseNet40 & 0.470 & 0.211 & 0.200 & \textcolor{blue}{0.183} & 0.254 & 0.187 & 0.186 & 0.225 & 0.186 & 0.253 & 0.186 & 0.188 & 0.186 & 0.221 & 0.213 & 0.186 & 0.194 & 0.188 & 0.277 & 0.187 & \textcolor{red}{0.181} & 0.189 \\
\hline
\end{tabular}
}
\end{table}

\subsubsection{CWECE$_a$}
\label{sec:tab-cmp-CwECE}
\begin{table}[H]
\centering
\caption{CWECE$_a$ across All Tasks (Best in \textcolor{red}{Red}, Second-best in \textcolor{blue}{Blue})}
\setlength{\tabcolsep}{2pt}
\resizebox{\linewidth}{!}{
\begin{tabular}{l l  c c c c c c c c c c c c c c c c c c c c c c c}
\hline
Dataset & Model & Uncal & HB & Iso & BBQ & ENIR & TS & VS & MS & Beta & ScaBin & Dir & GP & DIAG & DEC & IMax & SoftBin & Spline & EC & LECE & SCTL & TS+LECE & Ours &   \\
\hline
CIFAR10 & ResNet110 & 0.9829 & 0.4619 & 0.4302 & 0.4547 & 0.5616 & 0.4343 & 0.3939 & 0.4410 & \textcolor{blue}{0.3717} & 0.4083 & 0.3728 & 0.3874 & 0.3770 & \textcolor{red}{0.3230} & 0.4266 & 0.4318 & 0.4091 & 0.4367 & 0.6943 & 0.4361 & 0.3770 & 0.3973 & ($\times 10^{-2}$) \\
CIFAR10 & WideResNet32 & 0.9664 & 0.4820 & 0.3848 & 0.4986 & 0.5521 & 0.4924 & \textcolor{blue}{0.3039} & 0.4004 & \textcolor{red}{0.2898} & 0.4806 & 0.3146 & 0.4828 & 0.4982 & 0.3309 & 0.4837 & 0.5045 & 0.5155 & 0.4873 & 0.4853 & 0.5024 & 0.3518 & 0.5027 & ($\times 10^{-2}$) \\
CIFAR10 & DenseNet40 & 1.1403 & 0.5376 & 0.4841 & 0.5351 & 0.6653 & 0.4043 & 0.3833 & 0.5827 & 0.3723 & 0.5082 & 0.3827 & 0.3929 & 0.3939 & \textcolor{blue}{0.3697} & 0.4536 & 0.3727 & 0.4462 & 0.4154 & 0.7785 & 0.3815 & \textcolor{red}{0.3514} & 0.4148 & ($\times 10^{-2}$) \\
SVHN & ResNet152(SD) & 0.1981 & 0.1905 & 0.1575 & 0.2036 & 0.1737 & 0.1788 & 0.2208 & 0.2794 & 0.2208 & 0.4741 & 0.2122 & 0.1701 & \textcolor{blue}{0.1514} & 0.2425 & 0.1558 & 0.4966 & 0.2028 & 0.2211 & 0.1660 & 0.2062 & 0.2075 & \textcolor{red}{0.1138} & ($\times 10^{-2}$) \\
CIFAR100 & ResNet110 & 0.4172 & 0.2118 & 0.2194 & 0.1999 & 0.2655 & 0.2086 & 0.1973 & 0.2106 & 0.1943 & 0.2216 & \textcolor{blue}{0.1919} & 0.2098 & 0.2108 & 0.2171 & 0.2176 & 0.2076 & 0.2082 & 0.2077 & 0.2354 & 0.2083 & 0.1993 & \textcolor{red}{0.1065} & ($\times 10^{-2}$) \\
CIFAR100 & WideResNet32 & 0.4212 & 0.2016 & 0.2070 & 0.1836 & 0.2351 & 0.1877 & 0.1846 & 0.2187 & 0.1860 & 0.2467 & \textcolor{blue}{0.1824} & 0.1848 & 0.1881 & 0.2055 & 0.2105 & 0.1877 & 0.1938 & 0.1875 & 0.3578 & 0.1885 & \textcolor{red}{0.1753} & 0.1920 & ($\times 10^{-2}$) \\
\hline
\end{tabular}
}
\end{table}

\subsubsection{CWECE$_{r=2}$}
\label{sec:tab-cmp-marCE2}
\begin{table}[H]
\centering
\caption{CWECE$_{r=2}$ across All Tasks (Best in \textcolor{red}{Red}, Second-best in \textcolor{blue}{Blue})}
\setlength{\tabcolsep}{2pt}
\resizebox{\linewidth}{!}{
\begin{tabular}{l l  c c c c c c c c c c c c c c c c c c c c c c c}
\hline
Dataset & Model & Uncal & HB & Iso & BBQ & ENIR & TS & VS & MS & Beta & ScaBin & Dir & GP & DIAG & DEC & IMax & SoftBin & Spline & EC & LECE & SCTL & TS+LECE & Ours &   \\
\hline
CIFAR10 & ResNet110 & 0.2725 & 0.2129 & 0.1974 & 0.2203 & 0.2205 & 0.1764 & 0.1993 & 0.2585 & 0.1864 & \textcolor{blue}{0.1579} & 0.1837 & 0.1590 & \textcolor{red}{0.1531} & 0.1591 & 0.1810 & 0.1657 & 0.1808 & 0.1763 & 0.2558 & 0.1745 & 0.1651 & 0.1609 & ($\times 10^{-1}$) \\
CIFAR10 & WideResNet32 & 0.2954 & 0.2267 & 0.1899 & 0.2306 & 0.2190 & 0.2224 & \textcolor{red}{0.1534} & 0.2445 & \textcolor{blue}{0.1536} & 0.1931 & 0.1654 & 0.2188 & 0.2274 & 0.1742 & 0.2262 & 0.2284 & 0.2310 & 0.2204 & 0.2263 & 0.2285 & 0.1802 & 0.2247 & ($\times 10^{-1}$) \\
CIFAR10 & DenseNet40 & 0.3201 & 0.2648 & 0.1966 & 0.2473 & 0.2193 & 0.1691 & 0.1677 & 0.3011 & 0.1676 & 0.2002 & 0.1628 & 0.1661 & 0.1656 & 0.1768 & 0.1909 & 0.1619 & 0.2015 & 0.1740 & 0.2766 & \textcolor{blue}{0.1606} & 0.1640 & \textcolor{red}{0.1530} & ($\times 10^{-1}$) \\
SVHN & ResNet152(SD) & 0.1296 & 0.1232 & 0.1151 & 0.1199 & 0.1250 & 0.1065 & 0.1309 & 0.1564 & 0.1308 & 0.1828 & 0.1212 & 0.1072 & \textcolor{red}{0.1006} & 0.1070 & 0.1111 & 0.1406 & 0.1171 & 0.1104 & 0.1114 & 0.1094 & 0.1135 & \textcolor{blue}{0.1011} & ($\times 10^{-1}$) \\
CIFAR100 & ResNet110 & 0.2641 & 0.1891 & 0.1806 & 0.1722 & 0.1934 & 0.1663 & 0.1732 & 0.1850 & 0.1712 & \textcolor{red}{0.1418} & 0.1689 & 0.1656 & \textcolor{blue}{0.1619} & 0.1633 & 0.1790 & 0.1638 & 0.1622 & 0.1649 & 0.1822 & 0.1644 & 0.1640 & 0.1857 & ($\times 10^{-1}$) \\
CIFAR100 & WideResNet32 & 0.2725 & 0.1849 & 0.1793 & 0.1632 & 0.1861 & 0.1621 & 0.1634 & 0.1922 & 0.1633 & \textcolor{red}{0.1476} & 0.1621 & \textcolor{blue}{0.1571} & 0.1616 & 0.1581 & 0.1824 & 0.1620 & 0.1664 & 0.1616 & 0.2587 & 0.1631 & 0.1581 & 0.1695 & ($\times 10^{-1}$) \\
\hline
\end{tabular}
}
\end{table}

\subsubsection{$t$CWECE}
\label{sec:tab-cmp-CwECEthr}
\begin{table}[H]
\centering
\caption{$t$CWECE across All Tasks (Best in \textcolor{red}{Red}, Second-best in \textcolor{blue}{Blue})}
\setlength{\tabcolsep}{2pt}
\resizebox{\linewidth}{!}{
\begin{tabular}{l l  c c c c c c c c c c c c c c c c c c c c c c}
\hline
Dataset & Model & Uncal & HB & Iso & BBQ & ENIR & TS & VS & MS & Beta & ScaBin & Dir & GP & DIAG & DEC & IMax & SoftBin & Spline & EC & LECE & SCTL & TS+LECE & Ours \\
\hline
CIFAR10 & ResNet110 & 0.0523 & 0.0319 & 0.0289 & 0.0315 & 0.0338 & 0.0289 & 0.0277 & 0.0281 & 0.0255 & 0.0269 & 0.0250 & 0.0249 & \textcolor{blue}{0.0240} & \textcolor{red}{0.0215} & 0.0283 & 0.0276 & 0.0268 & 0.0287 & 0.0399 & 0.0287 & 0.0246 & 0.0258 \\
CIFAR10 & WideResNet32 & 0.0516 & 0.0341 & 0.0262 & 0.0360 & 0.0322 & 0.0325 & \textcolor{blue}{0.0210} & 0.0256 & \textcolor{red}{0.0207} & 0.0321 & 0.0219 & 0.0317 & 0.0330 & 0.0219 & 0.0320 & 0.0336 & 0.0343 & 0.0323 & 0.0295 & 0.0335 & 0.0249 & 0.0325 \\
CIFAR10 & DenseNet40 & 0.0610 & 0.0397 & 0.0291 & 0.0370 & 0.0356 & 0.0247 & 0.0250 & 0.0377 & 0.0239 & 0.0332 & 0.0245 & 0.0240 & 0.0243 & 0.0247 & 0.0283 & \textcolor{blue}{0.0231} & 0.0281 & 0.0256 & 0.0438 & 0.0236 & \textcolor{red}{0.0229} & 0.0247 \\
SVHN & ResNet152(SD) & 0.0133 & 0.0113 & \textcolor{blue}{0.0099} & 0.0120 & 0.0114 & 0.0118 & 0.0136 & 0.0151 & 0.0137 & 0.0300 & 0.0127 & 0.0121 & 0.0115 & 0.0130 & 0.0113 & 0.0262 & 0.0139 & 0.0136 & 0.0116 & 0.0129 & 0.0127 & \textcolor{red}{0.0094} \\
CIFAR100 & ResNet110 & 0.1299 & 0.0876 & 0.0461 & 0.0908 & 0.0542 & 0.0350 & 0.0387 & 0.0297 & 0.0382 & 0.0288 & 0.0399 & 0.0362 & 0.0375 & 0.0324 & 0.0369 & 0.0338 & 0.0359 & 0.0335 & \textcolor{blue}{0.0273} & 0.0340 & 0.0342 & \textcolor{red}{0.0150} \\
CIFAR100 & WideResNet32 & 0.1579 & 0.0801 & 0.0476 & 0.0923 & 0.0526 & 0.0375 & 0.0400 & \textcolor{red}{0.0314} & 0.0400 & \textcolor{blue}{0.0332} & 0.0402 & 0.0368 & 0.0361 & 0.0358 & 0.0505 & 0.0375 & 0.0394 & 0.0369 & 0.1388 & 0.0368 & 0.0342 & 0.0376 \\
\hline
\end{tabular}
}
\end{table}

\subsubsection{$t$CWECE$^k$}
\label{sec:tab-cmp-kCWECEthr}
\begin{table}[H]
\centering
\caption{$t$CWECE$^k$ across All Tasks (Best in \textcolor{red}{Red}, Second-best in \textcolor{blue}{Blue})}
\setlength{\tabcolsep}{2pt}
\resizebox{\linewidth}{!}{
\begin{tabular}{l l  c c c c c c c c c c c c c c c c c c c c c c}
\hline
Dataset & Model & Uncal & HB & Iso & BBQ & ENIR & TS & VS & MS & Beta & ScaBin & Dir & GP & DIAG & DEC & IMax & SoftBin & Spline & EC & LECE & SCTL & TS+LECE & Ours \\
\hline
CIFAR10 & ResNet110 & 0.0498 & 0.0265 & 0.0204 & 0.0266 & 0.0268 & 0.0241 & 0.0213 & 0.0219 & 0.0193 & 0.0235 & 0.0205 & 0.0191 & 0.0176 & \textcolor{red}{0.0162} & \textcolor{blue}{0.0175} & 0.0232 & 0.0216 & 0.0226 & 0.0379 & 0.0221 & 0.0185 & 0.0195 \\
CIFAR10 & WideResNet32 & 0.0480 & 0.0284 & 0.0206 & 0.0298 & 0.0260 & 0.0288 & \textcolor{red}{0.0160} & 0.0204 & 0.0173 & 0.0276 & 0.0172 & 0.0297 & 0.0293 & \textcolor{blue}{0.0161} & 0.0274 & 0.0298 & 0.0305 & 0.0290 & 0.0260 & 0.0291 & 0.0189 & 0.0291 \\
CIFAR10 & DenseNet40 & 0.0587 & 0.0325 & 0.0209 & 0.0317 & 0.0284 & 0.0194 & 0.0187 & 0.0270 & 0.0197 & 0.0292 & \textcolor{blue}{0.0185} & 0.0204 & 0.0191 & 0.0199 & 0.0225 & \textcolor{red}{0.0170} & 0.0247 & 0.0188 & 0.0400 & 0.0193 & 0.0198 & 0.0199 \\
SVHN & ResNet152(SD) & 0.0113 & 0.0080 & \textcolor{red}{0.0069} & 0.0079 & 0.0079 & 0.0091 & 0.0111 & 0.0144 & 0.0108 & 0.0279 & 0.0099 & 0.0092 & 0.0088 & 0.0101 & 0.0081 & 0.0236 & 0.0115 & 0.0116 & 0.0089 & 0.0102 & 0.0099 & \textcolor{blue}{0.0074} \\
CIFAR100 & ResNet110 & 0.1177 & 0.0539 & 0.0331 & 0.0660 & 0.0431 & 0.0279 & 0.0295 & \textcolor{blue}{0.0210} & 0.0300 & 0.0262 & 0.0310 & 0.0280 & 0.0296 & 0.0243 & 0.0259 & 0.0272 & 0.0273 & 0.0273 & \textcolor{red}{0.0187} & 0.0275 & 0.0264 & 0.0309 \\
CIFAR100 & WideResNet32 & 0.1442 & 0.0508 & 0.0335 & 0.0677 & 0.0409 & 0.0292 & 0.0292 & \textcolor{red}{0.0220} & 0.0296 & 0.0303 & 0.0302 & 0.0297 & 0.0288 & 0.0280 & 0.0342 & 0.0297 & 0.0298 & 0.0295 & 0.1229 & 0.0293 & \textcolor{blue}{0.0271} & 0.0296 \\
\hline
\end{tabular}
}
\end{table}

}
\end{landscape}

\begin{landscape}
\subsection{Canonical Calibration Results}
\label{sec:apdx-canonicalcalibrationnumber}
{\mdseries
\subsubsection{DKDE-CE}
\label{sec:tab-cmp-DKDE-CE}

\begin{table}[H]
\centering
\caption{DKDE-CE across All Tasks (Best in \textcolor{red}{Red}, Second-best in \textcolor{blue}{Blue})}
\setlength{\tabcolsep}{2pt}
\resizebox{\linewidth}{!}{
\begin{tabular}{l l  c c c c c c c c c c c c c c c c c c c c c c c}
\hline
Dataset & Model & Uncal & HB & Iso & BBQ & ENIR & TS & VS & MS & Beta & ScaBin & Dir & GP & DIAG & DEC & IMax & SoftBin & Spline & EC & LECE & SCTL & TS+LECE & Ours &   \\
\hline
CIFAR10 & ResNet110 & 0.0216 & 0.0517 & \textcolor{red}{0.0075} & 0.0542 & \textcolor{blue}{0.0088} & 0.0143 & 0.0149 & 0.0614 & 0.0161 & 0.0102 & 0.0157 & 0.0133 & 0.0133 & 0.0200 & 0.0159 & 0.0152 & 0.0094 & 0.0146 & 0.0186 & 0.0146 & 0.0134 & 0.0089 & ($\times 10^{-1}$) \\
CIFAR10 & WideResNet32 & 0.0211 & 0.0482 & \textcolor{red}{0.0052} & 0.0504 & \textcolor{blue}{0.0059} & 0.0111 & 0.0106 & 0.0581 & 0.0111 & 0.0082 & 0.0114 & 0.0106 & 0.0111 & 0.0157 & 0.0161 & 0.0113 & 0.0085 & 0.0110 & 0.0099 & 0.0112 & 0.0111 & 0.0129 & ($\times 10^{-1}$) \\
CIFAR10 & DenseNet40 & 0.0208 & 0.0606 & \textcolor{red}{0.0081} & 0.0621 & 0.0092 & 0.0129 & 0.0140 & 0.0724 & 0.0149 & 0.0104 & 0.0137 & 0.0123 & 0.0128 & 0.0149 & 0.0181 & 0.0135 & 0.0120 & 0.0125 & 0.0142 & 0.0129 & 0.0137 & \textcolor{blue}{0.0091} & ($\times 10^{-1}$) \\
SVHN & ResNet152(SD) & \textcolor{blue}{0.0071} & 0.0085 & 0.0089 & 0.0084 & 0.0102 & 0.0115 & 0.0101 & \textcolor{red}{0.0054} & 0.0101 & 0.0484 & 0.0101 & 0.0118 & 0.0117 & 0.0094 & 0.0117 & 0.0176 & 0.0110 & 0.0122 & 0.0083 & 0.0120 & 0.0119 & 0.0105 & ($\times 10^{-1}$) \\
CIFAR100 & ResNet110 & 0.3131 & 1.5355 & \textcolor{blue}{0.1116} & 1.7831 & 0.1287 & 0.2143 & 0.2045 & 0.3300 & 0.2098 & 0.1839 & 0.2082 & 0.2382 & 0.2369 & 0.1479 & 0.3476 & 0.2177 & 0.1512 & 0.2186 & 0.1787 & 0.2171 & 0.2076 & \textcolor{red}{0.0167} & ($\times 10^{-1}$) \\
CIFAR100 & WideResNet32 & 0.2878 & 1.3526 & \textcolor{red}{0.1133} & 1.6085 & \textcolor{blue}{0.1217} & 0.1717 & 0.1648 & 0.3067 & 0.1715 & 0.2564 & 0.1692 & 0.1795 & 0.1811 & 0.1339 & 0.3609 & 0.1693 & 0.1502 & 0.1707 & 0.3460 & 0.1716 & 0.1670 & 0.1405 & ($\times 10^{-1}$) \\
\hline
\end{tabular}
}
\end{table}

\subsubsection{SKCE}
\label{sec:tab-cmp-SKCE}

\begin{table}[H]
\centering
\caption{SKCE across All Tasks (Best in \textcolor{red}{Red}, Second-best in \textcolor{blue}{Blue})}
\label{tab:apdx-tab-cmp-SKCE}
\setlength{\tabcolsep}{2pt}
\resizebox{\linewidth}{!}{
\begin{tabular}{l l  c c c c c c c c c c c c c c c c c c c c c c c}
\hline
Dataset & Model & Uncal & HB & Iso & BBQ & ENIR & TS & VS & MS & Beta & ScaBin & Dir & GP & DIAG & DEC & IMax & SoftBin & Spline & EC & LECE & SCTL & TS+LECE & Ours &   \\
\hline
CIFAR10 & ResNet110 & 2.2186 & 1.0775 & 0.3357 & 1.1132 & 0.5579 & 0.2702 & 0.2144 & 0.3553 & 0.1979 & 0.1752 & 0.1962 & 0.1626 & 0.1425 & \textcolor{red}{0.0950} & 0.2606 & 0.2040 & 0.2401 & 0.2359 & 0.8381 & 0.2342 & \textcolor{blue}{0.1288} & 0.2231 & ($\times 10^{-4}$) \\
CIFAR10 & WideResNet32 & 2.6548 & 1.3939 & 0.3970 & 1.4529 & 0.5773 & 1.1141 & 0.1741 & 0.3464 & 0.1724 & 0.5951 & \textcolor{blue}{0.1144} & 1.1408 & 1.1245 & \textcolor{red}{0.1103} & 1.0105 & 1.1333 & 1.2074 & 1.1129 & 0.4296 & 1.1278 & 0.2576 & 1.1002 & ($\times 10^{-4}$) \\
CIFAR10 & DenseNet40 & 2.7452 & 1.0192 & 0.2732 & 1.0410 & 0.5711 & 0.3071 & 0.2328 & 0.6287 & 0.2257 & 0.9156 & \textcolor{red}{0.1491} & 0.2803 & 0.2933 & \textcolor{blue}{0.1744} & 0.3996 & 0.2676 & 0.3603 & 0.3218 & 1.0870 & 0.2867 & 0.2161 & 0.3380 & ($\times 10^{-4}$) \\
SVHN & ResNet152(SD) & 0.0600 & 0.1059 & 0.0668 & 0.1061 & 0.0456 & 0.0501 & 0.2284 & 0.3580 & 0.2300 & 1.3093 & 0.2030 & 0.0379 & \textcolor{blue}{0.0278} & 0.2325 & 0.0335 & 0.4929 & 0.0492 & 0.0807 & 0.0398 & 0.0698 & 0.1139 & \textcolor{red}{0.0165} & ($\times 10^{-4}$) \\
CIFAR100 & ResNet110 & 2.2982 & 1.0178 & 0.4997 & 1.1285 & 0.7876 & 0.1914 & 0.3517 & 0.6000 & 0.3885 & 0.9294 & 0.3600 & \textcolor{blue}{0.1509} & 0.1514 & 0.3359 & 0.2923 & 0.1804 & 0.1859 & 0.1785 & 0.5405 & 0.1818 & \textcolor{red}{0.1247} & 0.8543 & ($\times 10^{-4}$) \\
CIFAR100 & WideResNet32 & 2.7656 & 1.0249 & 0.4533 & 1.0668 & 0.6239 & 0.4649 & \textcolor{red}{0.2535} & 0.4528 & 0.3185 & 1.7450 & 0.3308 & 0.4511 & 0.4600 & 0.2731 & 0.5077 & 0.4653 & 0.4764 & 0.4634 & 1.9628 & 0.4628 & \textcolor{blue}{0.2620} & 0.5036 & ($\times 10^{-4}$) \\
\hline
\end{tabular}
}
\end{table}
}
\end{landscape}

\begin{landscape}
\subsection{Negative Log-Likelihood Results}
\label{sec:tab-cmp-NLL}
{\mdseries

\begin{table}[H]
\centering
\caption{NLL across All Tasks (Best in \textcolor{red}{Red}, Second-best in \textcolor{blue}{Blue})}
\setlength{\tabcolsep}{2pt}
\resizebox{\linewidth}{!}{
\begin{tabular}{l l  c c c c c c c c c c c c c c c c c c c c c c}
\hline
Dataset & Model & Uncal & HB & Iso & BBQ & ENIR & TS & VS & MS & Beta & ScaBin & Dir & GP & DIAG & DEC & IMax & SoftBin & Spline & EC & LECE & SCTL & TS+LECE & Ours \\
\hline
CIFAR10 & ResNet110 & 0.3583 & 0.5022 & 0.3215 & 0.3269 & 0.4431 & 0.2093 & 0.2078 & 0.2986 & 0.2094 & 0.2741 & \textcolor{red}{0.2051} & 0.2056 & 0.2080 & 0.2358 & 0.2306 & 0.2091 & 0.3240 & 0.2089 & 0.2667 & 0.2089 & \textcolor{blue}{0.2055} & 0.2137 \\
CIFAR10 & WideResNet32 & 0.3821 & 0.5126 & 0.2696 & 0.3125 & 0.3953 & 0.1914 & \textcolor{blue}{0.1833} & 0.2853 & 0.1837 & 0.2594 & \textcolor{red}{0.1819} & 0.1907 & 0.1914 & 0.2374 & 0.2162 & 0.1914 & 0.2588 & 0.1915 & 0.2648 & 0.1914 & 0.1847 & 0.1997 \\
CIFAR10 & DenseNet40 & 0.4301 & 0.4060 & 0.3164 & 0.3666 & 0.6019 & 0.2245 & 0.2239 & 0.3408 & 0.2257 & 0.2849 & \textcolor{blue}{0.2238} & 0.2260 & 0.2246 & 0.2628 & 0.2544 & \textcolor{blue}{0.2238} & 0.3044 & 0.2256 & 0.3118 & 0.2244 & \textcolor{red}{0.2230} & 0.2252 \\
SVHN & ResNet152(SD) & 0.0854 & 0.2418 & 0.1304 & 0.1078 & 0.1910 & 0.0786 & 0.0815 & 0.1148 & 0.0815 & 0.1850 & 0.0811 & \textcolor{blue}{0.0785} & \textcolor{red}{0.0783} & 0.1037 & 0.0798 & 0.0898 & 0.1155 & 0.0797 & 0.0849 & 0.0793 & 0.0808 & 0.0787 \\
CIFAR100 & ResNet110 & 1.6937 & 4.1295 & 2.2883 & 1.7527 & 2.7629 & 1.0917 & 1.0957 & 1.9349 & 1.1104 & 1.5443 & \textcolor{red}{1.0782} & \textcolor{blue}{1.0816} & 1.1041 & 1.4734 & 1.1662 & 1.0914 & 1.3623 & 1.0915 & 1.4982 & 1.0914 & 1.0847 & 1.0917 \\
CIFAR100 & WideResNet32 & 1.8022 & 3.6208 & 2.0061 & 1.6836 & 2.4585 & 0.9448 & 0.9481 & 1.9090 & 0.9738 & 1.4640 & 0.9512 & \textcolor{blue}{0.9391} & 0.9466 & 1.2409 & 1.0877 & 0.9446 & 1.1181 & 0.9444 & 1.5759 & 0.9444 & \textcolor{red}{0.9371} & 0.9450 \\
\hline
\end{tabular}
}
\end{table}

}
\end{landscape}

\clearpage
\begin{landscape}
\section{Classification Accuracy Variations of Different Post-hoc Recalibration Methods}
\label{sec:apdx-tab-accuracy-variation}
{\mdseries
\begin{table}[H]
\centering
\caption{Classification Accuracy Variations Induced by Different Post-hoc Calibrators (\%). `-' indicates no accuracy variation.}
\label{tab:apdx-tab-accuracy-variation}
\setlength{\tabcolsep}{2pt}
\resizebox{\linewidth}{!}{
\begin{tabular}{c | l l  c c c c c c c c c c c c c c c c c c c c c}
\hline
Reference Accuracy (Uncal) & Dataset & Model & HB & Iso & BBQ & ENIR & TS & VS & MS & Beta & ScaBin & Dir & GP & DIAG & DEC & IMax & SoftBin & Spline & EC & LECE & SCTL & TS+LECE & Ours \\
\hline
$93.56$ &CIFAR10 & ResNet110&$-0.14$ & $+0.08$ & $+0.01$ & $-0.44$ & - & $-0.03$ & $-0.05$ & $+0.06$ & $-5.37$ & $-0.05$ & $-0.02$ & - & $+0.05$ & $-0.03$ & - & $-0.01$ & - & $+0.00$ & - & $+0.04$ & - \\
$93.93$ &CIFAR10 & WideResNet32&$-0.04$ & $+0.09$ & $+0.04$ & $-0.13$ & - & $+0.25$ & $+0.14$ & $+0.19$ & $-5.38$ & $+0.31$ & $+0.01$ & - & $+0.25$ & $-0.23$ & - & $-0.17$ & - & $+0.10$ & - & $+0.24$ & - \\
$92.42$ &CIFAR10 & DenseNet40&$-0.21$ & $-0.07$ & $-0.05$ & $-0.46$ & - & $-0.05$ & $+0.04$ & $+0.08$ & $-5.15$ & $+0.06$ & $-0.02$ & - & $+0.04$ & $-0.09$ & - & $-0.11$ & - & $+0.01$ & - & $+0.06$ & - \\
$98.15$ &SVHN & ResNet152(SD)&$-0.10$ & $-0.06$ & $-0.03$ & $-0.12$ & - & $-0.03$ & $+0.03$ & $-0.04$ & $-7.06$ & $-0.01$ & $-0.00$ & - & $+0.03$ & $+0.01$ & - & $+0.02$ & - & $+0.03$ & - & $+0.03$ & - \\
$71.48$ &CIFAR100 & ResNet110&$-1.02$ & $-0.23$ & $-0.18$ & $-1.98$ & - & $+0.07$ & $-0.68$ & $-0.06$ & $-26.09$ & $+0.21$ & $-0.01$ & - & $+0.01$ & $+0.06$ & - & $+0.02$ & - & $+0.04$ & - & $+0.10$ & - \\
$73.82$ &CIFAR100 & WideResNet32&$-0.73$ & $-0.05$ & $+0.00$ & $-1.35$ & - & $+0.20$ & $-0.29$ & $+0.13$ & $-28.41$ & $+0.13$ & $+0.00$ & - & $+0.10$ & $+0.02$ & - & $-0.07$ & - & $+0.01$ & - & $+0.17$ & - \\
$70.00$ &CIFAR100 & DenseNet40&$-0.80$ & $+0.04$ & $+0.18$ & $-1.88$ & - & $+0.08$ & $-0.59$ & $+0.07$ & $-31.00$ & $+0.33$ & $+0.00$ & - & $+0.32$ & $+0.04$ & - & $-0.04$ & - & $+0.05$ & - & $+0.42$ & - \\
\hline
\end{tabular}
}
\end{table}
}
\end{landscape}

\section{Sample Reliability Diagrams from ImageNet-SwinTransformer Experiment}
\label{sec:apdx-reliabilitydiagrams}
{\mdseries
\begin{figure}[H]
\centering
\includegraphics[width=0.95\textwidth]{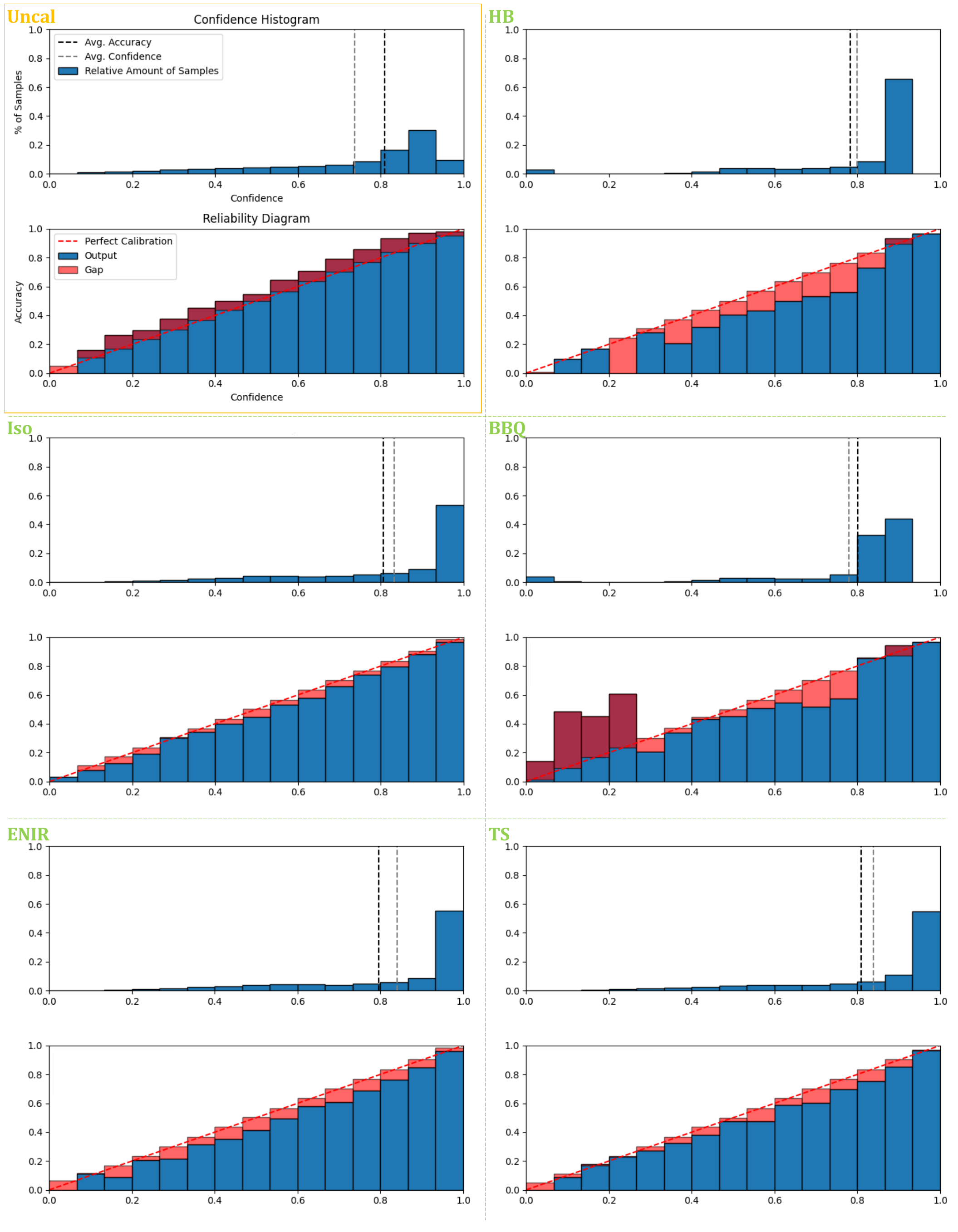}
\caption{Reliability Diagrams - Part I}
\label{reliability1}
\end{figure}
}
\clearpage
{\mdseries
\begin{figure}[H]
\centering
\includegraphics[width=0.95\textwidth]{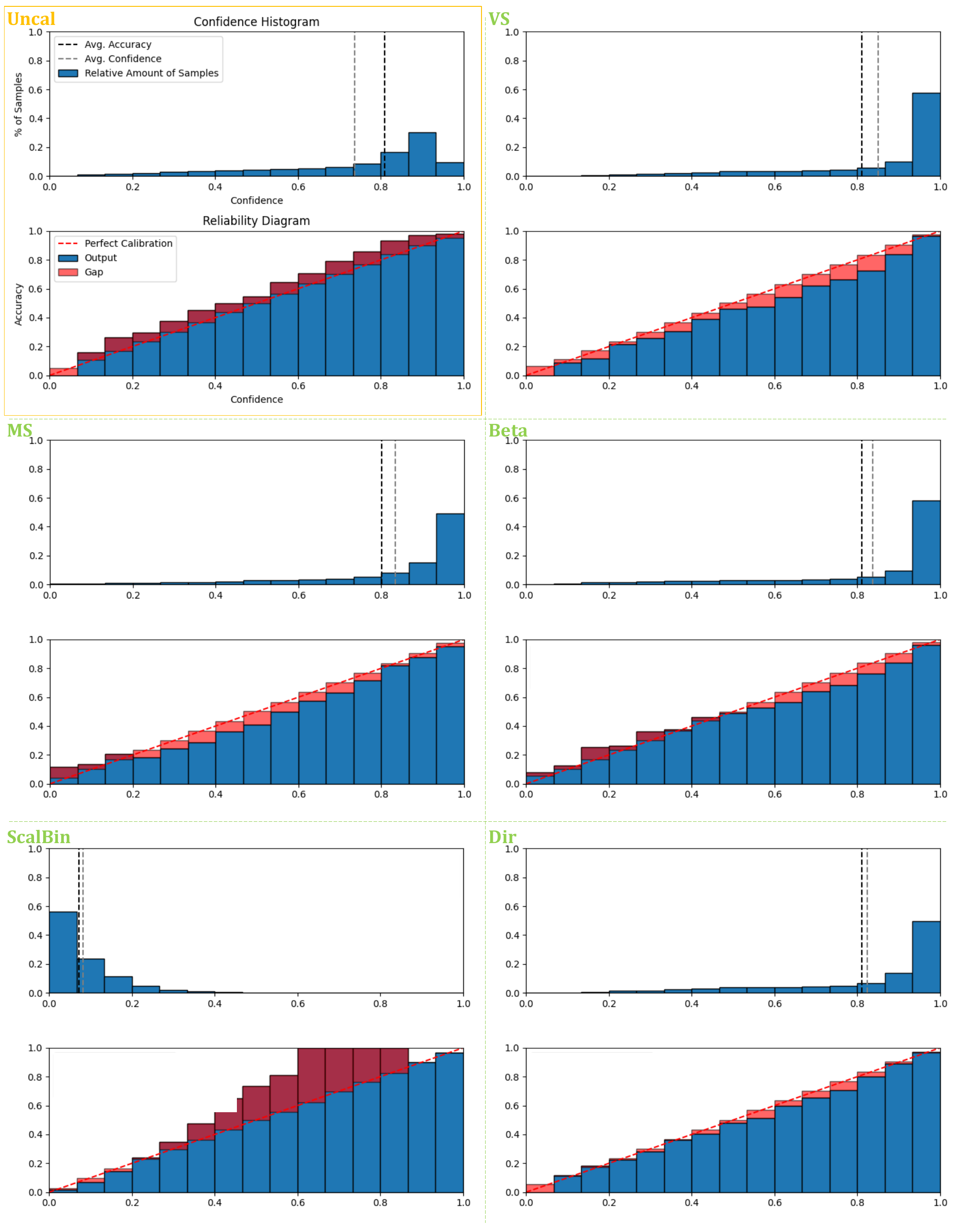}
\caption{Reliability Diagrams - Part II}
\label{reliability2}
\end{figure}

\begin{figure}[H]
\centering
\includegraphics[width=0.95\textwidth]{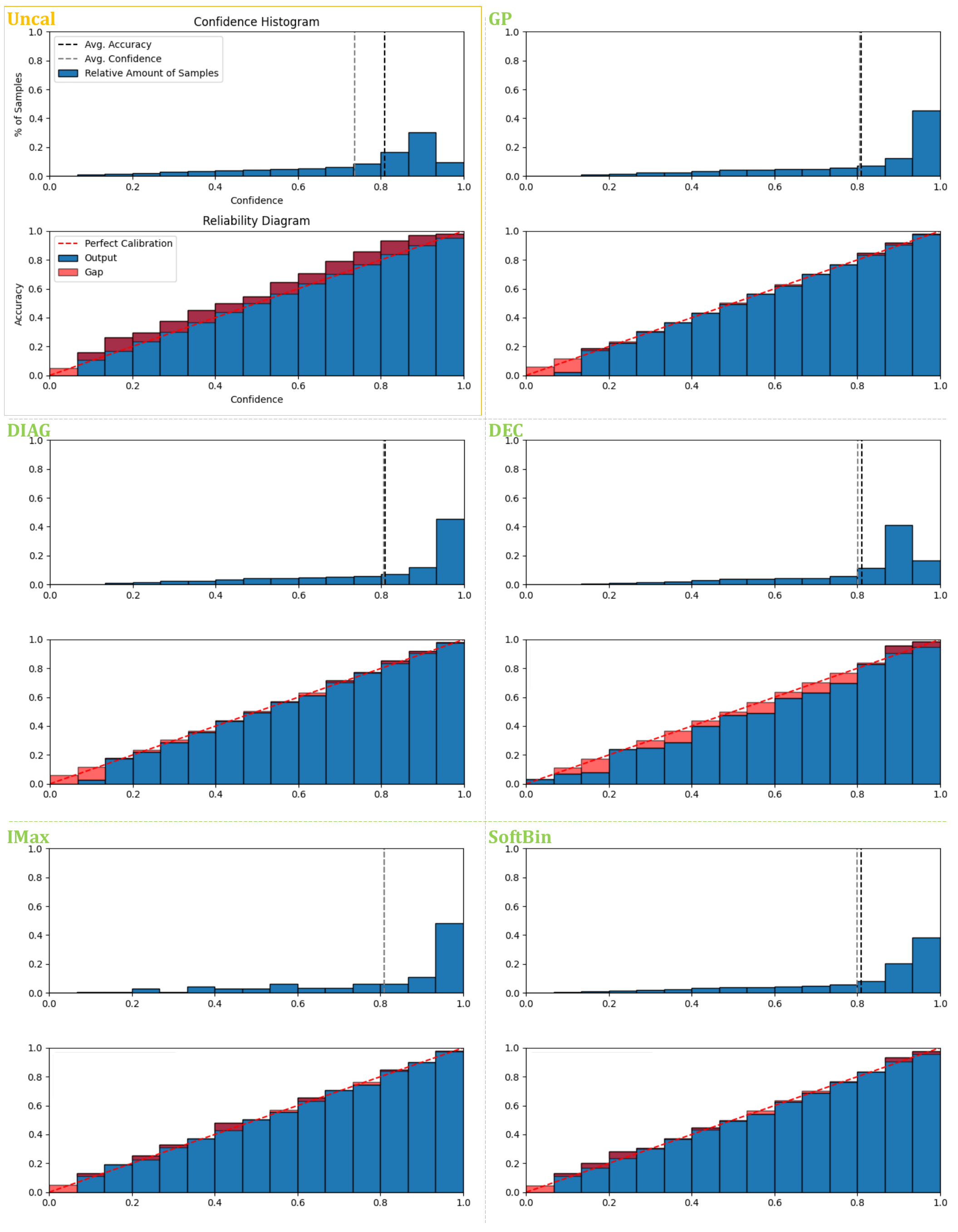}
\caption{Reliability Diagrams - Part III}
\label{reliability3}
\end{figure}

\begin{figure}[H]
\centering
\includegraphics[width=0.95\textwidth]{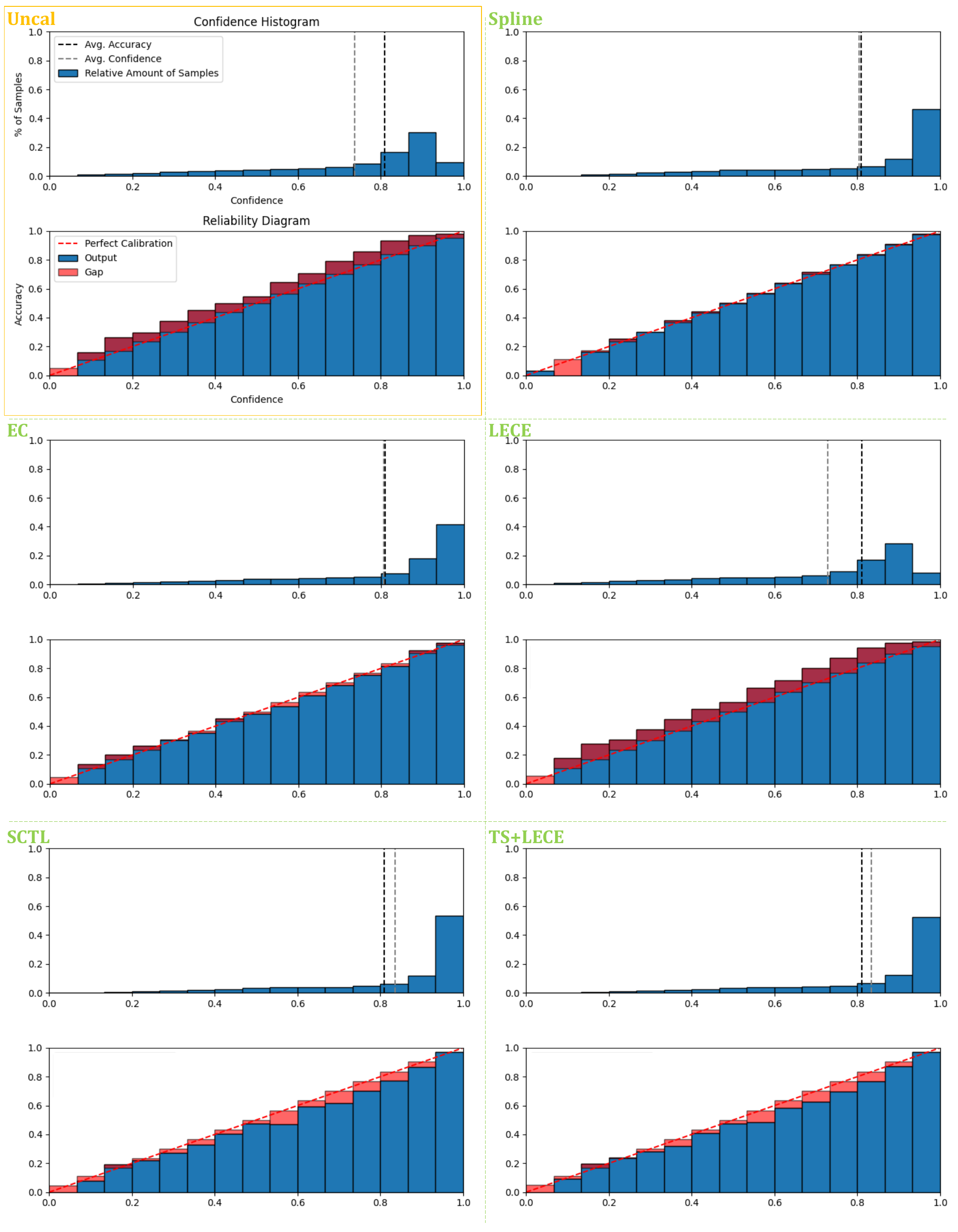}
\caption{Reliability Diagrams - Part IV}
\label{reliability4}
\end{figure}

\begin{figure}[H]
\centering
\includegraphics[width=0.95\textwidth]{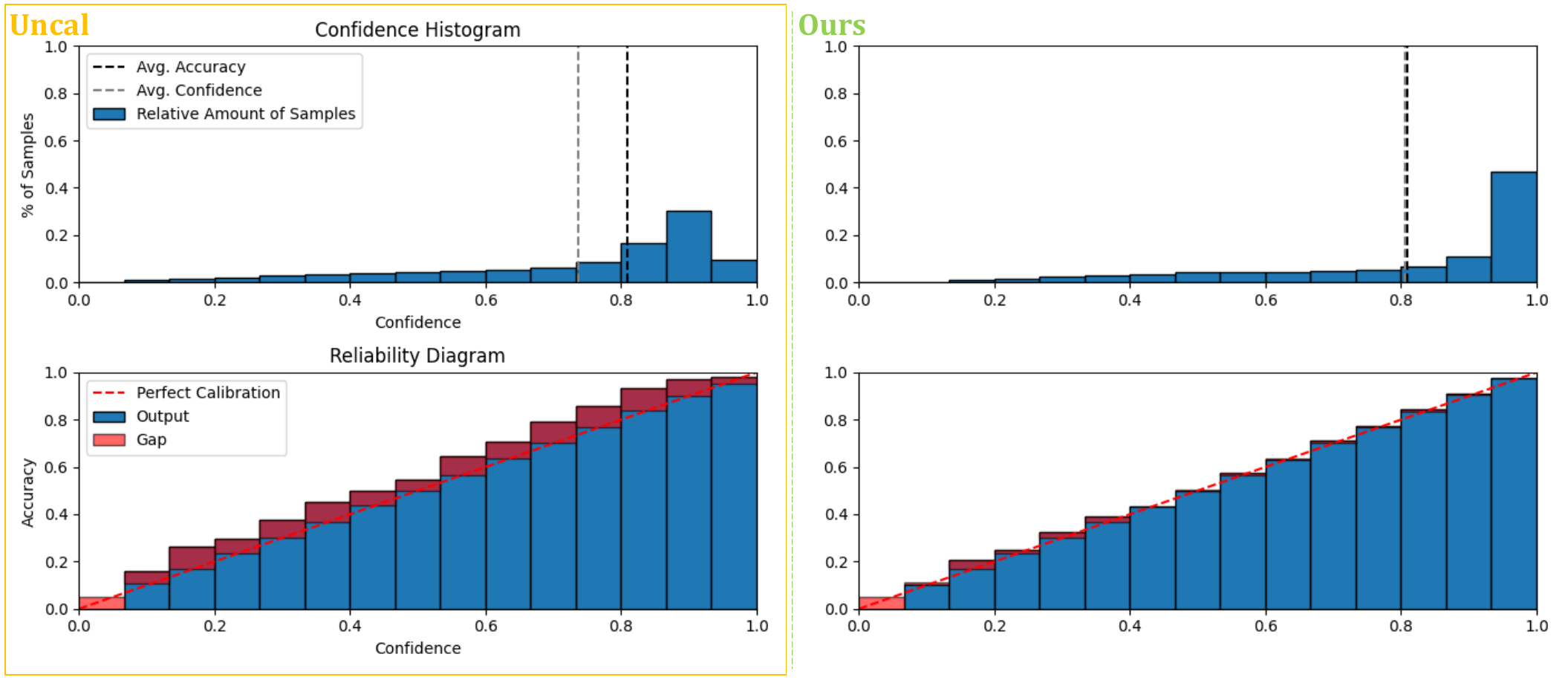}
\caption{Reliability Diagrams - Part V}
\label{reliability5}
\end{figure}
}
\clearpage

\section{Results Comparing Our Objective and Proper Scoring Rules}
\label{sec:apdx-psrcomparison}
{\mdseries

\begin{table}[H]
\centering
\caption{Metric-Specific Results across All Tasks - Part I}
\label{tab:apdx-psrcomparison1}

\end{table}

\begin{table}[H]
\centering
\caption{Metric-Specific Results across All Tasks - Part II}
\label{tab:apdx-psrcomparison2}
%
\end{table}

\begin{table}[H]
\centering
\caption{Metric-Specific Results across All Tasks - Part III}
\label{tab:apdx-psrcomparison3}
%
\end{table}

}

\section{Results Analyzing Hyperparameter Robustness and Future Work Discussion}
\subsection{Results Analyzing Hyperparameter Robustness}
\label{sec:apdx-robustnessanalysis}
{\mdseries

\begin{table}[H]
\centering
\caption{Metric-Specific Results for Robustness Analysis (`$\checkmark$' denotes the default values) - Part I}
%
\end{table}

\begin{table}[H]
\centering
\caption{Metric-Specific Results for Robustness Analysis (`$\checkmark$' denotes the default values) - Part II}
%
%
}%
\end{minipage}%
\vspace{5pt}
\hspace{5pt}
\end{tabular}
\end{tabular}
\end{table}

\subsection{Further Remarks on Future Work}
\label{sec:apdx-futurework}

While this study mainly focuses on building a calibrator, our equivalent formulation of bounded calibration error also offers the potential for developing canonical evaluation metrics. In contrast, current metrics are mostly non-canonical, and some canonical ones can suffer from the curse of dimensionality due to high-dimensional distribution estimation, which is avoided in our error statistics. However, given the plethora of existing metrics, systematically exploring their differences and advantages is beyond the scope of this study and is reserved for future work.

Lastly, our experiments were conducted for calibrating supervised discriminative deep networks. Recently, significant attention has been given to generative foundation models, such as large language models (LLMs). It has been highlighted that calibration biases in LLMs are closely related to model hallucination and generation quality \cite{no.209,no.210,no.211,no.212}. Specific training stages in foundation models, such as instruction tuning, can negatively impact calibration \cite{no.209}. Accordingly, some traditional calibration techniques, such as temperature scaling, have been successfully extended to suit the calibration demands of these large generative models \cite{no.210,no.212}. Our research potentially offers similar prospects for applying calibration in these models, which we plan to investigate in future studies.


}

\end{document}